\documentclass[a4paper,12pt]{article}
\usepackage{geometry}
\usepackage{cuted}
\usepackage{lipsum}
\usepackage{graphicx}
\usepackage{authblk}
\usepackage{chessfss}
\usepackage{xkeyval}
\usepackage{xifthen}
\usepackage{pgfcore}
\usepgfmodule{shapes}
\usepackage{pst-node}
\usepackage{chessboard}
\usepackage{float}
\usepackage[english]{babel}
\usepackage{csquotes}
\usepackage{tikz}
\usepackage[vlined,ruled]{algorithm2e}
\usepackage{xcolor}
\usepackage{algorithm2e}
\usepackage[title]{appendix}
\usepackage{graphicx}
\usepackage{lineno}
\usepackage{amsmath,amsfonts,amsthm}
\newtheorem{theorem}{Theorem}
\usepackage[numbers,sort,compress]{natbib}
\usepackage[normal,footnotesize,labelfont=bf,up,textfont=it,up]{caption}	% Custom captions under/above floats
\usepackage{booktabs}												% Nicer tables
\usepackage{fix-cm}													% Custom fontsizes
\usepackage{todonotes}
\usepackage[section]{placeins}
\let\Oldsection\section
\renewcommand{\section}{\FloatBarrier\Oldsection}

\let\Oldsubsection\subsection
\renewcommand{\subsection}{\FloatBarrier\Oldsubsection}

\let\Oldsubsubsection\subsubsection
\renewcommand{\subsubsection}{\FloatBarrier\Oldsubsubsection}

\usepackage{color}   %May be necessary if you want to color links
\usepackage{hyperref}
\hypersetup{
    colorlinks=true, %set true if you want colored links
    linktoc=all,     %set to all if you want both sections and subsections linked
    linkcolor=blue,  %choose some color if you want links to stand out
}

\hypersetup{
    colorlinks=true,
    linkcolor=blue,
    filecolor=magenta,      
    urlcolor=blue,
}
\title {\textbf{Algebraic Machine Learning}}
\author[1,*]{Fernando Martin-Maroto}
\author[2,+]{Gonzalo G. de Polavieja}
\affil[1]{\small Algebraic AI Inc. Santa Cruz, CA, USA}
\affil[2]{\small Champalimaud Research, Lisbon, Portugal}

\affil[*]{\small martin.maroto@algebraic.ai}
\affil[+]{\small gonzalo.polavieja@neuro.fchampalimaud.org}

\begin{document}
\maketitle
\textbf{Machine learning algorithms use error function minimization to fit a large set of parameters in a preexisting model. However, error minimization eventually leads to a memorization of the training dataset, losing the ability to generalize to other datasets. To achieve generalization something else is needed, for example a regularization method or stopping the training when error in a validation dataset is minimal. Here we propose a different approach to learning and generalization that is parameter-free, fully discrete and that does not use function minimization. We use the training data to find an algebraic representation with minimal size and maximal freedom, explicitly expressed as a product of irreducible components. This algebraic representation is shown to directly generalize, giving high accuracy in test data, more so the smaller the representation. We prove that the number of generalizing representations can be very large and the algebra only needs to find one. We also derive and test a relationship between compression and error rate. We give results for a simple problem solved step by step, hand-written character recognition, and the Queens Completion problem as an example of unsupervised learning. As an alternative to statistical learning, algebraic learning may offer advantages in combining bottom-up and top-down information, formal concept derivation from data and large-scale parallelization.
}
\newpage

\tableofcontents

%\linenumbers

\newpage

\section{Introduction}

Algebras have played an important role in logic and top-down approaches in Artificial Intelligence (AI) \cite{Nilsson1991}. They are still an active area of research in information systems, for example in knowledge representation, queries and inference \cite{Pouly2011}. Machine learning (ML) branched out from AI as a bottom-up approach of learning from data. Here we show how to use an algebraic structure \cite{Burris} to learn from data. This research programme may then be seen as a proposal to naturally combine top-down and bottom-up approaches. More specifically, we are interested in an approach to learning from data that is parameter-free and transparent to make analysis and formal proofs easier. Also, we want to explore the formation of concepts from data as transformations that lead to a large reduction of the size of an algebraic representation.  

We show how to express learning problems as elements and relationships in an extended semilattice algebra. We give a concrete algebraic algorithm, the \emph{Sparse Crossing}, that finds solutions as sets of \enquote{atomic} elements, or atoms. Learning takes place by algebraic transformations that minimize the number of atoms. The algorithm is stochastic and discrete with no floating point operations.

The algebraic approach has important differences to more standard approaches. It does not use function minimization. Minimizing functions has proven very useful in ML. However, the functions typically used have complex geometries with local minima. Navigating these surfaces often requires large datasets and special methods to avoid getting stuck in the local minima. These surfaces depend on many parameters that might need tuning with heuristic procedures. 

Instead of function minimization, our algebraic algorithm uses cardinal minimization. i.e. minimization of the number of atoms. It learns smoothly, with error rates in the test set decreasing with the number of training examples and with no risk of getting trapped in local minima. We found no evidence of overfitting using algebraic learning, so we do not use a validation dataset. Also, it is parameter-free, so there is no need to prealocate parameter values like a network architecture, with the algebra growing by itself using the training data. 

We studied Algebraic Learning in four examples to illustrate different properties. We start with the toy supervised problem of learning to classify images by whether they contain a vertical bar or not. The simplicity of this problem allows for analysis. We show how algebraic learning explicitly finds that the positive examples are indeed those that contain a vertical bar. We show also that the number of solutions with low error is astronomically large and the learning algorithm just needs to find one of them. This might also be the case in other systems, but for algebras it can be demonstrated.

Algebraic Learning is designed to \enquote{compress} training examples into atoms and not directly aimed at reducing the error. For this reason, we had to establish a relationship between compression and accuracy. We found that an algebra picked at random among the ones obeying the training examples has an error rate in test data inversely proportional to compression. We tested this theoretical result against experimental data obtained applying the Sparse Crossing algorithm to the problem of distinguishing images with even number of vertical bars from those with an odd number of bars. We found that Sparse Crossing is as at least as efficient in transforming compression into accuracy and fits very well the theoretical result when error rate is small.

We also tested the performance of algebraic learning in handwritten character classification. We used a single abstraction stage (a single processing \emph{layer}) operating in raw data, without preprocessing and with a training set that contains miss-labels. Algebraic learning achieves a good accuracy of about $99\%$ when distinguishing a digit from the rest. This is done with no overfitting and even when accuracy is not an explicit target of the algorithm. 

Our last example is the $N$-blocked $M\times M$ Queens Completion problem. Starting from $N$ blocked queens on an $M\times M$ chessboard, we need to place $M-N$ queens on the board in non-attacking positions. We encode board and attack rules as algebraic relations, and show that Algebraic Learning generates complete solutions for the standard $8 \times 8$ board and also in larger boards. Learning in this example is unsupervised, with the algebra learning the structure of the search space.

\section{The embedding algorithm}

\subsection{A toy problem illustrating algebraic learning}

Consider the very simple problem of learning how to classify $2 \times 2$ images in which pixels can be in black or white. We will learn how to classify these images into two classes using as training data the following five examples
\begin{center}
\begin{tikzpicture}[scale=0.5]

% First positive example

\draw[thick] (0,0)--(2,0);    
\draw[thick] (1,0)--(1,2);    
\draw[thick] (2,0)--(2,2);    
\draw[thick] (0,0)--(0,2);  
\draw[thick] (0,2)--(2,2);
\draw[thick] (0,1)--(2,1);

%\draw[fill=black] (-0.5,1.5) rectangle (0.5,2.5);  
\draw[fill=black] (0,0) rectangle (1,2);  
%\draw[fill=black] (1,0) rectangle (2,1);  

% Second positive example

\draw[thick] (3,0)--(5,0);    
\draw[thick] (4,0)--(4,2);    
\draw[thick] (5,0)--(5,2);    
\draw[thick] (3,0)--(3,2);  
\draw[thick] (3,2)--(5,2);
\draw[thick] (3,1)--(5,1);

\draw[fill=black] (4,0) rectangle (5,2);  

% Separator between positive and negative examples

\draw[thick] (5.5,-0.5)--(5.5,2.5);

% First negative example

\draw[thick] (6,0)--(8,0);    
\draw[thick] (6,0)--(6,2);    
\draw[thick] (8,0)--(8,2);    
\draw[thick] (6,1)--(8,1);  
\draw[thick] (6,2)--(8,2);
\draw[thick] (7,0)--(7,2);

\draw[fill=black] (6,1) rectangle (7,2);  
\draw[fill=black] (7,0) rectangle (8,1);  

% Second negative example

\draw[thick] (9,0)--(11,0);    
\draw[thick] (10,0)--(10,2);    
\draw[thick] (11,0)--(11,2);    
\draw[thick] (9,0)--(9,2);  
\draw[thick] (9,2)--(11,2);
\draw[thick] (9,1)--(11,1);

\draw[fill=black] (10,1) rectangle (11,2);  

% third negative example

\draw[thick] (12,0)--(14,0);    
\draw[thick] (13,0)--(13,2);    
\draw[thick] (14,0)--(14,2);    
\draw[thick] (12,0)--(12,2);  
\draw[thick] (12,2)--(14,2);
\draw[thick] (12,1)--(14,1);

\draw[fill=black] (13,0) rectangle (14,1);  

\end{tikzpicture}
\end{center}
We label the two examples on the left as belonging to the  \enquote{positive class} because they include a black vertical bar, and name them as $T_{1}^{+}$ and $T_{2}^{+}$. The three examples on the right are the  \enquote{negative} class, $T_{1}^{-}$, $T_{2}^{-}$ and $T_{3}^{-}$. Our goal is to build an algebra that can learn from the training how to classify new images as belonging to the positive or negative class.

\subsection{Elements of the algebra}

To embed a problem into an algebra we need the algebra to have at least one operator that is idempotent, associative and commutative. In this paper we use semilattices, the simplest algebraic structures with such operator. 

We will have three types of elements: constants, terms and atoms. Constants are the primitive description elements of our embedding problem. For images, for example, constants can be each of the pixels in black or white. For our $2 \times 2$ images we would then have the $8$ constants\newline
\begin{center}
\begin{tikzpicture}[scale=0.5]
% Now the constants
% First constant

\draw[thick] (-0.25,-5)--(0.75,-5);    
\draw[thick] (-0.25,-4)--(0.75,-4);    
\draw[thick] (-0.25,-5)--(-0.25,-4);    
\draw[thick] (0.75,-5)--(0.75,-4);  

\draw[fill=black] (-0.50,-5.25) rectangle (0,-4.75);  

% Second constant

\draw[thick] (1.5,-5)--(2.5,-5);    
\draw[thick] (1.5,-4)--(2.5,-4);    
\draw[thick] (1.5,-5)--(1.5,-4);    
\draw[thick] (2.5,-5)--(2.5,-4);  

\draw[fill=black] (1.25,-4.25) rectangle (1.75,-3.75); 

% Third constant

\draw[thick] (3.25,-5)--(4.25,-5);    
\draw[thick] (3.25,-4)--(4.25,-4);    
\draw[thick] (3.25,-5)--(3.25,-4);    
\draw[thick] (4.25,-5)--(4.25,-4);  

\draw[fill=black] (4,-4.25) rectangle (4.5,-3.75);  

% Fourth constant

\draw[thick] (5,-5)--(6,-5);    
\draw[thick] (5,-4)--(6,-4);    
\draw[thick] (5,-5)--(5,-4);    
\draw[thick] (6,-5)--(6,-4);  

\draw[fill=black] (5.75,-5.25) rectangle (6.25,-4.75);  

% 5th constant (add 6.25 in x coordinate, and black goes to white)

\draw[thick] (6.75,-5)--(7.75,-5);    
\draw[thick] (6.75,-4)--(7.75,-4);    
\draw[thick] (6.75,-5)--(6.75,-4);    
\draw[thick] (7.75,-5)--(7.75,-4);  

\draw[fill=white] (6.5,-5.25) rectangle (7,-4.75);  

% 6th constant

\draw[thick] (8.50,-5)--(9.5,-5);    
\draw[thick] (8.50,-4)--(9.5,-4);    
\draw[thick] (8.50,-5)--(8.5,-4);    
\draw[thick] (9.50,-5)--(9.50,-4);  

\draw[fill=white] (8.25,-4.25) rectangle (8.75,-3.75); 

% 7th constant

\draw[thick] (10.25,-5)--(11.25,-5);    
\draw[thick] (10.25,-4)--(11.25,-4);    
\draw[thick] (10.25,-5)--(10.25,-4);    
\draw[thick] (11.25,-5)--(11.25,-4);  

\draw[fill=white] (11,-4.25) rectangle (11.5,-3.75);  

% 8th constant

\draw[thick] (12,-5)--(13,-5);    
\draw[thick] (12,-4)--(13,-4);    
\draw[thick] (12,-5)--(12,-4);    
\draw[thick] (13,-5)--(13,-4);  

\draw[fill=white] (12.75,-5.25) rectangle (13.25,-4.75);  
\end{tikzpicture}
\end{center}
that we write as $c_{1}$ to $c_8$. 

The terms are formed by operating constants with the \enquote{merge} (or \enquote{idempotent summation}) operation, for which we use the symbol $\odot$.  This is our binary operation that is commutative, associative and idempotent. In the case of terms describing images, terms are sets of pixels. For example, the first example in the training set is a term that can be expressed as the merge of four constants as  
\begin{center}
\begin{tikzpicture}[scale=0.5]

% First positive example

\draw[thick] (0,0)--(2,0);    
\draw[thick] (1,0)--(1,2);    
\draw[thick] (2,0)--(2,2);    
\draw[thick] (0,0)--(0,2);  
\draw[thick] (0,2)--(2,2);
\draw[thick] (0,1)--(2,1);

%\draw[fill=black] (-0.5,1.5) rectangle (0.5,2.5);  
\draw[fill=black] (0,0) rectangle (1,2);  
%\draw[fill=black] (1,0) rectangle (2,1);  

% Equal symbol

\node[text width=0.1cm] at (2.9,0.9) 
   {$\LARGE{=}$};

% First constant

\draw[thick] (4.25,1.5)--(5.25,1.5);    
\draw[thick] (4.25,0.5)--(5.25,0.5);    
\draw[thick] (4.25,0.5)--(4.25,1.5);    
\draw[thick] (5.25,0.5)--(5.25,1.5);  

\draw[fill=black] (4,0.25) rectangle (4.5,0.75);  

% MERGING SYMBOL

\node[text width=0.1cm] at (6,0.9) 
   {$\LARGE{\odot}$};

% Second constant

\draw[thick] (7.25,1.5)--(8.25,1.5);    
\draw[thick] (7.25,0.5)--(8.25,0.5);    
\draw[thick] (8.25,0.5)--(8.25,1.5);    
\draw[thick] (7.25,0.5)--(7.25,1.5);  

\draw[fill=black] (7,1.25) rectangle (7.5,1.75); 

% MERGING SYMBOL

\node[text width=0.1cm] at (9,0.9) 
   {$\LARGE{\odot}$};

% 7th constant

\draw[thick] (10.25,1.5)--(11.25,1.5);    
\draw[thick] (10.25,0.5)--(11.25,0.5);    
\draw[thick] (11.25,0.5)--(11.25,1.5);    
\draw[thick] (10.25,0.5)--(10.25,1.5);  

\draw[fill=white] (11,1.25) rectangle (11.5,1.75);  

\node[text width=0.1cm] at (12.25,0.9) 
   {$\LARGE{\odot}$};
   
% 8th constant

\draw[thick] (13.5,1.5)--(14.5,1.5);    
\draw[thick] (13.5,0.5)--(14.5,0.5);    
\draw[thick] (14.5,0.5)--(14.5,1.5);    
\draw[thick] (13.5,0.5)--(13.5,1.5);  

\draw[fill=white] (14.25,0.25) rectangle (14.75,0.75);  
\end{tikzpicture}
\end{center}
Atoms are elements created by the learning algorithm, and we reserve greek letters for them. Similarly to terms being a merge of constants, $\odot_{i} c_{i}$, each constant is a merge of atoms, $\odot_{i} \phi_{i}$. A term is therefore also a merge of atoms.

An idempotent operator defines a partial order. Specifically, the merge operator allows us to establish the inclusion relationship \enquote{$<$} between elements $a$ and $b$ of the algebra, $a<b$, iff $a \odot b = b$. Take as example our first training image, which was the merge of four constants, $T_{1}^{+}=c_{1} \odot c_{2} \odot c_{7} \odot c_{8}$. Any of these constants, say $c_{1}$, obeys $c_{1} < T_{1}^{+}$, because $c_{1} \odot T_{1}^{+}=T_{1}^{+}$. Similarly, for a constant made of atoms, each of these atoms is \enquote{in} or \enquote{included in} the constant.

The \enquote{training set} of the algebra consists of a set $R$ of positive and negative relations of the form $(v<T_{1}^{+})$ or $\neg(v<T_{3}^{-})$ where $v$ is a constant, one we want to describe to our algebra by using examples and counterexamples. 
\newline
\begin{center}
\begin{tikzpicture}[scale=0.35]

% First positive example

\draw[thick] (0,0)--(2,0);    
\draw[thick] (1,0)--(1,2);    
\draw[thick] (2,0)--(2,2);    
\draw[thick] (0,0)--(0,2);  
\draw[thick] (0,2)--(2,2);
\draw[thick] (0,1)--(2,1);

%\draw[fill=black] (-0.5,1.5) rectangle (0.5,2.5);  
\draw[fill=black] (0,0) rectangle (1,2);  
%\draw[fill=black] (1,0) rectangle (2,1);  

% Vertical constant and atoms

\node[text width=0.1cm] at (-3,1) 
    {v};
\node[text width=0.1cm] at (-1.7,1) 
    {{$<$}};

% Third negative example

\draw[thick] (12,0)--(14,0);    
\draw[thick] (13,0)--(13,2);    
\draw[thick] (14,0)--(14,2);    
\draw[thick] (12,0)--(12,2);  
\draw[thick] (12,2)--(14,2);
\draw[thick] (12,1)--(14,1);

\draw[fill=black] (13,0) rectangle (14,1);  

\node[text width=0.1cm] at (9,1) 
    {v};
\node[text width=0.1cm] at (10.7,1) 
    {{$\not<$}};

\end{tikzpicture}
\end{center}

The learning algorithm transforms semilattices into other semilattices in a series of steps until finding one that satisfies the training set $R$. Using Model Theory\cite{Tent2012} jargon, we want to find a model of the theory of semilattices extended with a set of literals (the training set $R$).

\subsection{Graph of the algebra}

We use a graph to make the abstract notion of algebra more concrete and computationally amenable. Nodes in the graph $G$ are elements of the algebra. An enormous amount of terms can be defined from a set of constants. The graph has nodes only for the subset of terms mentioned in the training relations plus the \enquote{pinning terms}, terms that are calculated by the embedding algorithm and that we introduce later. We do not need to have a node for each possible term or element of the algebra. 

A directed edge $a \rightarrow b$ is used to represent some of the inclusion relationships between elements, but not all,
\begin{linenomath}
\begin{equation}
  a\rightarrow b \,\, \Rightarrow \,\, a < b,
\end{equation}
\end{linenomath}
where the implication only holds left to right. We add to the graph edges pointing from the component constants of a term to the node of the term. If a term $T$ is \emph{defined} as the merge of constants $c_i$ then
\begin{linenomath}
\begin{equation}
T\equiv \odot_i c_i \,\, \Rightarrow \,\, \forall i(c_i \rightarrow T).
\end{equation}
\end{linenomath}
If all the component constants of a term $T$ are also component constants of another term $S$ we add the edge $T \rightarrow S$.  We always use edges if any of the elements involved are atoms, 
\begin{linenomath}
\begin{equation}
 \phi \rightarrow b \,\, \Leftrightarrow \,\, \phi < b.
\end{equation}
\end{linenomath}
Graph edges can be seen as a graphical representation of an additional relation defined in our algebra that is transitive but not commutative. Graphs represent algebras only when they are transitively closed with respect to the edges. Directed edges are typically represented with arrows. However, to avoid clutter we use simple straight lines instead of arrows pointing upwards in the figures, which is unambiguous because $G$ is acyclic.  Also to avoid clutter, in the drawings we do not plot all implicit edges (for example, from atoms to terms). We also add a \enquote{$0$} atom included in all constants. This is not strictly necessary but will make exposition simpler. Our starting graph has already the form
\begin{center}
\begin{tikzpicture}[scale=0.35]

% First positive example

\draw[thick] (0,0)--(2,0);    
\draw[thick] (1,0)--(1,2);    
\draw[thick] (2,0)--(2,2);    
\draw[thick] (0,0)--(0,2);  
\draw[thick] (0,2)--(2,2);
\draw[thick] (0,1)--(2,1);

%\draw[fill=black] (-0.5,1.5) rectangle (0.5,2.5);  
\draw[fill=black] (0,0) rectangle (1,2);  
%\draw[fill=black] (1,0) rectangle (2,1);  

% Second positive example

\draw[thick] (3,0)--(5,0);    
\draw[thick] (4,0)--(4,2);    
\draw[thick] (5,0)--(5,2);    
\draw[thick] (3,0)--(3,2);  
\draw[thick] (3,2)--(5,2);
\draw[thick] (3,1)--(5,1);

\draw[fill=black] (4,0) rectangle (5,2);  

% Separator between positive and negative examples

\draw[thick] (5.5,-0.5)--(5.5,2.5);

% First negative example

\draw[thick] (6,0)--(8,0);    
\draw[thick] (6,0)--(6,2);    
\draw[thick] (8,0)--(8,2);    
\draw[thick] (6,1)--(8,1);  
\draw[thick] (6,2)--(8,2);
\draw[thick] (7,0)--(7,2);

\draw[fill=black] (6,1) rectangle (7,2);  
\draw[fill=black] (7,0) rectangle (8,1);  

% Second negative example

\draw[thick] (9,0)--(11,0);    
\draw[thick] (10,0)--(10,2);    
\draw[thick] (11,0)--(11,2);    
\draw[thick] (9,0)--(9,2);  
\draw[thick] (9,2)--(11,2);
\draw[thick] (9,1)--(11,1);

\draw[fill=black] (10,1) rectangle (11,2);  

% third negative example

\draw[thick] (12,0)--(14,0);    
\draw[thick] (13,0)--(13,2);    
\draw[thick] (14,0)--(14,2);    
\draw[thick] (12,0)--(12,2);  
\draw[thick] (12,2)--(14,2);
\draw[thick] (12,1)--(14,1);

\draw[fill=black] (13,0) rectangle (14,1);  

% Now the constants

% First constant

\draw[thick] (-0.25,-5)--(0.75,-5);    
\draw[thick] (-0.25,-4)--(0.75,-4);    
\draw[thick] (-0.25,-5)--(-0.25,-4);    
\draw[thick] (0.75,-5)--(0.75,-4);  

\draw[fill=black] (-0.50,-5.25) rectangle (0,-4.75);  

% Second constant

\draw[thick] (1.5,-5)--(2.5,-5);    
\draw[thick] (1.5,-4)--(2.5,-4);    
\draw[thick] (1.5,-5)--(1.5,-4);    
\draw[thick] (2.5,-5)--(2.5,-4);  

\draw[fill=black] (1.25,-4.25) rectangle (1.75,-3.75); 

% Third constant

\draw[thick] (3.25,-5)--(4.25,-5);    
\draw[thick] (3.25,-4)--(4.25,-4);    
\draw[thick] (3.25,-5)--(3.25,-4);    
\draw[thick] (4.25,-5)--(4.25,-4);  

\draw[fill=black] (4,-4.25) rectangle (4.5,-3.75);  

% Fourth constant

\draw[thick] (5,-5)--(6,-5);    
\draw[thick] (5,-4)--(6,-4);    
\draw[thick] (5,-5)--(5,-4);    
\draw[thick] (6,-5)--(6,-4);  

\draw[fill=black] (5.75,-5.25) rectangle (6.25,-4.75);  

% 5th constant (add 6.25 in x coordinate, and black goes to white)

\draw[thick] (6.75,-5)--(7.75,-5);    
\draw[thick] (6.75,-4)--(7.75,-4);    
\draw[thick] (6.75,-5)--(6.75,-4);    
\draw[thick] (7.75,-5)--(7.75,-4);  

\draw[fill=white] (6.5,-5.25) rectangle (7,-4.75);  

% 6th constant

\draw[thick] (8.50,-5)--(9.5,-5);    
\draw[thick] (8.50,-4)--(9.5,-4);    
\draw[thick] (8.50,-5)--(8.5,-4);    
\draw[thick] (9.50,-5)--(9.50,-4);  

\draw[fill=white] (8.25,-4.25) rectangle (8.75,-3.75); 

% 7th constant

\draw[thick] (10.25,-5)--(11.25,-5);    
\draw[thick] (10.25,-4)--(11.25,-4);    
\draw[thick] (10.25,-5)--(10.25,-4);    
\draw[thick] (11.25,-5)--(11.25,-4);  

\draw[fill=white] (11,-4.25) rectangle (11.5,-3.75);  

% 8th constant

\draw[thick] (12,-5)--(13,-5);    
\draw[thick] (12,-4)--(13,-4);    
\draw[thick] (12,-5)--(12,-4);    
\draw[thick] (13,-5)--(13,-4);  

\draw[fill=white] (12.75,-5.25) rectangle (13.25,-4.75);  
    
% Connectors from examples to constants

\draw[thin] (1,-0.3)--(0.25,-3.7);  
\draw[thin] (1,-0.3)--(2,-3.7);  
\draw[thin] (1,-0.3)--(10.75,-3.7);  
\draw[thin] (1,-0.3)--(12.5,-3.7);  

\draw[thin] (4,-0.3)--(5.5,-3.7); 
\draw[thin] (4,-0.3)--(3.75,-3.7); 
\draw[thin] (4,-0.3)--(7.25,-3.7); 
\draw[thin] (4,-0.3)--(9,-3.7); 

\draw[thin] (7,-0.3)--(2,-3.7); 
\draw[thin] (7,-0.3)--(5.5,-3.7); 
\draw[thin] (7,-0.3)--(7.25,-3.7); 
\draw[thin] (7,-0.3)--(10.75,-3.7); 

\draw[thin] (10,-0.3)--(3.75,-3.7); 
\draw[thin] (10,-0.3)--(7.25,-3.7); 
\draw[thin] (10,-0.3)--(9,-3.7); 
\draw[thin] (10,-0.3)--(12.5,-3.7); 

\draw[thin] (13,-0.3)--(5.5,-3.7); 
\draw[thin] (13,-0.3)--(7.25,-3.7); 
\draw[thin] (13,-0.3)--(9,-3.7); 
\draw[thin] (13,-0.3)--(10.75,-3.7); 

% Vertical constant and atoms

\node[text width=0.1cm] at (-3,-4.5) 
    {{$v$}};

\node[text width=0.1cm] at (6.7,-7.6) 
    {$0$};
    
% Connectors to atoms

% 0-v

\draw[thin] (7,-7)--(-3,-5.3);

% 0-contant 8

\draw[thin] (7,-7)--(12.5,-5.3);
\draw[thin] (7,-7)--(10.75,-5.3); 
\draw[thin] (7,-7)--(9,-5.3); 
\draw[thin] (7,-7)--(7.25,-5.3); 
\draw[thin] (7,-7)--(5.5,-5.3); 
\draw[thin] (7,-7)--(3.75,-5.3); 
\draw[thin] (7,-7)--(2,-5.3); 
\draw[thin] (7,-7)--(0.25,-5.3); 

% alpha

%\draw[thin] (-2.8,-5.3)--(-2.8,-9.6); 
%\draw[thin] (0.25,-5.3)--(-2.8,-9.6); 
%\draw[thin] (5.5,-5.3)--(-2.8,-9.6); 

% beta

%\draw[thin] (-2.8,-5.3)--(0.25,-9.6); 
%\draw[thin] (0.25,-5.3)--(0.25,-9.6); 
%\draw[thin] (3.75,-5.3)--(0.25,-9.6); 

% gamma

%\draw[thin] (0.25,-5.3)--(2,-9.6); 

% delta

%\draw[thin] (3.75,-5.3)--(3.75,-9.6); 

% delta

%\draw[thin] (5.5,-5.3)--(5.5,-9.6); 

%\draw[thin] (0.25,-5.3)--(0.25,-9.6); 
%\draw[thin] (0.25,-5.3)--(6.4,-9.6); 
%\draw[thin] (12.5,-5.3)--(6.4,-9.6); 

\end{tikzpicture}
\end{center}
From the edges we define the partial order $<$ as
\begin{linenomath}
\begin{equation}
\label{inclusionDef}
\forall\phi ((\phi \not\rightarrow a) \vee (\phi \rightarrow b)) \,\, \Leftrightarrow \,\, a < b,
\end{equation}
\end{linenomath}
where the universal quantifier runs over all atoms. The formula says that $a<b$ if and only if all the atoms edged to $a$ are also edged to $b$. 

When the graph is transitively closed it describes an algebra we call $M$. This algebra evolves during the learning process producing a model of the training relations $R$ at the end of the embedding. When we talk about $M$ we mean the algebra described by the graph at a given stage of the algorithm. 

\subsection{The dual algebra}

The algebraic manipulations we need to do are easier to perform using not only the algebra $M$ but also an auxiliary structure $M^{*}$. This $M^{*}$ is a semilattice closely related (but different) to the dual of $M$ \cite{Burris}, that we still call \enquote{the dual} and whose properties we detail in this section. We also use an extended algebraic structure $S$ that contains both semilattices $M$ and $M^*$, which have universes that are disjoint sets, i.e, an element of $S$ is either and element of $M$ or an element of $M^*$. The unary function $[ \ ]$ defined for $S$ maps the elements of $M$, say $a$ and $b$, into the elements $[a]$ and $[b]$ in $M^*$, that we call duals of $a$ and $b$. The duals of constants and terms are always constants and the dual of atoms are a new kind of element we name \enquote{dual-of-atom}. $M^{*}$ has constants, dual-of-atoms and atoms but it does not contain terms. Atoms of $M^{*}$ are not duals of any element of $M$. We refer to $M^{*}$ as the dual algebra and to $M$ as \enquote{the master} algebra.

Our algebra $S$ is characterized by the transitive, noncommutative relation \enquote{$\rightarrow$}, the partial order \enquote{$<$} and the unary operator $[ \ ]$. Besides the transitivity of \enquote{$\rightarrow$} and the definition of \enquote{$<$} given by Equation (\ref{inclusionDef}) we introduce the additional axiom,
\begin{linenomath}
\begin{equation}
a\rightarrow b \,\, \Rightarrow \,\, [b] \rightarrow [a],
\end{equation}
\end{linenomath}
that, again, only works from left to right. It means that the edges of the graph of $M$ are also edges of the graph of $M^*$ albeit reversed.

The auxiliary semilattice $M^*$ contains the images of the elements of $M$ under the unary operator $[ \ ]$, and has the reversed edges of $M$ plus some additional edges of its own and its own atoms. We introduced edges in $M$ to encode definitional relations like how a given training image (a term) is made up of particular pixels constants. In $M^*$ we add additional edges for the positive order relations of $R$ such as $v<T_{1}^{+}$,
\begin{linenomath}
\begin{equation}
[T_{1}^{+}] \rightarrow [v].
\end{equation}
\end{linenomath}
Positive order relations of our choosing are encoded with edges in $M^*$ and emerge in $M$ as reversed order relations, i.e. we get $(v<T_{1}^{+})$ from $[T_{1}^{+}] \rightarrow [v]$ at some point of the embedding process.

The graph of the dual $M^*$ has all the reversed edges of $M$ plus the edges corresponding to the positive order relations of $R$ and it should be also transitively closed. In this classification example, our training relations establish that $v$ is included in the positive training terms $T_{1}^{+}$ and $T_{2}^{+}$, so there are edges from the duals of both terms to the dual of $v$. Note again that these type of edges for relations of $R$ are not in the graph of $M$.
\begin{center}
\begin{tikzpicture}[scale=0.35]

% First positive example

\draw[thick] (0,-11)--(2,-11);    
\draw[thick] (1,-11)--(1,-9);    
\draw[thick] (2,-11)--(2,-9);    
\draw[thick] (0,-11)--(0,-9);  
\draw[thick] (0,-9)--(2,-9);
\draw[thick] (0,-10)--(2,-10);

%\draw[fill=black] (-0.5,1.5) rectangle (0.5,2.5);  
\draw[fill=black] (0,-11) rectangle (1,-9);  
%\draw[fill=black] (1,0) rectangle (2,1);  

% Second positive example

\draw[thick] (3,-11)--(5,-11);    
\draw[thick] (4,-11)--(4,-9);    
\draw[thick] (5,-11)--(5,-9);    
\draw[thick] (3,-11)--(3,-9);  
\draw[thick] (3,-9)--(5,-9);
\draw[thick] (3,-10)--(5,-10);

\draw[fill=black] (4,-11) rectangle (5,-9);  

% Separator between positive and negative examples

\draw[thick] (5.5,-11.5)--(5.5,-8.5);

% First negative example

\draw[thick] (6,-11)--(8,-11);    
\draw[thick] (7,-11)--(7,-9);    
\draw[thick] (8,-11)--(8,-9);    
\draw[thick] (6,-11)--(6,-9);  
\draw[thick] (6,-9)--(8,-9);
\draw[thick] (6,-10)--(8,-10);

\draw[fill=black] (6,-10) rectangle (7,-9);  
\draw[fill=black] (7,-11) rectangle (8,-10);  

% Second negative example

\draw[thick] (9,-11)--(11,-11);    
\draw[thick] (10,-11)--(10,-9);    
\draw[thick] (11,-11)--(11,-9);    
\draw[thick] (9,-11)--(9,-9);  
\draw[thick] (9,-9)--(11,-9);
\draw[thick] (9,-10)--(11,-10);

\draw[fill=black] (10,-10) rectangle (11,-9);  

% third negative example

\draw[thick] (12,-11)--(14,-11);    
\draw[thick] (13,-11)--(13,-9);    
\draw[thick] (14,-11)--(14,-9);    
\draw[thick] (12,-11)--(12,-9);  
\draw[thick] (12,-9)--(14,-9);
\draw[thick] (12,-10)--(14,-10);

\draw[fill=black] (13,-11) rectangle (14,-10);  

% Now the constants

% First atom

\draw[thick] (-0.25,-5)--(0.75,-5);    
\draw[thick] (-0.25,-4)--(0.75,-4);    
\draw[thick] (-0.25,-5)--(-0.25,-4);    
\draw[thick] (0.75,-5)--(0.75,-4);  

\draw[fill=black] (-0.50,-5.25) rectangle (0,-4.75);  

% Second constant

\draw[thick] (1.5,-5)--(2.5,-5);    
\draw[thick] (1.5,-4)--(2.5,-4);    
\draw[thick] (1.5,-5)--(1.5,-4);    
\draw[thick] (2.5,-5)--(2.5,-4);  

\draw[fill=black] (1.25,-4.25) rectangle (1.75,-3.75); 

% Third constant

\draw[thick] (3.25,-5)--(4.25,-5);    
\draw[thick] (3.25,-4)--(4.25,-4);    
\draw[thick] (3.25,-5)--(3.25,-4);    
\draw[thick] (4.25,-5)--(4.25,-4);  

\draw[fill=black] (4,-4.25) rectangle (4.5,-3.75);  

% Fourth constant

\draw[thick] (5,-5)--(6,-5);    
\draw[thick] (5,-4)--(6,-4);    
\draw[thick] (5,-5)--(5,-4);    
\draw[thick] (6,-5)--(6,-4);  

\draw[fill=black] (5.75,-5.25) rectangle (6.25,-4.75);  

% 5th constant (add 6.25 in x coordinate, and black goes to white)

\draw[thick] (6.75,-5)--(7.75,-5);    
\draw[thick] (6.75,-4)--(7.75,-4);    
\draw[thick] (6.75,-5)--(6.75,-4);    
\draw[thick] (7.75,-5)--(7.75,-4);  

\draw[fill=white] (6.5,-5.25) rectangle (7,-4.75);  

% 6th constant

\draw[thick] (8.50,-5)--(9.5,-5);    
\draw[thick] (8.50,-4)--(9.5,-4);    
\draw[thick] (8.50,-5)--(8.5,-4);    
\draw[thick] (9.50,-5)--(9.50,-4);  

\draw[fill=white] (8.25,-4.25) rectangle (8.75,-3.75); 

% 7th constant

\draw[thick] (10.25,-5)--(11.25,-5);    
\draw[thick] (10.25,-4)--(11.25,-4);    
\draw[thick] (10.25,-5)--(10.25,-4);    
\draw[thick] (11.25,-5)--(11.25,-4);  

\draw[fill=white] (11,-4.25) rectangle (11.5,-3.75);  

% 8th constant

\draw[thick] (12,-5)--(13,-5);    
\draw[thick] (12,-4)--(13,-4);    
\draw[thick] (12,-5)--(12,-4);    
\draw[thick] (13,-5)--(13,-4);  

\draw[fill=white] (12.75,-5.25) rectangle (13.25,-4.75);  
    
% Connectors from examples to constants

\draw[thin] (1,-8.7)--(0.25,-5.3);  
\draw[thin] (1,-8.7)--(2,-5.3);  
\draw[thin] (1,-8.7)--(10.75,-5.3);  
\draw[thin] (1,-8.7)--(12.5,-5.3);  

\draw[thin] (4,-8.7)--(3.75,-5.3); 
\draw[thin] (4,-8.7)--(5.5,-5.3); 
\draw[thin] (4,-8.7)--(7.25,-5.3); 
\draw[thin] (4,-8.7)--(9,-5.3); 

\draw[thin] (7,-8.7)--(2,-5.3); 
\draw[thin] (7,-8.7)--(5.5,-5.3); 
\draw[thin] (7,-8.7)--(7.25,-5.3); 
\draw[thin] (7,-8.7)--(10.75,-5.3); 

\draw[thin] (10,-8.7)--(3.75,-5.3); 
\draw[thin] (10,-8.7)--(7.25,-5.3); 
\draw[thin] (10,-8.7)--(9,-5.3); 
\draw[thin] (10,-8.7)--(12.5,-5.3); 

\draw[thin] (13,-8.7)--(5.5,-5.3); 
\draw[thin] (13,-8.7)--(7.25,-5.3); 
\draw[thin] (13,-8.7)--(9,-5.3); 
\draw[thin] (13,-8.7)--(10.75,-5.3); 

% Vertical constant and atoms

\node[text width=0.1cm] at (-3.25,-4.5) 
    {$[v]$};

\node[text width=0.1cm] at (7,-13.5) 
    {$0^{*}$};

\node[text width=0.1cm] at (6.4,-2) 
    {$[0]$};

%\node[text width=0.1cm] at (7,-14.5) 
%   {$\zeta_1$};

%\node[text width=0.1cm] at (10,-14.5) 
%    {$\zeta_2$};
    
%\node[text width=0.1cm] at (13,-14.5) 
%   {$\zeta_3$};    

%\node[text width=0.1cm] at (-3,-10) 
%    {$\Large{\alpha}$};

%\node[text width=0.1cm] at (0.25,-10) 
%    {$\Large{\beta}$};

%\node[text width=0.1cm] at (2,-10) 
%    {$\Large{\gamma}$};
    
%\node[text width=0.1cm] at (3.75,-10) 
%    {$\Large{\delta}$};

%\node[text width=0.1cm] at (5.5,-10) 
%    {$\Large{\epsilon}$};
    
%\node[text width=0.1cm] at (0.25,-10) 
%    {$\LARGE{\alpha}$};

%\node[text width=0.1cm] at (6.4,-10) 
%   {$\LARGE{0}$};
    
% Connectors to atoms

% 0*-Examples

\draw[thin] (7,-13)--(1,-11.3); 
\draw[thin] (7,-13)--(13,-11.3); 
\draw[thin] (7,-13)--(4,-11.3); 
\draw[thin] (7,-13)--(7,-11.3); 
\draw[thin] (7,-13)--(10,-11.3); 
%\draw[thin] (1,-13.9)--(13,-11.3); 
% First two examples -v

% eta-Examples
%\draw[thin] (7,-13.9)--(7,-11.3); 
%\draw[thin] (10,-13.9)--(10,-11.3); 
%\draw[thin] (13,-13.9)--(13,-11.3); 

% First two examples -v

\draw[thin] (1,-8.7)--(-2.8,-5.3); 
\draw[thin] (4,-8.7)--(-2.8,-5.3); 

% [0] to constants

\draw[thin] (7,-2.7)--(-2.8,-3.7);
\draw[thin] (7,-2.7)--(12.5,-3.7);
\draw[thin] (7,-2.7)--(10.75,-3.7);
\draw[thin] (7,-2.7)--(9,-3.7);
\draw[thin] (7,-2.7)--(7.25,-3.7);
\draw[thin] (7,-2.7)--(5.5,-3.7);
\draw[thin] (7,-2.7)--(3.75,-3.7);
\draw[thin] (7,-2.7)--(2,-3.7);---
\draw[thin] (7,-2.7)--(0.25,-3.7);
% alpha

%\draw[thin] (-2.8,-5.3)--(-2.8,-9.6); 
%\draw[thin] (0.25,-5.3)--(-2.8,-9.6); 
%\draw[thin] (5.5,-5.3)--(-2.8,-9.6); 

% beta

%\draw[thin] (-2.8,-5.3)--(0.25,-9.6); 
%\draw[thin] (0.25,-5.3)--(0.25,-9.6); 
%\draw[thin] (5.5,-5.3)--(0.25,-9.6); 

% gamma

%\draw[thin] (0.25,-5.3)--(2,-9.6); 

% delta

%\draw[thin] (3.75,-5.3)--(3.75,-9.6); 

% delta

%\draw[thin] (5.5,-5.3)--(5.5,-9.6); 

%\draw[thin] (0.25,-5.3)--(0.25,-9.6); 
%\draw[thin] (0.25,-5.3)--(6.4,-9.6); 
%\draw[thin] (12.5,-5.3)--(6.4,-9.6); 

\end{tikzpicture}
\end{center}
At the top of the graph of $M^{*}$ we draw the duals of the atoms of $M$, here only $[0]$, and at the bottom of the graph we draw the atoms of $M^{*}$, here $0^*$, again included to make our exposition simpler.

\subsection{Atomized models}

Equation (\ref{inclusionDef}) defines how to derive the partial order from the transitive, noncommutative edge relation \enquote{$\rightarrow$} and an special kind of elements we call \enquote{atoms}. We say that a model for which there is a description of the partial order in terms of a set of atoms is an \enquote{atomized} model.  In an atomized model all elements are sets of atoms. Using the language of Universal Algebra, when an algebra is atomized it explicitly becomes a direct product of directly idecomposable algebras \cite{Burris}. This does not mean, however, that we are restricting ourselves to some subset of possible models. The Stone theorem grants that any semilattice model can be described as an atomized model \cite{Burris}. 

We know how to derive the partial order from the atoms and edges but we have not given yet a definition for the idempotent operator.  The merge (or idempotent summation) of $a$ and $b$ is the element of the algebra atomized by a set of atoms that is the union of the atoms edged to $a$ and the atoms edged to $b$. The idempotent operator becomes a trivial set union of atoms. Obviously this operation is idempotent, commutative and associative. It is also consistent with our partial order given in equation \ref{inclusionDef} that satisfies $a<b$ iff $a \odot b=b$. Consistently, the partial order becomes the set inclusion. 

Before we continue with the embedding algorithm we are going to introduce some notation and redefine the problem we are trying to solve in terms of sets of atoms. In \textbf{Appendix \ref{Notation}} we define some useful sets. For the moment it is enough to consider the set ${\bf{GL}^{a}}(x)$ which is simply the set of atoms edged to element $x$ that is defined, as always, only when the graph is transitively closed. The \enquote{\textbf{G}} refers to the graph, the \enquote{\textbf{L}} to the lower segment and the superscript \enquote{\textbf{a}} to the atoms. The merge of $a$ and $b$ corresponds with the set of atoms
\begin{linenomath}
\begin{equation}
{\bf{GL}^{a}}(a \odot b)={\bf{GL}^{a}}(a) \cup {\bf{GL}^{a}}(b).
\label{eq:positive_encoding}
\end{equation}
\end{linenomath}
For our toy problem, we want a description of the constant $v$ and for the pixels (also constants) as a sets of atoms. Specifically, we want a model for which $v$ is a set included in the positive training images, ${T_{1}^{+}}$ and ${T_{2}^{+}}$ as
\begin{linenomath}
\begin{equation}
v<T_{i}^{+} \,\, \Leftrightarrow \,\, {\bf{GL}^{a}}(v) \subset {\bf{GL}^{a}}(T_{i}^{+}),
\label{eq:positive_encoding}
\end{equation}
\end{linenomath}
where the atoms of a term are the union of the atoms of its component constants. We are also looking for a particular atomic model for which the atoms of constant $v$ are not all in the terms corresponding with negative training examples
\begin{linenomath}
\begin{equation}
v\not <T_{i}^{-} \,\, \Leftrightarrow \,\,  {\bf{GL}^{a}}(v) \not \subset {\bf{GL}^{a}}(T_{i}^{-}).
\label{eq:negative_encoding}
\end{equation}
\end{linenomath}
The difficulty in finding the model lies in enforcing positive and negative training relations simultaneously, which translates in resolving a large system of equations and inequations over sets. The sets are made of elements we create in the process, the atoms, and there is the added difficulty of finding sets as small and as random (or as free) as possible. In {\bf{Sections \ref{memorizingGeneralizing}}} we introduce the concept of algebraic freedom and discuss its connection with randomness. 

We will use an operation, the \textit{crossing}, to enforce positive relations one by one. By doing so the model evolves through a series of semilattice models, all atomized, until becoming the model we want.  We can build the model step by step thanks to an invariance property related to a construct we name \textit{trace}. In the next sections we explain the trace and the crossing operation. After this we will show how to further reduce the size of model with a \textit{reduction} operation and how to do batch training. We will explain these operations for our toy example explicitly, and also give an analysis of the exact and approximate solutions.

\subsection{Trace and trace constraints} \label{traceConstraints}
The trace is central for the embedding procedure as a guiding tool for algebraic transformations. By operating the algebra while keeping the trace of some elements invariant, we can control the global effects caused by our local changes. 

The \textit{trace} $\textbf{Tr} (x)$ maps an element $x \in M$ to a set of atoms in $M^{*}$. To calculate the trace of $x$, we find first its atoms in the graph of $M$, which we write as ${\bf{GL}^{a}}(x)$. Say these are $N$ atoms ${\phi_i}$, with $\phi_{i} \rightarrow x$. Since atoms are minima of $M$, dual of atoms are maxima of $M^{*}$, so for each atom $\phi_{i}$ of $x$ there is a dual-of-atom at the top of the graph of $M^{*}$, $[\phi_i]$. Each of these $[\phi_i]$ also have an associated set of atoms in $M^{*}$, ${\bf{GL}^{a}}([\phi_{i}])$. The \textit{trace} of $x$ is defined as the intersection of these $N$ sets, ${\bf{Tr}}(x)=\bigcap_{i =1,2,...,N} {\bf{GL}^{a}}([\phi_i])$. Consistently, the trace of an atom $\phi$ equals ${\bf{Tr}}(\phi) \equiv {\bf{GL}^{a}}([\phi])$. In general we can write the trace as
\begin{linenomath}
\begin{equation}
\label{deftrace}
{\bf{Tr}}(x) \equiv \cap_{\phi \in {\bf{GL}^{a}}(x)} {\bf{GL}^{a}}([\phi]).
\end{equation}
\end{linenomath}
From this definition it follows that the \textit{trace} has the linearity property
\begin{linenomath}
\begin{equation}
{\bf{Tr}}(a \odot b)={\bf{Tr}}(a) \cap {\bf{Tr}}(b),
\end{equation}
\end{linenomath}
as the atoms in $M$ for $a \odot b$ are the union of the atoms of $a$ and the atoms of $b$ and therefore the trace is the intersection of the traces of $a$ and $b$. From this linearity and the definition of the order relation, $a<b$ iff $a \odot b=b$, it follows that an order relation is related to the traces as
\begin{linenomath}
\begin{equation}
a < b \,\, \Rightarrow \,\, {\bf{Tr}}(b) \subset {\bf{Tr}}(a).
\end{equation}
\end{linenomath}
This makes a correspondence between order relations in $M$ and trace interrelations between $M$ and $M^*$ that we call trace constraints. For our toy problem, we are interested in obeying trace constraints for the positive training examples, $v<T_{i}^{+}$, for which we then need to enforce ${{\bf{Tr}}(T_{i}^{+}) \subset \bf{Tr}}(v)$,
\begin{linenomath}
\begin{equation}
\textnormal{for} \ v<T_{i}^{+} \ \textnormal{enforce} \ {\bf{Tr}}(T_{i}^{+}) \subset {\bf{Tr}}(v).
\end{equation}
\end{linenomath}
This does not cause $v<T_{i}^{+}$ but it provides a necessary starting point. For negative training examples, $T_{i}^{-}$, we want to obey that $\neg(v < T_{i}^{-})$. This inclusion does not follow from (\ref{deftrace}), however, it can always be enforced if the embedding strategy is consistent as 
\begin{linenomath}
\begin{equation}
\textnormal{for} \ \neg(v < T_{i}^{-}) \ \textnormal{enforce} \ {\bf{Tr}}(T_{i}^{-}) \not\subset {\bf{Tr}}(v).
\end{equation}
\end{linenomath}
Once the trace constraint is met, no transformation of $M$ can produce $v < T_{i}^{-}$ unless it alters the traces. This constraint prevents positive relations to appear in $M$ in places where we do not want them. 

While the operator $[ \ ]$ does not really map $M$ into its dual semilattice, the traces of the elements of $M$ form an algebra that very much resembles the dual of $M$. This new algebra has trace constraints in the place of order relations and set intersections in the place of set unions. There are still some subtle differences between a proper dual of $M$ and the dual algebra provided by the trace. For example, the trace is defined with the atoms of $M$ instead of the constants of $M$, so it depends on the particular atomization of $M$. While finding a proper dual of $M$ amounts in difficulty to calculate $M$ itself, enforcing the trace constraints is easier because we have the extra freedom of introducing new atoms in $M$. In addition, we do not have restrictions for the size of the traces. We do not care if the traces are large or small. 

We want an atomization for $M$ but first we have to calculate an atomization for $M^{*}$. The atomization we are going to build for $M$ does not correspond with the dual of $M^{*}$, neither it corresponds with the dual of the algebra defined by the trace. It corresponds with an algebra freer than the algebra described by the traces. In \textbf{Section \ref{memorizingGeneralizing}} we explain the role that algebraic freedom plays as a counterbalance to cardinal minimization. 

Enforcing the trace constraints might look challenging but it is relatively simple. We are aided by the encoding of training relations $R$ as directed edges in the graph of $M^{*}$ so when the graph is transitively closed the \enquote{reverted} positive relations $[T_{i}^{+}] < [v]$ are always satisfied. We can start, although this step is optional, by first requiring $M^*$ to satisfy the \enquote{reverted} negative relations, positive and negative. That is, if we want to enforce $\neg(v < T_{i}^{-})$ in $M$, we enforce $\neg( [T_{i}^{-}] < [v])$ by adding an atom to $[T_{i}^{-}]$ in $M^{*}$. In our toy example, for every negative example $T_{i}^{-}$ we then add an atom $\xi_{i} \rightarrow [T_{i}^{-}]$, so in our example we introduce three atoms $\zeta_1$, $\zeta_2$ and $\zeta_3$ in the graph of $M^{*}$,
\begin{center}
\begin{tikzpicture}[scale=0.35]

% First positive example

\draw[thick] (0,-11)--(2,-11);    
\draw[thick] (1,-11)--(1,-9);    
\draw[thick] (2,-11)--(2,-9);    
\draw[thick] (0,-11)--(0,-9);  
\draw[thick] (0,-9)--(2,-9);
\draw[thick] (0,-10)--(2,-10);

%\draw[fill=black] (-0.5,1.5) rectangle (0.5,2.5);  
\draw[fill=black] (0,-11) rectangle (1,-9);  
%\draw[fill=black] (1,0) rectangle (2,1);  

% Second positive example

\draw[thick] (3,-11)--(5,-11);    
\draw[thick] (4,-11)--(4,-9);    
\draw[thick] (5,-11)--(5,-9);    
\draw[thick] (3,-11)--(3,-9);  
\draw[thick] (3,-9)--(5,-9);
\draw[thick] (3,-10)--(5,-10);

\draw[fill=black] (4,-11) rectangle (5,-9);  

% Separator between positive and negative examples

\draw[thick] (5.5,-11.5)--(5.5,-8.5);

% First negative example

\draw[thick] (6,-11)--(8,-11);    
\draw[thick] (7,-11)--(7,-9);    
\draw[thick] (8,-11)--(8,-9);    
\draw[thick] (6,-11)--(6,-9);  
\draw[thick] (6,-9)--(8,-9);
\draw[thick] (6,-10)--(8,-10);

\draw[fill=black] (6,-10) rectangle (7,-9);  
\draw[fill=black] (7,-11) rectangle (8,-10);  

% Second negative example

\draw[thick] (9,-11)--(11,-11);    
\draw[thick] (10,-11)--(10,-9);    
\draw[thick] (11,-11)--(11,-9);    
\draw[thick] (9,-11)--(9,-9);  
\draw[thick] (9,-9)--(11,-9);
\draw[thick] (9,-10)--(11,-10);

\draw[fill=black] (10,-10) rectangle (11,-9);  

% third negative example

\draw[thick] (12,-11)--(14,-11);    
\draw[thick] (13,-11)--(13,-9);    
\draw[thick] (14,-11)--(14,-9);    
\draw[thick] (12,-11)--(12,-9);  
\draw[thick] (12,-9)--(14,-9);
\draw[thick] (12,-10)--(14,-10);

\draw[fill=black] (13,-11) rectangle (14,-10);  

% Now the constants

% First constant

\draw[thick] (-0.25,-5)--(0.75,-5);    
\draw[thick] (-0.25,-4)--(0.75,-4);    
\draw[thick] (-0.25,-5)--(-0.25,-4);    
\draw[thick] (0.75,-5)--(0.75,-4);

\draw[fill=black] (-0.50,-5.25) rectangle (0,-4.75);  

% Second constant

\draw[thick] (1.5,-5)--(2.5,-5);    
\draw[thick] (1.5,-4)--(2.5,-4);    
\draw[thick] (1.5,-5)--(1.5,-4);    
\draw[thick] (2.5,-5)--(2.5,-4);  

\draw[fill=black] (1.25,-4.25) rectangle (1.75,-3.75); 

% Third constant

\draw[thick] (3.25,-5)--(4.25,-5);    
\draw[thick] (3.25,-4)--(4.25,-4);    
\draw[thick] (3.25,-5)--(3.25,-4);    
\draw[thick] (4.25,-5)--(4.25,-4);  

\draw[fill=black] (4,-4.25) rectangle (4.5,-3.75);  

% Fourth constant

\draw[thick] (5,-5)--(6,-5);    
\draw[thick] (5,-4)--(6,-4);    
\draw[thick] (5,-5)--(5,-4);    
\draw[thick] (6,-5)--(6,-4);  

\draw[fill=black] (5.75,-5.25) rectangle (6.25,-4.75);  

% 5th constant (add 6.25 in x coordinate, and black goes to white)

\draw[thick] (6.75,-5)--(7.75,-5);    
\draw[thick] (6.75,-4)--(7.75,-4);    
\draw[thick] (6.75,-5)--(6.75,-4);    
\draw[thick] (7.75,-5)--(7.75,-4);  

\draw[fill=white] (6.5,-5.25) rectangle (7,-4.75);  

% 6th constant

\draw[thick] (8.50,-5)--(9.5,-5);    
\draw[thick] (8.50,-4)--(9.5,-4);    
\draw[thick] (8.50,-5)--(8.5,-4);    
\draw[thick] (9.50,-5)--(9.50,-4);  

\draw[fill=white] (8.25,-4.25) rectangle (8.75,-3.75); 

% 7th constant

\draw[thick] (10.25,-5)--(11.25,-5);    
\draw[thick] (10.25,-4)--(11.25,-4);    
\draw[thick] (10.25,-5)--(10.25,-4);    
\draw[thick] (11.25,-5)--(11.25,-4);  

\draw[fill=white] (11,-4.25) rectangle (11.5,-3.75);  

% 8th constant

\draw[thick] (12,-5)--(13,-5);    
\draw[thick] (12,-4)--(13,-4);    
\draw[thick] (12,-5)--(12,-4);    
\draw[thick] (13,-5)--(13,-4);  

\draw[fill=white] (12.75,-5.25) rectangle (13.25,-4.75);  
    
% Connectors from examples to constants

\draw[thin] (1,-8.7)--(0.25,-5.3);  
\draw[thin] (1,-8.7)--(2,-5.3);  
\draw[thin] (1,-8.7)--(10.75,-5.3);  
\draw[thin] (1,-8.7)--(12.5,-5.3);  

\draw[thin] (4,-8.7)--(3.75,-5.3); 
\draw[thin] (4,-8.7)--(5.5,-5.3); 
\draw[thin] (4,-8.7)--(7.25,-5.3); 
\draw[thin] (4,-8.7)--(9,-5.3); 

\draw[thin] (7,-8.7)--(2,-5.3); 
\draw[thin] (7,-8.7)--(5.5,-5.3); 
\draw[thin] (7,-8.7)--(7.25,-5.3); 
\draw[thin] (7,-8.7)--(10.75,-5.3); 

\draw[thin] (10,-8.7)--(3.75,-5.3); 
\draw[thin] (10,-8.7)--(7.25,-5.3); 
\draw[thin] (10,-8.7)--(9,-5.3); 
\draw[thin] (10,-8.7)--(12.5,-5.3); 

\draw[thin] (13,-8.7)--(5.5,-5.3); 
\draw[thin] (13,-8.7)--(7.25,-5.3); 
\draw[thin] (13,-8.7)--(9,-5.3); 
\draw[thin] (13,-8.7)--(10.75,-5.3); 

% Vertical constant and atoms

\node[text width=0.1cm] at (-3,-4.5) 
    {$v$};

\node[text width=0.1cm] at (1,-14.5) 
    {$0^{*}$};

\node[text width=0.1cm] at (6.4,-2) 
    {$[0]$};

\node[text width=0.1cm] at (7,-14.5) 
   {$\zeta_1$};

\node[text width=0.1cm] at (10,-14.5) 
    {$\zeta_2$};
    
\node[text width=0.1cm] at (13,-14.5) 
   {$\zeta_3$};    

%\node[text width=0.1cm] at (-3,-10) 
%    {$\Large{\alpha}$};

%\node[text width=0.1cm] at (0.25,-10) 
%    {$\Large{\beta}$};

%\node[text width=0.1cm] at (2,-10) 
%    {$\Large{\gamma}$};
    
%\node[text width=0.1cm] at (3.75,-10) 
%    {$\Large{\delta}$};

%\node[text width=0.1cm] at (5.5,-10) 
%    {$\Large{\epsilon}$};
    
%\node[text width=0.1cm] at (0.25,-10) 
%    {$\LARGE{\alpha}$};

%\node[text width=0.1cm] at (6.4,-10) 
%   {$\LARGE{0}$};
    
% Connectors to atoms

% 0*-Examples

\draw[thin] (1,-13.9)--(1,-11.3); 
\draw[thin] (1,-13.9)--(13,-11.3); 
\draw[thin] (1,-13.9)--(4,-11.3); 
\draw[thin] (1,-13.9)--(7,-11.3); 
\draw[thin] (1,-13.9)--(10,-11.3); 
%\draw[thin] (1,-13.9)--(13,-11.3); 
% First two examples -v

% eta-Examples
\draw[thin] (7,-13.9)--(7,-11.3); 
\draw[thin] (10,-13.9)--(10,-11.3); 
\draw[thin] (13,-13.9)--(13,-11.3); 

% First two examples -v

\draw[thin] (1,-8.7)--(-2.8,-5.3); 
\draw[thin] (4,-8.7)--(-2.8,-5.3); 

% [0] to constants

\draw[thin] (7,-2.7)--(-2.8,-3.7);
\draw[thin] (7,-2.7)--(12.5,-3.7);
\draw[thin] (7,-2.7)--(10.75,-3.7);
\draw[thin] (7,-2.7)--(9,-3.7);
\draw[thin] (7,-2.7)--(7.25,-3.7);
\draw[thin] (7,-2.7)--(5.5,-3.7);
\draw[thin] (7,-2.7)--(3.75,-3.7);
\draw[thin] (7,-2.7)--(2,-3.7);
\draw[thin] (7,-2.7)--(0.25,-3.7);
% alpha

%\draw[thin] (-2.8,-5.3)--(-2.8,-9.6); 
%\draw[thin] (0.25,-5.3)--(-2.8,-9.6); 
%\draw[thin] (5.5,-5.3)--(-2.8,-9.6); 

% beta

%\draw[thin] (-2.8,-5.3)--(0.25,-9.6); 
%\draw[thin] (0.25,-5.3)--(0.25,-9.6); 
%\draw[thin] (5.5,-5.3)--(0.25,-9.6); 

% gamma

%\draw[thin] (0.25,-5.3)--(2,-9.6); 

% delta

%\draw[thin] (3.75,-5.3)--(3.75,-9.6); 

% delta

%\draw[thin] (5.5,-5.3)--(5.5,-9.6); 

%\draw[thin] (0.25,-5.3)--(0.25,-9.6); 
%\draw[thin] (0.25,-5.3)--(6.4,-9.6); 
%\draw[thin] (12.5,-5.3)--(6.4,-9.6); 

\end{tikzpicture}
\end{center} 
The new atoms are not in the set ${\bf{GL}^{a}}([v])$ so the reverted negative relations are satisfied. In fact all reverted relations, positive and negative, are satisfied at this point. We have now the chance to detect if the input order relations are inconsistent. First, make sure that for each couple of terms $T_{1}$ and $T_{2}$ mentioned in the input order relations such that the component constants of $T_{1}$ are a subset of those of $T_{2}$ we have added the edge $[T_{2}] \rightarrow [T_{1}]$. At this point, after transitive closure, the reverted order relations are satisfied if and only if the embedding is consistent. 

If there are edges pointing in both directions between two elements of $M^{*}$ we can identify them as the same element. Two ore more elements of $M$ may share the same dual.

We have completed the preprocessing step that speeds up the enforcing of trace constraints and validates the consistency of the embedding. We start now enforcing the trace constraints for the negative examples, ${\bf{Tr}}(T_{i}^{-}) \not\subset {\bf{Tr}}(v)$. To compute the \textit{trace}, we place the graph for $M$ and for $M^{*}$ side to side, to left and right, respectively
\begin{center} 
%
%
% Model and dual together
%
%
\begin{tikzpicture}[scale=0.35]

% First positive example

\draw[thick] (0,0)--(2,0);    
\draw[thick] (1,0)--(1,2);    
\draw[thick] (2,0)--(2,2);    
\draw[thick] (0,0)--(0,2);  
\draw[thick] (0,2)--(2,2);
\draw[thick] (0,1)--(2,1);

%\draw[fill=black] (-0.5,1.5) rectangle (0.5,2.5);  
\draw[fill=black] (0,0) rectangle (1,2);  
%\draw[fill=black] (1,0) rectangle (2,1);  

% Second positive example

\draw[thick] (3,0)--(5,0);    
\draw[thick] (4,0)--(4,2);    
\draw[thick] (5,0)--(5,2);    
\draw[thick] (3,0)--(3,2);  
\draw[thick] (3,2)--(5,2);
\draw[thick] (3,1)--(5,1);

\draw[fill=black] (4,0) rectangle (5,2);  

% Separator between positive and negative examples

\draw[thick] (5.5,-0.5)--(5.5,2.5);

% First negative example

\draw[thick] (6,0)--(8,0);    
\draw[thick] (6,0)--(6,2);    
\draw[thick] (8,0)--(8,2);    
\draw[thick] (6,1)--(8,1);  
\draw[thick] (6,2)--(8,2);
\draw[thick] (7,0)--(7,2);

\draw[fill=black] (6,1) rectangle (7,2);  
\draw[fill=black] (7,0) rectangle (8,1);  

% Second negative example

\draw[thick] (9,0)--(11,0);    
\draw[thick] (10,0)--(10,2);    
\draw[thick] (11,0)--(11,2);    
\draw[thick] (9,0)--(9,2);  
\draw[thick] (9,2)--(11,2);
\draw[thick] (9,1)--(11,1);

\draw[fill=black] (10,1) rectangle (11,2);  

% third negative example

\draw[thick] (12,0)--(14,0);    
\draw[thick] (13,0)--(13,2);    
\draw[thick] (14,0)--(14,2);    
\draw[thick] (12,0)--(12,2);  
\draw[thick] (12,2)--(14,2);
\draw[thick] (12,1)--(14,1);

\draw[fill=black] (13,0) rectangle (14,1);  

% Now the constants

% First constant

\draw[thick] (-0.25,-5)--(0.75,-5);    
\draw[thick] (-0.25,-4)--(0.75,-4);    
\draw[thick] (-0.25,-5)--(-0.25,-4);    
\draw[thick] (0.75,-5)--(0.75,-4);  

\draw[fill=black] (-0.50,-5.25) rectangle (0,-4.75);  

% Second constant

\draw[thick] (1.5,-5)--(2.5,-5);    
\draw[thick] (1.5,-4)--(2.5,-4);    
\draw[thick] (1.5,-5)--(1.5,-4);    
\draw[thick] (2.5,-5)--(2.5,-4);  

\draw[fill=black] (1.25,-4.25) rectangle (1.75,-3.75); 

% Third constant

\draw[thick] (3.25,-5)--(4.25,-5);    
\draw[thick] (3.25,-4)--(4.25,-4);    
\draw[thick] (3.25,-5)--(3.25,-4);    
\draw[thick] (4.25,-5)--(4.25,-4);  

\draw[fill=black] (4,-4.25) rectangle (4.5,-3.75);  

% Fourth constant

\draw[thick] (5,-5)--(6,-5);    
\draw[thick] (5,-4)--(6,-4);    
\draw[thick] (5,-5)--(5,-4);    
\draw[thick] (6,-5)--(6,-4);  

\draw[fill=black] (5.75,-5.25) rectangle (6.25,-4.75);  

% 5th constant (add 6.25 in x coordinate, and black goes to white)

\draw[thick] (6.75,-5)--(7.75,-5);    
\draw[thick] (6.75,-4)--(7.75,-4);    
\draw[thick] (6.75,-5)--(6.75,-4);    
\draw[thick] (7.75,-5)--(7.75,-4);  

\draw[fill=white] (6.5,-5.25) rectangle (7,-4.75);  

% 6th constant

\draw[thick] (8.50,-5)--(9.5,-5);    
\draw[thick] (8.50,-4)--(9.5,-4);    
\draw[thick] (8.50,-5)--(8.5,-4);    
\draw[thick] (9.50,-5)--(9.50,-4);  

\draw[fill=white] (8.25,-4.25) rectangle (8.75,-3.75); 

% 7th constant

\draw[thick] (10.25,-5)--(11.25,-5);    
\draw[thick] (10.25,-4)--(11.25,-4);    
\draw[thick] (10.25,-5)--(10.25,-4);    
\draw[thick] (11.25,-5)--(11.25,-4);  

\draw[fill=white] (11,-4.25) rectangle (11.5,-3.75);  

% 8th constant

\draw[thick] (12,-5)--(13,-5);    
\draw[thick] (12,-4)--(13,-4);    
\draw[thick] (12,-5)--(12,-4);    
\draw[thick] (13,-5)--(13,-4);  

\draw[fill=white] (12.75,-5.25) rectangle (13.25,-4.75);  
    
% Connectors from examples to constants

\draw[thin] (1,-0.3)--(0.25,-3.7);  
\draw[thin] (1,-0.3)--(2,-3.7);  
\draw[thin] (1,-0.3)--(10.75,-3.7);  
\draw[thin] (1,-0.3)--(12.5,-3.7);  

\draw[thin] (4,-0.3)--(5.5,-3.7); 
\draw[thin] (4,-0.3)--(3.75,-3.7); 
\draw[thin] (4,-0.3)--(7.25,-3.7); 
\draw[thin] (4,-0.3)--(9,-3.7); 

\draw[thin] (7,-0.3)--(2,-3.7); 
\draw[thin] (7,-0.3)--(5.5,-3.7); 
\draw[thin] (7,-0.3)--(7.25,-3.7); 
\draw[thin] (7,-0.3)--(10.75,-3.7); 

\draw[thin] (10,-0.3)--(3.75,-3.7); 
\draw[thin] (10,-0.3)--(7.25,-3.7); 
\draw[thin] (10,-0.3)--(9,-3.7); 
\draw[thin] (10,-0.3)--(12.5,-3.7); 

\draw[thin] (13,-0.3)--(5.5,-3.7); 
\draw[thin] (13,-0.3)--(7.25,-3.7); 
\draw[thin] (13,-0.3)--(9,-3.7); 
\draw[thin] (13,-0.3)--(10.75,-3.7); 

% Vertical constant and atoms

\node[text width=0.1cm] at (-3,-4.5) 
    {{$v$}};

\node[text width=0.1cm] at (6.7,-7.9) 
    {$0$};
    
% Connectors to atoms

% 0-v

\draw[thin] (7,-7.3)--(-3,-5.3);

% 0-contant 8

\draw[thin] (7,-7.3)--(12.5,-5.3);
\draw[thin] (7,-7.3)--(10.75,-5.3); 
\draw[thin] (7,-7.3)--(9,-5.3); 
\draw[thin] (7,-7.3)--(7.25,-5.3); 
\draw[thin] (7,-7.3)--(5.5,-5.3); 
\draw[thin] (7,-7.3)--(3.75,-5.3); 
\draw[thin] (7,-7.3)--(2,-5.3); 
\draw[thin] (7,-7.3)--(0.25,-5.3); 

% alpha

%\draw[thin] (-2.8,-5.3)--(-2.8,-9.6); 
%\draw[thin] (0.25,-5.3)--(-2.8,-9.6); 
%\draw[thin] (5.5,-5.3)--(-2.8,-9.6); 

% beta

%\draw[thin] (-2.8,-5.3)--(0.25,-9.6); 
%\draw[thin] (0.25,-5.3)--(0.25,-9.6); 
%\draw[thin] (3.75,-5.3)--(0.25,-9.6); 

% gamma

%\draw[thin] (0.25,-5.3)--(2,-9.6); 

% delta

%\draw[thin] (3.75,-5.3)--(3.75,-9.6); 

% delta

%\draw[thin] (5.5,-5.3)--(5.5,-9.6); 

%\draw[thin] (0.25,-5.3)--(0.25,-9.6); 
%\draw[thin] (0.25,-5.3)--(6.4,-9.6); 
%\draw[thin] (12.5,-5.3)--(6.4,-9.6); 

%%%%
%%%% HERE DUAL
%%%%

% First positive example

\draw[thick] (20,-8)--(22,-8);    
\draw[thick] (21,-8)--(21,-6);    
\draw[thick] (22,-8)--(22,-6);    
\draw[thick] (20,-8)--(20,-6);  
\draw[thick] (20,-6)--(22,-6);
\draw[thick] (20,-7)--(22,-7);

%\draw[fill=black] (-0.5,1.5) rectangle (0.5,2.5);  
\draw[fill=black] (20,-8) rectangle (21,-6);  
%\draw[fill=black] (1,0) rectangle (2,1);  

% Second positive example

\draw[thick] (23,-8)--(25,-8);    
\draw[thick] (24,-8)--(24,-6);    
\draw[thick] (25,-8)--(25,-6);    
\draw[thick] (23,-8)--(23,-6);  
\draw[thick] (23,-6)--(25,-6);
\draw[thick] (23,-7)--(25,-7);

\draw[fill=black] (24,-8) rectangle (25,-6);  

% Separator between positive and negative examples

\draw[thick] (25.5,-8.5)--(25.5,-5.5);

% First negative example

\draw[thick] (26,-8)--(28,-8);    
\draw[thick] (27,-8)--(27,-6);    
\draw[thick] (28,-8)--(28,-6);    
\draw[thick] (26,-8)--(26,-6);  
\draw[thick] (26,-6)--(28,-6);
\draw[thick] (26,-7)--(28,-7);

\draw[fill=black] (26,-7) rectangle (27,-6);  
\draw[fill=black] (27,-8) rectangle (28,-7);  

% Second negative example

\draw[thick] (29,-8)--(31,-8);    
\draw[thick] (30,-8)--(30,-6);    
\draw[thick] (31,-8)--(31,-6);    
\draw[thick] (29,-8)--(29,-6);  
\draw[thick] (29,-6)--(31,-6);
\draw[thick] (29,-7)--(31,-7);

\draw[fill=black] (30,-7) rectangle (31,-6);  

% third negative example

\draw[thick] (32,-8)--(34,-8);    
\draw[thick] (33,-8)--(33,-6);    
\draw[thick] (34,-8)--(34,-6);    
\draw[thick] (32,-8)--(32,-6);  
\draw[thick] (32,-6)--(34,-6);
\draw[thick] (32,-7)--(34,-7);

\draw[fill=black] (33,-8) rectangle (34,-7);  

% Now the constants

% First constant

\draw[thick] (19.75,-2)--(20.75,-2);    
\draw[thick] (19.75,-1)--(20.75,-1);    
\draw[thick] (19.75,-2)--(19.75,-1);    
\draw[thick] (20.75,-2)--(20.75,-1);

\draw[fill=black] (19.5,-2.25) rectangle (20,-1.75);  

% Second constant

\draw[thick] (21.5,-2)--(22.5,-2);    
\draw[thick] (21.5,-1)--(22.5,-1);    
\draw[thick] (21.5,-2)--(21.5,-1);    
\draw[thick] (22.5,-2)--(22.5,-1);  

\draw[fill=black] (21.25,-1.25) rectangle (21.75,-0.75); 

% Third constant

\draw[thick] (23.25,-2)--(24.25,-2);    
\draw[thick] (23.25,-1)--(24.25,-1);    
\draw[thick] (23.25,-2)--(23.25,-1);    
\draw[thick] (24.25,-2)--(24.25,-1);  

\draw[fill=black] (24,-1.25) rectangle (24.5,-0.75);  

% Fourth constant

\draw[thick] (25,-2)--(26,-2);    
\draw[thick] (25,-1)--(26,-1);    
\draw[thick] (25,-2)--(25,-1);    
\draw[thick] (26,-2)--(26,-1);  

\draw[fill=black] (25.75,-2.25) rectangle (26.25,-1.75);  

% 5th constant (add 6.25 in x coordinate, and black goes to white)

\draw[thick] (26.75,-2)--(27.75,-2);    
\draw[thick] (26.75,-1)--(27.75,-1);    
\draw[thick] (26.75,-2)--(26.75,-1);    
\draw[thick] (27.75,-2)--(27.75,-1);  

\draw[fill=white] (26.5,-2.25) rectangle (27,-1.75);  

% 6th constant

\draw[thick] (28.50,-2)--(29.5,-2);    
\draw[thick] (28.50,-1)--(29.5,-1);    
\draw[thick] (28.50,-2)--(28.5,-1);    
\draw[thick] (29.50,-2)--(29.50,-1);  

\draw[fill=white] (28.25,-1.25) rectangle (28.75,-0.75); 

% 7th constant

\draw[thick] (30.25,-2)--(31.25,-2);    
\draw[thick] (30.25,-1)--(31.25,-1);    
\draw[thick] (30.25,-2)--(30.25,-1);    
\draw[thick] (31.25,-2)--(31.25,-1);  

\draw[fill=white] (31,-1.25) rectangle (31.5,-0.75);  

% 8th constant

\draw[thick] (32,-2)--(33,-2);    
\draw[thick] (32,-1)--(33,-1);    
\draw[thick] (32,-2)--(32,-1);    
\draw[thick] (33,-2)--(33,-1);  

\draw[fill=white] (32.75,-2.25) rectangle (33.25,-1.75);  
    
% Connectors from examples to constants

\draw[thin] (21,-5.7)--(20.25,-2.3);  
\draw[thin] (21,-5.7)--(22,-2.3);  
\draw[thin] (21,-5.7)--(30.75,-2.3);  
\draw[thin] (21,-5.7)--(32.5,-2.3); 

\draw[thin] (24,-5.7)--(23.75,-2.3); 
\draw[thin] (24,-5.7)--(25.5,-2.3); 
\draw[thin] (24,-5.7)--(27.25,-2.3); 
\draw[thin] (24,-5.7)--(29,-2.3); 

\draw[thin] (27,-5.7)--(22,-2.3); 
\draw[thin] (27,-5.7)--(25.5,-2.3); 
\draw[thin] (27,-5.7)--(27.25,-2.3); 
\draw[thin] (27,-5.7)--(30.75,-2.3); 

\draw[thin] (30,-5.7)--(23.75,-2.3); 
\draw[thin] (30,-5.7)--(27.25,-2.3); 
\draw[thin] (30,-5.7)--(29,-2.3); 
\draw[thin] (30,-5.7)--(32.5,-2.3); 

\draw[thin] (33,-5.7)--(25.5,-2.3); 
\draw[thin] (33,-5.7)--(27.25,-2.3); 
\draw[thin] (33,-5.7)--(29,-2.3); 
\draw[thin] (33,-5.7)--(30.75,-2.3); 

% Vertical constant and atoms

\node[text width=0.1cm] at (16.7,-1.5) 
    {$[v]$};

\node[text width=0.1cm] at (21,-11.5) 
    {$0^{*}$};

\node[text width=0.1cm] at (26.6,2) 
    {$[0]$};

\node[text width=0.1cm] at (27,-11.5) 
   {$\zeta_1$};

\node[text width=0.1cm] at (30,-11.5) 
    {$\zeta_2$};
    
\node[text width=0.1cm] at (33,-11.5) 
   {$\zeta_3$};    

%\node[text width=0.1cm] at (-3,-10) 
%    {$\Large{\alpha}$};

%\node[text width=0.1cm] at (0.25,-10) 
%    {$\Large{\beta}$};

%\node[text width=0.1cm] at (2,-10) 
%    {$\Large{\gamma}$};
    
%\node[text width=0.1cm] at (3.75,-10) 
%    {$\Large{\delta}$};

%\node[text width=0.1cm] at (5.5,-10) 
%    {$\Large{\epsilon}$};
    
%\node[text width=0.1cm] at (0.25,-10) 
%    {$\LARGE{\alpha}$};

%\node[text width=0.1cm] at (6.4,-10) 
%   {$\LARGE{0}$};
    
% Connectors to atoms

% 0*-Examples

\draw[thin] (21,-10.9)--(21,-8.3); 
\draw[thin] (21,-10.9)--(33,-8.3); 
\draw[thin] (21,-10.9)--(24,-8.3); 
\draw[thin] (21,-10.9)--(27,-8.3); 
\draw[thin] (21,-10.9)--(30,-8.3); 
%\draw[thin] (1,-13.9)--(13,-11.3); 
% First two examples -v

\draw[thin] (21,-5.7)--(17.2,-2.3); 
\draw[thin] (24,-5.7)--(17.2,-2.3); 

% eta-Examples
\draw[thin] (27,-10.9)--(27,-8.3); 
\draw[thin] (30,-10.9)--(30,-8.3); 
\draw[thin] (33,-10.9)--(33,-8.3); 

% [0] to constants

\draw[thin] (27,1.1)--(17.2,-0.7);
\draw[thin] (27,1.1)--(32.5,-0.7);
\draw[thin] (27,1.1)--(30.75,-0.7);
\draw[thin] (27,1.1)--(29,-0.7);
\draw[thin] (27,1.1)--(27.25,-0.7);
\draw[thin] (27,1.1)--(25.5,-0.7);
\draw[thin] (27,1.1)--(23.75,-0.7);
\draw[thin] (27,1.1)--(22,-0.7);
\draw[thin] (27,1.1)--(20.25,-0.7);
% alpha

%\draw[thin] (-2.8,-5.3)--(-2.8,-9.6); 
%\draw[thin] (0.25,-5.3)--(-2.8,-9.6); 
%\draw[thin] (5.5,-5.3)--(-2.8,-9.6); 

% beta

%\draw[thin] (-2.8,-5.3)--(0.25,-9.6); 
%\draw[thin] (0.25,-5.3)--(0.25,-9.6); 
%\draw[thin] (5.5,-5.3)--(0.25,-9.6); 

% gamma

%\draw[thin] (0.25,-5.3)--(2,-9.6); 

% delta

%\draw[thin] (3.75,-5.3)--(3.75,-9.6); 

% delta

%\draw[thin] (5.5,-5.3)--(5.5,-9.6); 

%\draw[thin] (0.25,-5.3)--(0.25,-9.6); 
%\draw[thin] (0.25,-5.3)--(6.4,-9.6); 
%\draw[thin] (12.5,-5.3)--(6.4,-9.6); 

\end{tikzpicture}
\end{center}
We are now going to apply \textbf{Algorithms \ref{negativeTrace} and \ref{positiveTrace}} in \textbf{Appendix \ref{algorithms}} to enforce the trace constraints. We start with the negative trace constraints, \textbf{Algorithm \ref{negativeTrace}}. The trace for the negative training examples is $\textbf{Tr} (T_{i}^{-})={\bf{GL}^{a}}([0])=\{0^{*},\zeta_1, \zeta_2, \zeta_3\}$, and for constant $v$ is also $\textbf{Tr} (v)={\bf{GL}^{a}}([0])=\{0^{*},\zeta_1, \zeta_2, \zeta_3\}$. Now it is not obeyed that ${\bf{Tr}}(T_{i}^{-}) \not\subset {\bf{Tr}}(v)$ so we need to enforce it. For this we need to choose a constant $c \in M$ equal to $v$ or such that $[c]$ receives edges from $[v]$ and not from $[T_{i}^{-}]$. We then need to add an atom $\phi \rightarrow c$. The condition is fulfilled directly by $v$ so we add $\phi \rightarrow v$, and the corresponding dual-of-atom $[v] \rightarrow [\phi]$ in $M^{*}$.
\begin{center}
%
%
% Model and dual together with PHI in M
%
%
\begin{tikzpicture}[scale=0.35]

% First positive example

\draw[thick] (0,0)--(2,0);    
\draw[thick] (1,0)--(1,2);    
\draw[thick] (2,0)--(2,2);    
\draw[thick] (0,0)--(0,2);  
\draw[thick] (0,2)--(2,2);
\draw[thick] (0,1)--(2,1);

%\draw[fill=black] (-0.5,1.5) rectangle (0.5,2.5);  
\draw[fill=black] (0,0) rectangle (1,2);  
%\draw[fill=black] (1,0) rectangle (2,1);  

% Second positive example

\draw[thick] (3,0)--(5,0);    
\draw[thick] (4,0)--(4,2);    
\draw[thick] (5,0)--(5,2);    
\draw[thick] (3,0)--(3,2);  
\draw[thick] (3,2)--(5,2);
\draw[thick] (3,1)--(5,1);

\draw[fill=black] (4,0) rectangle (5,2);  

% Separator between positive and negative examples

\draw[thick] (5.5,-0.5)--(5.5,2.5);

% First negative example

\draw[thick] (6,0)--(8,0);    
\draw[thick] (6,0)--(6,2);    
\draw[thick] (8,0)--(8,2);    
\draw[thick] (6,1)--(8,1);  
\draw[thick] (6,2)--(8,2);
\draw[thick] (7,0)--(7,2);

\draw[fill=black] (6,1) rectangle (7,2);  
\draw[fill=black] (7,0) rectangle (8,1);  

% Second negative example

\draw[thick] (9,0)--(11,0);    
\draw[thick] (10,0)--(10,2);    
\draw[thick] (11,0)--(11,2);    
\draw[thick] (9,0)--(9,2);  
\draw[thick] (9,2)--(11,2);
\draw[thick] (9,1)--(11,1);

\draw[fill=black] (10,1) rectangle (11,2);  

% third negative example

\draw[thick] (12,0)--(14,0);    
\draw[thick] (13,0)--(13,2);    
\draw[thick] (14,0)--(14,2);    
\draw[thick] (12,0)--(12,2);  
\draw[thick] (12,2)--(14,2);
\draw[thick] (12,1)--(14,1);

\draw[fill=black] (13,0) rectangle (14,1);  

% Now the constants

% First constant

\draw[thick] (-0.25,-5)--(0.75,-5);    
\draw[thick] (-0.25,-4)--(0.75,-4);    
\draw[thick] (-0.25,-5)--(-0.25,-4);    
\draw[thick] (0.75,-5)--(0.75,-4);  

\draw[fill=black] (-0.50,-5.25) rectangle (0,-4.75);  

% Second constant

\draw[thick] (1.5,-5)--(2.5,-5);    
\draw[thick] (1.5,-4)--(2.5,-4);    
\draw[thick] (1.5,-5)--(1.5,-4);    
\draw[thick] (2.5,-5)--(2.5,-4);  

\draw[fill=black] (1.25,-4.25) rectangle (1.75,-3.75); 

% Third constant

\draw[thick] (3.25,-5)--(4.25,-5);    
\draw[thick] (3.25,-4)--(4.25,-4);    
\draw[thick] (3.25,-5)--(3.25,-4);    
\draw[thick] (4.25,-5)--(4.25,-4);  

\draw[fill=black] (4,-4.25) rectangle (4.5,-3.75);  

% Fourth constant

\draw[thick] (5,-5)--(6,-5);    
\draw[thick] (5,-4)--(6,-4);    
\draw[thick] (5,-5)--(5,-4);    
\draw[thick] (6,-5)--(6,-4);  

\draw[fill=black] (5.75,-5.25) rectangle (6.25,-4.75);  

% 5th constant (add 6.25 in x coordinate, and black goes to white)

\draw[thick] (6.75,-5)--(7.75,-5);    
\draw[thick] (6.75,-4)--(7.75,-4);    
\draw[thick] (6.75,-5)--(6.75,-4);    
\draw[thick] (7.75,-5)--(7.75,-4);  

\draw[fill=white] (6.5,-5.25) rectangle (7,-4.75);  

% 6th constant

\draw[thick] (8.50,-5)--(9.5,-5);    
\draw[thick] (8.50,-4)--(9.5,-4);    
\draw[thick] (8.50,-5)--(8.5,-4);    
\draw[thick] (9.50,-5)--(9.50,-4);  

\draw[fill=white] (8.25,-4.25) rectangle (8.75,-3.75); 

% 7th constant

\draw[thick] (10.25,-5)--(11.25,-5);    
\draw[thick] (10.25,-4)--(11.25,-4);    
\draw[thick] (10.25,-5)--(10.25,-4);    
\draw[thick] (11.25,-5)--(11.25,-4);  

\draw[fill=white] (11,-4.25) rectangle (11.5,-3.75);  

% 8th constant

\draw[thick] (12,-5)--(13,-5);    
\draw[thick] (12,-4)--(13,-4);    
\draw[thick] (12,-5)--(12,-4);    
\draw[thick] (13,-5)--(13,-4);  

\draw[fill=white] (12.75,-5.25) rectangle (13.25,-4.75);  
    
% Connectors from examples to constants

\draw[thin] (1,-0.3)--(0.25,-3.7);  
\draw[thin] (1,-0.3)--(2,-3.7);  
\draw[thin] (1,-0.3)--(10.75,-3.7);  
\draw[thin] (1,-0.3)--(12.5,-3.7);  

\draw[thin] (4,-0.3)--(5.5,-3.7); 
\draw[thin] (4,-0.3)--(3.75,-3.7); 
\draw[thin] (4,-0.3)--(7.25,-3.7); 
\draw[thin] (4,-0.3)--(9,-3.7); 

\draw[thin] (7,-0.3)--(2,-3.7); 
\draw[thin] (7,-0.3)--(5.5,-3.7); 
\draw[thin] (7,-0.3)--(7.25,-3.7); 
\draw[thin] (7,-0.3)--(10.75,-3.7); 

\draw[thin] (10,-0.3)--(3.75,-3.7); 
\draw[thin] (10,-0.3)--(7.25,-3.7); 
\draw[thin] (10,-0.3)--(9,-3.7); 
\draw[thin] (10,-0.3)--(12.5,-3.7); 

\draw[thin] (13,-0.3)--(5.5,-3.7); 
\draw[thin] (13,-0.3)--(7.25,-3.7); 
\draw[thin] (13,-0.3)--(9,-3.7); 
\draw[thin] (13,-0.3)--(10.75,-3.7); 

% Vertical constant and atoms

\node[text width=0.1cm] at (-3,-4.5) 
    {{$v$}};

\node[text width=0.1cm] at (6.7,-7.9) 
    {$0$};

\node[text width=0.1cm] at (-3,-8) 
    {{$\phi$}};
% Connectors to atoms

% phi-v

\draw[thin] (-3,-5.3)--(-3,-7.6);

% 0-v

\draw[thin] (7,-7.3)--(-3,-5.3);

% 0-contant 8

\draw[thin] (7,-7.3)--(12.5,-5.3);
\draw[thin] (7,-7.3)--(10.75,-5.3); 
\draw[thin] (7,-7.3)--(9,-5.3); 
\draw[thin] (7,-7.3)--(7.25,-5.3); 
\draw[thin] (7,-7.3)--(5.5,-5.3); 
\draw[thin] (7,-7.3)--(3.75,-5.3); 
\draw[thin] (7,-7.3)--(2,-5.3); 
\draw[thin] (7,-7.3)--(0.25,-5.3); 

% alpha

%\draw[thin] (-2.8,-5.3)--(-2.8,-9.6); 
%\draw[thin] (0.25,-5.3)--(-2.8,-9.6); 
%\draw[thin] (5.5,-5.3)--(-2.8,-9.6); 

% beta

%\draw[thin] (-2.8,-5.3)--(0.25,-9.6); 
%\draw[thin] (0.25,-5.3)--(0.25,-9.6); 
%\draw[thin] (3.75,-5.3)--(0.25,-9.6); 

% gamma

%\draw[thin] (0.25,-5.3)--(2,-9.6); 

% delta

%\draw[thin] (3.75,-5.3)--(3.75,-9.6); 

% delta

%\draw[thin] (5.5,-5.3)--(5.5,-9.6); 

%\draw[thin] (0.25,-5.3)--(0.25,-9.6); 
%\draw[thin] (0.25,-5.3)--(6.4,-9.6); 
%\draw[thin] (12.5,-5.3)--(6.4,-9.6); 

%%%%
%%%% HERE DUAL
%%%%

% First positive example

\draw[thick] (20,-8)--(22,-8);    
\draw[thick] (21,-8)--(21,-6);    
\draw[thick] (22,-8)--(22,-6);    
\draw[thick] (20,-8)--(20,-6);  
\draw[thick] (20,-6)--(22,-6);
\draw[thick] (20,-7)--(22,-7);

%\draw[fill=black] (-0.5,1.5) rectangle (0.5,2.5);  
\draw[fill=black] (20,-8) rectangle (21,-6);  
%\draw[fill=black] (1,0) rectangle (2,1);  

% Second positive example

\draw[thick] (23,-8)--(25,-8);    
\draw[thick] (24,-8)--(24,-6);    
\draw[thick] (25,-8)--(25,-6);    
\draw[thick] (23,-8)--(23,-6);  
\draw[thick] (23,-6)--(25,-6);
\draw[thick] (23,-7)--(25,-7);

\draw[fill=black] (24,-8) rectangle (25,-6);  

% Separator between positive and negative examples

\draw[thick] (25.5,-8.5)--(25.5,-5.5);

% First negative example

\draw[thick] (26,-8)--(28,-8);    
\draw[thick] (27,-8)--(27,-6);    
\draw[thick] (28,-8)--(28,-6);    
\draw[thick] (26,-8)--(26,-6);  
\draw[thick] (26,-6)--(28,-6);
\draw[thick] (26,-7)--(28,-7);

\draw[fill=black] (26,-7) rectangle (27,-6);  
\draw[fill=black] (27,-8) rectangle (28,-7);  

% Second negative example

\draw[thick] (29,-8)--(31,-8);    
\draw[thick] (30,-8)--(30,-6);    
\draw[thick] (31,-8)--(31,-6);    
\draw[thick] (29,-8)--(29,-6);  
\draw[thick] (29,-6)--(31,-6);
\draw[thick] (29,-7)--(31,-7);

\draw[fill=black] (30,-7) rectangle (31,-6);  

% third negative example

\draw[thick] (32,-8)--(34,-8);    
\draw[thick] (33,-8)--(33,-6);    
\draw[thick] (34,-8)--(34,-6);    
\draw[thick] (32,-8)--(32,-6);  
\draw[thick] (32,-6)--(34,-6);
\draw[thick] (32,-7)--(34,-7);

\draw[fill=black] (33,-8) rectangle (34,-7);  

% Now the constants

% First constant

\draw[thick] (19.75,-2)--(20.75,-2);    
\draw[thick] (19.75,-1)--(20.75,-1);    
\draw[thick] (19.75,-2)--(19.75,-1);    
\draw[thick] (20.75,-2)--(20.75,-1);

\draw[fill=black] (19.5,-2.25) rectangle (20,-1.75);  

% Second constant

\draw[thick] (21.5,-2)--(22.5,-2);    
\draw[thick] (21.5,-1)--(22.5,-1);    
\draw[thick] (21.5,-2)--(21.5,-1);    
\draw[thick] (22.5,-2)--(22.5,-1);  

\draw[fill=black] (21.25,-1.25) rectangle (21.75,-0.75); 

% Third constant

\draw[thick] (23.25,-2)--(24.25,-2);    
\draw[thick] (23.25,-1)--(24.25,-1);    
\draw[thick] (23.25,-2)--(23.25,-1);    
\draw[thick] (24.25,-2)--(24.25,-1);  

\draw[fill=black] (24,-1.25) rectangle (24.5,-0.75);  

% Fourth constant

\draw[thick] (25,-2)--(26,-2);    
\draw[thick] (25,-1)--(26,-1);    
\draw[thick] (25,-2)--(25,-1);    
\draw[thick] (26,-2)--(26,-1);  

\draw[fill=black] (25.75,-2.25) rectangle (26.25,-1.75);  

% 5th constant (add 6.25 in x coordinate, and black goes to white)

\draw[thick] (26.75,-2)--(27.75,-2);    
\draw[thick] (26.75,-1)--(27.75,-1);    
\draw[thick] (26.75,-2)--(26.75,-1);    
\draw[thick] (27.75,-2)--(27.75,-1);  

\draw[fill=white] (26.5,-2.25) rectangle (27,-1.75);  

% 6th constant

\draw[thick] (28.50,-2)--(29.5,-2);    
\draw[thick] (28.50,-1)--(29.5,-1);    
\draw[thick] (28.50,-2)--(28.5,-1);    
\draw[thick] (29.50,-2)--(29.50,-1);  

\draw[fill=white] (28.25,-1.25) rectangle (28.75,-0.75); 

% 7th constant

\draw[thick] (30.25,-2)--(31.25,-2);    
\draw[thick] (30.25,-1)--(31.25,-1);    
\draw[thick] (30.25,-2)--(30.25,-1);    
\draw[thick] (31.25,-2)--(31.25,-1);  

\draw[fill=white] (31,-1.25) rectangle (31.5,-0.75);  

% 8th constant

\draw[thick] (32,-2)--(33,-2);    
\draw[thick] (32,-1)--(33,-1);    
\draw[thick] (32,-2)--(32,-1);    
\draw[thick] (33,-2)--(33,-1);  

\draw[fill=white] (32.75,-2.25) rectangle (33.25,-1.75);  
    
% Connectors from examples to constants

\draw[thin] (21,-5.7)--(20.25,-2.3);  
\draw[thin] (21,-5.7)--(22,-2.3);  
\draw[thin] (21,-5.7)--(30.75,-2.3);  
\draw[thin] (21,-5.7)--(32.5,-2.3); 

\draw[thin] (24,-5.7)--(23.75,-2.3); 
\draw[thin] (24,-5.7)--(25.5,-2.3); 
\draw[thin] (24,-5.7)--(27.25,-2.3); 
\draw[thin] (24,-5.7)--(29,-2.3); 

\draw[thin] (27,-5.7)--(22,-2.3); 
\draw[thin] (27,-5.7)--(25.5,-2.3); 
\draw[thin] (27,-5.7)--(27.25,-2.3); 
\draw[thin] (27,-5.7)--(30.75,-2.3); 

\draw[thin] (30,-5.7)--(23.75,-2.3); 
\draw[thin] (30,-5.7)--(27.25,-2.3); 
\draw[thin] (30,-5.7)--(29,-2.3); 
\draw[thin] (30,-5.7)--(32.5,-2.3); 

\draw[thin] (33,-5.7)--(25.5,-2.3); 
\draw[thin] (33,-5.7)--(27.25,-2.3); 
\draw[thin] (33,-5.7)--(29,-2.3); 
\draw[thin] (33,-5.7)--(30.75,-2.3); 

% Vertical constant and atoms

\node[text width=0.1cm] at (16.7,2) 
    {$[\phi]$}; 
    
\node[text width=0.1cm] at (16.7,-1.5) 
    {$[v]$};

\node[text width=0.1cm] at (21,-11.5) 
    {$0^{*}$};

\node[text width=0.1cm] at (26.6,2) 
    {$[0]$};

\node[text width=0.1cm] at (27,-11.5) 
   {$\zeta_1$};

\node[text width=0.1cm] at (30,-11.5) 
    {$\zeta_2$};
    
\node[text width=0.1cm] at (33,-11.5) 
   {$\zeta_3$};    

%\node[text width=0.1cm] at (-3,-10) 
%    {$\Large{\alpha}$};

%\node[text width=0.1cm] at (0.25,-10) 
%    {$\Large{\beta}$};

%\node[text width=0.1cm] at (2,-10) 
%    {$\Large{\gamma}$};
    
%\node[text width=0.1cm] at (3.75,-10) 
%    {$\Large{\delta}$};

%\node[text width=0.1cm] at (5.5,-10) 
%    {$\Large{\epsilon}$};
    
%\node[text width=0.1cm] at (0.25,-10) 
%    {$\LARGE{\alpha}$};

%\node[text width=0.1cm] at (6.4,-10) 
%   {$\LARGE{0}$};
    
% Connectors to atoms

%  From [phi] to v
\draw[thin] (17.2,1.3)--(17.2,-0.7); 

% 0*-Examples

\draw[thin] (21,-10.9)--(21,-8.3); 
\draw[thin] (21,-10.9)--(33,-8.3); 
\draw[thin] (21,-10.9)--(24,-8.3); 
\draw[thin] (21,-10.9)--(27,-8.3); 
\draw[thin] (21,-10.9)--(30,-8.3); 
%\draw[thin] (1,-13.9)--(13,-11.3); 
% First two examples -v

\draw[thin] (21,-5.7)--(17.2,-2.3); 
\draw[thin] (24,-5.7)--(17.2,-2.3); 

% eta-Examples
\draw[thin] (27,-10.9)--(27,-8.3); 
\draw[thin] (30,-10.9)--(30,-8.3); 
\draw[thin] (33,-10.9)--(33,-8.3); 

% [0] to constants

\draw[thin] (27,1.1)--(17.2,-0.7);
\draw[thin] (27,1.1)--(32.5,-0.7);
\draw[thin] (27,1.1)--(30.75,-0.7);
\draw[thin] (27,1.1)--(29,-0.7);
\draw[thin] (27,1.1)--(27.25,-0.7);
\draw[thin] (27,1.1)--(25.5,-0.7);
\draw[thin] (27,1.1)--(23.75,-0.7);
\draw[thin] (27,1.1)--(22,-0.7);
\draw[thin] (27,1.1)--(20.25,-0.7);
% alpha

%\draw[thin] (-2.8,-5.3)--(-2.8,-9.6); 
%\draw[thin] (0.25,-5.3)--(-2.8,-9.6); 
%\draw[thin] (5.5,-5.3)--(-2.8,-9.6); 

% beta

%\draw[thin] (-2.8,-5.3)--(0.25,-9.6); 
%\draw[thin] (0.25,-5.3)--(0.25,-9.6); 
%\draw[thin] (5.5,-5.3)--(0.25,-9.6); 

% gamma

%\draw[thin] (0.25,-5.3)--(2,-9.6); 

% delta

%\draw[thin] (3.75,-5.3)--(3.75,-9.6); 

% delta

%\draw[thin] (5.5,-5.3)--(5.5,-9.6); 

%\draw[thin] (0.25,-5.3)--(0.25,-9.6); 
%\draw[thin] (0.25,-5.3)--(6.4,-9.6); 
%\draw[thin] (12.5,-5.3)--(6.4,-9.6); 

\end{tikzpicture}
\end{center}
We now re-check the traces, $\textbf{Tr}(T_{i}^{-})=\textbf{Tr}(0)=\{0^{*},\zeta_1, \zeta_2, \zeta_3\}$ and 
$\textbf{Tr}(v)=\textbf{Tr}(0) \cap \textbf{Tr} (\phi)=\{0^{*}\}$, thus obeying  
$\textbf{Tr}(T_{i}^{-}) \not\subset \textbf{Tr}(v)$, as required.

Now, for positive trace constraints, we apply \textbf{Algorithm \ref{positiveTrace}}. For the positive relations $v<T_{i}^{+}$, we need to enforce $\textbf{Tr}(T_{i}^{+}) \subset \textbf{Tr}(v)$. First we check the values of the traces, 
$\textbf{Tr}(T_{i}^{+})=\textbf{Tr}(0)=\{0^{*},\zeta_1, \zeta_2, \zeta_3\}$, and $\textbf{Tr}(v)=\textbf{Tr}(0)\cap \textbf{Tr}(\phi)=\textbf{Tr}(\phi)=\{0^{*}\}$. This means that $\textbf{Tr}(T_{i}^{+}) \not \subset \textbf{Tr}(v)$, so we need to enforce the trace constraint. We add atoms $\epsilon_i$ to the constants $c_i$ until $\textbf{Tr}(T_{i}^{+})=\textbf{Tr}(0)\cap_i \textbf{Tr}(\epsilon_i)$ equals $\textbf{Tr}(v)$. For the first term $T_1^{+}$ we just need to edge one atom to the first constant, $\epsilon_1 \rightarrow c_1$, so $\textbf{Tr}(T_{1}^{+})=\textbf{Tr}(0) \cap \textbf{Tr}(\epsilon_1)=\{0^{*}, \zeta_1,\zeta_2,\zeta_3\} \cap \{0^{*}\}=\{0^{*}\}$. For the second training example, $T_2^{+}$, we need to add one atom for each of its constants, $\epsilon_2 \rightarrow c_3$ and $\epsilon_3 \rightarrow c_4$. Doing this, we have 
$\textbf{Tr}(T_{2}^{+})=\textbf{Tr}(0) \cap \textbf{Tr}(\epsilon_2) \cap \textbf{Tr}(\epsilon_3)=\{0^{*},\zeta_1,\zeta_2,\zeta_3\} \cap \{0^{*},\zeta_2\} \cap \{0^{*},\zeta_1,\zeta_3\}=\{0^{*}\}$, as required.

After imposing the trace constraints the graphs look like this
\begin{center}
\begin{tikzpicture}[scale=0.35]

% First positive example

\draw[thick] (0,0)--(2,0);    
\draw[thick] (1,0)--(1,2);    
\draw[thick] (2,0)--(2,2);    
\draw[thick] (0,0)--(0,2);  
\draw[thick] (0,2)--(2,2);
\draw[thick] (0,1)--(2,1);

%\draw[fill=black] (-0.5,1.5) rectangle (0.5,2.5);  
\draw[fill=black] (0,0) rectangle (1,2);  
%\draw[fill=black] (1,0) rectangle (2,1);  

% Second positive example

\draw[thick] (3,0)--(5,0);    
\draw[thick] (4,0)--(4,2);    
\draw[thick] (5,0)--(5,2);    
\draw[thick] (3,0)--(3,2);  
\draw[thick] (3,2)--(5,2);
\draw[thick] (3,1)--(5,1);

\draw[fill=black] (4,0) rectangle (5,2);  

% Separator between positive and negative examples

\draw[thick] (5.5,-0.5)--(5.5,2.5);

% First negative example

\draw[thick] (6,0)--(8,0);    
\draw[thick] (6,0)--(6,2);    
\draw[thick] (8,0)--(8,2);    
\draw[thick] (6,1)--(8,1);  
\draw[thick] (6,2)--(8,2);
\draw[thick] (7,0)--(7,2);

\draw[fill=black] (6,1) rectangle (7,2);  
\draw[fill=black] (7,0) rectangle (8,1);  

% Second negative example

\draw[thick] (9,0)--(11,0);    
\draw[thick] (10,0)--(10,2);    
\draw[thick] (11,0)--(11,2);    
\draw[thick] (9,0)--(9,2);  
\draw[thick] (9,2)--(11,2);
\draw[thick] (9,1)--(11,1);

\draw[fill=black] (10,1) rectangle (11,2);  

% third negative example

\draw[thick] (12,0)--(14,0);    
\draw[thick] (13,0)--(13,2);    
\draw[thick] (14,0)--(14,2);    
\draw[thick] (12,0)--(12,2);  
\draw[thick] (12,2)--(14,2);
\draw[thick] (12,1)--(14,1);

\draw[fill=black] (13,0) rectangle (14,1);  

% Now the constants

% First constant

\draw[thick] (-0.25,-5)--(0.75,-5);    
\draw[thick] (-0.25,-4)--(0.75,-4);    
\draw[thick] (-0.25,-5)--(-0.25,-4);    
\draw[thick] (0.75,-5)--(0.75,-4);  

\draw[fill=black] (-0.50,-5.25) rectangle (0,-4.75);  

% Second constant

\draw[thick] (1.5,-5)--(2.5,-5);    
\draw[thick] (1.5,-4)--(2.5,-4);    
\draw[thick] (1.5,-5)--(1.5,-4);    
\draw[thick] (2.5,-5)--(2.5,-4);  

\draw[fill=black] (1.25,-4.25) rectangle (1.75,-3.75); 

% Third constant

\draw[thick] (3.25,-5)--(4.25,-5);    
\draw[thick] (3.25,-4)--(4.25,-4);    
\draw[thick] (3.25,-5)--(3.25,-4);    
\draw[thick] (4.25,-5)--(4.25,-4);  

\draw[fill=black] (4,-4.25) rectangle (4.5,-3.75);  

% Fourth constant

\draw[thick] (5,-5)--(6,-5);    
\draw[thick] (5,-4)--(6,-4);    
\draw[thick] (5,-5)--(5,-4);    
\draw[thick] (6,-5)--(6,-4);  

\draw[fill=black] (5.75,-5.25) rectangle (6.25,-4.75);  

% 5th constant (add 6.25 in x coordinate, and black goes to white)

\draw[thick] (6.75,-5)--(7.75,-5);    
\draw[thick] (6.75,-4)--(7.75,-4);    
\draw[thick] (6.75,-5)--(6.75,-4);    
\draw[thick] (7.75,-5)--(7.75,-4);  

\draw[fill=white] (6.5,-5.25) rectangle (7,-4.75);  

% 6th constant

\draw[thick] (8.50,-5)--(9.5,-5);    
\draw[thick] (8.50,-4)--(9.5,-4);    
\draw[thick] (8.50,-5)--(8.5,-4);    
\draw[thick] (9.50,-5)--(9.50,-4);  

\draw[fill=white] (8.25,-4.25) rectangle (8.75,-3.75); 

% 7th constant

\draw[thick] (10.25,-5)--(11.25,-5);    
\draw[thick] (10.25,-4)--(11.25,-4);    
\draw[thick] (10.25,-5)--(10.25,-4);    
\draw[thick] (11.25,-5)--(11.25,-4);  

\draw[fill=white] (11,-4.25) rectangle (11.5,-3.75);  

% 8th constant

\draw[thick] (12,-5)--(13,-5);    
\draw[thick] (12,-4)--(13,-4);    
\draw[thick] (12,-5)--(12,-4);    
\draw[thick] (13,-5)--(13,-4);  

\draw[fill=white] (12.75,-5.25) rectangle (13.25,-4.75);  
    
% Connectors from examples to constants

\draw[thin] (1,-0.3)--(0.25,-3.7);  
\draw[thin] (1,-0.3)--(2,-3.7);  
\draw[thin] (1,-0.3)--(10.75,-3.7);  
\draw[thin] (1,-0.3)--(12.5,-3.7);  

\draw[thin] (4,-0.3)--(5.5,-3.7); 
\draw[thin] (4,-0.3)--(3.75,-3.7); 
\draw[thin] (4,-0.3)--(7.25,-3.7); 
\draw[thin] (4,-0.3)--(9,-3.7); 

\draw[thin] (7,-0.3)--(2,-3.7); 
\draw[thin] (7,-0.3)--(5.5,-3.7); 
\draw[thin] (7,-0.3)--(7.25,-3.7); 
\draw[thin] (7,-0.3)--(10.75,-3.7); 

\draw[thin] (10,-0.3)--(3.75,-3.7); 
\draw[thin] (10,-0.3)--(7.25,-3.7); 
\draw[thin] (10,-0.3)--(9,-3.7); 
\draw[thin] (10,-0.3)--(12.5,-3.7); 

\draw[thin] (13,-0.3)--(5.5,-3.7); 
\draw[thin] (13,-0.3)--(7.25,-3.7); 
\draw[thin] (13,-0.3)--(9,-3.7); 
\draw[thin] (13,-0.3)--(10.75,-3.7); 

% Vertical constant and atoms

\node[text width=0.1cm] at (-3,-4.5) 
    {{$v$}};

\node[text width=0.1cm] at (7,-7.9) 
    {$0$};

\node[text width=0.1cm] at (-3,-8) 
    {{$\phi$}};

\node[text width=0.1cm] at (0.25,-8) 
    {{$\epsilon_1$}};

\node[text width=0.1cm] at (3.75,-8) 
    {{$\epsilon_2$}};

\node[text width=0.1cm] at (5.5,-8) 
    {{$\epsilon_3$}};

% Connectors to atoms

% phi-v

\draw[thin] (-3,-5.3)--(-3,-7.6);

%  epsilon1-constant
\draw[thin] (0.25,-5.3)--(0.25,-7.6);
%  epsilon2-constant
\draw[thin] (3.75,-5.3)--(3.75,-7.6);
%  epsilon3-constant
\draw[thin] (5.5,-5.3)--(5.5,-7.6);

% 0-v

\draw[thin] (7,-7.3)--(-3,-5.3);

% 0-contant 8

\draw[thin] (7,-7.3)--(12.5,-5.3);
\draw[thin] (7,-7.3)--(10.75,-5.3); 
\draw[thin] (7,-7.3)--(9,-5.3); 
\draw[thin] (7,-7.3)--(7.25,-5.3); 
\draw[thin] (7,-7.3)--(5.5,-5.3); 
\draw[thin] (7,-7.3)--(3.75,-5.3); 
\draw[thin] (7,-7.3)--(2,-5.3); 
\draw[thin] (7,-7.3)--(0.25,-5.3); 

% alpha

%\draw[thin] (-2.8,-5.3)--(-2.8,-9.6); 
%\draw[thin] (0.25,-5.3)--(-2.8,-9.6); 
%\draw[thin] (5.5,-5.3)--(-2.8,-9.6); 

% beta

%\draw[thin] (-2.8,-5.3)--(0.25,-9.6); 
%\draw[thin] (0.25,-5.3)--(0.25,-9.6); 
%\draw[thin] (3.75,-5.3)--(0.25,-9.6); 

% gamma

%\draw[thin] (0.25,-5.3)--(2,-9.6); 

% delta

%\draw[thin] (3.75,-5.3)--(3.75,-9.6); 

% delta

%\draw[thin] (5.5,-5.3)--(5.5,-9.6); 

%\draw[thin] (0.25,-5.3)--(0.25,-9.6); 
%\draw[thin] (0.25,-5.3)--(6.4,-9.6); 
%\draw[thin] (12.5,-5.3)--(6.4,-9.6); 

%%%%
%%%% HERE DUAL
%%%%

% First positive example

\draw[thick] (20,-8)--(22,-8);    
\draw[thick] (21,-8)--(21,-6);    
\draw[thick] (22,-8)--(22,-6);    
\draw[thick] (20,-8)--(20,-6);  
\draw[thick] (20,-6)--(22,-6);
\draw[thick] (20,-7)--(22,-7);

%\draw[fill=black] (-0.5,1.5) rectangle (0.5,2.5);  
\draw[fill=black] (20,-8) rectangle (21,-6);  
%\draw[fill=black] (1,0) rectangle (2,1);  

% Second positive example

\draw[thick] (23,-8)--(25,-8);    
\draw[thick] (24,-8)--(24,-6);    
\draw[thick] (25,-8)--(25,-6);    
\draw[thick] (23,-8)--(23,-6);  
\draw[thick] (23,-6)--(25,-6);
\draw[thick] (23,-7)--(25,-7);

\draw[fill=black] (24,-8) rectangle (25,-6);  

% Separator between positive and negative examples

\draw[thick] (25.5,-8.5)--(25.5,-5.5);

% First negative example

\draw[thick] (26,-8)--(28,-8);    
\draw[thick] (27,-8)--(27,-6);    
\draw[thick] (28,-8)--(28,-6);    
\draw[thick] (26,-8)--(26,-6);  
\draw[thick] (26,-6)--(28,-6);
\draw[thick] (26,-7)--(28,-7);

\draw[fill=black] (26,-7) rectangle (27,-6);  
\draw[fill=black] (27,-8) rectangle (28,-7);  

% Second negative example

\draw[thick] (29,-8)--(31,-8);    
\draw[thick] (30,-8)--(30,-6);    
\draw[thick] (31,-8)--(31,-6);    
\draw[thick] (29,-8)--(29,-6);  
\draw[thick] (29,-6)--(31,-6);
\draw[thick] (29,-7)--(31,-7);

\draw[fill=black] (30,-7) rectangle (31,-6);  

% third negative example

\draw[thick] (32,-8)--(34,-8);    
\draw[thick] (33,-8)--(33,-6);    
\draw[thick] (34,-8)--(34,-6);    
\draw[thick] (32,-8)--(32,-6);  
\draw[thick] (32,-6)--(34,-6);
\draw[thick] (32,-7)--(34,-7);

\draw[fill=black] (33,-8) rectangle (34,-7);  

% Now the constants

% First constant

\draw[thick] (19.75,-2)--(20.75,-2);    
\draw[thick] (19.75,-1)--(20.75,-1);    
\draw[thick] (19.75,-2)--(19.75,-1);    
\draw[thick] (20.75,-2)--(20.75,-1);

\draw[fill=black] (19.5,-2.25) rectangle (20,-1.75);  

% Second constant

\draw[thick] (21.5,-2)--(22.5,-2);    
\draw[thick] (21.5,-1)--(22.5,-1);    
\draw[thick] (21.5,-2)--(21.5,-1);    
\draw[thick] (22.5,-2)--(22.5,-1);  

\draw[fill=black] (21.25,-1.25) rectangle (21.75,-0.75); 

% Third constant

\draw[thick] (23.25,-2)--(24.25,-2);    
\draw[thick] (23.25,-1)--(24.25,-1);    
\draw[thick] (23.25,-2)--(23.25,-1);    
\draw[thick] (24.25,-2)--(24.25,-1);  

\draw[fill=black] (24,-1.25) rectangle (24.5,-0.75);  

% Fourth constant

\draw[thick] (25,-2)--(26,-2);    
\draw[thick] (25,-1)--(26,-1);    
\draw[thick] (25,-2)--(25,-1);    
\draw[thick] (26,-2)--(26,-1);  

\draw[fill=black] (25.75,-2.25) rectangle (26.25,-1.75);  

% 5th constant (add 6.25 in x coordinate, and black goes to white)

\draw[thick] (26.75,-2)--(27.75,-2);    
\draw[thick] (26.75,-1)--(27.75,-1);    
\draw[thick] (26.75,-2)--(26.75,-1);    
\draw[thick] (27.75,-2)--(27.75,-1);  

\draw[fill=white] (26.5,-2.25) rectangle (27,-1.75);  

% 6th constant

\draw[thick] (28.50,-2)--(29.5,-2);    
\draw[thick] (28.50,-1)--(29.5,-1);    
\draw[thick] (28.50,-2)--(28.5,-1);    
\draw[thick] (29.50,-2)--(29.50,-1);  

\draw[fill=white] (28.25,-1.25) rectangle (28.75,-0.75); 

% 7th constant

\draw[thick] (30.25,-2)--(31.25,-2);    
\draw[thick] (30.25,-1)--(31.25,-1);    
\draw[thick] (30.25,-2)--(30.25,-1);    
\draw[thick] (31.25,-2)--(31.25,-1);  

\draw[fill=white] (31,-1.25) rectangle (31.5,-0.75);  

% 8th constant

\draw[thick] (32,-2)--(33,-2);    
\draw[thick] (32,-1)--(33,-1);    
\draw[thick] (32,-2)--(32,-1);    
\draw[thick] (33,-2)--(33,-1);  

\draw[fill=white] (32.75,-2.25) rectangle (33.25,-1.75);  
    
% Connectors from examples to constants

\draw[thin] (21,-5.7)--(20.25,-2.3);  
\draw[thin] (21,-5.7)--(22,-2.3);  
\draw[thin] (21,-5.7)--(30.75,-2.3);  
\draw[thin] (21,-5.7)--(32.5,-2.3); 

\draw[thin] (24,-5.7)--(23.75,-2.3); 
\draw[thin] (24,-5.7)--(25.5,-2.3); 
\draw[thin] (24,-5.7)--(27.25,-2.3); 
\draw[thin] (24,-5.7)--(29,-2.3); 

\draw[thin] (27,-5.7)--(22,-2.3); 
\draw[thin] (27,-5.7)--(25.5,-2.3); 
\draw[thin] (27,-5.7)--(27.25,-2.3); 
\draw[thin] (27,-5.7)--(30.75,-2.3); 

\draw[thin] (30,-5.7)--(23.75,-2.3); 
\draw[thin] (30,-5.7)--(27.25,-2.3); 
\draw[thin] (30,-5.7)--(29,-2.3); 
\draw[thin] (30,-5.7)--(32.5,-2.3); 

\draw[thin] (33,-5.7)--(25.5,-2.3); 
\draw[thin] (33,-5.7)--(27.25,-2.3); 
\draw[thin] (33,-5.7)--(29,-2.3); 
\draw[thin] (33,-5.7)--(30.75,-2.3); 

% Vertical constant and atoms

\node[text width=0.1cm] at (16.7,2) 
    {$[\phi]$}; 

\node[text width=0.1cm] at (19.60,2) 
    {$[\epsilon_1]$};
    
\node[text width=0.1cm] at (23.1,2) 
    {$[\epsilon_2]$};
    
\node[text width=0.1cm] at (24.85,2) 
    {$[\epsilon_3]$};
    
\node[text width=0.1cm] at (16.7,-1.5) 
    {$[v]$};

\node[text width=0.1cm] at (21,-11.5) 
    {$0^{*}$};

\node[text width=0.1cm] at (26.6,2) 
    {$[0]$};

\node[text width=0.1cm] at (27,-11.5) 
   {$\zeta_1$};

\node[text width=0.1cm] at (30,-11.5) 
    {$\zeta_2$};
    
\node[text width=0.1cm] at (33,-11.5) 
   {$\zeta_3$};    

%\node[text width=0.1cm] at (-3,-10) 
%    {$\Large{\alpha}$};

%\node[text width=0.1cm] at (0.25,-10) 
%    {$\Large{\beta}$};

%\node[text width=0.1cm] at (2,-10) 
%    {$\Large{\gamma}$};
    
%\node[text width=0.1cm] at (3.75,-10) 
%    {$\Large{\delta}$};

%\node[text width=0.1cm] at (5.5,-10) 
%    {$\Large{\epsilon}$};
    
%\node[text width=0.1cm] at (0.25,-10) 
%    {$\LARGE{\alpha}$};

%\node[text width=0.1cm] at (6.4,-10) 
%   {$\LARGE{0}$};
    
% Connectors to atoms

%  From [phi] to v
\draw[thin] (17.2,1.3)--(17.2,-0.7); 
%  From [epsilon]to constants
\draw[thin] (20.25,1.3)--(20.25,-0.7); 
\draw[thin] (23.75,1.3)--(23.75,-0.7); 
\draw[thin] (25.5,1.3)--(25.5,-0.7); 
% 0*-Examples

\draw[thin] (21,-10.9)--(21,-8.3); 
\draw[thin] (21,-10.9)--(33,-8.3); 
\draw[thin] (21,-10.9)--(24,-8.3); 
\draw[thin] (21,-10.9)--(27,-8.3); 
\draw[thin] (21,-10.9)--(30,-8.3); 
%\draw[thin] (1,-13.9)--(13,-11.3); 
% First two examples -v

\draw[thin] (21,-5.7)--(17.2,-2.3); 
\draw[thin] (24,-5.7)--(17.2,-2.3); 

% eta-Examples
\draw[thin] (27,-10.9)--(27,-8.3); 
\draw[thin] (30,-10.9)--(30,-8.3); 
\draw[thin] (33,-10.9)--(33,-8.3); 

% [0] to constants

\draw[thin] (27,1.1)--(17.2,-0.7);
\draw[thin] (27,1.1)--(32.5,-0.7);
\draw[thin] (27,1.1)--(30.75,-0.7);
\draw[thin] (27,1.1)--(29,-0.7);
\draw[thin] (27,1.1)--(27.25,-0.7);
\draw[thin] (27,1.1)--(25.5,-0.7);
\draw[thin] (27,1.1)--(23.75,-0.7);
\draw[thin] (27,1.1)--(22,-0.7);
\draw[thin] (27,1.1)--(20.25,-0.7);
% alpha

%\draw[thin] (-2.8,-5.3)--(-2.8,-9.6); 
%\draw[thin] (0.25,-5.3)--(-2.8,-9.6); 
%\draw[thin] (5.5,-5.3)--(-2.8,-9.6); 

% beta

%\draw[thin] (-2.8,-5.3)--(0.25,-9.6); 
%\draw[thin] (0.25,-5.3)--(0.25,-9.6); 
%\draw[thin] (5.5,-5.3)--(0.25,-9.6); 

% gamma

%\draw[thin] (0.25,-5.3)--(2,-9.6); 

% delta

%\draw[thin] (3.75,-5.3)--(3.75,-9.6); 

% delta

%\draw[thin] (5.5,-5.3)--(5.5,-9.6); 

%\draw[thin] (0.25,-5.3)--(0.25,-9.6); 
%\draw[thin] (0.25,-5.3)--(6.4,-9.6); 
%\draw[thin] (12.5,-5.3)--(6.4,-9.6); 

\end{tikzpicture}
\end{center}
We have used our toy example to show how to enforce the trace constraints of \textbf{Algorithms \ref{negativeTrace}} and \textbf{\ref{positiveTrace}}, detailed in \textbf{Appendix \ref{algorithms}}. The general trace enforcing algorithm consists of repeatedly enforcing negative trace constraints and positive trace constraints until all constraints are satisfied, which usually occurs within a few iterations. The trace enforcing process always ends unless the embedding is inconsistent (see {\bf{Theorem \ref{traceConsistence}}}).  The number of times these algorithms loop within the do-while statements can be easily bounded by the cardinality of the sets involved and is no worse than linear with the size of the model.

\subsection{Full and Sparse Crossing operations}

After enforcing the trace constraints, all negative relations $v \not < T_{i}^{-}$ are already satisfied in $M$. This will always be the case. To build an atomized model that also satisfies the positive relations $v < T_{i}^{+}$, we use the \emph{Sparse Crossing} operation. This trace-invariant operation replaces the atoms of $v$ for others that are also in $T_i^+$ without interfering with previously enforced positive or negative relations and without changing the traces of any element of $M$. 

The Sparse Crossing can be seen as an sparse version of the \emph{Full Crossing} that is also a trace-invariant operation. Both operations are similar and can be represented with a two dimensional matrix as follows. Consider two elements $a$ and $b$ with atoms
\begin{linenomath}
\begin{align}
{\bf{GL}}^a (a) = \{ \alpha ,\beta, \chi \} \,\,\,\,\,\,\,\,\,\,\textnormal{and}\,\,\,\,\,\,\,\,\,\, {\bf{GL}}^a (b) = \{ \chi ,\delta ,\varepsilon \},
\end{align} 
\end{linenomath}
and suppose we want to enforce $a<b$. Extend the graph appending new atoms and edges as
\begin{linenomath}
\[
\begin{array}{*{20}c}
{\phi ,\varphi ,\gamma \to \alpha } \\
{\pi ,\omega ,\theta \to \beta } \\
\end{array}\,\,\,\,\,\,\,\,\,\,\begin{array}{*{20}c}
{\chi ',\phi ,\pi \to \chi } \\
{\delta ',\varphi ,\omega \to \delta } \\
{\varepsilon ',\gamma ,\theta \to \varepsilon } \\
\end{array}
\]
\end{linenomath}
Close the graph by transitive closure and delete $\alpha ,\beta ,\,\chi ,\,\delta$ and $\varepsilon$ from the graph. Then it holds that ${\bf{GL}}^a (a) \subset {\bf{GL}}^a (b)$ and therefore, $\,a<b$. This is the \emph{Full Crossing} of $a$ into $b$ and can be represented with the table
\begin{linenomath}
\[
\begin{array}{*{20}c}
{} & \chi & \delta & \varepsilon \\
\hline
\,\,\,\,\,\,\,\,\vline & \chi ' & \delta ' & \varepsilon ' \\
\alpha \,\,\,\,\vline & \phi & \varphi & \gamma \\
\beta \,\,\,\, \vline & \pi & \omega & \theta \\
\end{array}
\]
\end{linenomath}
We say that we have \emph{"crossed"} atoms $\alpha$ and $\beta$ into $b$. Note that we do not need to "cross" atom $\chi$ of $a$ as it is already in $b$.

Since the crossing is an expensive operation that multiplies the number of atoms, we instead do a Sparse Crossing. The idea is that, as long as we check that all involved atoms remain trace-invariant, the Sparse Crossing operation still enforces $a<b$ and preserves all positive relations. In addition, it also preserves negative relations as long as they are "protected" by its corresponding negative trace constraint. Later we give details of how to compute it, but is is intuitive to think of it as a Full Crossing but leaving empty spaces in the table, as in this example
\begin{linenomath}
\[
\begin{array}{*{20}c}
{} & \varepsilon_1 & \varepsilon_2 & \varepsilon_3 \\
\hline
\,\,\,\,\,\,\,\,\,\,\vline & \varepsilon_{1}' & \,\,\,\, & \varepsilon_{3}' \\
\phi _1 \,\,\,\,\vline & & \varphi_{12} & \\
\phi _2 \,\,\,\,\vline & & \varphi_{22} & \varphi_{23} \\
\end{array}
\] 
\end{linenomath}

The Full Crossing and Sparse Crossing operations transform one graph into another graph and, after transitive closure, the crossing operations also map one algebra into another. This mapping function commutes with both operations $\odot$ and $[\,\,\,]$ so it is an homomorphism. This means that if $q=r \odot s$ is true before crossing it is also true after. Since the partial order $<$ is defined using the idempotent operator $\odot$, crossing operations preserve all inclusions between elements. 

Full and Sparse Crossing operations keep unaltered all positive order relations, i.e. if $p<q$ is true before the crossing of $a$ into $b$, it is also true after. This applies to all positive relations, not only the positive training relations of $R^{+}$ (we use $R^{+}$ for the positive relations of $R$). Crossing, however, does not preserve negative order relations: a negative relation before crossing can turn positive after crossing. In fact, crossing is never an injective homomorphism because $a<b$, that is false before crossing, becomes true after the crossing.

It turns out that the enforcement of negative trace constraints and the positive trace constraint $ {\bf{Tr}}(b) \subset {\bf{Tr}}(a)$ is what is need to ensure that the crossing of $a$ into $b$ preserves the negative order relations $R^{-}$.  In addition, the atoms introduced during the trace constraint enforcement stage are precisely the ones we need to carry out the Full Crossing (or Sparse Crossing) operation and keep all atoms trace-invariant. This follows from {\bf{Theorem \ref{crossingTheorem}}} that states that the Full Crossing of $a$ into $b$ leaves the traces of all atoms unchanged if and only if the positive trace constraint for $a < b$, i.e. $ {\bf{Tr}}(b) \subset {\bf{Tr}}(a)$ is satisfied.

For a recipe on how to compute the Sparse Crossing, see the \textbf{Algorithm \ref{sparseCrossing}} in \textbf{Appendix \ref{algorithms}}. In the following, we apply it to our toy example to the two positive input relations.

The Sparse Crossing operation of $v$ into $T_i^+$ works by creating new atoms and linking them to atoms in $v$ and atoms in $T_i^+$ in a way that the traces of all atoms are kept unaltered. The linearity of the trace ensures that if the trace of the atoms remain unchanged so do the traces of all elements. 

After we edge new atoms, say $\phi_1$, $\phi_2$ and $\phi_3$, to an existing atom $\phi$ of $v$, the trace of $\phi$ should be recalculated using $\textbf{Tr}(\phi)=\textbf{Tr}(\phi_1) \cap \textbf{Tr}(\phi_2) \cap \textbf{Tr}(\phi_3)$. Since the new atoms we introduce during the Crossing operations are edged not only to $\phi$ but also to atoms of $T_i^+$, the trace of $\phi$ could change and we do not want that to happen. Before we get rid of atom $\phi$ and replace it with the new ones, we have to make sure that the trace $\textbf{Tr}(\phi)$ has not changed.

We want the Sparse Crossing to be as sparse as possible to produce a model as small as possible. We replace $\phi$ with as few new atoms as possible limited only by the need to keep $\textbf{Tr}(\phi)$ unchanged. This need may force us to introduce more than one new atom in the row of atom $\phi$ in the Sparse Crossing matrix. It is always possible to preserve $\textbf{Tr}(\phi)$ provided that we have enforced the positive trace constraint for $a < b$, i.e. $ {\bf{Tr}}(b) \subset {\bf{Tr}}(a)$. 

Keeping the trace of the atoms of $T_i^+$ unaltered is also necessary but it is easier than preserving the trace of $\phi$. Suppose an atom $\alpha < T_1^+$ and assume we add an edge $\phi_1 \rightarrow \alpha$. The trace of $\alpha$ may change but, if it does, now we have the extra freedom of appending another new atom edged just to $\alpha$, which always has the effect to leave the trace of $\alpha$ invariant. This freedom does not exist for atoms of $v$, like $\phi$. Keep in mind that the goal of the crossing is replacing the atoms of $v$ with new atoms that are in both $v$ and $T_i^+$. If we add an extra atom $\phi'$ edged only to $\phi$, then $\phi'$ is not in $T_i^+$.  

Consider the set of atoms $\Phi_1$ in $v$ and not in the first positive example, $T_{1}^{+}$, and choose one. At the point where we are in our toy example, there is only a single atom, $\phi \in \Phi_1$. Now we select one of the atoms in $T_{1}^{+}$, in this case only $\epsilon_1$. We create a new atom $\phi_{1}$, and we add an edge to $\phi$ and another to $\epsilon_1$, $\phi_{1} \rightarrow \phi$ and $\phi_{1} \rightarrow \epsilon_1$, obtaining the new graph \begin{center}
\begin{tikzpicture}[scale=0.35]

% First positive example

\draw[thick] (0,0)--(2,0); 
\draw[thick] (1,0)--(1,2); 
\draw[thick] (2,0)--(2,2); 
\draw[thick] (0,0)--(0,2); 
\draw[thick] (0,2)--(2,2);
\draw[thick] (0,1)--(2,1);

%\draw[fill=black] (-0.5,1.5) rectangle (0.5,2.5); 
\draw[fill=black] (0,0) rectangle (1,2); 
%\draw[fill=black] (1,0) rectangle (2,1); 

% Second positive example

\draw[thick] (3,0)--(5,0); 
\draw[thick] (4,0)--(4,2); 
\draw[thick] (5,0)--(5,2); 
\draw[thick] (3,0)--(3,2); 
\draw[thick] (3,2)--(5,2);
\draw[thick] (3,1)--(5,1);

\draw[fill=black] (4,0) rectangle (5,2); 

% Separator between positive and negative examples

\draw[thick] (5.5,-0.5)--(5.5,2.5);

% First negative example

\draw[thick] (6,0)--(8,0); 
\draw[thick] (6,0)--(6,2); 
\draw[thick] (8,0)--(8,2); 
\draw[thick] (6,1)--(8,1); 
\draw[thick] (6,2)--(8,2);
\draw[thick] (7,0)--(7,2);

\draw[fill=black] (6,1) rectangle (7,2); 
\draw[fill=black] (7,0) rectangle (8,1); 

% Second negative example

\draw[thick] (9,0)--(11,0); 
\draw[thick] (10,0)--(10,2); 
\draw[thick] (11,0)--(11,2); 
\draw[thick] (9,0)--(9,2); 
\draw[thick] (9,2)--(11,2);
\draw[thick] (9,1)--(11,1);

\draw[fill=black] (10,1) rectangle (11,2); 

% third negative example

\draw[thick] (12,0)--(14,0); 
\draw[thick] (13,0)--(13,2); 
\draw[thick] (14,0)--(14,2); 
\draw[thick] (12,0)--(12,2); 
\draw[thick] (12,2)--(14,2);
\draw[thick] (12,1)--(14,1);

\draw[fill=black] (13,0) rectangle (14,1); 

% Now the constants

% First constant

\draw[thick] (-0.25,-5)--(0.75,-5); 
\draw[thick] (-0.25,-4)--(0.75,-4); 
\draw[thick] (-0.25,-5)--(-0.25,-4); 
\draw[thick] (0.75,-5)--(0.75,-4); 

\draw[fill=black] (-0.50,-5.25) rectangle (0,-4.75); 

% Second constant

\draw[thick] (1.5,-5)--(2.5,-5); 
\draw[thick] (1.5,-4)--(2.5,-4); 
\draw[thick] (1.5,-5)--(1.5,-4); 
\draw[thick] (2.5,-5)--(2.5,-4); 

\draw[fill=black] (1.25,-4.25) rectangle (1.75,-3.75); 

% Third constant

\draw[thick] (3.25,-5)--(4.25,-5); 
\draw[thick] (3.25,-4)--(4.25,-4); 
\draw[thick] (3.25,-5)--(3.25,-4); 
\draw[thick] (4.25,-5)--(4.25,-4); 

\draw[fill=black] (4,-4.25) rectangle (4.5,-3.75); 

% Fourth constant

\draw[thick] (5,-5)--(6,-5); 
\draw[thick] (5,-4)--(6,-4); 
\draw[thick] (5,-5)--(5,-4); 
\draw[thick] (6,-5)--(6,-4); 

\draw[fill=black] (5.75,-5.25) rectangle (6.25,-4.75); 

% 5th constant (add 6.25 in x coordinate, and black goes to white)

\draw[thick] (6.75,-5)--(7.75,-5); 
\draw[thick] (6.75,-4)--(7.75,-4); 
\draw[thick] (6.75,-5)--(6.75,-4); 
\draw[thick] (7.75,-5)--(7.75,-4); 

\draw[fill=white] (6.5,-5.25) rectangle (7,-4.75); 

% 6th constant

\draw[thick] (8.50,-5)--(9.5,-5); 
\draw[thick] (8.50,-4)--(9.5,-4); 
\draw[thick] (8.50,-5)--(8.5,-4); 
\draw[thick] (9.50,-5)--(9.50,-4); 

\draw[fill=white] (8.25,-4.25) rectangle (8.75,-3.75); 

% 7th constant

\draw[thick] (10.25,-5)--(11.25,-5); 
\draw[thick] (10.25,-4)--(11.25,-4); 
\draw[thick] (10.25,-5)--(10.25,-4); 
\draw[thick] (11.25,-5)--(11.25,-4); 

\draw[fill=white] (11,-4.25) rectangle (11.5,-3.75); 

% 8th constant

\draw[thick] (12,-5)--(13,-5); 
\draw[thick] (12,-4)--(13,-4); 
\draw[thick] (12,-5)--(12,-4); 
\draw[thick] (13,-5)--(13,-4); 

\draw[fill=white] (12.75,-5.25) rectangle (13.25,-4.75); 
% Connectors from examples to constants

\draw[thin] (1,-0.3)--(0.25,-3.7); 
\draw[thin] (1,-0.3)--(2,-3.7); 
\draw[thin] (1,-0.3)--(10.75,-3.7); 
\draw[thin] (1,-0.3)--(12.5,-3.7); 

\draw[thin] (4,-0.3)--(5.5,-3.7); 
\draw[thin] (4,-0.3)--(3.75,-3.7); 
\draw[thin] (4,-0.3)--(7.25,-3.7); 
\draw[thin] (4,-0.3)--(9,-3.7); 

\draw[thin] (7,-0.3)--(2,-3.7); 
\draw[thin] (7,-0.3)--(5.5,-3.7); 
\draw[thin] (7,-0.3)--(7.25,-3.7); 
\draw[thin] (7,-0.3)--(10.75,-3.7); 

\draw[thin] (10,-0.3)--(3.75,-3.7); 
\draw[thin] (10,-0.3)--(7.25,-3.7); 
\draw[thin] (10,-0.3)--(9,-3.7); 
\draw[thin] (10,-0.3)--(12.5,-3.7); 

\draw[thin] (13,-0.3)--(5.5,-3.7); 
\draw[thin] (13,-0.3)--(7.25,-3.7); 
\draw[thin] (13,-0.3)--(9,-3.7); 
\draw[thin] (13,-0.3)--(10.75,-3.7); 

% Vertical constant and atoms

\node[text width=0.1cm] at (-3,-4.5) 
{{$v$}};

\node[text width=0.1cm] at (7,-7.9) 
{$0$};

\node[text width=0.1cm] at (-3,-8) 
{{$\phi$}};

\node[text width=0.1cm] at (0.25,-8) 
{{$\epsilon_1$}};

\node[text width=0.1cm] at (3.75,-8) 
{{$\epsilon_2$}};

\node[text width=0.1cm] at (5.5,-8) 
{{$\epsilon_3$}};
\node[text width=0.1cm] at (-1.875,-10) 
{{$\phi_{1}$}};
% Connectors to atoms

% phi-v

\draw[thin] (-3,-5.3)--(-3,-7.6);

% phi11-constants

\draw[thin] (-1.5,-9.3)--(-3,-8.7);
\draw[thin] (-1.5,-9.3)--(0.25,-8.7);

% epsilon1-constant
\draw[thin] (0.25,-5.3)--(0.25,-7.6);
% epsilon2-constant
\draw[thin] (3.75,-5.3)--(3.75,-7.6);
% epsilon3-constant
\draw[thin] (5.5,-5.3)--(5.5,-7.6);

% 0-v

\draw[thin] (7,-7.3)--(-3,-5.3);

% 0-contant 8

\draw[thin] (7,-7.3)--(12.5,-5.3);
\draw[thin] (7,-7.3)--(10.75,-5.3); 
\draw[thin] (7,-7.3)--(9,-5.3); 
\draw[thin] (7,-7.3)--(7.25,-5.3); 
\draw[thin] (7,-7.3)--(5.5,-5.3); 
\draw[thin] (7,-7.3)--(3.75,-5.3); 
\draw[thin] (7,-7.3)--(2,-5.3); 
\draw[thin] (7,-7.3)--(0.25,-5.3); 

% alpha

%\draw[thin] (-2.8,-5.3)--(-2.8,-9.6); 
%\draw[thin] (0.25,-5.3)--(-2.8,-9.6); 
%\draw[thin] (5.5,-5.3)--(-2.8,-9.6); 

% beta

%\draw[thin] (-2.8,-5.3)--(0.25,-9.6); 
%\draw[thin] (0.25,-5.3)--(0.25,-9.6); 
%\draw[thin] (3.75,-5.3)--(0.25,-9.6); 

% gamma

%\draw[thin] (0.25,-5.3)--(2,-9.6); 

% delta

%\draw[thin] (3.75,-5.3)--(3.75,-9.6); 

% delta

%\draw[thin] (5.5,-5.3)--(5.5,-9.6); 

%\draw[thin] (0.25,-5.3)--(0.25,-9.6); 
%\draw[thin] (0.25,-5.3)--(6.4,-9.6); 
%\draw[thin] (12.5,-5.3)--(6.4,-9.6); 

%%%%
%%%% HERE DUAL
%%%%

% First positive example

\draw[thick] (20,-8)--(22,-8); 
\draw[thick] (21,-8)--(21,-6); 
\draw[thick] (22,-8)--(22,-6); 
\draw[thick] (20,-8)--(20,-6); 
\draw[thick] (20,-6)--(22,-6);
\draw[thick] (20,-7)--(22,-7);

%\draw[fill=black] (-0.5,1.5) rectangle (0.5,2.5); 
\draw[fill=black] (20,-8) rectangle (21,-6); 
%\draw[fill=black] (1,0) rectangle (2,1); 

% Second positive example

\draw[thick] (23,-8)--(25,-8); 
\draw[thick] (24,-8)--(24,-6); 
\draw[thick] (25,-8)--(25,-6); 
\draw[thick] (23,-8)--(23,-6); 
\draw[thick] (23,-6)--(25,-6);
\draw[thick] (23,-7)--(25,-7);

\draw[fill=black] (24,-8) rectangle (25,-6); 

% Separator between positive and negative examples

\draw[thick] (25.5,-8.5)--(25.5,-5.5);

% First negative example

\draw[thick] (26,-8)--(28,-8); 
\draw[thick] (27,-8)--(27,-6); 
\draw[thick] (28,-8)--(28,-6); 
\draw[thick] (26,-8)--(26,-6); 
\draw[thick] (26,-6)--(28,-6);
\draw[thick] (26,-7)--(28,-7);

\draw[fill=black] (26,-7) rectangle (27,-6); 
\draw[fill=black] (27,-8) rectangle (28,-7); 

% Second negative example

\draw[thick] (29,-8)--(31,-8); 
\draw[thick] (30,-8)--(30,-6); 
\draw[thick] (31,-8)--(31,-6); 
\draw[thick] (29,-8)--(29,-6); 
\draw[thick] (29,-6)--(31,-6);
\draw[thick] (29,-7)--(31,-7);

\draw[fill=black] (30,-7) rectangle (31,-6); 

% third negative example

\draw[thick] (32,-8)--(34,-8); 
\draw[thick] (33,-8)--(33,-6); 
\draw[thick] (34,-8)--(34,-6); 
\draw[thick] (32,-8)--(32,-6); 
\draw[thick] (32,-6)--(34,-6);
\draw[thick] (32,-7)--(34,-7);

\draw[fill=black] (33,-8) rectangle (34,-7); 

% Now the constants

% First constant

\draw[thick] (19.75,-2)--(20.75,-2); 
\draw[thick] (19.75,-1)--(20.75,-1); 
\draw[thick] (19.75,-2)--(19.75,-1); 
\draw[thick] (20.75,-2)--(20.75,-1);

\draw[fill=black] (19.5,-2.25) rectangle (20,-1.75); 

% Second constant

\draw[thick] (21.5,-2)--(22.5,-2); 
\draw[thick] (21.5,-1)--(22.5,-1); 
\draw[thick] (21.5,-2)--(21.5,-1); 
\draw[thick] (22.5,-2)--(22.5,-1); 

\draw[fill=black] (21.25,-1.25) rectangle (21.75,-0.75); 

% Third constant

\draw[thick] (23.25,-2)--(24.25,-2); 
\draw[thick] (23.25,-1)--(24.25,-1); 
\draw[thick] (23.25,-2)--(23.25,-1); 
\draw[thick] (24.25,-2)--(24.25,-1); 

\draw[fill=black] (24,-1.25) rectangle (24.5,-0.75); 

% Fourth constant

\draw[thick] (25,-2)--(26,-2); 
\draw[thick] (25,-1)--(26,-1); 
\draw[thick] (25,-2)--(25,-1); 
\draw[thick] (26,-2)--(26,-1); 

\draw[fill=black] (25.75,-2.25) rectangle (26.25,-1.75); 

% 5th constant (add 6.25 in x coordinate, and black goes to white)

\draw[thick] (26.75,-2)--(27.75,-2); 
\draw[thick] (26.75,-1)--(27.75,-1); 
\draw[thick] (26.75,-2)--(26.75,-1); 
\draw[thick] (27.75,-2)--(27.75,-1); 

\draw[fill=white] (26.5,-2.25) rectangle (27,-1.75); 

% 6th constant

\draw[thick] (28.50,-2)--(29.5,-2); 
\draw[thick] (28.50,-1)--(29.5,-1); 
\draw[thick] (28.50,-2)--(28.5,-1); 
\draw[thick] (29.50,-2)--(29.50,-1); 

\draw[fill=white] (28.25,-1.25) rectangle (28.75,-0.75); 

% 7th constant

\draw[thick] (30.25,-2)--(31.25,-2); 
\draw[thick] (30.25,-1)--(31.25,-1); 
\draw[thick] (30.25,-2)--(30.25,-1); 
\draw[thick] (31.25,-2)--(31.25,-1); 

\draw[fill=white] (31,-1.25) rectangle (31.5,-0.75); 

% 8th constant

\draw[thick] (32,-2)--(33,-2); 
\draw[thick] (32,-1)--(33,-1); 
\draw[thick] (32,-2)--(32,-1); 
\draw[thick] (33,-2)--(33,-1); 

\draw[fill=white] (32.75,-2.25) rectangle (33.25,-1.75); 
% Connectors from examples to constants

\draw[thin] (21,-5.7)--(20.25,-2.3); 
\draw[thin] (21,-5.7)--(22,-2.3); 
\draw[thin] (21,-5.7)--(30.75,-2.3); 
\draw[thin] (21,-5.7)--(32.5,-2.3); 

\draw[thin] (24,-5.7)--(23.75,-2.3); 
\draw[thin] (24,-5.7)--(25.5,-2.3); 
\draw[thin] (24,-5.7)--(27.25,-2.3); 
\draw[thin] (24,-5.7)--(29,-2.3); 

\draw[thin] (27,-5.7)--(22,-2.3); 
\draw[thin] (27,-5.7)--(25.5,-2.3); 
\draw[thin] (27,-5.7)--(27.25,-2.3); 
\draw[thin] (27,-5.7)--(30.75,-2.3); 

\draw[thin] (30,-5.7)--(23.75,-2.3); 
\draw[thin] (30,-5.7)--(27.25,-2.3); 
\draw[thin] (30,-5.7)--(29,-2.3); 
\draw[thin] (30,-5.7)--(32.5,-2.3); 

\draw[thin] (33,-5.7)--(25.5,-2.3); 
\draw[thin] (33,-5.7)--(27.25,-2.3); 
\draw[thin] (33,-5.7)--(29,-2.3); 
\draw[thin] (33,-5.7)--(30.75,-2.3); 

% Vertical constant and atoms

\node[text width=0.1cm] at (16.7,2) 
{$[\phi]$}; 
\node[text width=0.1cm] at (17.8,4.1) 
{$[\phi_{1}]$};
\node[text width=0.1cm] at (19.60,2) 
{$[\epsilon_1]$};
\node[text width=0.1cm] at (23.1,2) 
{$[\epsilon_2]$};
\node[text width=0.1cm] at (24.85,2) 
{$[\epsilon_3]$};
\node[text width=0.1cm] at (16.7,-1.5) 
{$[v]$};

\node[text width=0.1cm] at (21,-11.5) 
{$0^{*}$};

\node[text width=0.1cm] at (26.6,2) 
{$[0]$};

\node[text width=0.1cm] at (27,-11.5) 
{$\zeta_1$};

\node[text width=0.1cm] at (30,-11.5) 
{$\zeta_2$};
\node[text width=0.1cm] at (33,-11.5) 
{$\zeta_3$}; 

%\node[text width=0.1cm] at (-3,-10) 
% {$\Large{\alpha}$};

%\node[text width=0.1cm] at (0.25,-10) 
% {$\Large{\beta}$};

%\node[text width=0.1cm] at (2,-10) 
% {$\Large{\gamma}$};
%\node[text width=0.1cm] at (3.75,-10) 
% {$\Large{\delta}$};

%\node[text width=0.1cm] at (5.5,-10) 
% {$\Large{\epsilon}$};
%\node[text width=0.1cm] at (0.25,-10) 
% {$\LARGE{\alpha}$};

%\node[text width=0.1cm] at (6.4,-10) 
% {$\LARGE{0}$};
% Connectors to atoms

% From [phi] to v
\draw[thin] (17.2,1.3)--(17.2,-0.7); 

% From [phi11] to atoms

\draw[thin] (18.5,3.3)--(17.2,2.7); 
\draw[thin] (18.5,3.3)--(20,2.7); 

% From [epsilon]to constants
\draw[thin] (20.25,1.3)--(20.25,-0.7); 
\draw[thin] (23.75,1.3)--(23.75,-0.7); 
\draw[thin] (25.5,1.3)--(25.5,-0.7); 
% 0*-Examples

\draw[thin] (21,-10.9)--(21,-8.3); 
\draw[thin] (21,-10.9)--(33,-8.3); 
\draw[thin] (21,-10.9)--(24,-8.3); 
\draw[thin] (21,-10.9)--(27,-8.3); 
\draw[thin] (21,-10.9)--(30,-8.3); 
%\draw[thin] (1,-13.9)--(13,-11.3); 
% First two examples -v

\draw[thin] (21,-5.7)--(17.2,-2.3); 
\draw[thin] (24,-5.7)--(17.2,-2.3); 

% eta-Examples
\draw[thin] (27,-10.9)--(27,-8.3); 
\draw[thin] (30,-10.9)--(30,-8.3); 
\draw[thin] (33,-10.9)--(33,-8.3); 

% [0] to constants

\draw[thin] (27,1.1)--(17.2,-0.7);
\draw[thin] (27,1.1)--(32.5,-0.7);
\draw[thin] (27,1.1)--(30.75,-0.7);
\draw[thin] (27,1.1)--(29,-0.7);
\draw[thin] (27,1.1)--(27.25,-0.7);
\draw[thin] (27,1.1)--(25.5,-0.7);
\draw[thin] (27,1.1)--(23.75,-0.7);
\draw[thin] (27,1.1)--(22,-0.7);
\draw[thin] (27,1.1)--(20.25,-0.7);
% alpha

%\draw[thin] (-2.8,-5.3)--(-2.8,-9.6); 
%\draw[thin] (0.25,-5.3)--(-2.8,-9.6); 
%\draw[thin] (5.5,-5.3)--(-2.8,-9.6); 

% beta

%\draw[thin] (-2.8,-5.3)--(0.25,-9.6); 
%\draw[thin] (0.25,-5.3)--(0.25,-9.6); 
%\draw[thin] (5.5,-5.3)--(0.25,-9.6); 

% gamma

%\draw[thin] (0.25,-5.3)--(2,-9.6); 

% delta

%\draw[thin] (3.75,-5.3)--(3.75,-9.6); 

% delta

%\draw[thin] (5.5,-5.3)--(5.5,-9.6); 

%\draw[thin] (0.25,-5.3)--(0.25,-9.6); 
%\draw[thin] (0.25,-5.3)--(6.4,-9.6); 
%\draw[thin] (12.5,-5.3)--(6.4,-9.6); 

\end{tikzpicture}
\end{center}
We have to check that the traces of all the atoms remain unchanged. We can start with atom $\phi$. Initially, we have that $\textbf{Tr}(\phi)_{i}= \{0^{*}\}$, and after the crossing $\textbf{Tr}(\phi)_{f}=\textbf{Tr}(\phi_{1})= \{0^{*}\}$, so it has not changed. For atom $\epsilon_1$, we have initially $\textbf{Tr}(\epsilon_1)_{i}=\{0^{*}\}$ and after the crossing $\textbf{Tr}(\epsilon_1)_{f}=\textbf{Tr}(\phi_{1})= \{0^{*}\}$, so it has not changed either. Now that we have checked for trace invariance of the crossing, we can eliminate the original atoms $\phi$ and $\epsilon_1$, giving
\begin{center}
\begin{tikzpicture}[scale=0.35]

% First positive example

\draw[thick] (0,0)--(2,0); 
\draw[thick] (1,0)--(1,2); 
\draw[thick] (2,0)--(2,2); 
\draw[thick] (0,0)--(0,2); 
\draw[thick] (0,2)--(2,2);
\draw[thick] (0,1)--(2,1);

%\draw[fill=black] (-0.5,1.5) rectangle (0.5,2.5); 
\draw[fill=black] (0,0) rectangle (1,2); 
%\draw[fill=black] (1,0) rectangle (2,1); 

% Second positive example

\draw[thick] (3,0)--(5,0); 
\draw[thick] (4,0)--(4,2); 
\draw[thick] (5,0)--(5,2); 
\draw[thick] (3,0)--(3,2); 
\draw[thick] (3,2)--(5,2);
\draw[thick] (3,1)--(5,1);

\draw[fill=black] (4,0) rectangle (5,2); 

% Separator between positive and negative examples

\draw[thick] (5.5,-0.5)--(5.5,2.5);

% First negative example

\draw[thick] (6,0)--(8,0); 
\draw[thick] (6,0)--(6,2); 
\draw[thick] (8,0)--(8,2); 
\draw[thick] (6,1)--(8,1); 
\draw[thick] (6,2)--(8,2);
\draw[thick] (7,0)--(7,2);

\draw[fill=black] (6,1) rectangle (7,2); 
\draw[fill=black] (7,0) rectangle (8,1); 

% Second negative example

\draw[thick] (9,0)--(11,0); 
\draw[thick] (10,0)--(10,2); 
\draw[thick] (11,0)--(11,2); 
\draw[thick] (9,0)--(9,2); 
\draw[thick] (9,2)--(11,2);
\draw[thick] (9,1)--(11,1);

\draw[fill=black] (10,1) rectangle (11,2); 

% third negative example

\draw[thick] (12,0)--(14,0); 
\draw[thick] (13,0)--(13,2); 
\draw[thick] (14,0)--(14,2); 
\draw[thick] (12,0)--(12,2); 
\draw[thick] (12,2)--(14,2);
\draw[thick] (12,1)--(14,1);

\draw[fill=black] (13,0) rectangle (14,1); 

% Now the constants

% First constant

\draw[thick] (-0.25,-5)--(0.75,-5); 
\draw[thick] (-0.25,-4)--(0.75,-4); 
\draw[thick] (-0.25,-5)--(-0.25,-4); 
\draw[thick] (0.75,-5)--(0.75,-4); 

\draw[fill=black] (-0.50,-5.25) rectangle (0,-4.75); 

% Second constant

\draw[thick] (1.5,-5)--(2.5,-5); 
\draw[thick] (1.5,-4)--(2.5,-4); 
\draw[thick] (1.5,-5)--(1.5,-4); 
\draw[thick] (2.5,-5)--(2.5,-4); 

\draw[fill=black] (1.25,-4.25) rectangle (1.75,-3.75); 

% Third constant

\draw[thick] (3.25,-5)--(4.25,-5); 
\draw[thick] (3.25,-4)--(4.25,-4); 
\draw[thick] (3.25,-5)--(3.25,-4); 
\draw[thick] (4.25,-5)--(4.25,-4); 

\draw[fill=black] (4,-4.25) rectangle (4.5,-3.75); 

% Fourth constant

\draw[thick] (5,-5)--(6,-5); 
\draw[thick] (5,-4)--(6,-4); 
\draw[thick] (5,-5)--(5,-4); 
\draw[thick] (6,-5)--(6,-4); 

\draw[fill=black] (5.75,-5.25) rectangle (6.25,-4.75); 

% 5th constant (add 6.25 in x coordinate, and black goes to white)

\draw[thick] (6.75,-5)--(7.75,-5); 
\draw[thick] (6.75,-4)--(7.75,-4); 
\draw[thick] (6.75,-5)--(6.75,-4); 
\draw[thick] (7.75,-5)--(7.75,-4); 

\draw[fill=white] (6.5,-5.25) rectangle (7,-4.75); 

% 6th constant

\draw[thick] (8.50,-5)--(9.5,-5); 
\draw[thick] (8.50,-4)--(9.5,-4); 
\draw[thick] (8.50,-5)--(8.5,-4); 
\draw[thick] (9.50,-5)--(9.50,-4); 

\draw[fill=white] (8.25,-4.25) rectangle (8.75,-3.75); 

% 7th constant

\draw[thick] (10.25,-5)--(11.25,-5); 
\draw[thick] (10.25,-4)--(11.25,-4); 
\draw[thick] (10.25,-5)--(10.25,-4); 
\draw[thick] (11.25,-5)--(11.25,-4); 

\draw[fill=white] (11,-4.25) rectangle (11.5,-3.75); 

% 8th constant

\draw[thick] (12,-5)--(13,-5); 
\draw[thick] (12,-4)--(13,-4); 
\draw[thick] (12,-5)--(12,-4); 
\draw[thick] (13,-5)--(13,-4); 

\draw[fill=white] (12.75,-5.25) rectangle (13.25,-4.75); 
% Connectors from examples to constants

\draw[thin] (1,-0.3)--(0.25,-3.7); 
\draw[thin] (1,-0.3)--(2,-3.7); 
\draw[thin] (1,-0.3)--(10.75,-3.7); 
\draw[thin] (1,-0.3)--(12.5,-3.7); 

\draw[thin] (4,-0.3)--(5.5,-3.7); 
\draw[thin] (4,-0.3)--(3.75,-3.7); 
\draw[thin] (4,-0.3)--(7.25,-3.7); 
\draw[thin] (4,-0.3)--(9,-3.7); 

\draw[thin] (7,-0.3)--(2,-3.7); 
\draw[thin] (7,-0.3)--(5.5,-3.7); 
\draw[thin] (7,-0.3)--(7.25,-3.7); 
\draw[thin] (7,-0.3)--(10.75,-3.7); 

\draw[thin] (10,-0.3)--(3.75,-3.7); 
\draw[thin] (10,-0.3)--(7.25,-3.7); 
\draw[thin] (10,-0.3)--(9,-3.7); 
\draw[thin] (10,-0.3)--(12.5,-3.7); 

\draw[thin] (13,-0.3)--(5.5,-3.7); 
\draw[thin] (13,-0.3)--(7.25,-3.7); 
\draw[thin] (13,-0.3)--(9,-3.7); 
\draw[thin] (13,-0.3)--(10.75,-3.7); 

% Vertical constant and atoms

\node[text width=0.1cm] at (-3,-4.5) 
{{$v$}};

\node[text width=0.1cm] at (7,-7.9) 
{$0$};

%\node[text width=0.1cm] at (-3,-8) 
% {{$\phi$}};

%\node[text width=0.1cm] at (0.25,-8) 
% {{$\epsilon_1$}};

\node[text width=0.1cm] at (3.75,-8) 
{{$\epsilon_2$}};

\node[text width=0.1cm] at (5.5,-8) 
{{$\epsilon_3$}};
\node[text width=0.1cm] at (-1.8,-8) 
{{$\phi_{1}$}};

% Connectors to atoms

% phi-v

%\draw[thin] (-3,-5.3)--(-3,-7.6);

% phi11-constants

\draw[thin] (-1.5,-7.4)--(-3,-5.3);
\draw[thin] (-1.5,-7.4)--(0.25,-5.3);

% epsilon1-constant
%\draw[thin] (0.25,-5.3)--(0.25,-7.6);
% epsilon2-constant
\draw[thin] (3.75,-5.3)--(3.75,-7.6);
% epsilon3-constant
\draw[thin] (5.5,-5.3)--(5.5,-7.6);

% 0-v

\draw[thin] (7,-7.3)--(-3,-5.3);

% 0-contant 8

\draw[thin] (7,-7.3)--(12.5,-5.3);
\draw[thin] (7,-7.3)--(10.75,-5.3); 
\draw[thin] (7,-7.3)--(9,-5.3); 
\draw[thin] (7,-7.3)--(7.25,-5.3); 
\draw[thin] (7,-7.3)--(5.5,-5.3); 
\draw[thin] (7,-7.3)--(3.75,-5.3); 
\draw[thin] (7,-7.3)--(2,-5.3); 
\draw[thin] (7,-7.3)--(0.25,-5.3); 

% alpha

%\draw[thin] (-2.8,-5.3)--(-2.8,-9.6); 
%\draw[thin] (0.25,-5.3)--(-2.8,-9.6); 
%\draw[thin] (5.5,-5.3)--(-2.8,-9.6); 

% beta

%\draw[thin] (-2.8,-5.3)--(0.25,-9.6); 
%\draw[thin] (0.25,-5.3)--(0.25,-9.6); 
%\draw[thin] (3.75,-5.3)--(0.25,-9.6); 

% gamma

%\draw[thin] (0.25,-5.3)--(2,-9.6); 

% delta

%\draw[thin] (3.75,-5.3)--(3.75,-9.6); 

% delta

%\draw[thin] (5.5,-5.3)--(5.5,-9.6); 

%\draw[thin] (0.25,-5.3)--(0.25,-9.6); 
%\draw[thin] (0.25,-5.3)--(6.4,-9.6); 
%\draw[thin] (12.5,-5.3)--(6.4,-9.6); 

%%%%
%%%% HERE DUAL
%%%%

% First positive example

\draw[thick] (20,-8)--(22,-8); 
\draw[thick] (21,-8)--(21,-6); 
\draw[thick] (22,-8)--(22,-6); 
\draw[thick] (20,-8)--(20,-6); 
\draw[thick] (20,-6)--(22,-6);
\draw[thick] (20,-7)--(22,-7);

%\draw[fill=black] (-0.5,1.5) rectangle (0.5,2.5); 
\draw[fill=black] (20,-8) rectangle (21,-6); 
%\draw[fill=black] (1,0) rectangle (2,1); 

% Second positive example

\draw[thick] (23,-8)--(25,-8); 
\draw[thick] (24,-8)--(24,-6); 
\draw[thick] (25,-8)--(25,-6); 
\draw[thick] (23,-8)--(23,-6); 
\draw[thick] (23,-6)--(25,-6);
\draw[thick] (23,-7)--(25,-7);

\draw[fill=black] (24,-8) rectangle (25,-6); 

% Separator between positive and negative examples

\draw[thick] (25.5,-8.5)--(25.5,-5.5);

% First negative example

\draw[thick] (26,-8)--(28,-8); 
\draw[thick] (27,-8)--(27,-6); 
\draw[thick] (28,-8)--(28,-6); 
\draw[thick] (26,-8)--(26,-6); 
\draw[thick] (26,-6)--(28,-6);
\draw[thick] (26,-7)--(28,-7);

\draw[fill=black] (26,-7) rectangle (27,-6); 
\draw[fill=black] (27,-8) rectangle (28,-7); 

% Second negative example

\draw[thick] (29,-8)--(31,-8); 
\draw[thick] (30,-8)--(30,-6); 
\draw[thick] (31,-8)--(31,-6); 
\draw[thick] (29,-8)--(29,-6); 
\draw[thick] (29,-6)--(31,-6);
\draw[thick] (29,-7)--(31,-7);

\draw[fill=black] (30,-7) rectangle (31,-6); 

% third negative example

\draw[thick] (32,-8)--(34,-8); 
\draw[thick] (33,-8)--(33,-6); 
\draw[thick] (34,-8)--(34,-6); 
\draw[thick] (32,-8)--(32,-6); 
\draw[thick] (32,-6)--(34,-6);
\draw[thick] (32,-7)--(34,-7);

\draw[fill=black] (33,-8) rectangle (34,-7); 

% Now the constants

% First constant

\draw[thick] (19.75,-2)--(20.75,-2); 
\draw[thick] (19.75,-1)--(20.75,-1); 
\draw[thick] (19.75,-2)--(19.75,-1); 
\draw[thick] (20.75,-2)--(20.75,-1);

\draw[fill=black] (19.5,-2.25) rectangle (20,-1.75); 

% Second constant

\draw[thick] (21.5,-2)--(22.5,-2); 
\draw[thick] (21.5,-1)--(22.5,-1); 
\draw[thick] (21.5,-2)--(21.5,-1); 
\draw[thick] (22.5,-2)--(22.5,-1); 

\draw[fill=black] (21.25,-1.25) rectangle (21.75,-0.75); 

% Third constant

\draw[thick] (23.25,-2)--(24.25,-2); 
\draw[thick] (23.25,-1)--(24.25,-1); 
\draw[thick] (23.25,-2)--(23.25,-1); 
\draw[thick] (24.25,-2)--(24.25,-1); 

\draw[fill=black] (24,-1.25) rectangle (24.5,-0.75); 

% Fourth constant

\draw[thick] (25,-2)--(26,-2); 
\draw[thick] (25,-1)--(26,-1); 
\draw[thick] (25,-2)--(25,-1); 
\draw[thick] (26,-2)--(26,-1); 

\draw[fill=black] (25.75,-2.25) rectangle (26.25,-1.75); 

% 5th constant (add 6.25 in x coordinate, and black goes to white)

\draw[thick] (26.75,-2)--(27.75,-2); 
\draw[thick] (26.75,-1)--(27.75,-1); 
\draw[thick] (26.75,-2)--(26.75,-1); 
\draw[thick] (27.75,-2)--(27.75,-1); 

\draw[fill=white] (26.5,-2.25) rectangle (27,-1.75); 

% 6th constant

\draw[thick] (28.50,-2)--(29.5,-2); 
\draw[thick] (28.50,-1)--(29.5,-1); 
\draw[thick] (28.50,-2)--(28.5,-1); 
\draw[thick] (29.50,-2)--(29.50,-1); 

\draw[fill=white] (28.25,-1.25) rectangle (28.75,-0.75); 

% 7th constant

\draw[thick] (30.25,-2)--(31.25,-2); 
\draw[thick] (30.25,-1)--(31.25,-1); 
\draw[thick] (30.25,-2)--(30.25,-1); 
\draw[thick] (31.25,-2)--(31.25,-1); 

\draw[fill=white] (31,-1.25) rectangle (31.5,-0.75); 

% 8th constant

\draw[thick] (32,-2)--(33,-2); 
\draw[thick] (32,-1)--(33,-1); 
\draw[thick] (32,-2)--(32,-1); 
\draw[thick] (33,-2)--(33,-1); 

\draw[fill=white] (32.75,-2.25) rectangle (33.25,-1.75); 
% Connectors from examples to constants

\draw[thin] (21,-5.7)--(20.25,-2.3); 
\draw[thin] (21,-5.7)--(22,-2.3); 
\draw[thin] (21,-5.7)--(30.75,-2.3); 
\draw[thin] (21,-5.7)--(32.5,-2.3); 

\draw[thin] (24,-5.7)--(23.75,-2.3); 
\draw[thin] (24,-5.7)--(25.5,-2.3); 
\draw[thin] (24,-5.7)--(27.25,-2.3); 
\draw[thin] (24,-5.7)--(29,-2.3); 

\draw[thin] (27,-5.7)--(22,-2.3); 
\draw[thin] (27,-5.7)--(25.5,-2.3); 
\draw[thin] (27,-5.7)--(27.25,-2.3); 
\draw[thin] (27,-5.7)--(30.75,-2.3); 

\draw[thin] (30,-5.7)--(23.75,-2.3); 
\draw[thin] (30,-5.7)--(27.25,-2.3); 
\draw[thin] (30,-5.7)--(29,-2.3); 
\draw[thin] (30,-5.7)--(32.5,-2.3); 

\draw[thin] (33,-5.7)--(25.5,-2.3); 
\draw[thin] (33,-5.7)--(27.25,-2.3); 
\draw[thin] (33,-5.7)--(29,-2.3); 
\draw[thin] (33,-5.7)--(30.75,-2.3); 

% Vertical constant and atoms

%\node[text width=0.1cm] at (16.7,2) 
% {$[\phi]$}; 
\node[text width=0.1cm] at (17.8,2) 
{$[\phi_{1}]$};
%\node[text width=0.1cm] at (19.60,2) 
% {$[\epsilon_1]$};
\node[text width=0.1cm] at (23.1,2) 
{$[\epsilon_2]$};
\node[text width=0.1cm] at (24.85,2) 
{$[\epsilon_3]$};
\node[text width=0.1cm] at (16.7,-1.5) 
{$[v]$};

\node[text width=0.1cm] at (21,-11.5) 
{$0^{*}$};

\node[text width=0.1cm] at (26.6,2) 
{$[0]$};

\node[text width=0.1cm] at (27,-11.5) 
{$\zeta_1$};

\node[text width=0.1cm] at (30,-11.5) 
{$\zeta_2$};
\node[text width=0.1cm] at (33,-11.5) 
{$\zeta_3$}; 

%\node[text width=0.1cm] at (-3,-10) 
% {$\Large{\alpha}$};

%\node[text width=0.1cm] at (0.25,-10) 
% {$\Large{\beta}$};

%\node[text width=0.1cm] at (2,-10) 
% {$\Large{\gamma}$};
%\node[text width=0.1cm] at (3.75,-10) 
% {$\Large{\delta}$};

%\node[text width=0.1cm] at (5.5,-10) 
% {$\Large{\epsilon}$};
%\node[text width=0.1cm] at (0.25,-10) 
% {$\LARGE{\alpha}$};

%\node[text width=0.1cm] at (6.4,-10) 
% {$\LARGE{0}$};
% Connectors to atoms

% From [phi] to v
%\draw[thin] (17.2,1.3)--(17.2,-0.7); 

% From [phi11] to atoms

\draw[thin] (18.5,1.3)--(17.2,-0.7); 
\draw[thin] (18.5,1.3)--(20.25,-0.7); 

% From [epsilon]to constants
%\draw[thin] (20.25,1.3)--(20.25,-0.7); 
\draw[thin] (23.75,1.3)--(23.75,-0.7); 
\draw[thin] (25.5,1.3)--(25.5,-0.7); 
% 0*-Examples

\draw[thin] (21,-10.9)--(21,-8.3); 
\draw[thin] (21,-10.9)--(33,-8.3); 
\draw[thin] (21,-10.9)--(24,-8.3); 
\draw[thin] (21,-10.9)--(27,-8.3); 
\draw[thin] (21,-10.9)--(30,-8.3); 
%\draw[thin] (1,-13.9)--(13,-11.3); 
% First two examples -v

\draw[thin] (21,-5.7)--(17.2,-2.3); 
\draw[thin] (24,-5.7)--(17.2,-2.3); 

% eta-Examples
\draw[thin] (27,-10.9)--(27,-8.3); 
\draw[thin] (30,-10.9)--(30,-8.3); 
\draw[thin] (33,-10.9)--(33,-8.3); 

% [0] to constants

\draw[thin] (27,1.1)--(17.2,-0.7);
\draw[thin] (27,1.1)--(32.5,-0.7);
\draw[thin] (27,1.1)--(30.75,-0.7);
\draw[thin] (27,1.1)--(29,-0.7);
\draw[thin] (27,1.1)--(27.25,-0.7);
\draw[thin] (27,1.1)--(25.5,-0.7);
\draw[thin] (27,1.1)--(23.75,-0.7);
\draw[thin] (27,1.1)--(22,-0.7);
\draw[thin] (27,1.1)--(20.25,-0.7);
% alpha

%\draw[thin] (-2.8,-5.3)--(-2.8,-9.6); 
%\draw[thin] (0.25,-5.3)--(-2.8,-9.6); 
%\draw[thin] (5.5,-5.3)--(-2.8,-9.6); 

% beta

%\draw[thin] (-2.8,-5.3)--(0.25,-9.6); 
%\draw[thin] (0.25,-5.3)--(0.25,-9.6); 
%\draw[thin] (5.5,-5.3)--(0.25,-9.6); 

% gamma

%\draw[thin] (0.25,-5.3)--(2,-9.6); 

% delta

%\draw[thin] (3.75,-5.3)--(3.75,-9.6); 

% delta

%\draw[thin] (5.5,-5.3)--(5.5,-9.6); 

%\draw[thin] (0.25,-5.3)--(0.25,-9.6); 
%\draw[thin] (0.25,-5.3)--(6.4,-9.6); 
%\draw[thin] (12.5,-5.3)--(6.4,-9.6); 

\end{tikzpicture}
\end{center}
We now perform the crossing between $v$ and the second positive example, $T_2^{+}$. We cross the only atom in $\Phi_2$ (the atoms of $v$ that are not atoms of $T_2^{+}$), $\phi_{1}$, with one of the two atoms in $T_2^{+}$, say $\epsilon_2$. We create a new atom $\phi_{2}$ and edges $\phi_{2} \rightarrow \phi_{1}$ and $\phi_{2} \rightarrow \epsilon_2$. Before the crossing, the trace of atom $\phi_{1}$ is $\textbf{Tr}(\phi_{1})_{i}=\{0^{*}\}$, but after the crossing is $\textbf{Tr}(\phi_{1})_{f}=\{0^{*},\zeta_2 \}$. Since the trace has changed, we proceed to select another atom in $T_2^{+}$, that is, $\epsilon_3$. We then create a new atom $\phi_{3}$ and edges $\phi_{3} \rightarrow \phi_{1}$ and $\phi_{3} \rightarrow \epsilon_3$. After appending $\phi_{3}$ with this new edges to the graph the trace of $\phi_{1}$ becomes $\textbf{Tr}(\phi_{1})_{f}=\textbf{Tr}(\phi_{2}) \cap \textbf{Tr}(\phi_{3})=\{0^{*}\}$, so we now have the trace invariance we were looking for. The trace of atom $\epsilon_{2}$ also remains unchanged as $\textbf{Tr}(\epsilon_{2})_{i}=\{0^{*},\zeta_{1}\}$ equals $\textbf{Tr}(\epsilon_{2})_{f}=\textbf{Tr}(\phi_{2})=\{0^{*},\zeta_{1}\}$. For the trace of $\epsilon_{3}$ we have $\textbf{Tr}(\epsilon_{3})_{i}=\textbf{Tr}(\epsilon_{3})_{f}=\{0^{*},\zeta_{1},\zeta_{3}\}$ so it also remains unaltered.
Eliminating the initial atoms $\phi_{1}, \epsilon_2$ and $\epsilon_3$ we have the graph
\begin{center}
\begin{tikzpicture}[scale=0.35]

% First positive example

\draw[thick] (0,0)--(2,0); 
\draw[thick] (1,0)--(1,2); 
\draw[thick] (2,0)--(2,2); 
\draw[thick] (0,0)--(0,2); 
\draw[thick] (0,2)--(2,2);
\draw[thick] (0,1)--(2,1);

%\draw[fill=black] (-0.5,1.5) rectangle (0.5,2.5); 
\draw[fill=black] (0,0) rectangle (1,2); 
%\draw[fill=black] (1,0) rectangle (2,1); 

% Second positive example

\draw[thick] (3,0)--(5,0); 
\draw[thick] (4,0)--(4,2); 
\draw[thick] (5,0)--(5,2); 
\draw[thick] (3,0)--(3,2); 
\draw[thick] (3,2)--(5,2);
\draw[thick] (3,1)--(5,1);

\draw[fill=black] (4,0) rectangle (5,2); 

% Separator between positive and negative examples

\draw[thick] (5.5,-0.5)--(5.5,2.5);

% First negative example

\draw[thick] (6,0)--(8,0); 
\draw[thick] (6,0)--(6,2); 
\draw[thick] (8,0)--(8,2); 
\draw[thick] (6,1)--(8,1); 
\draw[thick] (6,2)--(8,2);
\draw[thick] (7,0)--(7,2);

\draw[fill=black] (6,1) rectangle (7,2); 
\draw[fill=black] (7,0) rectangle (8,1); 

% Second negative example

\draw[thick] (9,0)--(11,0); 
\draw[thick] (10,0)--(10,2); 
\draw[thick] (11,0)--(11,2); 
\draw[thick] (9,0)--(9,2); 
\draw[thick] (9,2)--(11,2);
\draw[thick] (9,1)--(11,1);

\draw[fill=black] (10,1) rectangle (11,2); 

% third negative example

\draw[thick] (12,0)--(14,0); 
\draw[thick] (13,0)--(13,2); 
\draw[thick] (14,0)--(14,2); 
\draw[thick] (12,0)--(12,2); 
\draw[thick] (12,2)--(14,2);
\draw[thick] (12,1)--(14,1);

\draw[fill=black] (13,0) rectangle (14,1); 

% Now the constants

% First constant

\draw[thick] (-0.25,-5)--(0.75,-5); 
\draw[thick] (-0.25,-4)--(0.75,-4); 
\draw[thick] (-0.25,-5)--(-0.25,-4); 
\draw[thick] (0.75,-5)--(0.75,-4); 

\draw[fill=black] (-0.50,-5.25) rectangle (0,-4.75); 

% Second constant

\draw[thick] (1.5,-5)--(2.5,-5); 
\draw[thick] (1.5,-4)--(2.5,-4); 
\draw[thick] (1.5,-5)--(1.5,-4); 
\draw[thick] (2.5,-5)--(2.5,-4); 

\draw[fill=black] (1.25,-4.25) rectangle (1.75,-3.75); 

% Third constant

\draw[thick] (3.25,-5)--(4.25,-5); 
\draw[thick] (3.25,-4)--(4.25,-4); 
\draw[thick] (3.25,-5)--(3.25,-4); 
\draw[thick] (4.25,-5)--(4.25,-4); 

\draw[fill=black] (4,-4.25) rectangle (4.5,-3.75); 

% Fourth constant

\draw[thick] (5,-5)--(6,-5); 
\draw[thick] (5,-4)--(6,-4); 
\draw[thick] (5,-5)--(5,-4); 
\draw[thick] (6,-5)--(6,-4); 

\draw[fill=black] (5.75,-5.25) rectangle (6.25,-4.75); 

% 5th constant (add 6.25 in x coordinate, and black goes to white)

\draw[thick] (6.75,-5)--(7.75,-5); 
\draw[thick] (6.75,-4)--(7.75,-4); 
\draw[thick] (6.75,-5)--(6.75,-4); 
\draw[thick] (7.75,-5)--(7.75,-4); 

\draw[fill=white] (6.5,-5.25) rectangle (7,-4.75); 

% 6th constant

\draw[thick] (8.50,-5)--(9.5,-5); 
\draw[thick] (8.50,-4)--(9.5,-4); 
\draw[thick] (8.50,-5)--(8.5,-4); 
\draw[thick] (9.50,-5)--(9.50,-4); 

\draw[fill=white] (8.25,-4.25) rectangle (8.75,-3.75); 

% 7th constant

\draw[thick] (10.25,-5)--(11.25,-5); 
\draw[thick] (10.25,-4)--(11.25,-4); 
\draw[thick] (10.25,-5)--(10.25,-4); 
\draw[thick] (11.25,-5)--(11.25,-4); 

\draw[fill=white] (11,-4.25) rectangle (11.5,-3.75); 

% 8th constant

\draw[thick] (12,-5)--(13,-5); 
\draw[thick] (12,-4)--(13,-4); 
\draw[thick] (12,-5)--(12,-4); 
\draw[thick] (13,-5)--(13,-4); 

\draw[fill=white] (12.75,-5.25) rectangle (13.25,-4.75); 
% Connectors from examples to constants

\draw[thin] (1,-0.3)--(0.25,-3.7); 
\draw[thin] (1,-0.3)--(2,-3.7); 
\draw[thin] (1,-0.3)--(10.75,-3.7); 
\draw[thin] (1,-0.3)--(12.5,-3.7); 

\draw[thin] (4,-0.3)--(5.5,-3.7); 
\draw[thin] (4,-0.3)--(3.75,-3.7); 
\draw[thin] (4,-0.3)--(7.25,-3.7); 
\draw[thin] (4,-0.3)--(9,-3.7); 

\draw[thin] (7,-0.3)--(2,-3.7); 
\draw[thin] (7,-0.3)--(5.5,-3.7); 
\draw[thin] (7,-0.3)--(7.25,-3.7); 
\draw[thin] (7,-0.3)--(10.75,-3.7); 

\draw[thin] (10,-0.3)--(3.75,-3.7); 
\draw[thin] (10,-0.3)--(7.25,-3.7); 
\draw[thin] (10,-0.3)--(9,-3.7); 
\draw[thin] (10,-0.3)--(12.5,-3.7); 

\draw[thin] (13,-0.3)--(5.5,-3.7); 
\draw[thin] (13,-0.3)--(7.25,-3.7); 
\draw[thin] (13,-0.3)--(9,-3.7); 
\draw[thin] (13,-0.3)--(10.75,-3.7); 

% Vertical constant and atoms

\node[text width=0.1cm] at (-3,-4.5) 
{{$v$}};

\node[text width=0.1cm] at (7,-7.9) 
{$0$};

%\node[text width=0.1cm] at (-3,-8) 
% {{$\phi$}};

%\node[text width=0.1cm] at (0.25,-8) 
% {{$\epsilon_1$}};

%\node[text width=0.1cm] at (3.75,-8) 
% {{$\epsilon_2$}};

%\node[text width=0.1cm] at (5.5,-8) 
% {{$\epsilon_3$}};
%\node[text width=0.1cm] at (-1.8,-8) 
% {{$\phi_{1}$}};

\node[text width=0.1cm] at (0.25,-7.9) 
{{$\phi_{2}$}};

\node[text width=0.1cm] at (3.75,-7.9) 
{{$\phi_{3}$}};
% Connectors to atoms

% phi-v

%\draw[thin] (-3,-5.3)--(-3,-7.6);

% phi12-constants

\draw[thin] (0.25,-7.3)--(-3,-5.3);
\draw[thin] (0.25,-7.3)--(0.25,-5.3);
\draw[thin] (0.25,-7.3)--(3.75,-5.3);

% phi13-constants

\draw[thin] (3.75,-7.3)--(-3,-5.3);
\draw[thin] (3.75,-7.3)--(0.25,-5.3);
\draw[thin] (3.75,-7.3)--(5.5,-5.3);

% epsilon1-constant
%\draw[thin] (0.25,-5.3)--(0.25,-7.6);
% epsilon2-constant
%\draw[thin] (3.75,-5.3)--(3.75,-7.6);
% epsilon3-constants
%\draw[thin] (5.5,-5.3)--(5.5,-7.6);

% 0-v

\draw[thin] (7,-7.3)--(-3,-5.3);

% 0-constant 8

\draw[thin] (7,-7.3)--(12.5,-5.3);
\draw[thin] (7,-7.3)--(10.75,-5.3); 
\draw[thin] (7,-7.3)--(9,-5.3); 
\draw[thin] (7,-7.3)--(7.25,-5.3); 
\draw[thin] (7,-7.3)--(5.5,-5.3); 
\draw[thin] (7,-7.3)--(3.75,-5.3); 
\draw[thin] (7,-7.3)--(2,-5.3); 
\draw[thin] (7,-7.3)--(0.25,-5.3); 

% alpha

%\draw[thin] (-2.8,-5.3)--(-2.8,-9.6); 
%\draw[thin] (0.25,-5.3)--(-2.8,-9.6); 
%\draw[thin] (5.5,-5.3)--(-2.8,-9.6); 

% beta

%\draw[thin] (-2.8,-5.3)--(0.25,-9.6); 
%\draw[thin] (0.25,-5.3)--(0.25,-9.6); 
%\draw[thin] (3.75,-5.3)--(0.25,-9.6); 

% gamma

%\draw[thin] (0.25,-5.3)--(2,-9.6); 

% delta

%\draw[thin] (3.75,-5.3)--(3.75,-9.6); 

% delta

%\draw[thin] (5.5,-5.3)--(5.5,-9.6); 

%\draw[thin] (0.25,-5.3)--(0.25,-9.6); 
%\draw[thin] (0.25,-5.3)--(6.4,-9.6); 
%\draw[thin] (12.5,-5.3)--(6.4,-9.6); 

%%%%
%%%% HERE DUAL
%%%%

% First positive example

\draw[thick] (20,-8)--(22,-8); 
\draw[thick] (21,-8)--(21,-6); 
\draw[thick] (22,-8)--(22,-6); 
\draw[thick] (20,-8)--(20,-6); 
\draw[thick] (20,-6)--(22,-6);
\draw[thick] (20,-7)--(22,-7);

%\draw[fill=black] (-0.5,1.5) rectangle (0.5,2.5); 
\draw[fill=black] (20,-8) rectangle (21,-6); 
%\draw[fill=black] (1,0) rectangle (2,1); 

% Second positive example

\draw[thick] (23,-8)--(25,-8); 
\draw[thick] (24,-8)--(24,-6); 
\draw[thick] (25,-8)--(25,-6); 
\draw[thick] (23,-8)--(23,-6); 
\draw[thick] (23,-6)--(25,-6);
\draw[thick] (23,-7)--(25,-7);

\draw[fill=black] (24,-8) rectangle (25,-6); 

% Separator between positive and negative examples

\draw[thick] (25.5,-7.3)--(25.5,-5.5);

% First negative example

\draw[thick] (26,-8)--(28,-8); 
\draw[thick] (27,-8)--(27,-6); 
\draw[thick] (28,-8)--(28,-6); 
\draw[thick] (26,-8)--(26,-6); 
\draw[thick] (26,-6)--(28,-6);
\draw[thick] (26,-7)--(28,-7);

\draw[fill=black] (26,-7) rectangle (27,-6); 
\draw[fill=black] (27,-8) rectangle (28,-7); 

% Second negative example

\draw[thick] (29,-8)--(31,-8); 
\draw[thick] (30,-8)--(30,-6); 
\draw[thick] (31,-8)--(31,-6); 
\draw[thick] (29,-8)--(29,-6); 
\draw[thick] (29,-6)--(31,-6);
\draw[thick] (29,-7)--(31,-7);

\draw[fill=black] (30,-7) rectangle (31,-6); 

% third negative example

\draw[thick] (32,-8)--(34,-8); 
\draw[thick] (33,-8)--(33,-6); 
\draw[thick] (34,-8)--(34,-6); 
\draw[thick] (32,-8)--(32,-6); 
\draw[thick] (32,-6)--(34,-6);
\draw[thick] (32,-7)--(34,-7);

\draw[fill=black] (33,-8) rectangle (34,-7); 

% Now the constants

% First constant

\draw[thick] (19.75,-2)--(20.75,-2); 
\draw[thick] (19.75,-1)--(20.75,-1); 
\draw[thick] (19.75,-2)--(19.75,-1); 
\draw[thick] (20.75,-2)--(20.75,-1);

\draw[fill=black] (19.5,-2.25) rectangle (20,-1.75); 

% Second constant

\draw[thick] (21.5,-2)--(22.5,-2); 
\draw[thick] (21.5,-1)--(22.5,-1); 
\draw[thick] (21.5,-2)--(21.5,-1); 
\draw[thick] (22.5,-2)--(22.5,-1); 

\draw[fill=black] (21.25,-1.25) rectangle (21.75,-0.75); 

% Third constant

\draw[thick] (23.25,-2)--(24.25,-2); 
\draw[thick] (23.25,-1)--(24.25,-1); 
\draw[thick] (23.25,-2)--(23.25,-1); 
\draw[thick] (24.25,-2)--(24.25,-1); 

\draw[fill=black] (24,-1.25) rectangle (24.5,-0.75); 

% Fourth constant

\draw[thick] (25,-2)--(26,-2); 
\draw[thick] (25,-1)--(26,-1); 
\draw[thick] (25,-2)--(25,-1); 
\draw[thick] (26,-2)--(26,-1); 

\draw[fill=black] (25.75,-2.25) rectangle (26.25,-1.75); 

% 5th constant (add 6.25 in x coordinate, and black goes to white)

\draw[thick] (26.75,-2)--(27.75,-2); 
\draw[thick] (26.75,-1)--(27.75,-1); 
\draw[thick] (26.75,-2)--(26.75,-1); 
\draw[thick] (27.75,-2)--(27.75,-1); 

\draw[fill=white] (26.5,-2.25) rectangle (27,-1.75); 

% 6th constant

\draw[thick] (28.50,-2)--(29.5,-2); 
\draw[thick] (28.50,-1)--(29.5,-1); 
\draw[thick] (28.50,-2)--(28.5,-1); 
\draw[thick] (29.50,-2)--(29.50,-1); 

\draw[fill=white] (28.25,-1.25) rectangle (28.75,-0.75); 

% 7th constant

\draw[thick] (30.25,-2)--(31.25,-2); 
\draw[thick] (30.25,-1)--(31.25,-1); 
\draw[thick] (30.25,-2)--(30.25,-1); 
\draw[thick] (31.25,-2)--(31.25,-1); 

\draw[fill=white] (31,-1.25) rectangle (31.5,-0.75); 

% 8th constant

\draw[thick] (32,-2)--(33,-2); 
\draw[thick] (32,-1)--(33,-1); 
\draw[thick] (32,-2)--(32,-1); 
\draw[thick] (33,-2)--(33,-1); 

\draw[fill=white] (32.75,-2.25) rectangle (33.25,-1.75); 
% Connectors from examples to constants

\draw[thin] (21,-5.7)--(20.25,-2.3); 
\draw[thin] (21,-5.7)--(22,-2.3); 
\draw[thin] (21,-5.7)--(30.75,-2.3); 
\draw[thin] (21,-5.7)--(32.5,-2.3); 

\draw[thin] (24,-5.7)--(23.75,-2.3); 
\draw[thin] (24,-5.7)--(25.5,-2.3); 
\draw[thin] (24,-5.7)--(27.25,-2.3); 
\draw[thin] (24,-5.7)--(29,-2.3); 

\draw[thin] (27,-5.7)--(22,-2.3); 
\draw[thin] (27,-5.7)--(25.5,-2.3); 
\draw[thin] (27,-5.7)--(27.25,-2.3); 
\draw[thin] (27,-5.7)--(30.75,-2.3); 

\draw[thin] (30,-5.7)--(23.75,-2.3); 
\draw[thin] (30,-5.7)--(27.25,-2.3); 
\draw[thin] (30,-5.7)--(29,-2.3); 
\draw[thin] (30,-5.7)--(32.5,-2.3); 

\draw[thin] (33,-5.7)--(25.5,-2.3); 
\draw[thin] (33,-5.7)--(27.25,-2.3); 
\draw[thin] (33,-5.7)--(29,-2.3); 
\draw[thin] (33,-5.7)--(30.75,-2.3); 

% Vertical constant and atoms

%\node[text width=0.1cm] at (16.7,2) 
% {$[\phi]$}; 
%\node[text width=0.1cm] at (17.8,2) 
% {$[\phi_{1}]$};
%\node[text width=0.1cm] at (19.60,2) 
% {$[\epsilon_1]$};
\node[text width=0.1cm] at (23.1,2) 
{$[\phi_{3}]$};
\node[text width=0.1cm] at (19.6,2) 
{$[\phi_{2}]$};
\node[text width=0.1cm] at (16.7,-1.5) 
{$[v]$};

\node[text width=0.1cm] at (21,-11.5) 
{$0^{*}$};

\node[text width=0.1cm] at (26.6,2) 
{$[0]$};

\node[text width=0.1cm] at (27,-11.5) 
{$\zeta_1$};

\node[text width=0.1cm] at (30,-11.5) 
{$\zeta_2$};
\node[text width=0.1cm] at (33,-11.5) 
{$\zeta_3$}; 

%\node[text width=0.1cm] at (-3,-10) 
% {$\Large{\alpha}$};

%\node[text width=0.1cm] at (0.25,-10) 
% {$\Large{\beta}$};

%\node[text width=0.1cm] at (2,-10) 
% {$\Large{\gamma}$};
%\node[text width=0.1cm] at (3.75,-10) 
% {$\Large{\delta}$};

%\node[text width=0.1cm] at (5.5,-10) 
% {$\Large{\epsilon}$};
%\node[text width=0.1cm] at (0.25,-10) 
% {$\LARGE{\alpha}$};

%\node[text width=0.1cm] at (6.4,-10) 
% {$\LARGE{0}$};
% Connectors to atoms

% From [phi] to v
%\draw[thin] (17.2,1.3)--(17.2,-0.7); 

% From [phi12] to constants

\draw[thin] (20.25,1.1)--(17.2,-0.7); 
\draw[thin] (20.25,1.1)--(20.25,-0.7); 
\draw[thin] (20.25,1.1)--(23.75,-0.7); 

% From [phi13] to constants

\draw[thin] (23.75,1.1)--(17.2,-0.7); 
\draw[thin] (23.75,1.1)--(20.25,-0.7); 
\draw[thin] (23.75,1.1)--(25.5,-0.7); 

% From [epsilon]to constants
%\draw[thin] (20.25,1.3)--(20.25,-0.7); 
%\draw[thin] (23.75,1.3)--(23.75,-0.7); 
%\draw[thin] (25.5,1.3)--(25.5,-0.7); 
% 0*-Examples

\draw[thin] (21,-10.9)--(21,-8.3); 
\draw[thin] (21,-10.9)--(33,-8.3); 
\draw[thin] (21,-10.9)--(24,-8.3); 
\draw[thin] (21,-10.9)--(27,-8.3); 
\draw[thin] (21,-10.9)--(30,-8.3); 
%\draw[thin] (1,-13.9)--(13,-11.3); 
% First two examples -v

\draw[thin] (21,-5.7)--(17.2,-2.3); 
\draw[thin] (24,-5.7)--(17.2,-2.3); 

% eta-Examples
\draw[thin] (27,-10.9)--(27,-8.3); 
\draw[thin] (30,-10.9)--(30,-8.3); 
\draw[thin] (33,-10.9)--(33,-8.3); 

% [0] to constants

\draw[thin] (27,1.1)--(17.2,-0.7);
\draw[thin] (27,1.1)--(32.5,-0.7);
\draw[thin] (27,1.1)--(30.75,-0.7);
\draw[thin] (27,1.1)--(29,-0.7);
\draw[thin] (27,1.1)--(27.25,-0.7);
\draw[thin] (27,1.1)--(25.5,-0.7);
\draw[thin] (27,1.1)--(23.75,-0.7);
\draw[thin] (27,1.1)--(22,-0.7);
\draw[thin] (27,1.1)--(20.25,-0.7);
% alpha

%\draw[thin] (-2.8,-5.3)--(-2.8,-9.6); 
%\draw[thin] (0.25,-5.3)--(-2.8,-9.6); 
%\draw[thin] (5.5,-5.3)--(-2.8,-9.6); 

% beta

%\draw[thin] (-2.8,-5.3)--(0.25,-9.6); 
%\draw[thin] (0.25,-5.3)--(0.25,-9.6); 
%\draw[thin] (5.5,-5.3)--(0.25,-9.6); 

% gamma

%\draw[thin] (0.25,-5.3)--(2,-9.6); 

% delta

%\draw[thin] (3.75,-5.3)--(3.75,-9.6); 

% delta

%\draw[thin] (5.5,-5.3)--(5.5,-9.6); 

%\draw[thin] (0.25,-5.3)--(0.25,-9.6); 
%\draw[thin] (0.25,-5.3)--(6.4,-9.6); 
%\draw[thin] (12.5,-5.3)--(6.4,-9.6); 
\end{tikzpicture}
\end{center}
If we had more atoms in $v$, we would repeat the same procedure for these atoms. In our present case, we have finished. 

The Sparse Crossing operation has worked; the positive training examples all obey $v<T_{i}^{+}$, and the negative ones $ v \not< T_{j}^{-}$. The atoms of $v$ are ${\bf{GL}^{a}}(v)=\{0,\phi_{2},\phi_{3}\}$ and the atoms for the positive training examples are also ${\bf{GL}^{a}}(T_{1,2}^{+})=\{0,\phi_{2},\phi_{3}\}$, while for the negative examples we have ${\bf{GL}^{a}}(T_{1,3}^{-})=\{0,\phi_{3}\}$ and ${\bf{GL}^{a}}(T_{2}^{-})=\{0,\phi_{2}\}$. 

\subsection{Reduction operation} \label{reduction}

Models found by Sparse Crossing are much smaller than models found by Full Crossing.  We are still interested in further reducing their size. A suitable size reduction algorithm should be trace-invariant. Trace invariance preserves trace constraints and preserving trace constraints ensures that we will be able to carry out pending Sparse Crossing operations. This means that a trace-invariant reduction scheme can be called at any time during learning, as often as required.

While carrying out Sparse Crossing operations we were careful to keep the trace of all the atoms unaltered, for the reduction operation we will focus on constants instead. An operation that keeps the trace of all constants unaltered also keeps the trace of the terms unaltered and the trace constraints preserved. Since atoms are not mentioned in trace constraints we do not really need them to be trace-invariant; it is enough with keeping the constants trace-invariant.  Furthermore, we can remove atoms from a model as long as we keep the traces of all the constants unchanged. 

Our reduction scheme consists of finding a subset $Q$ of the atoms that produces the same traces for all constants. We can then discard the atoms that are not in $Q$. We start with $Q$ empty. Then we review the constants one by one in random order. For each constant $c$ we select a subset of its atoms such that the trace of $c$ calculated using only the atoms in this subset corresponds with the actual trace of $c$. The selected atoms are added to $Q$ before we move onto the next constant. When selecting atoms for the next constant $c$ we start with the intersection between its atoms and $Q$, i.e ${{\bf{GL}}^a}(c) \cap Q$, and continue adding atoms to $Q$ until the trace calculated taking into account the atoms of ${{\bf{GL}}^a}(c) \cap Q$ equates ${\bf{Tr}}(c)$. Once all constants are reviewed any atom that is not in $Q$ can be safely removed from the algebra. \textbf{Algorithm \ref{traceReduction}} in \textbf{Appendix \ref{algorithms}} is an improved version of this method. 

At the point where we are with our toy problem we have already obtained an atomized model that distinguishes well positive from negative examples. We are going to apply the reduction algorithm to see if we can get rid of some of its atoms. We are going to review the constants starting by the third one. In our example, the trace of the third constant only depends on a single atom $\textbf{Tr}(c_{3})=\textbf{Tr}(0) \cap \textbf{Tr}(\phi_{2})=\textbf{Tr}(\phi_{2})$, so we must keep $\phi_{2}$ for trace invariance of this constant. An analogous situation takes place for the fourth constant, for which we have that its trace only depends on $\phi_{3}$, $\textbf{Tr}(c_{4})=\textbf{Tr}(\phi_{3})$, so it cannot be eliminated, either. The model we have obtained cannot be reduced.

To illustrate a simple size reduction in action, let us add to the graphs an extra atom $\beta$ edged to the first constant, $\beta \rightarrow c_{1}$. This new atom does not change any traces and gives the graphs 
\begin{center}
\begin{tikzpicture}[scale=0.35]

% First positive example

\draw[thick] (0,0)--(2,0);    
\draw[thick] (1,0)--(1,2);    
\draw[thick] (2,0)--(2,2);    
\draw[thick] (0,0)--(0,2);  
\draw[thick] (0,2)--(2,2);
\draw[thick] (0,1)--(2,1);

%\draw[fill=black] (-0.5,1.5) rectangle (0.5,2.5);  
\draw[fill=black] (0,0) rectangle (1,2);  
%\draw[fill=black] (1,0) rectangle (2,1);  

% Second positive example

\draw[thick] (3,0)--(5,0);    
\draw[thick] (4,0)--(4,2);    
\draw[thick] (5,0)--(5,2);    
\draw[thick] (3,0)--(3,2);  
\draw[thick] (3,2)--(5,2);
\draw[thick] (3,1)--(5,1);

\draw[fill=black] (4,0) rectangle (5,2);  

% Separator between positive and negative examples

\draw[thick] (5.5,-0.5)--(5.5,2.5);

% First negative example

\draw[thick] (6,0)--(8,0);    
\draw[thick] (6,0)--(6,2);    
\draw[thick] (8,0)--(8,2);    
\draw[thick] (6,1)--(8,1);  
\draw[thick] (6,2)--(8,2);
\draw[thick] (7,0)--(7,2);

\draw[fill=black] (6,1) rectangle (7,2);  
\draw[fill=black] (7,0) rectangle (8,1);  

% Second negative example

\draw[thick] (9,0)--(11,0);    
\draw[thick] (10,0)--(10,2);    
\draw[thick] (11,0)--(11,2);    
\draw[thick] (9,0)--(9,2);  
\draw[thick] (9,2)--(11,2);
\draw[thick] (9,1)--(11,1);

\draw[fill=black] (10,1) rectangle (11,2);  

% third negative example

\draw[thick] (12,0)--(14,0);    
\draw[thick] (13,0)--(13,2);    
\draw[thick] (14,0)--(14,2);    
\draw[thick] (12,0)--(12,2);  
\draw[thick] (12,2)--(14,2);
\draw[thick] (12,1)--(14,1);

\draw[fill=black] (13,0) rectangle (14,1);  

% Now the constants

% First constant

\draw[thick] (-0.25,-5)--(0.75,-5);    
\draw[thick] (-0.25,-4)--(0.75,-4);    
\draw[thick] (-0.25,-5)--(-0.25,-4);    
\draw[thick] (0.75,-5)--(0.75,-4);  

\draw[fill=black] (-0.50,-5.25) rectangle (0,-4.75);  

% Second constant

\draw[thick] (1.5,-5)--(2.5,-5);    
\draw[thick] (1.5,-4)--(2.5,-4);    
\draw[thick] (1.5,-5)--(1.5,-4);    
\draw[thick] (2.5,-5)--(2.5,-4);  

\draw[fill=black] (1.25,-4.25) rectangle (1.75,-3.75); 

% Third constant

\draw[thick] (3.25,-5)--(4.25,-5);    
\draw[thick] (3.25,-4)--(4.25,-4);    
\draw[thick] (3.25,-5)--(3.25,-4);    
\draw[thick] (4.25,-5)--(4.25,-4);  

\draw[fill=black] (4,-4.25) rectangle (4.5,-3.75);  

% Fourth constant

\draw[thick] (5,-5)--(6,-5);    
\draw[thick] (5,-4)--(6,-4);    
\draw[thick] (5,-5)--(5,-4);    
\draw[thick] (6,-5)--(6,-4);  

\draw[fill=black] (5.75,-5.25) rectangle (6.25,-4.75);  

% 5th constant (add 6.25 in x coordinate, and black goes to white)

\draw[thick] (6.75,-5)--(7.75,-5);    
\draw[thick] (6.75,-4)--(7.75,-4);    
\draw[thick] (6.75,-5)--(6.75,-4);    
\draw[thick] (7.75,-5)--(7.75,-4);  

\draw[fill=white] (6.5,-5.25) rectangle (7,-4.75);  

% 6th constant

\draw[thick] (8.50,-5)--(9.5,-5);    
\draw[thick] (8.50,-4)--(9.5,-4);    
\draw[thick] (8.50,-5)--(8.5,-4);    
\draw[thick] (9.50,-5)--(9.50,-4);  

\draw[fill=white] (8.25,-4.25) rectangle (8.75,-3.75); 

% 7th constant

\draw[thick] (10.25,-5)--(11.25,-5);    
\draw[thick] (10.25,-4)--(11.25,-4);    
\draw[thick] (10.25,-5)--(10.25,-4);    
\draw[thick] (11.25,-5)--(11.25,-4);  

\draw[fill=white] (11,-4.25) rectangle (11.5,-3.75);  

% 8th constant

\draw[thick] (12,-5)--(13,-5);    
\draw[thick] (12,-4)--(13,-4);    
\draw[thick] (12,-5)--(12,-4);    
\draw[thick] (13,-5)--(13,-4);  

\draw[fill=white] (12.75,-5.25) rectangle (13.25,-4.75);  
    
% Connectors from examples to constants

\draw[thin] (1,-0.3)--(0.25,-3.7);  
\draw[thin] (1,-0.3)--(2,-3.7);  
\draw[thin] (1,-0.3)--(10.75,-3.7);  
\draw[thin] (1,-0.3)--(12.5,-3.7);  

\draw[thin] (4,-0.3)--(5.5,-3.7); 
\draw[thin] (4,-0.3)--(3.75,-3.7); 
\draw[thin] (4,-0.3)--(7.25,-3.7); 
\draw[thin] (4,-0.3)--(9,-3.7); 

\draw[thin] (7,-0.3)--(2,-3.7); 
\draw[thin] (7,-0.3)--(5.5,-3.7); 
\draw[thin] (7,-0.3)--(7.25,-3.7); 
\draw[thin] (7,-0.3)--(10.75,-3.7); 

\draw[thin] (10,-0.3)--(3.75,-3.7); 
\draw[thin] (10,-0.3)--(7.25,-3.7); 
\draw[thin] (10,-0.3)--(9,-3.7); 
\draw[thin] (10,-0.3)--(12.5,-3.7); 

\draw[thin] (13,-0.3)--(5.5,-3.7); 
\draw[thin] (13,-0.3)--(7.25,-3.7); 
\draw[thin] (13,-0.3)--(9,-3.7); 
\draw[thin] (13,-0.3)--(10.75,-3.7); 

% Vertical constant and atoms

\node[text width=0.1cm] at (-3,-4.5) 
    {{$v$}};

\node[text width=0.1cm] at (7,-7.9) 
    {$0$};

%\node[text width=0.1cm] at (-3,-8) 
%    {{$\phi$}};

%\node[text width=0.1cm] at (0.25,-8) 
%    {{$\epsilon_1$}};

%\node[text width=0.1cm] at (3.75,-8) 
%    {{$\epsilon_2$}};

%\node[text width=0.1cm] at (5.5,-8) 
%    {{$\epsilon_3$}};
    
%\node[text width=0.1cm] at (-1.8,-8) 
%    {{$\phi_{1}$}};

\node[text width=0.1cm] at (0.25,-7.9) 
    {{$\phi_{2}$}};

\node[text width=0.1cm] at (3.75,-7.9) 
    {{$\phi_{3}$}};

\node[text width=0.1cm] at (-3,-7.9) 
    {{$\beta$}};
% Connectors to atoms

% phi-v

%\draw[thin] (-3,-5.3)--(-3,-7.6);

% beta-constants

\draw[thin] (-3,-7.3)--(0.25,-5.3);

% phi12-constants

\draw[thin] (0.25,-7.3)--(-3,-5.3);
\draw[thin] (0.25,-7.3)--(0.25,-5.3);
\draw[thin] (0.25,-7.3)--(3.75,-5.3);

% phi13-constants

\draw[thin] (3.75,-7.3)--(-3,-5.3);
\draw[thin] (3.75,-7.3)--(0.25,-5.3);
\draw[thin] (3.75,-7.3)--(5.5,-5.3);

%  epsilon1-constant
%\draw[thin] (0.25,-5.3)--(0.25,-7.6);
%  epsilon2-constant
%\draw[thin] (3.75,-5.3)--(3.75,-7.6);
%  epsilon3-constants
%\draw[thin] (5.5,-5.3)--(5.5,-7.6);

% 0-v

\draw[thin] (7,-7.3)--(-3,-5.3);

% 0-constant 8

\draw[thin] (7,-7.3)--(12.5,-5.3);
\draw[thin] (7,-7.3)--(10.75,-5.3); 
\draw[thin] (7,-7.3)--(9,-5.3); 
\draw[thin] (7,-7.3)--(7.25,-5.3); 
\draw[thin] (7,-7.3)--(5.5,-5.3); 
\draw[thin] (7,-7.3)--(3.75,-5.3); 
\draw[thin] (7,-7.3)--(2,-5.3); 
\draw[thin] (7,-7.3)--(0.25,-5.3); 

% alpha

%\draw[thin] (-2.8,-5.3)--(-2.8,-9.6); 
%\draw[thin] (0.25,-5.3)--(-2.8,-9.6); 
%\draw[thin] (5.5,-5.3)--(-2.8,-9.6); 

% beta

%\draw[thin] (-2.8,-5.3)--(0.25,-9.6); 
%\draw[thin] (0.25,-5.3)--(0.25,-9.6); 
%\draw[thin] (3.75,-5.3)--(0.25,-9.6); 

% gamma

%\draw[thin] (0.25,-5.3)--(2,-9.6); 

% delta

%\draw[thin] (3.75,-5.3)--(3.75,-9.6); 

% delta

%\draw[thin] (5.5,-5.3)--(5.5,-9.6); 

%\draw[thin] (0.25,-5.3)--(0.25,-9.6); 
%\draw[thin] (0.25,-5.3)--(6.4,-9.6); 
%\draw[thin] (12.5,-5.3)--(6.4,-9.6); 

%%%%
%%%% HERE DUAL
%%%%

% First positive example

\draw[thick] (20,-8)--(22,-8);    
\draw[thick] (21,-8)--(21,-6);    
\draw[thick] (22,-8)--(22,-6);    
\draw[thick] (20,-8)--(20,-6);  
\draw[thick] (20,-6)--(22,-6);
\draw[thick] (20,-7)--(22,-7);

%\draw[fill=black] (-0.5,1.5) rectangle (0.5,2.5);  
\draw[fill=black] (20,-8) rectangle (21,-6);  
%\draw[fill=black] (1,0) rectangle (2,1);  

% Second positive example

\draw[thick] (23,-8)--(25,-8);    
\draw[thick] (24,-8)--(24,-6);    
\draw[thick] (25,-8)--(25,-6);    
\draw[thick] (23,-8)--(23,-6);  
\draw[thick] (23,-6)--(25,-6);
\draw[thick] (23,-7)--(25,-7);

\draw[fill=black] (24,-8) rectangle (25,-6);  

% Separator between positive and negative examples

\draw[thick] (25.5,-7.3)--(25.5,-5.5);

% First negative example

\draw[thick] (26,-8)--(28,-8);    
\draw[thick] (27,-8)--(27,-6);    
\draw[thick] (28,-8)--(28,-6);    
\draw[thick] (26,-8)--(26,-6);  
\draw[thick] (26,-6)--(28,-6);
\draw[thick] (26,-7)--(28,-7);

\draw[fill=black] (26,-7) rectangle (27,-6);  
\draw[fill=black] (27,-8) rectangle (28,-7);  

% Second negative example

\draw[thick] (29,-8)--(31,-8);    
\draw[thick] (30,-8)--(30,-6);    
\draw[thick] (31,-8)--(31,-6);    
\draw[thick] (29,-8)--(29,-6);  
\draw[thick] (29,-6)--(31,-6);
\draw[thick] (29,-7)--(31,-7);

\draw[fill=black] (30,-7) rectangle (31,-6);  

% third negative example

\draw[thick] (32,-8)--(34,-8);    
\draw[thick] (33,-8)--(33,-6);    
\draw[thick] (34,-8)--(34,-6);    
\draw[thick] (32,-8)--(32,-6);  
\draw[thick] (32,-6)--(34,-6);
\draw[thick] (32,-7)--(34,-7);

\draw[fill=black] (33,-8) rectangle (34,-7);  

% Now the constants

% First constant

\draw[thick] (19.75,-2)--(20.75,-2);    
\draw[thick] (19.75,-1)--(20.75,-1);    
\draw[thick] (19.75,-2)--(19.75,-1);    
\draw[thick] (20.75,-2)--(20.75,-1);

\draw[fill=black] (19.5,-2.25) rectangle (20,-1.75);  

% Second constant

\draw[thick] (21.5,-2)--(22.5,-2);    
\draw[thick] (21.5,-1)--(22.5,-1);    
\draw[thick] (21.5,-2)--(21.5,-1);    
\draw[thick] (22.5,-2)--(22.5,-1);  

\draw[fill=black] (21.25,-1.25) rectangle (21.75,-0.75); 

% Third constant

\draw[thick] (23.25,-2)--(24.25,-2);    
\draw[thick] (23.25,-1)--(24.25,-1);    
\draw[thick] (23.25,-2)--(23.25,-1);    
\draw[thick] (24.25,-2)--(24.25,-1);  

\draw[fill=black] (24,-1.25) rectangle (24.5,-0.75);  

% Fourth constant

\draw[thick] (25,-2)--(26,-2);    
\draw[thick] (25,-1)--(26,-1);    
\draw[thick] (25,-2)--(25,-1);    
\draw[thick] (26,-2)--(26,-1);  

\draw[fill=black] (25.75,-2.25) rectangle (26.25,-1.75);  

% 5th constant (add 6.25 in x coordinate, and black goes to white)

\draw[thick] (26.75,-2)--(27.75,-2);    
\draw[thick] (26.75,-1)--(27.75,-1);    
\draw[thick] (26.75,-2)--(26.75,-1);    
\draw[thick] (27.75,-2)--(27.75,-1);  

\draw[fill=white] (26.5,-2.25) rectangle (27,-1.75);  

% 6th constant

\draw[thick] (28.50,-2)--(29.5,-2);    
\draw[thick] (28.50,-1)--(29.5,-1);    
\draw[thick] (28.50,-2)--(28.5,-1);    
\draw[thick] (29.50,-2)--(29.50,-1);  

\draw[fill=white] (28.25,-1.25) rectangle (28.75,-0.75); 

% 7th constant

\draw[thick] (30.25,-2)--(31.25,-2);    
\draw[thick] (30.25,-1)--(31.25,-1);    
\draw[thick] (30.25,-2)--(30.25,-1);    
\draw[thick] (31.25,-2)--(31.25,-1);  

\draw[fill=white] (31,-1.25) rectangle (31.5,-0.75);  

% 8th constant

\draw[thick] (32,-2)--(33,-2);    
\draw[thick] (32,-1)--(33,-1);    
\draw[thick] (32,-2)--(32,-1);    
\draw[thick] (33,-2)--(33,-1);  

\draw[fill=white] (32.75,-2.25) rectangle (33.25,-1.75);  
    
% Connectors from examples to constants

\draw[thin] (21,-5.7)--(20.25,-2.3);  
\draw[thin] (21,-5.7)--(22,-2.3);  
\draw[thin] (21,-5.7)--(30.75,-2.3);  
\draw[thin] (21,-5.7)--(32.5,-2.3); 

\draw[thin] (24,-5.7)--(23.75,-2.3); 
\draw[thin] (24,-5.7)--(25.5,-2.3); 
\draw[thin] (24,-5.7)--(27.25,-2.3); 
\draw[thin] (24,-5.7)--(29,-2.3); 

\draw[thin] (27,-5.7)--(22,-2.3); 
\draw[thin] (27,-5.7)--(25.5,-2.3); 
\draw[thin] (27,-5.7)--(27.25,-2.3); 
\draw[thin] (27,-5.7)--(30.75,-2.3); 

\draw[thin] (30,-5.7)--(23.75,-2.3); 
\draw[thin] (30,-5.7)--(27.25,-2.3); 
\draw[thin] (30,-5.7)--(29,-2.3); 
\draw[thin] (30,-5.7)--(32.5,-2.3); 

\draw[thin] (33,-5.7)--(25.5,-2.3); 
\draw[thin] (33,-5.7)--(27.25,-2.3); 
\draw[thin] (33,-5.7)--(29,-2.3); 
\draw[thin] (33,-5.7)--(30.75,-2.3); 

% Vertical constant and atoms

%\node[text width=0.1cm] at (16.7,2) 
%    {$[\phi]$}; 
    
%\node[text width=0.1cm] at (17.8,2) 
%    {$[\phi_{1}]$};
    
%\node[text width=0.1cm] at (19.60,2) 
%    {$[\epsilon_1]$};

\node[text width=0.1cm] at (16.7,2) 
    {$[\beta]$};
    
\node[text width=0.1cm] at (23.1,2) 
    {$[\phi_{3}]$};
    
\node[text width=0.1cm] at (19.6,2) 
    {$[\phi_{2}]$};
    
\node[text width=0.1cm] at (16.7,-1.5) 
    {$[v]$};

\node[text width=0.1cm] at (21,-11.5) 
    {$0^{*}$};

\node[text width=0.1cm] at (26.6,2) 
    {$[0]$};

\node[text width=0.1cm] at (27,-11.5) 
   {$\zeta_1$};

\node[text width=0.1cm] at (30,-11.5) 
    {$\zeta_2$};
    
\node[text width=0.1cm] at (33,-11.5) 
   {$\zeta_3$};    

%\node[text width=0.1cm] at (-3,-10) 
%    {$\Large{\alpha}$};

%\node[text width=0.1cm] at (0.25,-10) 
%    {$\Large{\beta}$};

%\node[text width=0.1cm] at (2,-10) 
%    {$\Large{\gamma}$};
    
%\node[text width=0.1cm] at (3.75,-10) 
%    {$\Large{\delta}$};

%\node[text width=0.1cm] at (5.5,-10) 
%    {$\Large{\epsilon}$};
    
%\node[text width=0.1cm] at (0.25,-10) 
%    {$\LARGE{\alpha}$};

%\node[text width=0.1cm] at (6.4,-10) 
%   {$\LARGE{0}$};
    
% Connectors to atoms

%  From [phi] to v
%\draw[thin] (17.2,1.3)--(17.2,-0.7); 

% From [Beta] to c1

\draw[thin] (17.2,1.3)--(20.25,-0.7); 

%  From [phi12] to constants

\draw[thin] (20.25,1.1)--(17.2,-0.7); 
\draw[thin] (20.25,1.1)--(20.25,-0.7); 
\draw[thin] (20.25,1.1)--(23.75,-0.7); 

%  From [phi13] to constants

\draw[thin] (23.75,1.1)--(17.2,-0.7); 
\draw[thin] (23.75,1.1)--(20.25,-0.7); 
\draw[thin] (23.75,1.1)--(25.5,-0.7); 

%  From [epsilon] to constants
%\draw[thin] (20.25,1.3)--(20.25,-0.7); 
%\draw[thin] (23.75,1.3)--(23.75,-0.7); 
%\draw[thin] (25.5,1.3)--(25.5,-0.7); 
% 0*-Examples

\draw[thin] (21,-10.9)--(21,-8.3); 
\draw[thin] (21,-10.9)--(33,-8.3); 
\draw[thin] (21,-10.9)--(24,-8.3); 
\draw[thin] (21,-10.9)--(27,-8.3); 
\draw[thin] (21,-10.9)--(30,-8.3); 
%\draw[thin] (1,-13.9)--(13,-11.3); 
% First two examples -v

\draw[thin] (21,-5.7)--(17.2,-2.3); 
\draw[thin] (24,-5.7)--(17.2,-2.3); 

% eta-Examples
\draw[thin] (27,-10.9)--(27,-8.3); 
\draw[thin] (30,-10.9)--(30,-8.3); 
\draw[thin] (33,-10.9)--(33,-8.3); 

% [0] to constants

\draw[thin] (27,1.1)--(17.2,-0.7);
\draw[thin] (27,1.1)--(32.5,-0.7);
\draw[thin] (27,1.1)--(30.75,-0.7);
\draw[thin] (27,1.1)--(29,-0.7);
\draw[thin] (27,1.1)--(27.25,-0.7);
\draw[thin] (27,1.1)--(25.5,-0.7);
\draw[thin] (27,1.1)--(23.75,-0.7);
\draw[thin] (27,1.1)--(22,-0.7);
\draw[thin] (27,1.1)--(20.25,-0.7);
% alpha

%\draw[thin] (-2.8,-5.3)--(-2.8,-9.6); 
%\draw[thin] (0.25,-5.3)--(-2.8,-9.6); 
%\draw[thin] (5.5,-5.3)--(-2.8,-9.6); 

% beta

%\draw[thin] (-2.8,-5.3)--(0.25,-9.6); 
%\draw[thin] (0.25,-5.3)--(0.25,-9.6); 
%\draw[thin] (5.5,-5.3)--(0.25,-9.6); 

% gamma

%\draw[thin] (0.25,-5.3)--(2,-9.6); 

% delta

%\draw[thin] (3.75,-5.3)--(3.75,-9.6); 

% delta

%\draw[thin] (5.5,-5.3)--(5.5,-9.6); 

%\draw[thin] (0.25,-5.3)--(0.25,-9.6); 
%\draw[thin] (0.25,-5.3)--(6.4,-9.6); 
%\draw[thin] (12.5,-5.3)--(6.4,-9.6); 

\end{tikzpicture}
\end{center}
Revising the first constant, we then have that its trace depends on the trace of the three atoms and atom $0$ as $\textbf{Tr}(c_{1})=\textbf{Tr}(0) \cap \textbf{Tr}(\phi_{2}) \cap \textbf{Tr}(\phi_{3}) \cap \textbf{Tr}(\beta)$, with $\textbf{Tr}(0)=\{0^{*},\zeta_{1},\zeta_{2},\zeta_{3}\}$, 
$\textbf{Tr}(\phi_{2})=\{0^{*},\zeta_{2}\}$, 
$\textbf{Tr}(\phi_{3})=\{0^{*},\zeta_{1},\zeta_{3}\}$ and $\textbf{Tr}(\beta)=\{0^{*}\}$. The constant $c_1$ would remain trace-invariant if we eliminated atoms $\phi_{2}$ and $\phi_{3}$ or only $\beta$. But, as for the invariance of constants $c_3$ and $c_4$ we need atoms $\phi_{2}$ and $\phi_{3}$, it is $\beta$ the atom we can eliminate. The stochastic \textbf{Algorithm \ref{traceReduction}} in \textbf{Appendix \ref{algorithms}} typically would delete atom $\beta$ in a single call or within a few calls. 

There are other reduction schemes. For example, a size reduction scheme based on keeping just enough atoms to discriminate the set $R^{-}$ ensures an atomization with a size under that of set $R^{-}$. The problem with this reduction scheme is that it fails to produce good generalizing models as it seems to reduce algebraic freedom (see {\bf{Section \ref{memorizingGeneralizing}}}) more than one would wish, especially at the initial phases of learning when the error is large and the algebra should grow rather than shrink. In addition, since this scheme can violate negative trace constraints, it can only be used once the full embedding is completed. If used at some intermediate stage of the embedding, subsequent Sparse Crossing operations may produce models that do not satisfy $R^{-}$.  This is a problem because the model can become very large before we can reduce its size. However, this scheme can be successfully applied to the dual, $M^{*}$, right before trace constraints are enforced, ensuring that the number of atoms in $M^{*}$ is never larger than the size of the set $R^{-}$. This reduction scheme corresponds to {\bf{Algorithm \ref{dualReduction}}} in \textbf{Appendix \ref{algorithms}}.

The trace-preserving reduction scheme presented in this section works well in combination with Sparse Crossing and finds small, generalizing models efficiently. For this reduction scheme there is no guarantee that the size of $M$ is going to end up under the size of $M^{*}$ or under the size of $R^{-}$. In fact, it is often the case that the atomization of $M$ is a few times larger than the atomization of $M^{*}$. Even when is smaller then $M$, $M^{*}$ does not generalize because it is not sufficiently free.

\subsection{Batch training} \label{batchTraining}

We have seen how to learn an atomized model from a set $R$ of positive and negative examples. In practice, we would check the accuracy of the learned model in test data. If the accuracy is below some desired level, we would continue training with a new set of examples. Rather than a single set $R$ we have a series of batches $R_{0}, R_{1},...,R_{n}$ where the subscript corresponds with the training epoch.

Assume we are in epoch $1$. In order to keep the graph $G(S)$ manageable we want to remove from it the nodes corresponding to elements mentioned in $R_{0}$, as well as deleting the set $R_{0}$ itself to leave space for the new set of relations $R_{1}$. However, if we delete $R_{0}$ we run into the following problem; Often when we resume learning by embedding $R_{1}$ some relations of $R_{0}$ no longer hold. We need a method to minimize the likelihood for this to occur.

In the following we discuss how to do batch learning. Once learning epoch $0$ is completed we have $R_{0}$ encoded into atoms. What we do is replacing $R_{0}$ by a set of relations that $\emph{define}$ its atoms in the following way; For each atom $\phi$, we create one term $T_{\phi}$ equal to the idempotent summation of all the constants that do not contain atom $\phi$. For each constant $c$ so $(\phi < c)$, we create a new relation $\neg(c < T_{\phi})$. We call this set of terms and relations \enquote{the  pinning structure} of the algebra because they help to preserve knowledge when $R_{0}$ is deleted and additional learning takes place. We call terms $T_{\phi}$ \enquote{pinning terms} and the set of relations, \enquote{pinning relations} and refer to them with $R_{p}$. 

Following this procedure in our toy example, we would create two pinning terms. For atom $\phi_{2}$, with $\phi_{2} < c_{1}$ and $\phi_{2} < c_{3}$, we create $T_{\phi_{2}}=c_{2} \odot c_{4} \odot c_{5} \odot c_{6} \odot c_{7} \odot c_{8}$. For atom $\phi_{3}$, with $\phi_{3} < c_{1}$ and $\phi_{2} < c_{4}$, we create $T_{\phi_{3}}=c_{2} \odot c_{3} \odot c_{5} \odot c_{6} \odot c_{7} \odot c_{8}$. 

We then require the negative relations $\neg(c_{1}<T_{\phi_{2}})$, $\neg(c_{3}<T_{\phi_{3}})$, $\neg(v<T_{\phi_{2}})$, $\neg(c_{1}<T_{\phi_{3}})$, $\neg(c_{4}<T_{\phi_{3}})$ and $\neg(v<T_{\phi_{3}})$, and obtain
\begin{center}
\begin{tikzpicture}[scale=0.35]

% Positive example

\draw[thin] (0,0)--(2,0);    
\draw[thick] (1,0)--(1,2);    
\draw[thick] (2,0)--(2,2);    
\draw[thick] (0,0)--(0,2);  
\draw[thick] (0,2)--(2,2);
\draw[thick] (0,1)--(2,1);

\draw[fill=black] (0,0) rectangle (1,2);  
\draw[fill=black] (1,1) rectangle (2,2);  

\draw[thick] (2.5,-0.5)--(2.5,2.5);

% Negative example

\draw[thin] (3,0)--(5,0);    
\draw[thin] (4,0)--(4,2);    
\draw[thin] (5,0)--(5,2);    
\draw[thin] (3,0)--(3,2);  
\draw[thin] (3,2)--(5,2);
\draw[thin] (3,1)--(5,1);

\draw[fill=black] (3,1) rectangle (5,2);

\draw[thin] (2.5,-0.5)--(2.5,2.5);

% Now the constants

% First constant

\draw[thin] (-0.25,-5)--(0.75,-5);    
\draw[thin] (-0.25,-4)--(0.75,-4);    
\draw[thin] (-0.25,-5)--(-0.25,-4);    
\draw[thin] (0.75,-5)--(0.75,-4);  

\draw[fill=black] (-0.50,-5.25) rectangle (0,-4.75);  

% Second constant

\draw[thin] (1.5,-5)--(2.5,-5);    
\draw[thin] (1.5,-4)--(2.5,-4);    
\draw[thin] (1.5,-5)--(1.5,-4);    
\draw[thin] (2.5,-5)--(2.5,-4);  

\draw[fill=black] (1.25,-4.25) rectangle (1.75,-3.75); 

% Third constant

\draw[thin] (3.25,-5)--(4.25,-5);    
\draw[thin] (3.25,-4)--(4.25,-4);    
\draw[thin] (3.25,-5)--(3.25,-4);    
\draw[thin] (4.25,-5)--(4.25,-4);  

\draw[fill=black] (4,-4.25) rectangle (4.5,-3.75);  

% Fourth constant

\draw[thin] (5,-5)--(6,-5);    
\draw[thin] (5,-4)--(6,-4);    
\draw[thin] (5,-5)--(5,-4);    
\draw[thin] (6,-5)--(6,-4);  

\draw[fill=black] (5.75,-5.25) rectangle (6.25,-4.75);  

% 5th constant (add 6.25 in x coordinate, and black goes to white)

\draw[thin] (6.75,-5)--(7.75,-5);    
\draw[thin] (6.75,-4)--(7.75,-4);    
\draw[thin] (6.75,-5)--(6.75,-4);    
\draw[thin] (7.75,-5)--(7.75,-4);  

\draw[fill=white] (6.5,-5.25) rectangle (7,-4.75);  

% 6th constant

\draw[thin] (8.50,-5)--(9.5,-5);    
\draw[thin] (8.50,-4)--(9.5,-4);    
\draw[thin] (8.50,-5)--(8.5,-4);    
\draw[thin] (9.50,-5)--(9.50,-4);  

\draw[fill=white] (8.25,-4.25) rectangle (8.75,-3.75); 

% 7th constant

\draw[thin] (10.25,-5)--(11.25,-5);    
\draw[thin] (10.25,-4)--(11.25,-4);    
\draw[thin] (10.25,-5)--(10.25,-4);    
\draw[thin] (11.25,-5)--(11.25,-4);  

\draw[fill=white] (11,-4.25) rectangle (11.5,-3.75);  

% 8th constant

\draw[thin] (12,-5)--(13,-5);    
\draw[thin] (12,-4)--(13,-4);    
\draw[thin] (12,-5)--(12,-4);    
\draw[thin] (13,-5)--(13,-4);  

\draw[fill=white] (12.75,-5.25) rectangle (13.25,-4.75);  
    
% Connectors from examples to constants

\draw[thin] (1,-0.3)--(0.25,-3.7);  
\draw[thin] (1,-0.3)--(2,-3.7);  
%\draw[thin] (1,-0.3)--(10.75,-3.7);  
\draw[thin] (1,-0.3)--(12.5,-3.7);  
\draw[thin] (1,-0.3)--(3.75,-3.7); 

%\draw[thin] (4,-0.3)--(5.5,-3.7); 
\draw[thin] (4,-0.3)--(2,-3.7); 
\draw[thin] (4,-0.3)--(3.75,-3.7); 
\draw[thin] (4,-0.3)--(7.25,-3.7); 
\draw[thin] (4,-0.3)--(12.5,-3.7); 
%\draw[thin] (4,-0.3)--(9,-3.7); 

\draw[thin] (7,-0.3)--(2,-3.7); 
\draw[thin] (7,-0.3)--(5.5,-3.7); 
\draw[thin] (7,-0.3)--(7.25,-3.7); 
\draw[thin] (7,-0.3)--(9,-3.7); 
\draw[thin] (7,-0.3)--(10.75,-3.7); 
\draw[thin] (7,-0.3)--(12.5,-3.7); 

%\draw[thin] (10,-0.3)--(2,-3.7); 
%\draw[thin] (10,-0.3)--(3.75,-3.7); 
%\draw[thin] (10,-0.3)--(7.25,-3.7); 
%\draw[thin] (10,-0.3)--(9,-3.7); 
%\draw[thin] (10,-0.3)--(10.75,-3.7); 
%\draw[thin] (10,-0.3)--(12.5,-3.7);

\draw[thin] (10,3.3)--(4,2.5); 
%\draw[thin] (10,-0.3)--(3.75,-3.7); 
%\draw[thin] (10,-0.3)--(7.25,-3.7); 
\draw[thin] (10,3.3)--(9,-3.7); 
\draw[thin] (10,3.3)--(10.75,-3.7); 
%\draw[thin] (10,-0.3)--(12.5,-3.7);

% Vertical constant and atoms

\node[text width=0.1cm] at (6.5,1) 
    {{$T_{\phi_{2}}$}};

%\node[text width=0.1cm] at (9.5,1) 
 %   {{$T_{\phi_{3}}$}};

\node[text width=0.1cm] at (9.5,4) 
    {{$T_{\phi_{3}}$}};
 
\node[text width=0.1cm] at (-3,-4.5) 
    {{$v$}};

\node[text width=0.1cm] at (7,-7.9) 
    {$0$};

%\node[text width=0.1cm] at (-3,-8) 
%    {{$\phi$}};

%\node[text width=0.1cm] at (0.25,-8) 
%    {{$\epsilon_1$}};

%\node[text width=0.1cm] at (3.75,-8) 
%    {{$\epsilon_2$}};

%\node[text width=0.1cm] at (5.5,-8) 
%    {{$\epsilon_3$}};
    
%\node[text width=0.1cm] at (-1.8,-8) 
%    {{$\phi_{1}$}};

\node[text width=0.1cm] at (0.25,-7.9) 
    {{$\phi_{2}$}};

\node[text width=0.1cm] at (3.75,-7.9) 
    {{$\phi_{3}$}};
    
% Connectors to atoms

% phi-v

%\draw[thin] (-3,-5.3)--(-3,-7.6);

% phi12-constants

\draw[thin] (0.25,-7.3)--(-3,-5.3);
\draw[thin] (0.25,-7.3)--(0.25,-5.3);
\draw[thin] (0.25,-7.3)--(3.75,-5.3);

% phi13-constants

\draw[thin] (3.75,-7.3)--(-3,-5.3);
\draw[thin] (3.75,-7.3)--(0.25,-5.3);
\draw[thin] (3.75,-7.3)--(5.5,-5.3);

%  epsilon1-constant
%\draw[thin] (0.25,-5.3)--(0.25,-7.6);
%  epsilon2-constant
%\draw[thin] (3.75,-5.3)--(3.75,-7.6);
%  epsilon3-constants
%\draw[thin] (5.5,-5.3)--(5.5,-7.6);

% 0-v

\draw[thin] (7,-7.3)--(-3,-5.3);

% 0-constant 8

\draw[thin] (7,-7.3)--(12.5,-5.3);
\draw[thin] (7,-7.3)--(10.75,-5.3); 
\draw[thin] (7,-7.3)--(9,-5.3); 
\draw[thin] (7,-7.3)--(7.25,-5.3); 
\draw[thin] (7,-7.3)--(5.5,-5.3); 
\draw[thin] (7,-7.3)--(3.75,-5.3); 
\draw[thin] (7,-7.3)--(2,-5.3); 
\draw[thin] (7,-7.3)--(0.25,-5.3); 

% alpha

%\draw[thin] (-2.8,-5.3)--(-2.8,-9.6); 
%\draw[thin] (0.25,-5.3)--(-2.8,-9.6); 
%\draw[thin] (5.5,-5.3)--(-2.8,-9.6); 

% beta

%\draw[thin] (-2.8,-5.3)--(0.25,-9.6); 
%\draw[thin] (0.25,-5.3)--(0.25,-9.6); 
%\draw[thin] (3.75,-5.3)--(0.25,-9.6); 

% gamma

%\draw[thin] (0.25,-5.3)--(2,-9.6); 

% delta

%\draw[thin] (3.75,-5.3)--(3.75,-9.6); 

% delta

%\draw[thin] (5.5,-5.3)--(5.5,-9.6); 

%\draw[thin] (0.25,-5.3)--(0.25,-9.6); 
%\draw[thin] (0.25,-5.3)--(6.4,-9.6); 
%\draw[thin] (12.5,-5.3)--(6.4,-9.6); 

%%%%
%%%% HERE DUAL
%%%%

% Positive example

\draw[thin] (20,-8)--(22,-8);    
\draw[thin] (21,-8)--(21,-6);    
\draw[thin] (22,-8)--(22,-6);    
\draw[thin] (20,-8)--(20,-6);  
\draw[thin] (20,-6)--(22,-6);
\draw[thin] (20,-7)--(22,-7);

%\draw[fill=black] (-0.5,1.5) rectangle (0.5,2.5);  
\draw[fill=black] (20,-8) rectangle (21,-6);  
\draw[fill=black] (21,-7) rectangle (22,-6); 
%\draw[fill=black] (1,0) rectangle (2,1);  

% Separator between positive and negative examples

\draw[thin] (22.5,-7.3)--(22.5,-5.5);

% Negative example

\draw[thin] (23,-8)--(25,-8);    
\draw[thin] (24,-8)--(24,-6);    
\draw[thin] (25,-8)--(25,-6);    
\draw[thin] (23,-8)--(23,-6);  
\draw[thin] (23,-6)--(25,-6);
\draw[thin] (23,-7)--(25,-7);

\draw[fill=black] (23,-6) rectangle (25,-7);  

% Now the constants

% First constant

\draw[thin] (19.75,-2)--(20.75,-2);    
\draw[thin] (19.75,-1)--(20.75,-1);    
\draw[thin] (19.75,-2)--(19.75,-1);    
\draw[thin] (20.75,-2)--(20.75,-1);

\draw[fill=black] (19.5,-2.25) rectangle (20,-1.75);  

% Second constant

\draw[thin] (21.5,-2)--(22.5,-2);    
\draw[thin] (21.5,-1)--(22.5,-1);    
\draw[thin] (21.5,-2)--(21.5,-1);    
\draw[thin] (22.5,-2)--(22.5,-1);  

\draw[fill=black] (21.25,-1.25) rectangle (21.75,-0.75); 

% Third constant

\draw[thin] (23.25,-2)--(24.25,-2);    
\draw[thin] (23.25,-1)--(24.25,-1);    
\draw[thin] (23.25,-2)--(23.25,-1);    
\draw[thin] (24.25,-2)--(24.25,-1);  

\draw[fill=black] (24,-1.25) rectangle (24.5,-0.75);  

% Fourth constant

\draw[thin] (25,-2)--(26,-2);    
\draw[thin] (25,-1)--(26,-1);    
\draw[thin] (25,-2)--(25,-1);    
\draw[thin] (26,-2)--(26,-1);  

\draw[fill=black] (25.75,-2.25) rectangle (26.25,-1.75);  

% 5th constant (add 6.25 in x coordinate, and black goes to white)

\draw[thin] (26.75,-2)--(27.75,-2);    
\draw[thin] (26.75,-1)--(27.75,-1);    
\draw[thin] (26.75,-2)--(26.75,-1);    
\draw[thin] (27.75,-2)--(27.75,-1);  

\draw[fill=white] (26.5,-2.25) rectangle (27,-1.75);  

% 6th constant

\draw[thin] (28.50,-2)--(29.5,-2);    
\draw[thin] (28.50,-1)--(29.5,-1);    
\draw[thin] (28.50,-2)--(28.5,-1);    
\draw[thin] (29.50,-2)--(29.50,-1);  

\draw[fill=white] (28.25,-1.25) rectangle (28.75,-0.75); 

% 7th constant

\draw[thin] (30.25,-2)--(31.25,-2);    
\draw[thin] (30.25,-1)--(31.25,-1);    
\draw[thin] (30.25,-2)--(30.25,-1);    
\draw[thin] (31.25,-2)--(31.25,-1);  

\draw[fill=white] (31,-1.25) rectangle (31.5,-0.75);  

% 8th constant

\draw[thin] (32,-2)--(33,-2);    
\draw[thin] (32,-1)--(33,-1);    
\draw[thin] (32,-2)--(32,-1);    
\draw[thin] (33,-2)--(33,-1);  

\draw[fill=white] (32.75,-2.25) rectangle (33.25,-1.75);  
    
% Connectors from examples to constants

\draw[thin] (21,-5.7)--(20.25,-2.3);  
\draw[thin] (21,-5.7)--(22,-2.3);  
\draw[thin] (21,-5.7)--(23.75,-2.3);  
\draw[thin] (21,-5.7)--(32.5,-2.3); 

\draw[thin] (24,-5.7)--(23.75,-2.3); 
\draw[thin] (24,-5.7)--(22,-2.3); 
\draw[thin] (24,-5.7)--(32.5,-2.3); 
\draw[thin] (24,-5.7)--(27.25,-2.3); 

\draw[thin] (27,-5.7)--(22,-2.3); 
\draw[thin] (27,-5.7)--(25.5,-2.3); 
\draw[thin] (27,-5.7)--(27.25,-2.3); 
\draw[thin] (27,-5.7)--(29,-2.3); 
\draw[thin] (27,-5.7)--(30.75,-2.3); 
\draw[thin] (27,-5.7)--(32.5,-2.3); 

%\draw[thin] (27,-5.7)--(22,-2.3); 
%\draw[thin] (30,-5.7)--(23.75,-2.3); 
%\draw[thin] (30,-5.7)--(27.25,-2.3); 
%\draw[thin] (30,-5.7)--(29,-2.3); 
%\draw[thin] (30,-5.7)--(30.75,-2.3); 
%\draw[thin] (30,-5.7)--(32.5,-2.3); 

\draw[thin] (29.75,-10)--(24,-8.3); 
%\draw[thin] (29.75,-5.7)--(27.25,-2.3); 
\draw[thin] (29.75,-10)--(29,-2.3); 
\draw[thin] (29.75,-10)--(30.75,-2.3); 
%\draw[thin] (30,-5.7)--(32.5,-2.3); 

% Vertical constant and atoms

%\node[text width=0.1cm] at (16.7,2) 
%    {$[\phi]$}; 
    
%\node[text width=0.1cm] at (17.8,2) 
%    {$[\phi_{1}]$};
    
%\node[text width=0.1cm] at (19.60,2) 
%    {$[\epsilon_1]$};

\node[text width=0.1cm] at (26.5,-7) 
    {$T_{\phi_{2}}$};
    
%\node[text width=0.1cm] at (29.5,-7) 
   % {$T_{\phi_{3}}$};
    
\node[text width=0.1cm] at (29.5,-11) 
    {$T_{\phi_{3}}$};
    
\node[text width=0.1cm] at (19.6,2) 
    {$[\phi_{2}]$};

\node[text width=0.1cm] at (23.1,2) 
    {$[\phi_{3}]$};
    
\node[text width=0.1cm] at (16.7,-1.5) 
    {$[v]$};

%\node[text width=0.1cm] at (21,-11.5) 
%    {$0^{*}$};

\node[text width=0.1cm] at (21,-13) 
    {$0^{*}$};
    
\node[text width=0.1cm] at (26.6,2) 
    {$[0]$};

%\node[text width=0.1cm] at (27,-11.5) 
%   {$\zeta_2$};

%\node[text width=0.1cm] at (30,-11.5) 
%    {$\zeta_3$};
    
%\node[text width=0.1cm] at (24,-11.5) 
%   {$\zeta_1$};    

\node[text width=0.1cm] at (27,-13) 
   {$\zeta_1$};

\node[text width=0.1cm] at (29.7,-13) 
    {$\zeta_2$};
    
%\node[text width=0.1cm] at (24,-11.5) 
%   {$\zeta_1$};  

% Connectors to atoms

%  From [phi] to v
%\draw[thin] (17.2,1.3)--(17.2,-0.7); 

%  From [phi12] to constants

\draw[thin] (20.25,1.1)--(17.2,-0.7); 
\draw[thin] (20.25,1.1)--(20.25,-0.7); 
\draw[thin] (20.25,1.1)--(23.75,-0.7); 

%  From [phi13] to constants

\draw[thin] (23.75,1.1)--(17.2,-0.7); 
\draw[thin] (23.75,1.1)--(20.25,-0.7); 
\draw[thin] (23.75,1.1)--(25.5,-0.7); 

%  From [epsilon]to constants
%\draw[thin] (20.25,1.3)--(20.25,-0.7); 
%\draw[thin] (23.75,1.3)--(23.75,-0.7); 
%\draw[thin] (25.5,1.3)--(25.5,-0.7); 

% 0*-Examples

\draw[thin] (21,-12.3)--(21,-8.3); 
%\draw[thin] (21,-10.9)--(33,-8.3); 
\draw[thin] (21,-12.3)--(24,-8.3); 
\draw[thin] (21,-12.3)--(27,-8.3); 
\draw[thin] (21,-12.3)--(29.75,-10); 

%\draw[thin] (21,-10.9)--(21,-8.3); 
%\draw[thin] (21,-10.9)--(33,-8.3); 
%\draw[thin] (21,-10.9)--(24,-8.3); 
%\draw[thin] (21,-10.9)--(27,-8.3); 
%\draw[thin] (21,-10.9)--(30,-8.3);
%\draw[thin] (1,-13.9)--(13,-11.3); 
% First two examples -v

\draw[thin] (21,-5.7)--(17.2,-2.3); 
%\draw[thin] (24,-5.7)--(17.2,-2.3); 

% eta-Examples

\draw[thin] (27,-12.3)--(27,-8.3); 
\draw[thin] (29.75,-12.3)--(29.75,-11.5); 
%\draw[thin] (24,-10.9)--(24,-8.3); 
%\draw[thin] (27,-10.9)--(27,-8.3); 
%\draw[thin] (29.75,-12)--(29.75,-10);

%\draw[thin] (33,-10.9)--(33,-8.3); 

% [0] to constants

\draw[thin] (27,1.1)--(17.2,-0.7);
\draw[thin] (27,1.1)--(32.5,-0.7);
\draw[thin] (27,1.1)--(30.75,-0.7);
\draw[thin] (27,1.1)--(29,-0.7);
\draw[thin] (27,1.1)--(27.25,-0.7);
\draw[thin] (27,1.1)--(25.5,-0.7);
\draw[thin] (27,1.1)--(23.75,-0.7);
\draw[thin] (27,1.1)--(22,-0.7);
\draw[thin] (27,1.1)--(20.25,-0.7);
% alpha

%\draw[thin] (-2.8,-5.3)--(-2.8,-9.6); 
%\draw[thin] (0.25,-5.3)--(-2.8,-9.6); 
%\draw[thin] (5.5,-5.3)--(-2.8,-9.6); 

% beta

%\draw[thin] (-2.8,-5.3)--(0.25,-9.6); 
%\draw[thin] (0.25,-5.3)--(0.25,-9.6); 
%\draw[thin] (5.5,-5.3)--(0.25,-9.6); 

% gamma

%\draw[thin] (0.25,-5.3)--(2,-9.6); 

% delta

%\draw[thin] (3.75,-5.3)--(3.75,-9.6); 

% delta

%\draw[thin] (5.5,-5.3)--(5.5,-9.6); 

%\draw[thin] (0.25,-5.3)--(0.25,-9.6); 
%\draw[thin] (0.25,-5.3)--(6.4,-9.6); 
%\draw[thin] (12.5,-5.3)--(6.4,-9.6); 
\end{tikzpicture}
\end{center} 
New pinning terms and relations are formed at the end of each learning epoch. We do not need to replace pinning relations of epochs $0,1,...,n-1$ with the pinning relations derived from the atoms of epoch $n$, instead we can let them accumulate, epoch after epoch. Pinning relations do not grow until becoming unmanageable. One of the reasons why this occurs is that pinning terms and relations do not only get created, they also get discarded. Discarding pinning relations is needed because the set $R_n \cup R_p$ may be inconsistent. We use $R_p$ to refer to all the accumulated pinning relations. 

Inconsistencies are detected as explained in {\bf{Section \ref{traceConstraints}}}. When we try to enforce the set $R_n \cup R_p$ in $M^{*}$ often we find that some negative relations cannot be enforced. We enforce the relations in $M^{*}$ by adding atoms in the dual to discriminate all the (reversed) negative relations, i.e. the relations in the set $R_{n}^{-} \cup R_p$. Once we calculate the transitive closure of the graph of $M^{*}$ we may find that some negative relations do not hold. It doesn't matter how many times we try or how we choose to introduce the discriminating atoms in $M^{*}$ the resulting relations that do not hold are always the same and are always negative. If a relation that belongs to $R_{n}^{-}$ does not hold, then the set $R_{n}$ is inconsistent. In this case, something is wrong with our training set. If the relation that fails belongs to $R_p$ we just discard it deleting its pinning term and associated pinning relations. We can regard $R_{p}$ as a set of hypotheses; some hypotheses are eventually found inconsistent with new data. 

Once inconsistent pinning relations have been discarded and we have a consistent set $R_n \cup R_p$ we have a couple of possible strategies. One is enforcing $R_n \cup R_p$ in epoch $n$. The other strategy, the one used in all the experiments of this paper, consists of creating a new set of atoms for $M^{*}$ enforcing only the relations of $R_n$ but with the pinning terms of $R_p$ present in the graph of $M^{*}$. Then we enforce trace constraints for $R_n$ and also for $R_p$ but only for the relations of $R_p$ that happen to hold in $M^{*}$. In different epochs different relations hold and different pinning terms are used. We found this strategy more efficient and computationally lighter than the first.

Pinning relations are all negative and there are good reasons for this. One reason has to do with maximizing algebraic freedom and it is discussed in section {\bf{Sections \ref{memorizingGeneralizing}}} and {\bf{Theorem \ref{pinningFree}}}, the second reason is that inconsistent negative relations can be individually detected while inconsistent positive relations cannot be detected so easily. In {\bf{Theorem \ref{positiveEntail}}} of {\bf{Appendix \ref{theorems}}} we prove that any positive relation that is entailed by a set of positive and negative relations is also entailed by the positive relations of the set alone. Negative relations do not have positive relations as logical consequences; their consequences are all negative, so by restricting ourselves to negative pinning relations we can detect and isolate inconsistences introduced by the pinning relations as they arise. 

Introducing pinning relations to deal with batch learning has other important advantage. Pinning relations found after embedding a batch of order relations $R_{n+1}$ tend to be quite similar to those obtained for the previous batch $R_{n}$, more so the more the algebra has already learned. This means that the number of pinning terms and relations tend to converge to a fixed number with training or grow very slowly. This approach is then clearly superior to one combining all training sets together. 

Pinning terms and relations accumulate the knowledge of the training and can be shared with other algebras.

\section{Analysis of solutions}

\subsection{Finding a class definition for the toy problem}

So far we have obtained, with 2 positive and 3 negative examples, a model with two atoms. With a few more training examples, two more atoms are obtained. In terms of the constants they are edged to, we can plot these four atoms as
\begin{center}
\begin{tikzpicture}[scale=0.5]
\label{picture:atoms22}
% First atom

\draw[thick] (-0.25,-5)--(0.75,-5);    
\draw[thick] (-0.25,-4)--(0.75,-4);    
\draw[thick] (-0.25,-5)--(-0.25,-4);    
\draw[thick] (0.75,-5)--(0.75,-4);  

\draw[fill=black] (-0.50,-5.25) rectangle (0,-4.75);  
\draw[fill=black] (0.50,-5.25) rectangle (1,-4.75);  

% Second atom

\draw[thick] (1.5,-5)--(2.5,-5);    
\draw[thick] (1.5,-4)--(2.5,-4);    
\draw[thick] (1.5,-5)--(1.5,-4);    
\draw[thick] (2.5,-5)--(2.5,-4);  

\draw[fill=black] (2.25,-4.25) rectangle (2.75,-3.75); 
\draw[fill=black] (1.25,-5.25) rectangle (1.75,-4.75); 

% Third atom

\draw[thick] (3.25,-5)--(4.25,-5);    
\draw[thick] (3.25,-4)--(4.25,-4);    
\draw[thick] (3.25,-5)--(3.25,-4);    
\draw[thick] (4.25,-5)--(4.25,-4);  

\draw[fill=black] (4,-4.25) rectangle (4.5,-3.75);  
\draw[fill=black] (3,-4.25) rectangle (3.5,-3.75);  

% Fourth atom

\draw[thick] (5,-5)--(6,-5);    
\draw[thick] (5,-4)--(6,-4);    
\draw[thick] (5,-5)--(5,-4);    
\draw[thick] (6,-5)--(6,-4);  

\draw[fill=black] (4.75,-4.25) rectangle (5.25,-3.75);  
\draw[fill=black] (5.75,-5.25) rectangle (6.25,-4.75);  

\end{tikzpicture} 
\end{center}
where the first two atoms are atoms $\phi_{3}$ and $\phi_{2}$, that we derived in previous sections. Using more examples, trace-invariant reduction and Sparse Crossing the reader can find the entire solution by hand.

The $16$ possible $2 \times 2$ images can be correctly classified into those with a black vertical bar that contain the four atoms and those without a bar that have one or more atoms missing. A $2 \times 2$ image $I$ has the property $v$ of the positive class when its term contains all four atoms. Those images that do not contain a vertical bar have terms that only contain three atoms or less. Describing the four atoms in terms of the constants, that for clarity we have renamed as $c_{ij\textbf{b}}$ for {\bf{b}}lack pixels in row $i$ and column $j$:
\begin{linenomath}
\begin{equation}
\begin{aligned}
(v < I) & \Leftrightarrow (c_{11\textbf{b}} < I \vee c_{12\textbf{b}} < I) \wedge (c_{11\textbf{b}} < I \vee c_{22\textbf{b}} < I)  \\
   &\wedge \, \, \, (c_{21\textbf{b}} < I \vee c_{12\textbf{b}} < I) \wedge (c_{21\textbf{b}} < I \vee c_{22\textbf{b}} < I),
 \label{eq:solution_bar}
\end{aligned}
\end{equation}
\end{linenomath}
we get a first order expression that defines $v$.

Using the distributive law, we can rewrite the solution in Equation (\ref{eq:solution_bar}) as
\begin{linenomath}
\begin{equation}
(v < I) \Leftrightarrow (c_{11 \textbf{b}} <I \wedge c_{21 \textbf{b}} <I) \vee (c_{12\textbf{b}} < I \wedge c_{22 \textbf{b}} < I),
\end{equation}
\end{linenomath}
or more compactly as
\begin{linenomath}
\begin{align}
\label{eq:imagewithline_small}
(v < I) \Leftrightarrow \vee_{j=1}^{2} \ \wedge_{i=1}^{2} (c_{ij \textbf{b}} < I).
\end{align}
\end{linenomath}
This last expression in Equation (\ref{eq:imagewithline_small}) says that positive examples have or column 1 or column 2 (or both) with row 1 and row 2 in black, which is the simplest definition of an image including a vertical bar. We have been able to learn and derive from examples a closed expression defining the class of images that contain a vertical bar. 

Deriving formal class definitions form examples is an exciting subject but in this paper we are mainly interested in approximate solutions that are sufficient for Machine Learning. To understand the approximate solutions, however, we are going to walk backwards from exact, closed expressions to the atoms. Specifically, we will show that for the vertical bar problem in a grid of size $M \times N$ there is an astronomically large number of suitable approximate solutions with a desired error rate, and that our stochastic learning algorithm only needs to find one.

\subsection{Analysis of exact solutions}  \label{exactSolutions}

So far we have analyzed the vertical bar problem only for $ 2 \times 2$ images. In the following we derive the exact solution for the more general case
of $M \times N$ images. We are not going to learn the solution from examples, instead we are going to derive the form of the atoms from the known concept of a vertical bar. This exact solution will be used in the next section to show that, for algebraic learning to find an approximate solution, it needs to find some valid subset of atoms from the exact solution and there are an astronomically large number of valid subsets.  

An $M \times N$ image contains a black vertical bar if there is at least one of the $N$ columns in the image for which all $M$ pixels are black. Formally, image $I$ has a vertical line if either column $1$ has all rows in black, or column $2$ or any of the $N$ columns,
\begin{linenomath}
\begin{equation}
(v < I) \Leftrightarrow \vee_{j=1}^{N} \wedge_{i=1}^{M} (c_{ij \textbf{b}} < I),
\label{eq:imagewithline0}
\end{equation}
\end{linenomath}
where the dusjunction $\vee$ runs over the columns of the image $I$ and the conjunction $\wedge$ over the rows and $c_{ij \textbf{b}}$ stands for the pixel in row $i$ and column $j$ being in black. 

In order to compare Equation (\ref{eq:imagewithline0}) with the general form of a solution learned using algebras we are going to start by expressing the partial order $a < b$ affecting any two elements $a$ and $b$ of an atomized algebra in terms of its constants rather than its atoms. We know how to determine if $a < b$ using atoms: $a < b$ if and only if the atoms of $a$ are also atoms of $b$. Just to clarify, we say an atom $\phi$ is \enquote{in $a$} or is \enquote{of $a$} or \enquote{contained in $a$} if it is in the lower segment ${\bf{L}}(a)$ or, equivalently in the graph, if it is edged to $a$, i.e. if $\phi \in {\bf{GL}}(a)$ or, also equivalently, if $\phi < a$.  

Element $a$ of an atomized algebra is lower than $b$ (or \enquote{in $b$}) if and only if all the atoms of $a$ are also in $b$,
\begin{linenomath}
\begin{equation}
\label{eq:atomic1}
(a<b) \Leftrightarrow \wedge_{\phi \,<\, a} (\phi < b),
\end{equation}
\end{linenomath}
with the conjunction running over all atoms of $a$. An atom $\phi$ is contained in $b$ if any of the constants that contain that atom, $c_{\phi k}$, is in element $b$,
\begin{linenomath}
\begin{equation}
\label{eq:atomic2}
(\phi<b) \Leftrightarrow \vee_{k} (c_{\phi k} < b),
\end{equation}
\end{linenomath}
where index $k$ runs along all constants $c_{\phi k}$ that contain atom $\phi$.

Substituting Equation (\ref{eq:atomic2}) into Equation (\ref{eq:atomic1}), we can write
\begin{linenomath}
\begin{equation}
\label{eq:lineinimage_algebra}
(a<b) \Leftrightarrow \wedge_{\phi \,<\, a}  \vee_{k} (c_{\phi k} < b),
\end{equation}
\end{linenomath}
where we have expressed $a<b$ with the constants of the algebra. 

Going back to our problem of vertical bars, we can apply to $v < I$ what we have learned and write
\begin{linenomath}
\begin{equation}
\label{eq:imagewithline_algebra}
(v<I) \Leftrightarrow \wedge_{\phi \,<\, v}  \vee_{k} (c_{\phi k} < I).
\end{equation}
\end{linenomath}
where $I$ is a term describing an image, and then compare it with our first expression, Equation (\ref{eq:imagewithline0}). This formula has a conjuction followed by a disjunction so we are going to transform Equation (\ref{eq:imagewithline0}) to conjunctive normal form (CNF) to match the form of Equation (\ref{eq:imagewithline_algebra}).

We first do the transformation for images of size $3 \times 2$ to keep it intuitive. For this size of images, the expression for a vertical bar to be in one of these images in Equation (\ref{eq:imagewithline0}) is of the form
\begin{linenomath}
\begin{equation}
\begin{aligned}
(v < I) & \Leftrightarrow \vee_{j=1}^{2} \ \wedge_{i=1}^{3} (c_{ij \textbf{b}} < I) \\
  &=(c_{11 \textbf{b}<I}\wedge c_{21 \textbf{b}<I}\wedge c_{31 \textbf{b}}<I) \vee c_{12\textbf{b}<I}\wedge c_{22 \textbf{b}<I}\wedge c_{32\textbf{b}}<I).
\end{aligned}
\end{equation}
\end{linenomath}
Using that conjuction and disjuction are distributive: 
\begin{linenomath}
\begin{equation}
\begin{aligned}
(v < I) &\Leftrightarrow \vee_{j=1}^{2}  \wedge_{i=1}^{3} (c_{ij \textbf{b}} < I) \\
  &=(c_{11\textbf{b}} \vee c_{12\textbf{b}}) \wedge (c_{11\textbf{b}} \vee c_{22\textbf{b}}) \wedge (c_{11\textbf{b}} \vee c_{32\textbf{b}}) \\
  & \vee \,
  (c_{21\textbf{b}} \vee c_{12\textbf{b}}) \wedge (c_{21\textbf{b}} \vee c_{22\textbf{b}}) \wedge (c_{21\textbf{b}} \vee c_{32\textbf{b}}) \\
  & \vee \,
  (c_{31\textbf{b}} \vee c_{12\textbf{b}}) \wedge (c_{31\textbf{b}} \vee c_{22\textbf{b}}) \wedge (c_{31\textbf{b}} \vee c_{32\textbf{b}}),
\end{aligned}
\end{equation}
\end{linenomath}
where for compactness we are not writing $ < I$ after each pixel $c_{ij \textbf{b}}$.
We now have a conjunction of $9$ terms, each term a disjuntion of two  constants. For the column $j$ the indexes are clear as they simply run over $1$ and $2$. For rows, note that each of the nine terms is one of the $9$ possible ways to assign each of the rows to a column 
\begin{linenomath}
\begin{equation}
\begin{aligned}
(v < I)  \Leftrightarrow & \vee_{j=1}^{2}  \wedge_{i=1}^{3} (c_{ij \textbf{b}} < I) \\
  = & \wedge_{\sigma}^{9} \vee_{j=1}^{2} (c_{ \sigma(j) j\textbf{b}} < I ),
\end{aligned}
\end{equation}
\end{linenomath}
where we have introduced $\sigma$, an index that runs over all possible assignations of a row to each column. To express that we are covering all possible combinations (variations, in fact) we write the symbol $j \rightarrow i$ to represent a new index that runs over all mappings from $j$ to $i$. We can thus write 
\begin{linenomath}
\begin{equation}
\begin{aligned}
(v < I) \Leftrightarrow & \vee_{j=1}^{2}  \wedge_{i=1}^{3} (c_{ij \textbf{b}} < I) \\
  = & \wedge_{j\rightarrow i}^{9} \vee_{j=1}^{2} (c_{i(j) j\textbf{b}} < I),
\label{eq:imagewithline2}
\end{aligned}
\end{equation}
\end{linenomath}
which is the CNF form we were after.  We could conveniently use symbol $j \rightarrow i$ to simply switch $\vee_{j} \wedge_{i}$ for $\wedge_{j\rightarrow i} \vee_{j}$ to get to the same result. The constants $c_{i(j) j\textbf{b}}$ are now written with an index $i$ that depends on index $j$. Said dependency is different for each of the 9 possible values of the \enquote{map index} $j \rightarrow i$. Each value of the index is a possible mapping function $i(j)$ from $j$ to $i$. 

Comparing the formula of Equation (\ref{eq:imagewithline2}) with the formula for the algebra in Equation (\ref{eq:imagewithline_algebra}) we get that for these $3 \times 2$ images there are $9$ atoms of the form 
\begin{linenomath}
\begin{equation}
\phi_{j\rightarrow i}=\vee_{j} c_{i(j)j \textbf{b}},
\end{equation}
\end{linenomath}
each edged to a black pixel in column $1$ and a black pixel in column $2$. 

The same argument follows for images of size $M \times N$, for which there are $M^N$ atoms of size $M \times N$, each edged to a  black pixel in each of the $N$ columns. The $M^N$ different atoms come from the possible mappings from columns to rows,
\begin{linenomath}
\begin{equation}
(v < I) \Leftrightarrow  \wedge_{j\rightarrow i} \phi_{j\rightarrow i}.
\end{equation}
\end{linenomath}
For images of size $3 \times 2$ the exact atomization has to $3^{2}=9$ atoms. For our toy example of $2 \times 2$ images, the exact atomization corresponds to $2^{2}=4$ atoms with the same form we found using the algebraic learning algorithm. For this simple problem Sparse Crossing managed to find the exact atomization.

Following an analogous procedure, in {\bf{Appendix \ref{perfectAtomizationAppendix}}} we show how to derive the form of the exact atomization for any problem for which we know a first-order formula, with or without quantifiers.

\subsection{Analysis of approximated solutions} \label{approximated}

Consider the vertical line problem again but this time for images of size $15 \times 15$. According to the analysis in the previous section, the exact atomization has $15^{15} \approx 4 \times 10^{17}$ atoms, each having $15$ black pixels, one per column. Now we calculate the number of atoms we would need for an approximated model. 

Suppose we are dealing with a training dataset that has a $10 \%$ noise defined as a probability to have a white background pixel transformed into black. Assume we are interested in a model with a false positive error rate of $1$ in a $1000$, that is, out of $1000$ images without a vertical bar we accept, on average, one false positive. For an image without a vertical bar to be classified as positive, it needs to contain all atoms of constant '$v$'. Let us first compute the probability that a given image contains one of the atoms, say atom $\phi$, of $v$. This atom $\phi$ is in one black pixel per column. In total $\phi$ is edged to $15$ black pixels. For an image without a bar to have this atom, out of its $10 \%$ noise pixels in black it needs to have at least one black pixel in the same position than one of the $15$ black pixels of $\phi$. The probability of this happening is the probability that not all of the $15$ black pixels in $\phi$ are white in our image,
\begin{linenomath}
\begin{equation}
p(\phi < I)=1-0.9^{15},
\end{equation}
\end{linenomath}
with $\,0.9 = 1 - 0.1\,$ the probability that any background pixel is white.

Suppose $v$ has $A$ atoms. The probability for the term associated to the image to contain all the atoms of $v$, considering the probabilities for each of the atoms to be in the image as approximately independent, is then given by
\begin{linenomath}
\begin{equation}
p(v < I)\approx(1-0.9^{15})^A.
\end{equation}
\end{linenomath}
For this probability to be below $1/1000$, we obtain for $A$ that
\begin{linenomath}
\begin{equation}
A>\frac{\log(1/1000)}{\log(1-0.9^{15})} =29.9.
\end{equation}
\end{linenomath}
This means that approximately $30$ atoms suffice to have a false positive error under $1$ in $1000$ and no false negatives. We can choose these $30$ atoms among the $4\times 10^{17}$ atoms of the exact atomization. This gives of the order of $10^{529}$ solutions. Our learning algorithm only needs to find one among this astronomical number of solutions.  

If we use Sparse Crossing to resolve this problem we start seeing atoms with the right form after some few hundred examples. In fact, we find $30$ atoms (and more) of the exact solution which renders the error rate to less than one in a thousand within $50.000$ examples. Learning occurs fast: error rate is about $5\%$ after the first $1000$ examples. Hundreds of atoms with the right form are found at the end.

Identifying vertical bars is easy because the atoms of the exact solution are small (this is discussed in general in {\bf{Appendix \ref{perfectAtomizationAppendix}}}). An exact solution with small atoms is not a general property of all problems. When the atoms of the exact solution are large, an atom taken from an approximate solution often corresponds to a \enquote{subatom} rather than an atom of the exact solution. A subatom is a smaller atom cotained only in some of the containing constants of an atom in the exact solution. By replacing large exact atoms with smaller subatoms false negatives are tradeoff in exchange for fewer false positives.  If the subatoms are chosen appropriately the size of the algebra and the error rate can be kept small.  

For example, if we want to distinguish images with an even number of bars from images with an odd number of bars, the atoms of the exact solution have a variety of forms and sizes.  Each atom is contained in each of the white pixels of one or more complete white bars and also in one black pixel of each of the remaining bars. These atoms may be very large, some contained in almost half of the constants. The form of the exact solution for this problem can be derived by using the technique detailed in appendix \ref{perfectAtomizationAppendix}.  The atoms we obtain by Sparse Crossing correspond to sub atoms of the exact solution and have variable sizes. The atoms of the exact solution for this problem can be partitioned in classes and approximate solutions select atoms in all or most of these classes. We have resolved this problem using Sparse Crossing with error rates well under $1\%$. We have solved it for small grids like $10 \times 10$ with a low or moderate noise below $10\%$ and for smaller grids like $5 \times 5$ with a background noise as large as $50\%$. For example, in $10 \times 10$ with a noise level of $1\%$ the number of atoms needed is about $1,800$. If noise is increased to $2.5\%$ the number of atoms needed grows to $4,200$ and give an error rate under $1\%$ and lower error of about $0.3 \%$ if we use multiple atomizations compatible with the same dual, as explained in {\bf{Section \ref{multipleAtomizations}}}.

In the case of the MNIST handwritten character dataset \cite{Lecun1998} there is no proper way to define a exact solution, however, anything we would consider as a good candidate has atoms of all sizes, from a few constants to (almost) $784$. Again, small approximate solutions with about $1\%$ error rate can be found with much smaller atoms, most contained in about $4$ to $10$ constants only.

\subsection{Memorizing and generalizing models} 
\label{memorizingGeneralizing}

The abilities of memorizing and generalizing are not incompatible in humans, nor should they be for machine learning algorithms. However, it seems to exist some kind of fundamental trade-of between memorizing and learning. Memorizing, instead of generalizing, (also known as overfitting) is a frequent problem for statistical learning algorithms. 

Memorizing may not be bad per-se but it is always expensive. It comes at a cost. The cost of memorizing for algebras is growing the model. Memorizing a relation requires to add a new atom or to make some of the existing atoms \enquote{larger}. Specifically, an atom $\phi$ is larger than atom $\omega$ if for each constant $\omega < c$ we have $\phi < c$.

Learning, in the other hand, may yield a negative cost; learning a relation can make the atomization smaller or at least grow it by an amount that is less than the information needed to store the relation. 

A goal of algebraic learning is to find a small model. We want small models not only because large models are expensive but also because an atomized model with substantially fewer atoms than the number of independent input relations is going to generalize. In the next section and in \textbf{Appendix \ref{compressionAndError}} we establish a relationship between generalization and compression rate for random models.

Smallness alone, however, is not enough to guarantee a generalizing model. Data compressors produce small representations of data but do not generalize. Furthermore, in order to acquire the information needed to extract relevant features from data, generalizing models may need to \enquote{grow} in an initial learning stage. In this sense generalizing algorithms behave very differently than a data compressor. So, if smallness is not enough, what is missing?  

In order to answer this question we are going to study first the memorizing models. We start by describing two algorithms that produce models that act as memories of $R$. It is not necessary for the reader to understand how these algorithms work to follow the discussion below. 

To build a memory we need to encode the training positive order relations $R^{+}$ as directed edges in the graph of $M^{*}$, construct the set $\Lambda$ formed by the constants $[b]$ such that there is some $a$ and some relation $\neg(a < b) \in R^{-}$, add a different atom in $M^{*}$ under each constant of $\Lambda$, add to the dual of each term mentioned in any relation of $R$ the atoms in the intersection of the duals of the term's component constants and calculate the transitive closure of the graph of $M^{*}$. Then there is a one-to-one mapping between each atom of $M^{*}$ and one atom we can introduce in $M$. 

Now, for each atom $\,\xi \in M^{*}$ introduce an atom $\phi_{\xi}\in M \,$ that satisfies for each constant $c$ in $M$: 
\begin{linenomath}
\begin{equation}
\neg (\xi < [c])\, \Leftrightarrow (\,\phi _\xi < c).
\end{equation}
\end{linenomath} 
If $a$ and $b$ are constants it follows easily that
\begin{linenomath}
\begin{equation}
\phi _\xi \in {\bf{dis}}(a,b)\, \Leftrightarrow \xi \in {\bf{dis}}([b],[a]),
\end{equation}
\end{linenomath} 
so $\neg(a< b)$ holds in $M$ if and only if $\neg([b] < [a])$ holds in the dual (see the notation in {{\bf{Appendix} \ref{Notation}} for the definition of discriminant). Theorem \ref{termInversion} extends this correspondence between discriminants to terms and therefore proves that the atomization of $M$ satisfies the same positive and negative relations than $M^{*}$. 

The first thing we should realize is that this algorithm is not stochastic. There is a single model produced by the algorithm. The model memorizes the negative input relations $R^{-}$ so, for any pair of terms or constants $a$ and $b$, the relation $a < b$ holds unless $R \models \neg(a < b)$. To the query \enquote{$(a<b)$?} the model always answers \enquote{yes} unless the input relations imply otherwise. The model is therefore unable to distinguish most elements as it also satisfies $b<a$.

The second thing we should realize is that this model does not need more atoms than negative relations we have in $R^{-}$. The model may be small (particularly if there are many more relations is $R^{+}$) and despite that entirely unable to generalize.

The ability of a model to distinguish terms is captured by the concept of algebraic freedom. A model $N$ is freer them a model $M$ if for each pair $a$ and $b$ of elements, $\neg(a < b)$ is true in $M$ implies that it is also true in $N$. The freest model of an algebra (any algebra, not only a semilattice, a group for example) corresponds with its \enquote{term algebra}. The term algebra is a model that has an element for each possible term and two terms correspond with the same element only if their equality is entailed by the axioms of the algebra.

The axioms of our algebra are the axioms of a semilattice plus the input relations $R$. The model produced by the memorizing algorithm above is precisely the least-free model compatible with these axioms. What about the freest model? 

The freest model can also be easily built; we do not even need the auxiliary $M^{*}$. To build the freest model introduce one different atom under each constant of $M$ and then enforce all the positive relations of $R^{+}$ one by one using full crossing. 

Again, this algorithm is not stochastic and there is a single output model. Since the number of atoms of this model tends to grow geometrically with the number of positive relations in $R^{+}$ we usually end up with an atomization with many atoms. 

Not surprisingly, this model also behaves as a memory. This time it remembers the relations in $R^{+}$ and to the query $(a<b)?$ the model always answers \enquote{no} unless $R^{+} \models (a < b)$. The model distinguishes most terms but it does not generalize and it is so large in practical problems that usually cannot be computed.

The freest and least-free models are both memories. However, there are important differences; free models are very large while least-free models are small. In addition, least-free models tend to produce larger atoms than freer models. \textbf{Figures \ref{fig:atomsFreest}} and \textbf{\ref{fig:atomsLeastFree}} depict the atoms of the freest and least free models of the toy vertical-line problem in $2\times 2$ dimension.

We want to keep our generalizing models reasonably away from the memorizing models. Because we specifically seek for small models staying away from least-free models, that are also small, is fundamental. We do not need to worry about free models that behave as memories because they are large and cardinal minimization only finds small solutions. The ingredient we need to complement cardinal minimization is algebraic freedom. 

\bigskip
\bigskip

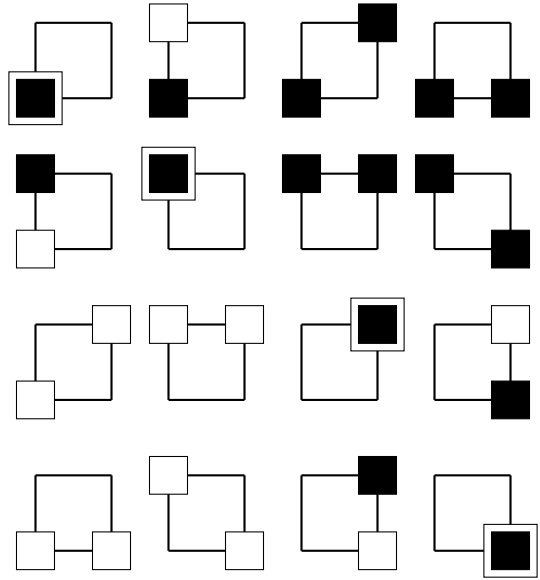
\begin{figure}[htb!]
\centering
\begin{tikzpicture}[scale=1]
% First atom

\draw[thick] (-0.25,-5)--(0.75,-5);    
\draw[thick] (-0.25,-4)--(0.75,-4);    
\draw[thick] (-0.25,-5)--(-0.25,-4);    
\draw[thick] (0.75,-5)--(0.75,-4);  

\draw[fill=white] (-0.6,-5.35) rectangle (0.1,-4.65);  % Down Left white
\draw[fill=black] (-0.50,-5.25) rectangle (0,-4.75);  % Down Left

%\draw[fill=black] (0.50,-5.25) rectangle (1,-4.75);   % Down Right
%\draw[fill=black] (-0.5,-4.25) rectangle (0,-3.75);   % Up Left
%\draw[fill=black] (0.5,-4.25) rectangle (1,-3.75);    % Up Right

% Second atom

\draw[thick] (1.5,-5)--(2.5,-5);    
\draw[thick] (1.5,-4)--(2.5,-4);    
\draw[thick] (1.5,-5)--(1.5,-4);    
\draw[thick] (2.5,-5)--(2.5,-4);  

%\draw[fill=white] (2.25,-4.25) rectangle (2.75,-3.75);      % Up Right
\draw[fill=white] (1.25,-4.25) rectangle (1.75,-3.75);      % Up Left
\draw[fill=black] (1.25,-5.25) rectangle (1.75,-4.75);      % Down Left
%\draw[fill=black] (2.25,-5.25) rectangle (2.75,-4.75);      % Down Right

% Third atom

\draw[thick] (3.25,-5)--(4.25,-5);    
\draw[thick] (3.25,-4)--(4.25,-4);    
\draw[thick] (3.25,-5)--(3.25,-4);    
\draw[thick] (4.25,-5)--(4.25,-4);  

\draw[fill=black] (4,-4.25) rectangle (4.5,-3.75);  % Up Right
%\draw[fill=black] (3,-4.25) rectangle (3.5,-3.75);  % Up Left
%\draw[fill=black] (4,-5.25) rectangle (4.5,-4.75);  % Down Right
\draw[fill=black] (3,-5.25) rectangle (3.5,-4.75);  % Down Left

% Fourth atom

\draw[thick] (5,-5)--(6,-5);    
\draw[thick] (5,-4)--(6,-4);    
\draw[thick] (5,-5)--(5,-4);    
\draw[thick] (6,-5)--(6,-4);  

%\draw[fill=black] (4.75,-4.25) rectangle (5.25,-3.75);  % Up Left
\draw[fill=black] (5.75,-5.25) rectangle (6.25,-4.75);  % Down Right
%\draw[fill=black] (5.75,-4.25) rectangle (6.25,-3.75);  % Up Right
\draw[fill=black] (4.75,-5.25) rectangle (5.25,-4.75);  % Down Left

% Fifth atom

\draw[thick] (-0.25,-7)--(0.75,-7);    
\draw[thick] (-0.25,-6)--(0.75,-6);    
\draw[thick] (-0.25,-7)--(-0.25,-6);    
\draw[thick] (0.75,-7)--(0.75,-6);  

\draw[fill=white] (-0.50,-7.25) rectangle (0,-6.75);  % Down Left
%\draw[fill=black] (0.50,-7.25) rectangle (1,-6.75);   % Down Right
\draw[fill=black] (-0.5,-6.25) rectangle (0,-5.75);   % Up Left
%\draw[fill=black] (0.5,-6.25) rectangle (1,-5.75);    % Up Right

% Sixth atom

\draw[thick] (1.5,-7)--(2.5,-7);    
\draw[thick] (1.5,-6)--(2.5,-6);    
\draw[thick] (1.5,-7)--(1.5,-6);    
\draw[thick] (2.5,-7)--(2.5,-6);  

%\draw[fill=black] (2.25,-6.25) rectangle (2.75,-5.75);      % Up Right
\draw[fill=white] (1.15,-6.35) rectangle (1.85,-5.65);      % Up Left White
\draw[fill=black] (1.25,-6.25) rectangle (1.75,-5.75);      % Up Left
%\draw[fill=black] (1.25,-7.25) rectangle (1.75,-6.75);      % Down Left
%\draw[fill=black] (2.25,-7.25) rectangle (2.75,-6.75);      % Down Right

% Seventh atom

\draw[thick] (3.25,-7)--(4.25,-7);    
\draw[thick] (3.25,-6)--(4.25,-6);    
\draw[thick] (3.25,-7)--(3.25,-6);    
\draw[thick] (4.25,-7)--(4.25,-6);  

\draw[fill=black] (4,-6.25) rectangle (4.5,-5.75);  % Up Right
\draw[fill=black] (3,-6.25) rectangle (3.5,-5.75);  % Up Left
%\draw[fill=black] (4,-7.25) rectangle (4.5,-6.75);  % Down Right
%\draw[fill=black] (3,-7.25) rectangle (3.5,-6.75);  % Down Left

% Eighth atom

\draw[thick] (5,-7)--(6,-7);    
\draw[thick] (5,-6)--(6,-6);    
\draw[thick] (5,-7)--(5,-6);    
\draw[thick] (6,-7)--(6,-6);  

\draw[fill=black] (4.75,-6.25) rectangle (5.25,-5.75);  % Up Left
\draw[fill=black] (5.75,-7.25) rectangle (6.25,-6.75);  % Down Right
%\draw[fill=black] (5.75,-6.25) rectangle (6.25,-5.75);  % Up Right
%\draw[fill=black] (4.75,-7.25) rectangle (5.25,-6.75);  % Down Left

% Ninth atom

\draw[thick] (-0.25,-9)--(0.75,-9);    
\draw[thick] (-0.25,-8)--(0.75,-8);    
\draw[thick] (-0.25,-9)--(-0.25,-8);    
\draw[thick] (0.75,-9)--(0.75,-8);  

\draw[fill=white] (-0.50,-9.25) rectangle (0,-8.75);  % Down Left
%\draw[fill=black] (0.50,-9.25) rectangle (1,-8.75);   % Down Right
%\draw[fill=white] (-0.5,-8.25) rectangle (0,-7.75);   % Up Left
\draw[fill=white] (0.5,-8.25) rectangle (1,-7.75);    % Up Right

% Tenth atom

\draw[thick] (1.5,-9)--(2.5,-9);    
\draw[thick] (1.5,-8)--(2.5,-8);    
\draw[thick] (1.5,-9)--(1.5,-8);    
\draw[thick] (2.5,-9)--(2.5,-8);  

\draw[fill=white] (2.25,-8.25) rectangle (2.75,-7.75);      % Up Right
\draw[fill=white] (1.25,-8.25) rectangle (1.75,-7.75);      % Up Left
%\draw[fill=black] (1.25,-9.25) rectangle (1.75,-8.75);      % Down Left
%\draw[fill=black] (2.25,-9.25) rectangle (2.75,-8.75);      % Down Right

% Eleventh atom

\draw[thick] (3.25,-9)--(4.25,-9);    
\draw[thick] (3.25,-8)--(4.25,-8);    
\draw[thick] (3.25,-9)--(3.25,-8);    
\draw[thick] (4.25,-9)--(4.25,-8);  

\draw[fill=white] (3.9,-8.35) rectangle (4.6,-7.65);  % Up Right White
\draw[fill=black] (4,-8.25) rectangle (4.5,-7.75);  % Up Right

%\draw[fill=black] (3,-8.25) rectangle (3.5,-7.75);  % Up Left
%\draw[fill=black] (4,-9.25) rectangle (4.5,-8.75);  % Down Right
%\draw[fill=black] (3,-9.25) rectangle (3.5,-8.75);  % Down Left

% 12th atom

\draw[thick] (5,-9)--(6,-9);    
\draw[thick] (5,-8)--(6,-8);    
\draw[thick] (5,-9)--(5,-8);    
\draw[thick] (6,-9)--(6,-8);  

%\draw[fill=black] (4.75,-8.25) rectangle (5.25,-7.75);  % Up Left
\draw[fill=black] (5.75,-9.25) rectangle (6.25,-8.75);  % Down Right
\draw[fill=white] (5.75,-8.25) rectangle (6.25,-7.75);  % Up Right
%\draw[fill=black] (4.75,-9.25) rectangle (5.25,-8.75);  % Down Left

% 13th atom

\draw[thick] (-0.25,-11)--(0.75,-11);    
\draw[thick] (-0.25,-10)--(0.75,-10);    
\draw[thick] (-0.25,-11)--(-0.25,-10);    
\draw[thick] (0.75,-11)--(0.75,-10);  

\draw[fill=white] (-0.50,-11.25) rectangle (0,-10.75);  % Down Left
\draw[fill=white] (0.50,-11.25) rectangle (1,-10.75);   % Down Right
%\draw[fill=black] (-0.5,-10.25) rectangle (0,-9.75);   % Up Left
%\draw[fill=black] (0.5,-10.25) rectangle (1,-9.75);    % Up Right

% 14th atom

\draw[thick] (1.5,-11)--(2.5,-11);    
\draw[thick] (1.5,-10)--(2.5,-10);    
\draw[thick] (1.5,-11)--(1.5,-10);    
\draw[thick] (2.5,-11)--(2.5,-10);  

%\draw[fill=black] (2.25,-10.25) rectangle (2.75,-9.75);      % Up Right
\draw[fill=white] (1.25,-10.25) rectangle (1.75,-9.75);      % Up Left
%\draw[fill=white] (1.25,-11.25) rectangle (1.75,-10.75);      % Down Left
\draw[fill=white] (2.25,-11.25) rectangle (2.75,-10.75);      % Down Right

% 15th atom

\draw[thick] (3.25,-11)--(4.25,-11);    
\draw[thick] (3.25,-10)--(4.25,-10);    
\draw[thick] (3.25,-11)--(3.25,-10);    
\draw[thick] (4.25,-11)--(4.25,-10);  

\draw[fill=black] (4,-10.25) rectangle (4.5,-9.75);  % Up Right
%\draw[fill=black] (3,-10.25) rectangle (3.5,-9.75);  % Up Left
\draw[fill=white] (4,-11.25) rectangle (4.5,-10.75);  % Down Right
%\draw[fill=black] (3,-11.25) rectangle (3.5,-10.75);  % Down Left

% 16th atom

\draw[thick] (5,-11)--(6,-11);    
\draw[thick] (5,-10)--(6,-10);    
\draw[thick] (5,-11)--(5,-10);    
\draw[thick] (6,-11)--(6,-10);  

%\draw[fill=black] (4.75,-10.25) rectangle (5.25,-9.75);  % Up Left
\draw[fill=white] (5.65,-11.35) rectangle (6.35,-10.65);  % Down Right White
\draw[fill=black] (5.75,-11.25) rectangle (6.25,-10.75);  % Down Right
%\draw[fill=black] (5.75,-10.25) rectangle (6.25,-9.75);  % Up Right
%\draw[fill=black] (4.75,-11.25) rectangle (5.25,-10.75);  % Down Left

\end{tikzpicture} 
\caption{Atoms of the freest model that satisfy the training examples for the toy problem of identifying vertical lines in dimension $2 \times 2$. A black square with a white border represents a pixel whose black color and white color constants contain the atom. }
\label{fig:atomsFreest}
\end{figure}

\bigskip
\bigskip

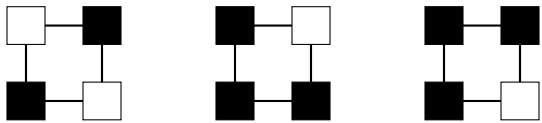
\begin{figure}[htb!]
\centering
\begin{tikzpicture}[scale=1]
% First atom

\draw[thick] (-0.25,-5)--(0.75,-5);    
\draw[thick] (-0.25,-4)--(0.75,-4);    
\draw[thick] (-0.25,-5)--(-0.25,-4);    
\draw[thick] (0.75,-5)--(0.75,-4);  

\draw[fill=black] (-0.50,-5.25) rectangle (0,-4.75);  % Down Left
\draw[fill=white] (0.50,-5.25) rectangle (1,-4.75);   % Down Right
\draw[fill=white] (-0.5,-4.25) rectangle (0,-3.75);   % Up Left
\draw[fill=black] (0.5,-4.25) rectangle (1,-3.75);    % Up Right

% Second atom

\draw[thick] (2.5,-5)--(3.5,-5);    
\draw[thick] (2.5,-4)--(3.5,-4);    
\draw[thick] (2.5,-5)--(2.5,-4);    
\draw[thick] (3.5,-5)--(3.5,-4);  

\draw[fill=white] (3.25,-4.25) rectangle (3.75,-3.75);      % Up Right
\draw[fill=black] (2.25,-4.25) rectangle (2.75,-3.75);      % Up Left
\draw[fill=black] (2.25,-5.25) rectangle (2.75,-4.75);      % Down Left
\draw[fill=black] (3.25,-5.25) rectangle (3.75,-4.75);      % Down Right

% Third atom

\draw[thick] (5.25,-5)--(6.25,-5);    
\draw[thick] (5.25,-4)--(6.25,-4);    
\draw[thick] (5.25,-5)--(5.25,-4);    
\draw[thick] (6.25,-5)--(6.25,-4);  

\draw[fill=black] (6,-4.25) rectangle (6.5,-3.75);  % Up Right
\draw[fill=black] (5,-4.25) rectangle (5.5,-3.75);  % Up Left
\draw[fill=white] (6,-5.25) rectangle (6.5,-4.75);  % Down Right
\draw[fill=black] (5,-5.25) rectangle (5.5,-4.75);  % Down Left

\end{tikzpicture} 
\caption{Atoms of the least free model that satisfy the training examples for the toy problem of identifying vertical lines in dimension $2 \times 2$.}
\label{fig:atomsLeastFree}
\end{figure}

\bigskip
\bigskip

\begin{figure}[htb!]
\centering
\begin{tikzpicture}[scale=1]
% First atom

% Third atom

\draw[thick] (3.25,-5)--(4.25,-5);    
\draw[thick] (3.25,-4)--(4.25,-4);    
\draw[thick] (3.25,-5)--(3.25,-4);    
\draw[thick] (4.25,-5)--(4.25,-4);  

\draw[fill=black] (4,-4.25) rectangle (4.5,-3.75);  % Up Right
%\draw[fill=black] (3,-4.25) rectangle (3.5,-3.75);  % Up Left
%\draw[fill=black] (4,-5.25) rectangle (4.5,-4.75);  % Down Right
\draw[fill=black] (3,-5.25) rectangle (3.5,-4.75);  % Down Left

% Fourth atom

\draw[thick] (5,-5)--(6,-5);    
\draw[thick] (5,-4)--(6,-4);    
\draw[thick] (5,-5)--(5,-4);    
\draw[thick] (6,-5)--(6,-4);  

%\draw[fill=black] (4.75,-4.25) rectangle (5.25,-3.75);  % Up Left
\draw[fill=black] (5.75,-5.25) rectangle (6.25,-4.75);  % Down Right
%\draw[fill=black] (5.75,-4.25) rectangle (6.25,-3.75);  % Up Right
\draw[fill=black] (4.75,-5.25) rectangle (5.25,-4.75);  % Down Left

\end{tikzpicture} 
\caption{Atoms of a generalizing model that satisfy the training examples for the toy problem of identifying vertical lines in dimension $2 \times 2$. }
\label{fig:atomsGeneralization}
\end{figure}
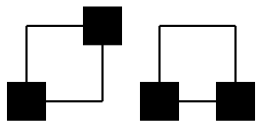

\bigskip
\bigskip

The Sparse Crossing algorithm enforces positive relations one by one using crossing just like the algorithm that finds the freest model. However the crossing is sparse in order to make the model smaller (Figure \ref{fig:atomsGeneralization}). In this way, a balance between cardinal minimization and freedom is naturally obtained. In this balance many generalizing models can be found. 

In both, the generalizing Sparse Crossing algorithm and the algorithm that finds the freest model, we start from an algebra that is very free and satisfies all the negative relations, and then we make it less free by enforcing the positive relations using crossing until we get a model of $R$. If we use full crossing, we reduce freedom just by the minimal amount needed to accommodate the positive order relations (see \textbf{Theorem \ref{crossingFree}}). If we use Sparse Crossing, we reach some compromise between freedom and size. 

We saw in \textbf{Section \ref{approximated}} that a few atoms can distinguish images with vertical lines from images without them. This is because the atoms are small, edged only to a few constants. Larger atoms are less discriminative, in the sense that they distinguish between fewer terms than smaller atoms. The atoms that resolve the MNIST dataset \cite{Lecun1998} are also small, most edged to fewer than 20 constants, many edged to as few as 5 or 6 constants which is much less than the 1568 constants needed to describe the images. Smaller atoms are more useful and increasing algebraic freedom pushes for smaller atoms.

For the toy problem the atoms produced by least-free models are contained in 4 constants each plus $v$. The atoms of the freest memory are only in two constants and $v$. The generalizing solution has 4 atoms all in two constants (plus $v$) each. 

Algebraic freedom is also at the core of the batch learning method based on adding pinning relations. Pinning relations are all negative. It can be proved that the pinning relations of a model $M$ capture all negative relations between terms. This means that whatever model we build compliant with this negative pinning relations is going to be able to distinguish between the pairs of terms that are discriminated in $M$. Any model that satisfies the pinning relations of $M$ is strictly freer than $M$. See {\bf{Theorem \ref{pinningFree}}}.

In favor of freedom maybe be argued that distinguishing between terms is by itself a desirable property of a model that understands the world. Irrespectively of whether algebraic freedom is fundamental or not, it is certainly a good counterbalance to cardinal minimization in semilattices.  

\subsection{Small and random algebraic models imply high accuracy} \label{ekIsConstant}

In the previous section we saw that Sparse Crossing balances cardinal minimization with algebraic freedom to stay away from memorizing models. In this section we prove that to find a good generalizing model we only need to pick at random a small model of the training set $R$. 
We also show that Sparse Crossing corresponds well with this theoretical result for small error. By seeking algebraic freedom, Sparse Crossing finds models of $R$ at random, away from easier-to-find, non-random, least-free models that act as memories. 

So far we have not worked with the idea of accuracy or error explicitly. Instead we have focused in finding small algebras that obey the training examples. We thus need to establish a formal link between small models and accuracy. There is some intuition from Physics, Statistics or Machine Learning that simpler models can generalize better, but it is still not obvious that the smaller the algebraic model the higher the accuracy.

We can demonstrate that for an algebra chosen at random among the ones that obey a large enough set $R$ of training examples, the expected error $\epsilon$ in a test example is (\textbf{Appendix \ref{compressionAndError}})

\begin{linenomath}
\begin{equation}
\epsilon=\frac{\ln{|\Omega_{Z}}|-\ln{|\Omega_{Z,R}}|}{|R|},
\end{equation}
\end{linenomath}

with $\Omega_{Z}$ the set of all possible atomizations with $Z$ atoms using $C$ constants and $\Omega_{Z,R}$ the set of atomizations that also satisfy $R$. The larger is the first set, $\Omega_{Z}$, the more training examples are going to be necessary to produce a desired error rate. On the other hand, the larger is the second set, $\Omega_{Z,R}$, the fewer training relations are needed. 

The quantity $\ln{|\Omega_{Z,R}}|$ measures the degeneracy of the solutions and is a subtracting term that works in to further reduce test error. $\ln{|\Omega_{Z}}|$ is an easy to calculate value that only depends upon $Z$ and the number of constants and determines an upper bound for the number of examples needed to produce an error rate $\epsilon$. Assuming that our constants come in pairs so the presence of one constant in a term implies the absence of the other (like the white pixel and the black pixel constant pair for images), the total number of atoms is $3^{p}$ with $p$ the number of pixels or $3^{C/2}$ for the constants $C=2 p$. Then $\ln{|\Omega_{Z}}|\approx \frac{ln(3)\,ZC}{2}$. Ignoring, for the moment, the beneficial effect of $\ln{|\Omega_{Z,R}}|$ and substituting above we can determine a worse-case relationship 
\newline
\begin{linenomath}
\begin{equation}
\epsilon=\frac{\ln{3}}{2} \, \frac{C}{\kappa},
\label{eq:prediction}
\end{equation}
\end{linenomath} \newline
where $\kappa$ is the compression ratio $\kappa=|R|/Z$. This expression implies that the more the algebra compresses the $R$ training examples into $Z$ atoms, the smaller the test error, and with a conversion between the two rates upper bounded by the factor $C\,\ln(3)/2$. Picking randomly an algebra that satisfies $R$ with a given compression ratio $\kappa$ would make for a good learning algorithm. It would have much better performance than using the non-stochastic, memory-like algebras of the previous section. The random picking, though, is an ideal algorithm we cannot efficiently compute.

The term $\ln{|\Omega_{Z,R}}|$ can be estimated using the symmetries of the problem (see \textbf{Appendix \ref{symmetries}}).
For the problem of detecting the presence of a vertical bar we can permute the rows and the columns of an input image without affecting its classification. For $d\times d$ image, we derived $\ln{|\Omega_{Z,R}}| > 2 \,Z\, \ln(d!)$, giving \newline 
\begin{linenomath}
\begin{equation}
\epsilon=\frac{d^{2} \ln{3}-2 \ln(d!)}{k}
\label{eq:prediction2}
\end{equation}
\end{linenomath} 
where $C=2p=2d^2$ has been used.

We have compared this prediction with the experimental results using the Sparse Crossing algorithm, and found that Sparse Crossing performs always better than this theoretical value. We have observed, however, that the higher the size of $R$ and the smaller $\epsilon$ becomes the closer are the observed values to this theoretical result. We also noticed that the harder the problem is, the faster the approach to the theoretical result as $|R|$ increases. In order to determine if we can find the theoretical prediction for a large value of $|R|$, we thus used a harder, albeit similar, problem to the detection of vertical bars. Instead of detecting the presence of vertical bars, we used the much harder problem of separating images with an even number of complete bars from images with an odd number of complete vertical bars in the presence of noise. This problem has the same symmetries than the simpler vertical bar problem so it should obey the same relation we have derived relating compression and error rates.

We generated a large number of training images. We did so by adding to an otherwise white image a random number of black vertical lines in random positions. The white background pixels are then turned into black with some probability. The training protocol started with $200$ training images. If the test error increased (decreased) in the next epoch we used $2\%$ more ($2 \%$ less) training images in the next batch. Experimental results show a clear proportionality between $\epsilon$ and $\kappa^{-1}$ with a proportionality constant that increases slowly as error rate decreases until clearly stabilizing (see \textbf{Figure \ref{fig:combinedResultsF}}) at a value that differs from the theoretical prediction in less than $10\%$ for $7 \times 7$ and about $5\%$ for $10 \times 10$ images. To get to this point we had to use as many as $37$ million examples.  

We plotted the theoretical prediction in Equation (\ref{eq:prediction2}) as a straight line, in green for $7 \times 7$ images and in blue for $10 \times 10$ (\textbf{Figure \ref{fig:accuracyfromcompression}}). The experimental results from Sparse Crossing are plotted as dots, again in green and blue for $7 \times 7$ and $10 \times 10$ images, respectively. \textbf{Figure \ref{fig:accuracyfromcompression}a} is a logarithmic plot that allows depicting the behavior of Sparse Crossing results for all values of error and compression obtained during learning. The linear scale in \textbf{Figure \ref{fig:accuracyfromcompression}b} is used to show the behavior at low errors, where algebraic learning and the theoretical expression show a good match.

We see a remarkable match between observed and theoretical values when error rates are small. For fewer training examples and higher error rates, Sparse Crossing appears more efficient that the theoretical result. This could be due to the assumptions made to derive the theoretical relation (high values of $R$ and low values of $\epsilon$) or, perhaps, Sparse Crossing is indeed more efficient than a random picking. Sparse Crossing searches the model space far away from memorizing least-free models and that can give it some edge over the random picking. However the volume taken by memorizing models compared with the overall volume of the space of models of $R$ is small, so we speculate that the measured superiority of Sparse Crossing over the theoretical derivation for the random picking is just due to the restricted validity of the theoretical result to very low error rates.

\begin{figure}[htb!]
\centering
%\hspace*{-4cm} 
\includegraphics[scale=0.5]{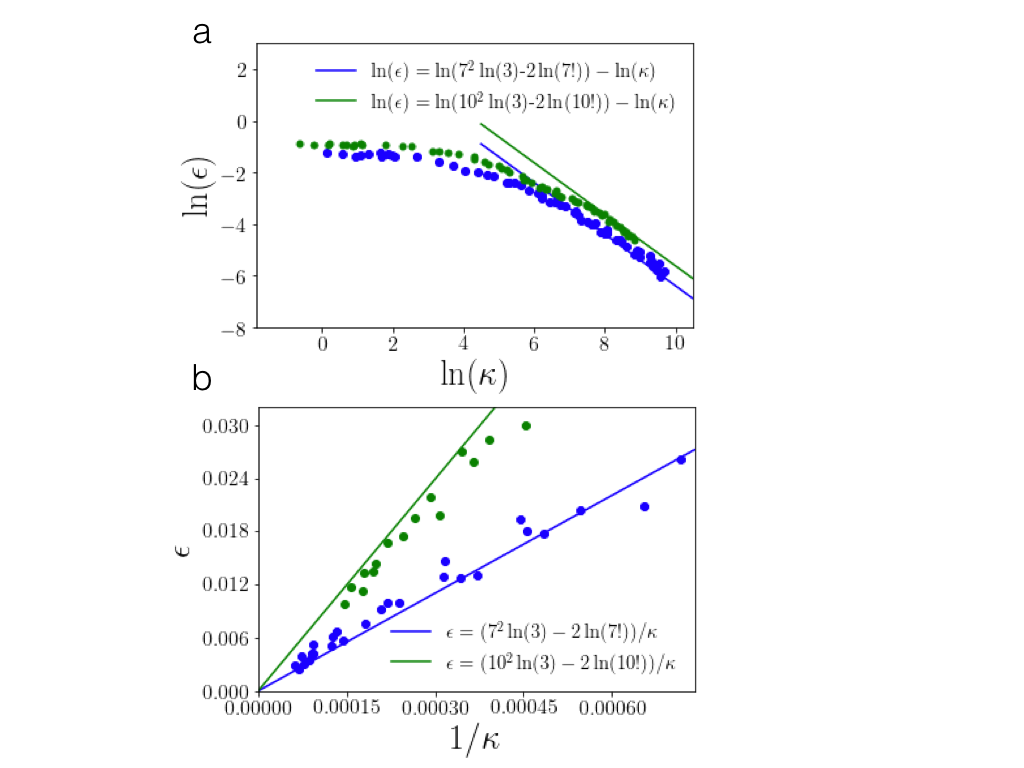} \caption{\textbf{Testing the relationship between compression and error rate in Equation (\ref{eq:prediction2})} \textbf{a}. Logarithm of error, $\ln (\epsilon)$ versus logarithm of compression, $\ln (\epsilon)$ for the problem of separating noisy images with even from odd number of vertical bars. Theoretical expression plotted as straight line, in green for $7 \times 7$ images (noise 5\%) and in blue for $10 \times 10$ images (noise 2.5\%). Algebraic learning results plotted as dots, using the same color scheme as for lines. \textbf{b} Same as a. but in linear scale to show results at low error and high compression.}
\label{fig:accuracyfromcompression}
\end{figure}

\section{Classification of hand-written digits}

Our first example of the vertical bar problem was simple enough to facilitate analysis. In this case there is a simple formula that separates positive from negative examples.  In this section we show that algebraic learning also works in real-world problems for which there is no formal or simple description. For this, we chose the standard example of hand-written digit recognition. A digit cannot be precisely defined in mathematical terms as was the case with the vertical bar, different people can write them differently and the standard dataset we use, MNIST \cite{Lecun1998}, has miss-labels in the training set, all factors making it a simple real-world case.

We used the $28 \times 28$ binary version of images of the MNIST dataset, with no pre-processing. The embedding technique is the same we applied to the toy problem of the vertical bar. An image is represented as an idempotent summation of $784$ constants representing pixels in black or white. Digits are treated as independent binary classifiers. The specific task is to learn to distinguish one digit from the rest in a supervised manner.  We use one constant per digit, each playing a similar role than the constant $v$ of the toy problem, in total $1,578$ constants. 

Our training protocol was as follows. We used $60,000$ images for training. Training epochs started with batches containing $100$ positive and $100$ negative examples. When identification accuracy in training did not increase with training epoch, the number of examples was increased by a $5\%$ until a maximum of $2,000$ positive and $2,000$ negative examples per batch. Increasing batch size and balancing of positive and negative examples seemed to accelerate convergence to some limited extent, but we did not find an impact in final accuracy values. Each digit was trained separately in a regular laptop.

For standard machine learning systems, data are separated into training, validation and test. Validation data is used to find the value of training hyperparameters that give highest accuracy in a dataset different to the one used in training. This is done to try to avoid overfitting, that is, learning specific 
features of the training data that decrease accuracy in the test set. We found no overfitting using algebraic learning (\textbf{Figure~\ref{fig:mnist1}}, top). This figure gives the error rate in the test set for the recognition of digits \enquote{0} to \enquote{9} as a function of the training epoch. The error decreases until training epoch $200$, from which it stays constant except for small fluctuations. As we did not find overfitting using algebraic learning, we did not need to use a validation dataset in our study of hand-written recognition. 

\begin{figure}[htb!]
%\centering
\hspace*{-0.5cm} 
 \includegraphics[scale=0.85]{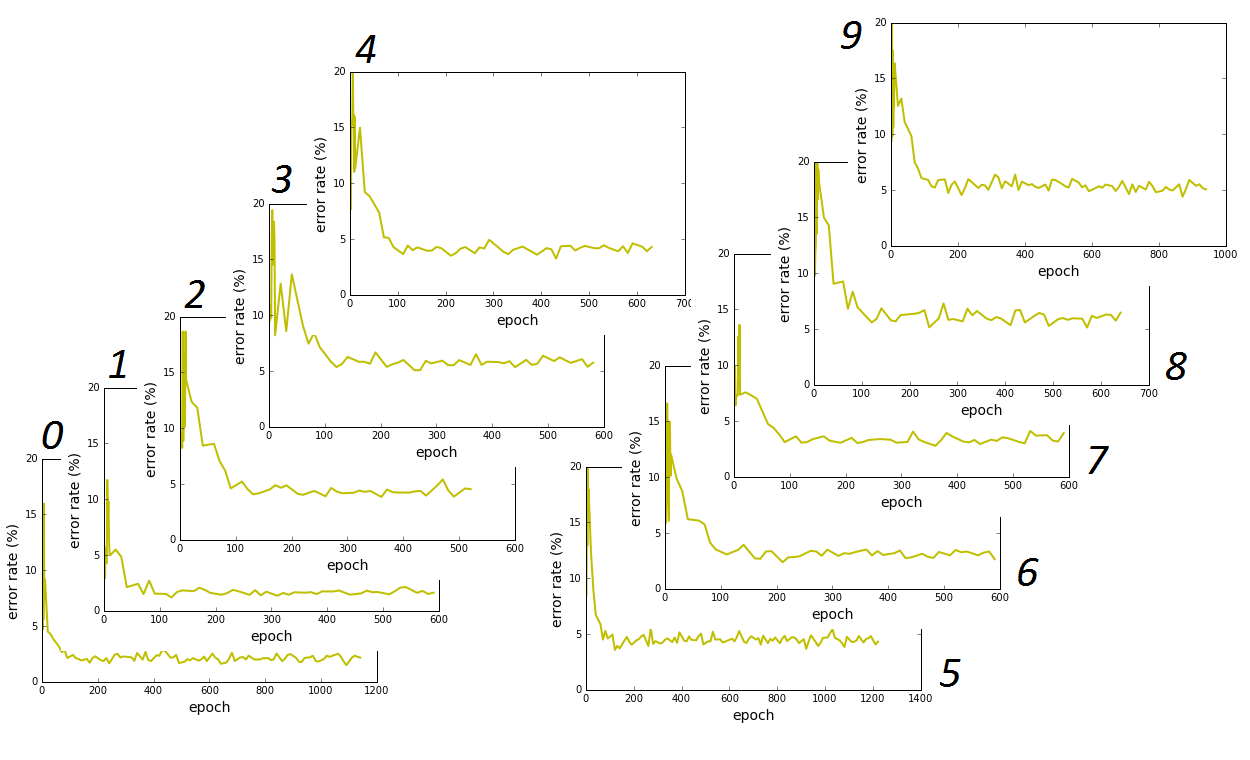}
	\caption{\textbf{Learning to distinguish one hand-written digit from the rest.} Error rate in test set for the recognition of MNIST digits \enquote{0} to \enquote{9} in dimension $28 \times 28$ at different training epochs using a single atomization. Digits are trained separately as binary classifiers. Learning takes place within the first two hundred epochs. Repeated training afterwards with the same examples does not affect error rate. }
    \label{fig:mnist1}
 \end{figure}

After training, we found that the error rate in the test dataset (a total of $10,000$ images) varies from $1.63\%$ for digit \enquote{1} to $6.46\%$ for digit \enquote{8} (see Table~\ref{tab:mnist}(A) for all digits), and an average error rate of $4.0\%$. 
 
Most atoms found consist of scattered white and black pixels, Figure~\ref{fig:mnist1b}. After training with a batch, the positive examples of the batch contain all master atoms while the negative examples contain less than all master atoms. This translates into master atomizations for which each atom is contained in at least one pixel of each positive example. Each atom corresponds to groups of pixels shared more frequently by positive examples (to give a minimum number of atoms) that appear less frequently in negative examples (to produce atoms of a minimum size so algebraic freedom is maximized). Also, pixels containing many atoms are correlated with the pixels more frequently found in the inverse of most negative examples. In this way the probability for a negative example to contain all atoms is small. 

Most atoms are contained in only a few pixels but we found a few atoms that resemble the inverse of rare versions of digits in the negative class. For example, a \enquote{6} that is very rotated in the third row and four column of \textbf{Figure~\ref{fig:mnist1b}} is an atom found during the algebraic training of digit \enquote{5} versus the rest of digits. These untypical training examples are learned by forming a specific memory with a single atom and in this way their influence in the form of the other atoms can be negligible. This may a reason for algebraic learning not being severely affected by mislabelings.

\begin{figure}
\centering
\hspace*{-0.5cm} 
 \includegraphics[scale=0.8]{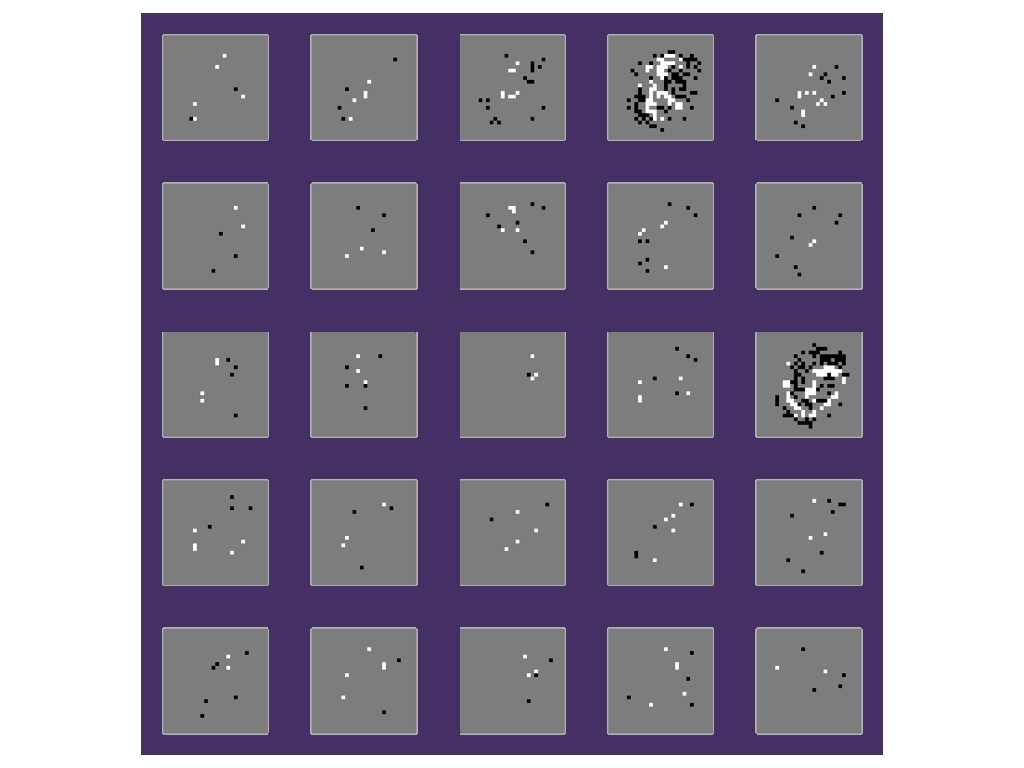}
    \caption{\textbf{Atoms in digit recognition}. Example atoms in the training for digit \enquote{5}.}
    \label{fig:mnist1b}
 \end{figure}
 
\newpage

\begin{table}[t]
\centering
\begin{minipage}{.7\textwidth}
\centering
\caption*{A: Single master atomization}
\begin{tabular}{@{}llll@{}}
\toprule
Digit      & Error (\%)  & FPR(\%)   & FNR (\%) \\ \midrule
0          &  2.17            & 2.15    &  2.35   \\
1          &  1.63          & 4.15      &  1.50   \\
2          &  4.54          & 5.47      &  7.94    \\
3          &  5.75          & 3.87      &  8.22    \\
4          &  4.29          & 4.02      &  8.15    \\    	  
5          &  4.32          & 2.40       &  7.40   \\
6          &  2.65          & 3.70       &  5.01    \\
7          &  3.92          & 6.08      &  5.84    \\
8          &  6.46          & 4.84      &  9.96    \\
9          &  5.06          & 2.15      &  7.04    \\
													\\
Average    &  4.08          & 3.83     &  6.34    \\ \bottomrule
\\
\end{tabular}
\end{minipage}
\begin{minipage}{.7\textwidth}
\centering
\caption*{B: 10 master atomizations }
%\caption{Errors, false positives and false negatives in recognition of hand-written digits in test set using a single master atomization. }
%\end{table}
%\begin{table}[]
\begin{tabular}{@{}llll@{}}
\toprule
Digit      & Error (\%)  & FPR(\%)   & FNR (\%) \\ \midrule
0          &  0.97            & 0.99    &  0.82   \\
1          &  0.69         & 0.65      &  0.97   \\
2          &  1.60         & 1.40      &  3.29    \\
3          &  2.44         & 2.40      &  2.77    \\
4          &  1.80         & 1.68     &  2.85   \\    	  
5          &  1.61        & 1.57      &  2.02  \\
6          &  1.48         & 1.41      &  2.09   \\
7          &  1.36          & 1.14    &  3.31    \\
8          &  2.54         & 2.47     &  3.18    \\
9          &  2.29         & 2.12     &  3.77    \\
													\\
Average    &  1,68          & 1.52      &  2.41    \\ \bottomrule
     \\
\end{tabular}
\end{minipage}
\caption{Errors, false positives and false negatives in recognition of hand-written digits in test set for (A) one master atomization, and (B) for 10 master atomizations requiring $5$ or more agreements to classify an example as positive and fewer than $5$ as negative. } 
\label{tab:mnist}
\end{table}

\subsection{Using several master atomizations} \label{multipleAtomizations}

The result of embedding a batch of training examples is an atomization satisfying all the examples in the batch. At each epoch, a suitable atomization of the dual is chosen of the many possible and then the Sparse Crossing algorithm produces an atomization of the master consistent with the training set and the pinning relations or a subset of them. Enforcing of the batch is carried out using a stochastic algorithm over the chosen atomization of the dual, which contains the pinning terms learned in previous epochs. The enforcing of a batch then results in one of the many suitable atomizations of the master algebra.

Nothing prevents us from using more than one atomization for a training or a test batch. Changing the atomization for the dual results in different trace constraints and, hence, in a different atomization for the master. For the hadwritten digits problem, we counted how many among $10$ atomizations classify the positive test images as positive. In \textbf{Figure~\ref{fig:mnist2}} we show results for digits \enquote{0} and \enquote{9}. For digit \enquote{0}, for example, approximately $90 \%$ of the positive test images have the $10$ atomizations agreeing in that a digit is indeed a digit \enquote{0}. For less than $8\%$ of the test cases, it is $9$ out of the $10$ atomizations that agree in that the digit is a \enquote{0}. Agreement of less of the algebras meet with even smaller percentages of the cases. We found that, for digit \enquote{0}, more algebras agree the more round the digit (\textbf{Figure~\ref{fig:mnist2}}, top insets). For digit \enquote{9}, disagreement exists for incomplete and rotated version of the digit (\textbf{Figure~\ref{fig:mnist2}}, top insets). Using more than one atomization is a simple procedure that extracts more information from the algebra. In \textbf{Table ~\ref{tab:mnist}B} we give test set results using the atomizations obtained from the last $10$ epochs of training.

\begin{figure}[!htb]
\centering
\hspace*{-1cm} 
 \includegraphics[scale=0.8]{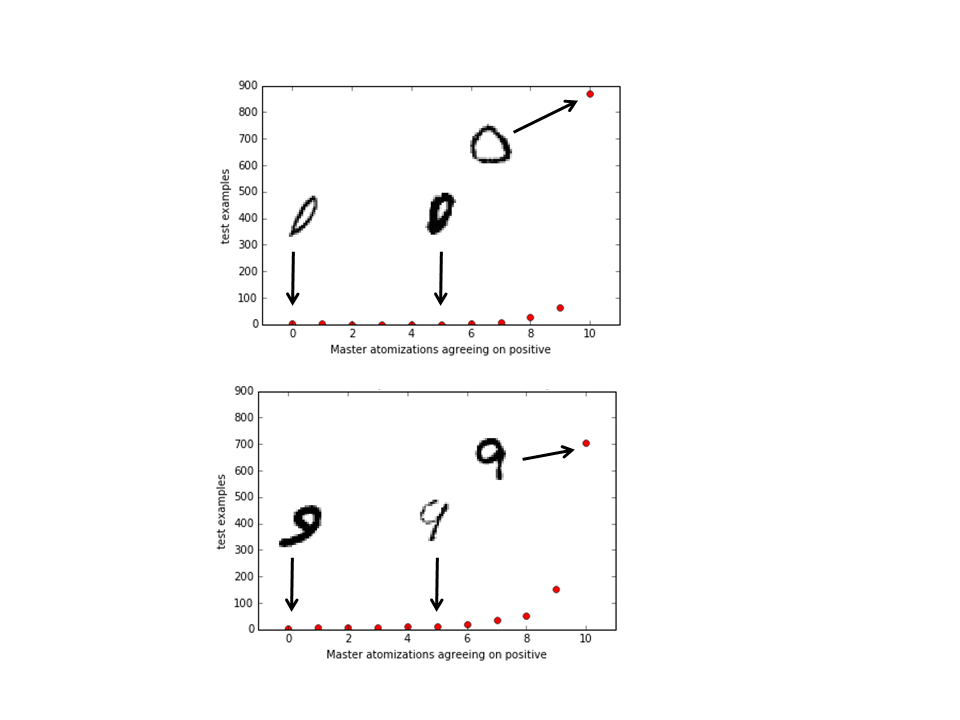}
    \caption{\textbf{Agreement of 10 master atomizations.} Top: Count of the number of test examples correctly classified as digit \enquote{0} by $0, 1, 2,...,10$ master atomizations out of 10 atomizations. Most images are correctly identified by all the atomizations. Insets are examples of images of digit \enquote{0} for complete agreement, disagreement and complete lack of identification. Bottom: Same but for digit \enquote{9}}
    \label{fig:mnist2}
 \end{figure}

The number of examples correctly classified as positive in Figure~\ref{fig:mnist2} increases approximately as an exponential with the number of agreements. A very similar exponentially-looking function is observed when we plot the number of negative examples versus the number of atomizations correctly agreeing on a negative (plot not shown). The exponential increase might be understood with each image $I$ having a probability $p_I$ of misclassification when using a single atomization. The value $p_I$ is typically low for most example images but it may be high (even closer to 1) for some difficult images. For each image, a binomial distribution describes the number of times it is misclassified among the $10$ tests corresponding with $10$ different atomizations. The distribution for all test images should be a mixture of binomial distributions with different values of $p_I$ and, with most images been easy to identify, the weight of the easily identifiable examples dominates producing the exponentially-looking distribution of \textbf{Figure~\ref{fig:mnist2}}.  

There is a subset of the test examples for which a small probability of misclassification exists even though they may look very clear to a human. However, the risk of misclassification due to this intrinsic probability of failure goes away exponentially if multiple atomizations are used. A few atomizations should suffice to classify correctly these examples with small $p_I$. On the other hand, doesn't matter how many atomizations we use we cannot expect to correctly classify the examples with high $p_I$. In this case only additional training with new examples can improve the rate of success. Training with the same examples neither increases nor decreases the error rate. 

Using the criterion that at least $7$ or more of the atomizations need to agree that an image is a \enquote{0} to declare it a \enquote{0}, obtains an error rate of $0.6\%$ for this digit, with false positive and negative ratios of $\textnormal{FPR}=0.5\%$ and $\textnormal{FNR}=1.22\%$, respectively. The same criterion finds for digit \enquote{9} an error of $1.36\%$, and $\textnormal{FPR}=0.8\%$ and $\textnormal{FNR}=6.5\%$. The average over all digits is found to give an error rate of $1.07\%$ and $\textnormal{FPR}=0.56\%$ and $\textnormal{FNR}=5.6\%$. 

A criterion consisting of requiring $5$ or more atomizations to agree that an example is positive gives more balanced false and negative ratios, $\textnormal{FPR}=1.5\%$ and $\textnormal{FNR}=2.4\%$ but a higher total error rate of $1.68\%$  (see \textbf{Table ~\ref{tab:mnist}B} for all digits). A higher false negative ratio is consistent with the fact then we have $10$ times more negative examples than positive examples. 

The MNIST dataset has a limited training set that does not allow to see the effect of multiple atomizations in test results at very low error rates. For the problem of separating noisy images with even vs odd number of vertical bars we have an unlimited supply of training examples. In this case we get the results of \textbf{Figure~\ref{fig:LogFPRFNR}}, on the right. The false positive and negative ratios using $10$ atomizations decrease with the training epochs and at all times during the training remain significantly smaller than the error rate obtained with a single atomization. An almost perfect exponential dependence of the example count with the number of agreements (like in \textbf{Figure~\ref{fig:mnist2}}) is also observed for both, the positive and negative examples (data not shown). It doesn't matter how much training we do there is always an advantage in using a few master atomizations to extract the most information from the algebra.

\begin{figure}
\centering
 \includegraphics[scale=0.9]{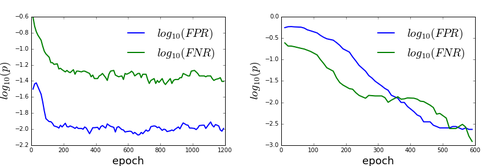}
    \caption{\textbf{ False positive and negative ratios using multiple master atomizations.} Left: The logarithm of the false positive and false negative ratios for digit \enquote{5} in dimension $28 \times 28$ requiring $5$ or more agreements out of $10$ master atomizations to classify an example as positive and fewer than $5$ as negative. False positive and negative ratios decrease fast at first but subsequent training using the same training examples does not improve, nor deteriorate, results. Right: Same criterion applied to the classification of even vs odd number of vertical bars in images of size $10 \times 10$ in the presence of 2.5\% noise. In this case new training examples are used at each epoch which results in monotonically decreasing false positive and negative ratios. }
 \label{fig:LogFPRFNR}
 \end{figure}

For the MNIST dataset, using $10$ master atomizations leads to a reduction of the overall error rate from $4.0\%$ to $1.07\%$. We asked if this is the best we can do. To answer this, in the following we investigate how much information can be extracted from the pinning terms. 

For the handwritten digits, a single atomization in the master has of the order of few hundred atoms. As a a consequence of cardinal minimization of the algebra, each negative example of a training batch contains typically all atoms except one. Cardinal minimization is finding the right atoms but is not optimizing error rate. Error rate decreases as a side effect of cardinal minimization. In fact there is no need other than reducing the size of the representation for requiring a single atom miss to separate negative from positive examples. 

To further reduce error rate, we may consider a separation of positive from negative examples using more than a single atom. Pinning terms are derived from atoms so atoms can be recovered from pinning terms. If we convert all pinning terms back into atoms, negative examples are separated from positive examples by many atom misses. However, many positive examples now also have a few atom misses. We thus proceed in the following way. Define \textit{misses cut-off} as the arbitrary maximum number of misses allowed for an example to be declared positive. For digit \enquote{0}, for example, we find an interval of \textit{misses cut-off} of $10-50$ with an error rate below $1\%$, with a minimum of $0.36\%$ error rate for $23$ misses. Digits differ in the optimal misses cut-off, with values from $13$ to $27$, but all have quite flat error rates in a wide interval. Different cut-offs could be defined to minimize error, false positive or false negative ratios. The best error rate obtained gives an overall $0.78\%$ for the $10$ digits. Error rates for all digits are given in the table of \textbf{Figure~\ref{fig:mnist3}} for the cut-offs that minimize error. 

This value of $0.78\%$ for the error rate is lower but similar to the $1.07\%$ error rate obtained using $10$ master atomizations. 

\begin{figure}[!htb]
%\centering
\hspace*{-1.4cm}  
\vspace*{-1.5cm} 
 \includegraphics[scale=0.7]{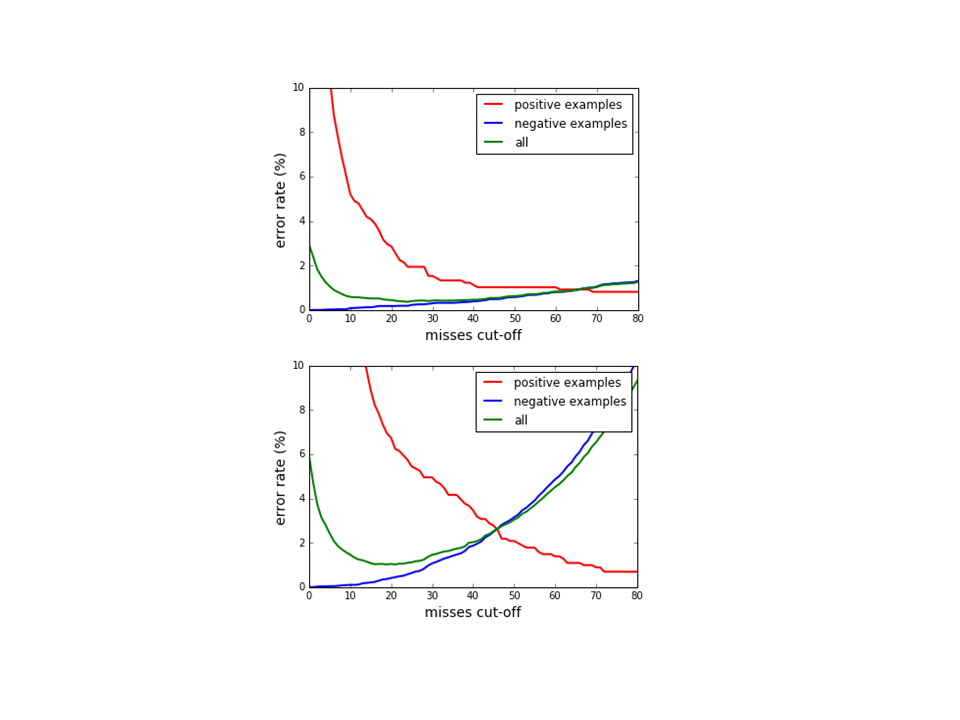}
    %\begin{table}[]
%\caption*{A: Single master atomization}
\centering
%\label{mnist_optimal}
\begin{tabular}{@{}lllll@{}} 
%\vspace*{-1cm}  
\toprule

%\multicolumn{2}{c}{Item} &            \\ \cmidrule(r){1-2}

Digit      & Error (\%) & Misses & FPR(\%)   & FNR (\%) \\ \midrule
0          &  0,36   &24     & 0,19    &  1,90   \\
1          &  0,28   &15    & 0,11     &  1,80   \\
2          &  0,75   &26     & 0,45      &  3,5    \\
3          &  1,083  &17     & 0,51     &  6,24    \\
4          &  0,841  &27       & 0,49     & 4,00   \\  
5          &  0,792  &27   & 0,28      &  5,40  \\
6          &  0,76   &13  & 0,20     &  5,80  \\
7          &  0,735  &20    & 0,33    &  4,38    \\
8          &  1,235  &22    & 0,75    &  5,60    \\
9          &  1,017  &19     & 0,36    &  6,93    \\
													\\
Average    & 0,78 &    & 0,37 & 4,55 \\ \bottomrule
\end{tabular}
%\end{table}
    
        \caption{\textbf{How pinning terms best distinguish positive from negative test images using different number of allowed misses (\textit{misses cut-off}).} Top: Error rate and false positive ratios for test images of digit \enquote{0} using all the atoms recovered from the pinning terms of the algebra. Errors are shown as a function the \textit{misses cut-off}. Middle: Same as top but for digit \enquote{9}. Bottom: For each digit, minimal error rate and corresponding false positive and negative ratios.} \label{fig:mnist3}
\end{figure}

\section{Solving the $N$-Queens Completion Problem}

So far we have used three supervised learning examples in which an algebra is trained to learn from data. Algebras can learn from examples but they can also incorporate formal relationships, for example symmetries or known constraints. They are also capable of learning in unsupervised manner.

In the following we study the $N$-Queens Completion problem as such a case in which we incorporate several relationships. In this problem, we fix $N$ queens to $N$ positions of a $M \times M$ chessboard, and we want the algebra to find how to add $M-N$ queens to the board so none of the $M$ queens attack each other.  In the following we detail how to embed this problem into the algebra.  

\subsection{Board description}

A simple and effective embedding uses of $2N^{2}$ constants to describe the board, two constants for each board square. A constant $Q_{xy}$ describes that board position $(x,y)$ contains a queen. A constant $E_{xy}$ describes that board position $(x,y)$ is empty. There is no need to distinguish white and black squares.  

A board or subset of the board is represented by a term $B$ that is an idempotent summation of some constants $Q_{xy}$ and $E_{xy}$.

\subsection{Attack rules}

To encode the queen attack rules we used an additional constant $U$. Let $A(x,y)$ be the set of board squares attacked by a queen at $(x,y)$, and let's agree that $A(x,y)$ does not include the square $(x,y)$. When avoiding attacks, represented by the presence of constant $U$, a queen at $(x,y)$ implies the presence of empty squares at each position in the set $A(x,y)$,
\begin{linenomath}
\begin{equation}
\forall x\forall y \forall i \forall j \,\left( (i,j) \in A(x,y) \, \Rightarrow \,  E_{ij} < U\odot Q_{xy}  \right).
\end{equation}
\end{linenomath}
By adding these rules to the training set, the idempotent summation of constant $U$ and a board $B$ (or a subset of a board) with one or more queens results in an extended board that has the empty square constants at positions attacked by the queens in $B$.  

\subsection{Rule to add queens}

The previous rule adds empty squares to a board subset. We can also write a rule that adds queens by extending the definition of constant $U$. When the term $B$ contains a subset of a board with a row or column of empty squares missing just one square, the summation of $U$ and $B$ completes the column or row by adding a queen, 
\begin{linenomath}
\begin{align}
&\forall x\forall y (Q_{xy} < U \odot_{i, i \neq x} E_{iy})   \\
&\forall x\forall y (Q_{xy} < U \odot_{j, j \neq y} E_{xj}), 
\end{align}
\end{linenomath}
where we used $\odot_{i, i \neq x}$ to represent the idempotent summation of all values of $i$ except $i = x$.   

\subsection{Definition of $R_{x}$ and $C_{y}$}

To place $M$ non-attacking queens on a $M \times M$ board, all rows and all columns should have a queen. We introduced an additional set of  $2N$ constants, $R_{x}$ and $C_{y}$, to require that a queen must be present at every row $x$ and at every column $y$ of the board $B$.  
We define these two constants with the help of the order relations
\begin{linenomath}
\begin{equation}
\forall x\forall y (R_{x} \odot C_{y} < Q_{xy}),
\end{equation}
\end{linenomath}
and the negative relations
\begin{linenomath}
\begin{equation}
\forall x  (R_{x} \not < (\odot_{ij, i \neq x } Q_{ij}) \odot (\odot_{ij} E_{ij})), 
\end{equation}
\end{linenomath}
and
\begin{linenomath}
\begin{equation}
\forall y  (C_{x} \not < (\odot_{ij, j \neq y } Q_{ij}) \odot (\odot_{ij} E_{ij})), 
\end{equation}
\end{linenomath}
where idempotent summations run along all possible values of indexes $i$ and $j$ and $\odot_{ij, i \neq x}$ represents a summation for all board positions except those with row equal to $x$.

\subsection{Independence rules}

Now let's encode the independence of board square constants. No term $B$ representing a complete or partial chessboard should contain a queen at $(x,y)$, $Q_{xy} \not < B$, unless $Q_{xy}$ is a component in the explicit definition of term $B$. To capture this we require that
\begin{linenomath}
\begin{equation}
\forall x\forall y  (Q_{xy} \not < (\odot_{ij, (i,j) \neq (x,y) } Q_{ij}) \odot (\odot_{ij} E_{ij})),  
\end{equation}
\end{linenomath}
and analogously for empty spaces,
\begin{linenomath}
\begin{equation}
\forall x\forall y  (E_{xy} \not < (\odot_{ij, (i,j) \neq (x,y) } E_{ij}) \odot (\odot_{ij} Q_{ij})). 
\end{equation} 
\end{linenomath}
If $E_{xy}$ is not one of the components defining term $B$ then it follows that $B < (\odot_{ij, (i,j) \neq (x,y) } E_{ij}) \odot (\odot_{ij} Q_{ij})$ and than $E_{xy} < B$ contradicts the independence relation above. Note that if the rule to add a queen and the attack rules were defined without using the extra constant $U$ they would contradict the independence rules.

Similar relations can be written for constants $R_{x}$ and $C_{y}$ representing any queen in a row $x$ or in a column $y$ as
\begin{linenomath}
\begin{align}
& \forall x (R_{x} \not < (\odot_{ij, i \neq x} Q_{ij}) \odot (\odot_{ij} E_{ij}))  \\ 
& \forall y (C_{y} \not < (\odot_{ij, j \neq y} Q_{ij}) \odot ( \odot_{ij} E_{ij}))
\end{align} 
\end{linenomath}
We can add additional independence order relations such as
\begin{linenomath}
\begin{align}
&\forall x\forall y  (Q_{xy} \not < U \odot E_{xy}), \\ &\forall x\forall y  (E_{xy} \not < U \odot Q_{xy}), \\
&\forall x\forall y  (Q_{xy} \not < U \odot_i R_{i} \odot_j C_{j}), \, \, \, \text{and} \\
&\forall x\forall y  (E_{xy} \not < U \odot_i R_{i} \odot_j C_{j}).
\end{align} 
\end{linenomath}
 
\subsection{Embedding an $N$ Queens Completion game}

We refer to all the above rules as \enquote{the rule set}. Now we are going to add additional relations to encode a particular $N$-completion game. We use a constant $S$ to represent the solution we are looking for.   We want to find a completion for a board already with, say, two queens fixed at positions $(p,q)$ and $(r,s)$. To require a solution with the two fixed queens we add the relation:
\begin{linenomath}
\begin{align}
&Q_{pq} \odot Q_{rs} < S.
\end{align}
\end{linenomath}
The solution $S$ should be a particular configuration of queens and empty positions on a board, and must therefore be contained in the set of all possible configurations, or equivalently in the idempotent summation of all board squares both empty and with a queen,
\begin{linenomath}
\begin{equation}
S <  \odot_{ij} \, (E_{ij} \odot Q_{ij}).
\end{equation}
\end{linenomath}
However, no board position can be simultaneously empty and with a queen, that is, it cannot contain both $E_{xy}$ and  $Q_{xy}$, 
\begin{linenomath}
\begin{equation}
\forall x\forall y  (E_{xy} \odot Q_{xy}) \not  <  U \odot  S.
\end{equation}
\end{linenomath}
We also know that $M$ non-attacking queens on a $M \times M$ chessboard must occupy each row, 
\begin{linenomath}
\begin{align}
&\forall x  (R_{x} < S),
\end{align}
\end{linenomath}
and also each column,
\begin{linenomath}
\begin{align}
&\forall y (C_{y} < S). 
\end{align}
\end{linenomath}

\subsection{Solving the $2$-blocked $8 \times 8$ completion problem}

The \enquote{rule set} and the rules to \enquote{embed an $N$-queen completion game} are the complete set $R$ of input order relations. We enforce $R$ at each epoch. {\bf{Figure}~\ref{fig:chessboard8}} gives the results of several epochs of algebraic learning. We chose the initial state to be two queens in positions \textbf{b4} and \textbf{d5} (\textbf{Figure~\ref{fig:chessboard8}}, queens in blue). After the first run of the Sparse Crossing algorithm, we get an incomplete board (\textbf{Figure~\ref{fig:chessboard8}}, epoch 1). The board is plotted by querying at each board square if relation $Q_{xy}<S$ or relation $E_{xy}<S$ is satisfied. When we find that neither of the two relations are satisfied we add a question mark to that position in \textbf{Figure~\ref{fig:chessboard8}}. In this first epoch, all positions except those attacked by the two initial queens are in question mark.   

The second epoch has some pinning terms and pinning relations defined from the atoms generated in the first epoch.  We then find a new atomization for the dual that also satisfies the pinning relations generated in the first epoch. $R$ is already satisfied in the master before the second epoch starts, however, the trace constraints are not because the atomization of the dual has changed. Enforcing trace constraints introduce new atoms in the master that have to be crossed. As a result we get a different atomization for the solution $S$ and the other constants. This is not very different from what we did for the handwriten character recognition. Using multiple master atomizations, extracts more information from the pinning terms. With each atomization of the master, the algebra looks at the solution from a \enquote{different angle} and it can learn from it by creating new pinning terms and relations.

Eventually, at epoch $12$, a complete board is found. It is remarkable that a solution is found without searching for a particular configuration. Following the chessboards generated at each epoch, the solution seems to appear \enquote{out of the blue} in the sense of not showing any intermediate boards. The algebra is learning the structure of the search space. When enough pinning terms are added, the algebra can produce board solutions. 

The approach seems to benefit form inserting idle cycles (epochs) for which no fixed queens are set and only the order relations of the rule set are enforced. In the problem of {\bf{Figure}~\ref{fig:chessboard8}}, iddle cycles were used in epochs 8,9 and 10 and 19, 20 and 21. In this way, we could find the two different completions compatible with the initial queens in epochs 12 and 28.

\begin{figure}
\centering
\hspace*{-1cm} 
\includegraphics[scale=0.8]{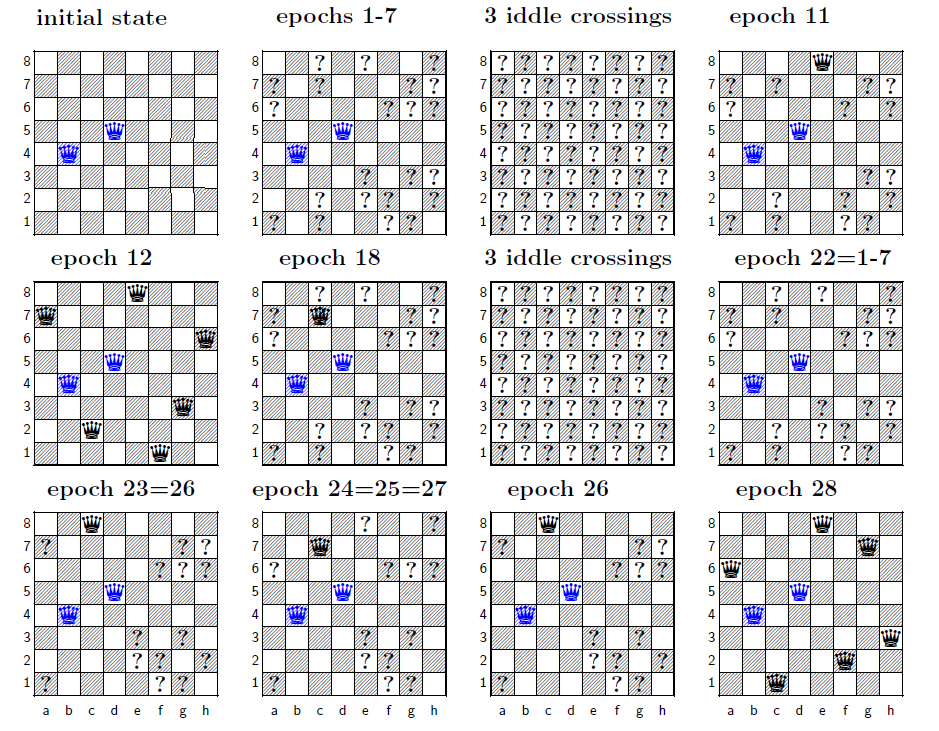} \caption{\textbf{Algebraic learning of the 2-blocked $8 \times 8$ queens problem.} Chessboards at different learning epochs arranged in increasing epoch order. They are generated by querying the algebra at each board position for presence of queen, presence of empty square or absence of both, marked as '?'.}
\label{fig:chessboard8}
\end{figure}

\begin{figure}[t]
\begin{center}
 \includegraphics[scale=0.8]{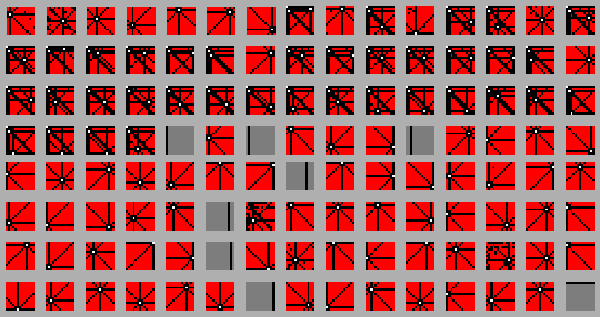}
\end{center}
    \caption{\textbf{A subset of atoms for the 2-blocked $13 \times 13$ queens problem.} Master atoms for epoch $42$, with white corresponding to empty, black to queen, red both, and gray none. }
    \label{fig:chessboard13}
 \end{figure}

\subsection{Algebraic learning in larger chessboards}

The straightforward approach of the previous section works well for $8 \times 8$ or larger boards, like $13 \times 13$. At least for low dimensions, algebraic learning is capable of finding complete boards \emph{at once}, without apprently following intermediate steps.

We found, not surprisingly, that building solutions step by step works better for larger boards. At each epoch we inserted a queen at some random but legal position. The algebra had no instructions regarding what to do with the inserted queens. It can keep them or eliminate them. A queen is inserted in the board by adding the relation $Q_{xy} < S$ in just one single epoch. In following epochs it is up to the algebra to keep it there or not.  

We studied the case of a $17 \times 17$ chessboard. When queens are located at legal but random positions the usual outcome is a board with fewer than $17$ queens. Adding more queens is not possible without attacking others. The situation is different for algebraic learning. 

Consider the following experiment. We ran 33 attempts at finding a complete board for 17 queens with one blocked. Each attempt consisted of 20 epochs, 17 epochs were a legal queen is added and 3 additional idle epochs. The first attempt starts with a blocked queen at position \textbf{c10} and in every epoch a legal queen is added, if possible. For the first 26 attempts no complete board was achieved. 

At attempt 27 a complete board was found ({\bf{Figure}~\ref{fig:chessboard17}}, left). In this figure, the initial blocked queen is in blue at position \textbf{c10}, in red the queens we randomly introduced and in black the queens the algebra found. 

The simplified dynamics of this experiment can be summarized in the following way. The board starts with a blocked queen (1 queen on board), algebraic learning rans for an epoch and resulted a board with two other queens (3), we then added a queen randomly to a legal position (4). A new learning epoch kept this queen but eliminated the previous ones (2), then five queens were inserted randomly one by one and kept (7 queens on board). Another queen was then inserted and the algebra added a new one (9 on board), then again three more queens were inserted (12). Finally, one queen was randomly inserted and 4 more created by the algebra at once, making a total of 17 and the board was completed. 

At attempt 33 another complete board is found ({\bf{Figure}~\ref{fig:chessboard17}}, right). The simplified dynamics at this attempt was: starting with a blocked queen at \textbf{c10} (1 queen on board), we added two queens (3 queens on board) and two were inserted by the algebra (5 on board), more queens were randomly inserted for 5 steps (10 on board) and then 7 appeared at once, added by the algebra to a total of 17 legal queens.

Many atoms obtained resemble legal or almost legal board subsets. Most of these atoms correspond to boards with a few queens, some a single one, but others look like boards with small groups and some with larger groups of queens ({\bf{Figure}~\ref{fig:chessboard17}}, bottom). Atoms look similar when no queens are manually inserted ({\bf{Figure}~\ref{fig:chessboard13}}).

We repeated the same experiment $10$ times with a $17 \times 17$ board. Each experiment consists on a number of attempts to produce a complete board, typically around $60$ (see {\bf{Table}~\ref{tab:queens_experiments}}). Each of these attempts consists in adding at most $17$ queens. We also tested that these results of algebraic learning cannot be explained by random placement of queens on legal positions. We compared the results of the $10$ experiments with a purely random placement of $17$ legal queens. The purely random case produces a board in an attempt with probability $p=0.008$. For each experiment we give in {\bf{Table}~\ref{tab:queens_experiments}} the probability to produce by chance a similar or better result. Overall, this probability is less than $p=7 \times 10^{-12}$. In contrast to the random case, algebraic learning seems in many cases to take a number of attempts to produce a full board (say epoch 24 in Experiment 1 or 37 in Experiment 4) and then can quickly produce more boards, sometimes 6, 7 and 8 complete boards are found in very few attempts. 

In some of the experiments the same full board configuration is found more than once (marked with an $r$ in the table). Since finding repeated boards by chance is unlikely this is probbaly due to the recall of previous attempts. To compute the p-values, repeated boards were not counted as valid.

\begin{figure}
\begin{center}
 \includegraphics[scale=0.83]{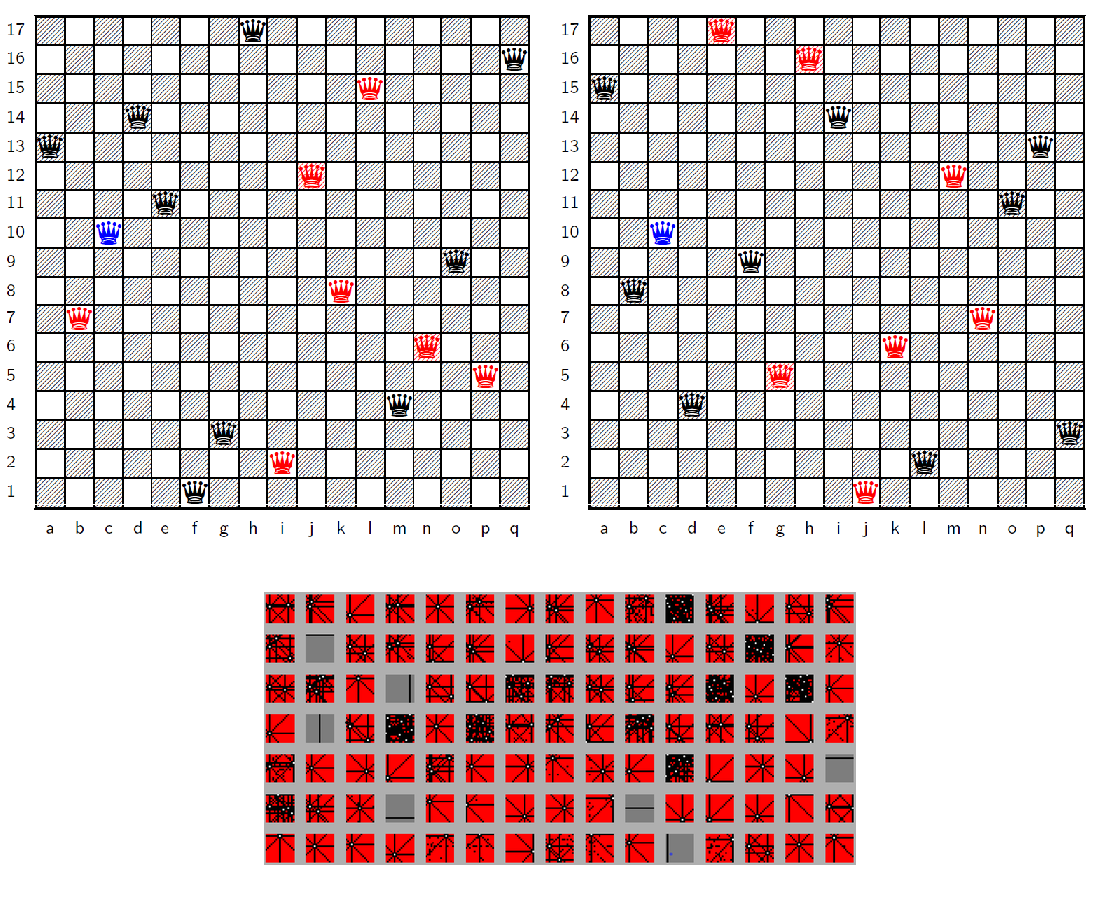}
\end{center}
    \caption{\textbf{Algebraic learning of the 1-blocked $17 \times 17$ queens problem.} Top: Two solutions found by algebraic learning. Starting from the blocked blue queen in each epoch a legal queen (red) is inserted. Inserted queens can be kept or discarded. In black queens placed by the algebra.  Bottom: A subset of atoms of $M$ for epoch 420. Each atom is represented by the constants that contain it, with white corresponding to the empty square constant, black to the queen, red both constants, and gray none.}
\label{fig:chessboard17}
\end{figure}

\begin{table}[b]
\begin{tabular}{@{}llllll@{}}
\toprule
Exp        & Epoch with full board              & Attempts & Full boards     & p-value \\ \midrule
 \\
1          &  24, 25, 26, 29, 32r, 39, 60, 67r         & 77       &  8 (2 repeated)  & $4\times 10^{-5}$  \\
2          &  19, 35                             & 51       &  2              & $0.06$       \\
3          &  39                                & 58       &  1              & $0.37$       \\
4          &  37, 40, 46, 49, 50r, 54                & 55       &  6 (1r)         & $8.5 \times 10^{-5}$ \\    5          &  None                              & 70       &  0              & $1$                   \\
6          &  9, 22                              & 66       &  2              & $0.1$    \\
7          &  18, 21, 26, 40, 44, 47r, 48r, 66r        & 79       &  8 (3r)         & $4.7\times 10^{-4}$  \\
8          &  34, 41, 52                          & 54       &  3              & $0.01$    \\
9          &  5, 6, 20, 41, 42r, 71, 74               & 77       &  7 (1r)          & $4.0 \times 10^{-5}$   \\
10         &  28, 34                             & 37      &  2  	          &$0.04$ \\ \bottomrule
\end{tabular}
\caption{\textbf{10 experiments in algebraic learning for the 1-blocked $17 \times 17$ queens problem}}\label{tab:queens_experiments}
\end{table}

\section{Discussion} \label{discussion}
 
We have shown how algebraic representations can be used to learn from data. The algebraic learning technique we propose is a parameter-free method. A small algebraic model grows out of operations on the data and it can generalize without overfitting. We used four examples to illustrate some of its properties. In the following we summarize the results obtained to point to open questions.

We used the simple case of learning whether an image contains a vertical bar to introduce the reader to Sparse Crossing. In this case the result of learning is a compression of examples into a few atoms that correctly classify the images. We also used this example as a simple case in which it is possible to obtain the definition of a class from the atoms, obtaining explicitly a definition of a bar in an image. This goes beyond the classification problem and we believe that further development of these techniques may lead to the ability to derive formal concepts from data. To understand the challenge, we suggest reading \textbf{Appendix \ref{perfectAtomizationAppendix}} as it deals with the inverse problem of predicting atoms from known formal descriptions.

The vertical bar problem was also used to demonstrate explicitly that the number of possible solutions of the system with low but non-zero error is astronomically large and that the algebra only needs to find one. This helps understanding why finding a generalizing model with algebraic learning is not as hard as finding a needle in a haystack. In addition, we show that these solutions can be small.

As algebraic learning is only aimed at producing small algebraic models we had to show that the compression of a large number of input constraints (training examples) into small representations translates into accuracy. We proved that an algebra picked at random among those that correctly model the training data has a test error that is inversely proportional to how much it compresses training examples into atoms. We have observed this inverse proportionality in all problems for which many training examples were available so error could be made arbitrarily small. We proved that for a randomly chosen model the dependence has the form $\epsilon=\log(3) \, C/2k$. We used the problem of distinguishing images with even or odd number of bars to show that algebraic learning using Sparse Crossing also obeys this theoretical result when error rates are low. 

We argued that the search for small generalizing models requires avoiding non-free models so increasing algebraic freedom should also be pursued when searching for models of small cardinal. If algebraic freedom is not sought, the models obtained are far from random and do not generalize.

We improved accuracy results using several master atomizations instead of a single one. The way algebraic learning works, the information learned is accumulated in pinning terms and pinning relations and not only in master atoms of a single batch. To better extract the information contained in the pinning relations we used 10 master atomizations. As different master atomizations express different pinning terms, using multiple atomizations stochastically samples the information gathered by pinning relations. We showed that using a majority voting of multiple atomizations gives a much higher accuracy for the problem of identifying hand-written digits. Each time the input data is represented in the algebra with a different atomization, additional information is extracted.

An alternative but not very different method would be to use many algebras in parallel. Indeed, one key advantage of algebras is that learning is crystalized into atoms and pinning relations, and these atoms and relations can be shared among algebras, which makes the system parallelizable at large scale. Learning can occur in disconnected algebras working in the same or different problems, in parallel, for as long as necessary before information is shared. 

Perhaps the more conceptually interesting example was how algebraic learning can solve the $N$-blocked $M \times M$ Queen Completion problem \cite{Gent2017}. At the technical level, it illustrates how algebraic learning can naturally incorporate any kind of extra relations, in this case teaching the system what the board, the legal moves and the goal of the game are. At the conceptual level, it stands out as an interesting case of unsupervised learning. In this case, algebraic learning learns the form of the search space within the given constraints and uses it to boost the search for a solution of the puzzle. For large chess boards, we accelerated the learning by using some stochastic input of legal queens. The method to find a solution is very different to others proposed before. There is no systematic search that guarantees a solution, there is no backtracking and the method is not specific for the task. Algebraic learning may be useful as a general purpose technique to explain to a machine which are the rules of a problem (a game) so the machine can learn to search for solutions compatible with the rules. These solutions are relevant as they correspond with small algebraic representations of the current game, its general rules and the previously gathered experience.

Algebraic Learning shows instances of human-to-algebra, algebra-to-algebra and algebra-to-human communication. Human-to-algebra communication was exemplified in how we taught the game rules of the $N$ Queens Completion problem to an algebra. Algebra-to-algebra communication takes place when algebras share pinning terms and relations. A simple example of algebra-to-human communication was given for an algebra conceptualizing that positive example images contained a vertical bar.

Our semilattices produced new atoms in learning, but a more general approach should also modify the constants, which should allow for an improved conversion of compression into accuracy. Note that test error is proportional to the number of constants $C$ in the formula that relates error to compression. While Sparse Crossing leaves the constants fixed the proportionality of error with $C$ indicates that further developments should involve learning new constants to reduce not only the number of atoms but also the number of constants they depend upon.

We have used semilattices because they have a rich structure \cite{Papert} despite their simplicity. They are simplest algebras with idempotent operators. Other idempotent algebras such as semilattices extended with unary operators may be relevant to machine learning. Unary operators can be used to easily extend the techniques in this paper to finitely-generated\cite{Burris} infinite models, which could be used to apply machine learning to more abstract domains.

While algebraic learning is not intended to be a model of a brain, there are some interesting parallels. First, atoms relate to constants by an OR operation, that parallels the activation of neurons by one (or a small subset) of its inputs. Second, algebraic learning uses a \enquote{dual} algebra that handles all the accumulated experience and a master algebra that conforms better to current inputs. This resembles the separation into working and long-term memory in the brain. The interactions between both algebras can be thought as feedback connectivity, also present in brains. Even more relevant, this interaction is used to make hypotheses, and a similar role has been proposed for feedbacks in the brain \cite{Bullier2001}. Third, identification of patterns in the cortex is fast compared to the firing speed of neurons \cite{Thorpe1996}. This hints to an important role of wide processing in the brain as that naturally produced by semilattice embeddings. Pattern identification in an algebra occurs with no more processing than a wide representation of the sensory information as a subset of learned atoms. Complex problems such as the Queen Completion problem can be solved despite the lack of sequential processing in layers.

These parallels make us think of the potential usefulness of extrapolating from Algebraic Learning to Artificial Neural Networks \cite{Bishop1995} and vice versa. We hope that our findings regarding the relationship between error and compression rate, as well as the role of balancing algebraic freedom with size minimization may have applications into neural networks, in particular to regularization or as an alternative deriving principle for neural processing different from error minimization. Also, combining or embedding algebras and neural networks might be useful to obtain the versatility of neural networks while having the ability of algebras to incorporate top-down information.

We have presented an approach to Machine Learning based on Abstract or Universal Algebra \cite{Burris}. While the relationship of Algebraic Learning with Universal Algebra is more direct, there are other areas in Mathematics with potential connexions. We mention the Minimum Description Length principle \cite{Grunwald}, Constraint Propagation Theory \cite{Marriott}, Compressive Sensing \cite{Donoho2006, Tao}, Formal Concept Analysis \cite{Skopljanac-Macina2014} or Ramsey Theory \cite{Todorcevic2010} as some candidate areas. 

\section*{Acknowledgements}
We acknowledge funding from Champalimaud Foundation (to G.G.d.P.).

\begin{appendices}

\section{Notation} \label{Notation}

We use ${\bf{C}}(S)$ and ${\bf{A}}(S)$ for the subset of constants and atoms respectively of a set $S$. We also use ${\bf{C}}(M)$ and ${\bf{A}}(M)$ for the constants or atoms respectively of a model $M$. The lower and upper segments of an element $x$ is defined as ${\bf{L}}(x) = \{y:y\leq x\}$ and ${\bf{U}}(x) = \{y: y > x\}$.  To distinguish the algebra from its graph we use the prefix ${\bf{G}}$, so the lower and upper segments for the graph are defined as: ${\bf{GL}}(x) = \{y: (y\rightarrow x) \vee (y = x)\}$ and ${\bf{GU}}(x) = \{y: x \rightarrow y\}$ assuming always the graphs are transitively closed. Finally, the superscript ${ \ }^{a}$ is used to denote the intersection with the atoms: ${{\bf{GL}}^{a}}(x) = {\bf{GL}}(x) \cap {\bf{A}}(M)$ assuming $x \in M$. If $x$ belongs to $M^{*}$ the intersection is with the atoms  ${\bf{A}}(M^{*})$.  We also use ${{\bf{L}}^{a}}(x) = {\bf{L}}(x) \cap {\bf{A}}(M)$. We also use the intersection with the constants ${{\bf{U}}^{c}}(x) = {\bf{U}}(x) \cap {\bf{C}}(M)$. 

The discriminant ${\bf{dis}}(a,b)$ is the set of atoms ${{\bf{GL}}^{a}}(a) \setminus {{\bf{GL}}^{a}}(b)$. Relation $a<b$ holds if and only if ${\bf{dis}}(a,b)$ is empty.

We say an atom $\phi$ is \emph{larger} than atom $\eta$ if for any constant $c$,  $(\eta < c) \Rightarrow (\phi < c)$.   

\section{Theorems} \label{theorems}

\begin{theorem} \label{crossingTheorem}
Let elements $a$ and $b$ satisfy $[b] \to [a]$.  Let $\Phi$ be the set of atoms involved in the full crossing of $a$ in $b$, e.g.  $\phi \in \Phi$ is edged to $a$ or to $b$ but not to both. Assume $\phi$ becomes $\phi =\odot _j \varphi_j$ as a result of the crossing. Then
\begin{linenomath}
\begin{equation*} 
 \forall \phi (\phi \in \Phi)   (  {\bf{Tr}}(\phi) = {\bf{Tr}}(\odot _j \varphi_j)   )   \,\,\,\,    \Leftrightarrow   \,\,\,\,    \{    {\bf{Tr}}(b) \subset {\bf{Tr}}(a)   \}  
\end{equation*} 
\end{linenomath}
\end{theorem}
\begin{proof}
Before crossing, $\phi$ is a minima and by definition ${{\bf{Tr}}(\phi) \equiv \bf{GL^a}}([\phi])$.  After the crossing the new minima are  $\varphi _j$  and then  ${\bf{Tr}}(\odot _j \varphi _j) \equiv \cap _j {\bf{GL}}^{\bf{a}} ([\varphi _j] )$. The left side of the equivalence in the theorem then becomes $\forall \phi (\phi \in \Phi)   (  {\bf{GL^a}}([\phi]) = \cap _j {\bf{GL^a}}([\varphi _j])  )$.

Since we assumed that $\phi$ is not in both $a$ and $b$, we have either $(\phi < a) \wedge \neg(\phi < b)$  or $\neg (\phi < a) \wedge (\phi < b)$. 

An atom $\phi$ initially in $b$ always preserves its trace due to the corresponding element $\phi '$  edged only to $\phi$, so   $ {\bf{Tr}}(\phi)= {\bf{Tr}}(\phi')$.  Therefore, 
$
\forall \phi (\phi \in {\bf{GL^a}}(b) ) (  {\bf{GL^a}}([\phi]) = \cap _j {\bf{GL^a}}([\varphi _j])  )  )
$
is true for any crossing.  The case $ (\phi < a) \wedge \neg(\phi < b)$ remains. We have to show
\begin{linenomath}
\begin{equation*}
\forall \phi (\phi \in {\bf{GL^a}}(a) ) \,\, \{  {\bf{GL^a}}([\phi]) = \cap _j {\bf{GL^a}}([\varphi_j])  )  \}  \,\,\,\,   \Leftrightarrow   \,\,\,\,    \{    {\bf{Tr}}(b) \subset {\bf{Tr}}(a)   \} .
\end{equation*}
\end{linenomath}
Assume that, before crossing,  $b$ was atomized as $b = \odot _j \varepsilon _j$. For each $\varepsilon _j$, a new atom $\varphi_j$ is created and edges 
$
(\varphi_j \to \phi ) \wedge (\varphi _j \to \varepsilon _j )
$
are appended to the graph. New edges are appended to the graph of $M^{*}$: $([\phi ] \to [\varphi_j ]) \wedge ([\varepsilon _j ] \to [\varphi _j ])
$. Since no other elements are edged to $[\varphi_j]$, then
\begin{linenomath}
\begin{equation*}
{\bf{GL}}^{\bf{a}} ([\varphi_j ]) = {\bf{GL}}^{\bf{a}} ([\phi ]) \cup {\bf{GL}}^{\bf{a}} ([\varepsilon _j ]),
\end{equation*}
\end{linenomath}
giving
\begin{linenomath}
\begin{equation*}
\begin{aligned}
{\bf{Tr}}( \odot _j \varphi _j ) &= \cap _j {\bf{GL}}^{\bf{a}} ([\varphi_j ]) \\
&= \cap _j \{ {\bf{GL}}^{\bf{a}} ([\phi ]) \cup {\bf{GL}}^{\bf{a}} ([\varepsilon _j ])\} \\
&= {\bf{GL}}^{\bf{a}} ([\phi ]) \cup \{ \cap _j {\bf{GL}}^{\bf{a}} ([\varepsilon _j ])\} \\
&= {\bf{GL}}^{\bf{a}} ([\phi ]) \cup {\bf{Tr}}(b),
\end{aligned}
\end{equation*}
\end{linenomath}
which says that when $\phi$ is crossed into $b$  the trace of $\phi$ gains the set ${\bf{Tr}}(b)$.  ${\bf{Tr}}(\phi)$  remains invariant if and only if it contained ${\bf{Tr}}(b)$ before crossing. Therefore, if all atoms of $a$ remain trace-invariant, then ${\bf{Tr}}(b) \subset {\bf{Tr}}(a)$. Conversely, if 
${\bf{Tr}}(b) \subset {\bf{Tr}}(a)$, each ${\bf{Tr}}(\phi)$ for any $\phi<a$ should contain the trace ${\bf{Tr}}(b)$, and therefore remain trace-invariant. 
\end{proof}

\bigskip

\begin{theorem}
Let term $\, k =  \odot _{i} c_i \,$ with component constants $c_i$. For any atom $\xi$ in $M^{*}$ it holds that $(\xi <[k])\, \Leftrightarrow \neg(\,\phi _\xi   < k)$ if and only if $k$ satisfies $\,{\bf{GL}}^{\bf{a}} ([k]) =  \cap _i {\bf{GL}}^{\bf{a}} ([c_i ])$. See {\bf{Section \ref{memorizingGeneralizing}}} for a definition of $\phi _\xi$.
\label{termInversion}
\end{theorem}

\begin{proof}
We built atom $\phi _\xi$ to satisfy for each constant $c$ of $M$ the relation $\neg  (\,\phi _\xi   < c) \Leftrightarrow (\,\xi   < [c] )$. For a term $k$:
\begin{linenomath}
\begin{equation*} 
\neg (\,\phi _\xi   < k) \Leftrightarrow \forall i\neg (\,\phi _\xi   < c_i ) \Leftrightarrow \forall i(\,\xi  < [c_i ]) \Leftrightarrow \forall i(\,\xi  \in {\bf{GL}}^{\bf{a}} ([c_i ])) \Leftrightarrow \xi  \in  \cap _i {\bf{GL}} ([c_i ]).
\end{equation*} 
\end{linenomath}
$\,{\bf{GL}}^{\bf{a}} ([k]) =  \cap _i {\bf{GL}}^{\bf{a}} ([c_i ])$  is equivalent to  $\forall \xi \,  \{\xi \in  \cap _i {\bf{GL}}^{\bf{a}} ([c_i ]) \Leftrightarrow  \xi  < [k]\}$, so it follows that 
\begin{linenomath}
\begin{equation*} 
\{{\bf{GL}}^{\bf{a}} ([k]) =  \cap _i {\bf{GL}}^{\bf{a}} ([c_i ])\} \Leftrightarrow \forall \xi  \, \{ \neg (\,\phi _\xi   < k) \Leftrightarrow  \xi  < [k]\},
\end{equation*} 
\end{linenomath}
which completes the proof. 
\end{proof}

\bigskip

\begin{theorem}  \label{pinningFree} 
Let $N$ and $M$ be two semilattices over the same set $C$ of constants and assume $M$ is atomized. If $N$ satisfies the pinning relations $R_p(M)$ then for any pair $a$, $b$ of terms over $C$ it holds: 

   i)  $N \models (a<b) \,\, \Rightarrow \,\, M \models (a<b)$. 

   ii) The set $R_p(M)$ captures all negative order relations of $M$. 

   iii) For each $\phi \in M$ there is at least one atom $\eta \in N$ such that $\phi$ is as large or larger than $\eta$. 

   iv) $N$ is as free or freer than $M$.
\end{theorem}
\begin{proof}
 Let $u$ and $v$ be two terms, and assume $M$ satisfies $\neg (u < v)$. There should be an atom in the discriminant $\phi \in {\bf{dis}}_{M}(u,v) \subset M$ and a component constant $c \in C$ of $u$ such that $\phi < c$ and $(c < u) \wedge (v < T_{\phi})$ where $T_{\phi}$ is the pinning term of $\phi$ (see {\bf{Section \ref{batchTraining}}} for a definition of $T_{\phi}$). This is not only true for $M$ it is also true for the term algebra over $C$ and therfore for any model, i.e. it is also satisfied by $N$. In addtion, $\phi < c$ implies $\neg(c<T_{\phi}) \in R_p(M)$ and because we have assumed $N$ satisfies $R_p(M)$ then $N$ also models $\neg(c<T_{\phi})$. Therefore $N$ satisfies $(c < u) \wedge (v < T_{\phi}) \wedge \neg(c<T_{\phi})$ which implies $\neg (u < v)$ and it follows that if $\neg (u < v)$ is true for $M$ is also true for $N$ which proves iv and also proves that $R_p(M)$ captures all negative relations of $M$. By negating both sides of this implication we get the equivalent $N \models (a<b) \,\, \Rightarrow \,\, M \models (a<b)$.  

Assume $N$ is atomized. To prove the third claim select any pinning relation of atom $\phi \in M$, e.g. $\neg(c < T_{\phi})$ where $c$ is some constant. We have assumed that $N \models \neg(c < T_{\phi})$ so there is some atom $\eta \in {\bf{dis}}_{N}(c, T_{\phi}) \subset N$, which implies $\neg(\eta <  T_{\phi})$ and inmediatelly follows $T_{\phi} \leq T_{\eta}$ and $\phi$ is as large or larger than $\eta$, i.e. for each constant $d$ such $\eta < d$ we have $\phi < d$.  
\end{proof}

\bigskip

\begin{theorem} \label{positiveEntail}
Assume $\,\neg p \wedge R \Rightarrow q$, where $p$ and $q$ are two positive order relations, $\,\neg p$ is a negative order relation, and $R$ is a set of positive and negative order relations. Then $\, R \Rightarrow q$. 
\end{theorem}
\begin{proof} 
Without loss of generality we may assume that $\,\neg p \wedge R \wedge q\,$ has a model $M_1$. The hypothesis requires $\,\neg p \wedge R \wedge \neg q\,$ has no model. Either $\, R \wedge \neg q\,$ has a model $M_2$, or $R$ alone implies $q$. Assume $M_2$ exists. We can always atomize both models with two disjoint atom sets, one set atomizing $M_1$ and the other $M_2$. Make a new model $M_3$ atomized by the union of the atoms in both models and defined by $\,{\bf{L}}_{M_3 }^{\bf{a}} (c) = {\bf{L}}_{M_1 }^{\bf{a}} (c) \cup {\bf{L}}_{M_2 }^{\bf{a}} (c)$ for each constant $c$. Immediately follows that $M_3$ is a model that satisfies $R$ and all the negative relations of $M_1$ and $M_2$. In fact $M_3 \models \neg p \wedge R \wedge \neg q\,$ contradicting $\,\neg p \wedge R \Rightarrow q$. Therefore $M_2$ does not exist and $\, R \Rightarrow q.$ \end{proof}

\bigskip

\begin{theorem} \label{redundantAtom}
Let atom $\phi$ be \emph{redundant} in model $M$ if for each constant $c$ such that $\phi < c$ there is at least one atom $\eta < c$ in $M$ such that $\phi$ is larger than $\eta$. An atom can be eliminated without altering $M$ if and only if it is redundant.
\end{theorem}
\begin{proof}
Let $R^{+}$ be the set of all positive relations satisfied by the constants and terms of $M$. Since positive relations do not become negative when atoms are eliminated, taking out $\phi$ from $M$ produces a model $N$ of $R^{+}$. 

To prove that a redundant atom can be eliminated let $a$ and $b$ be a pair of elements (constants or terms, not atoms) and $\neg(a<b)$ a negative relation satisfied by $M$ and discriminated by a redundant atom $\phi < c \leq a$ where $c$ is some constant. There is an atom $\eta < c$ in $M$ such that $\phi$ is larger than $\eta$. Suppose $\eta < b$. There is a constant $e$ such that $\eta < e \leq b$. Because $\phi$ is larger, $\phi < e \leq b$ contradicting our assumption that $\phi \in {{\bf{dis}}_M}(a, b)$. We have proved that $N \models \neg(\eta < b)$ so any negative relation of $M$ is also satisfied by $N$. If $N$ models the same positive and negative relations than $M$ then the subalgebras of $M$ and $N$ spawned by constants and terms are isomorphic.

Conversely, assume atom $\phi$ can be eliminated without altering $M$. For each constant $c$ such $\phi < c$ it holds $\phi \in {{\bf{dis}}_M}(c  < T_{\phi})$ where $T_{\phi}$ is the pinning term of $\phi$ (see {\bf{Section \ref{batchTraining}}} for a definition of $T_{\phi}$). If $\phi$ can be eliminated there should be some other atom $\eta_c < c$ discriminating $c \not< T_{\phi}$ which implies $T_{\phi} \leq T_{\eta_c}$ and $\phi$ is as large or larger than $\eta_c$. Since for each constant such $\phi < c$ there is an $\eta_c \in M$, $\phi$ is redundant.   \end{proof}	

\bigskip

\begin{theorem} \label{crossingFree}
Let $M_i$ be a model, $d$ and $e$ elements of $M_i$ such $\neg(d  < e)$ and $T^{+}(M_i)$ the set of all positive relations between terms formed with the constants of $M_i$. Let model $M_f$ be the result of enforcing relation $d  < e$ using Full or Sparse Crossing.

i)	$M_f$ is strictly less free than $M_i$, i.e. if $M_f  \models \neg(a < b)$  then $M_i \models \neg(a < b)$. 

ii)	$M_f  \models (a < b)$ if and only if  $\,T^{+}(M_i) \cup (d < e)  \Rightarrow (a < b)$ in case \emph{full crossing} is used.

\end{theorem}

\begin{proof}
Suppose $a < b$ is true before full crossing. The atoms of $a$ are a subset of the atoms of $b$ so any replacement of atoms for others affects both $a$ and $b$ and cannot introduce discriminating atoms. Hence, all positive relations of $M_i$ are true after crossing and $M_f$ is as free or less free than $M_i$. In addition, $d  < e$  is true after crossing and false before which proves that $M_f$ is strictly less free than $M_i$ and proves claim $i$.

Suppose $a < b$ is false before full crossing but it turns true after. Let $\phi \in {{\bf{dis}}_i}(a, b)$ a discriminating atom for this relation.  If atom $\phi$ is no longer discriminat in $M_f$ is because the atoms at $\phi$\textsc{\char13}s row in the crossing matrix are also edged to $b$. This can only occur if $e < b$ in $M_i$.  In addition, all discriminating atoms have been transformed by crossing which means that they were also atoms of $d$.  We have ${{\bf{dis}}_i}(a, b) \subset {\bf{L}^a}(d)$ which proves that $M_i \models (a \odot d < b \odot d)$. Model $M_i$ satisfies:
\begin{linenomath}
\begin{equation*} 
 (a \odot d < b \odot d)  \cup  (e < b)  \in  T^{+}(M_i).
\end{equation*} 
\end{linenomath}
Together with $d < e$ these relations imply:
\begin{linenomath}
\begin{equation*}
(a \odot d < b \odot d)  \wedge  (e < b) \wedge  (d < e)  \Rightarrow  (a < b),
\end{equation*} 
\end{linenomath}
which proves ii.   \end{proof}

\bigskip

\begin{theorem} \label{traceConsistence}

The trace constraints can be enforced using algorithms \ref{negativeTrace} and \ref{positiveTrace} if the relation set $R$ is consistent.

\end{theorem}
\begin{proof}

By adding a new atom to a constant $c \in M$ and only to this constant it is always possible to make ${\bf{Tr}}(c) =  {\bf{GL}^a}([c])$. In the same way, by adding new atoms to the component constants of a term $k$ (one new atom per constant) it is possible to enforce ${\bf{Tr}}(k) = {\bf{GL}^a}([k])$ unless there is an atom $\zeta \in M^{*}$ in the lower segment of all the duals of the component constants of $k$. In such case a new edge $\zeta \rightarrow [k]$ should be added to the graph of $M^{*}$, which we do while enforcing positive trance constraints, and obtain ${\bf{Tr}}(k) =  {\bf{GL}^a}([k])$. Therefore, if $x$ is a constant or term of $M$ we can make ${\bf{Tr}}(x) = {\bf{GL}^a}([x])$ by adding atoms to $M$ and edges to $M^{*}$.

We want to enforce trace constraints for positive relations $(d < e) \in R^{+}$ and negative relations $\neg(a < b) \in R^{-}$. By adding new atoms to $M$ and edges to $M^{*}$ we can enforce the trace constraints ${\bf{Tr}}(e) \subset {\bf{Tr}}(d)$ and ${\bf{Tr}}(b) \not\subset {\bf{Tr}}(a)$ if we can enforce the simpler constraints ${{\bf{GL}}^a}([e]) \subset {\bf{Tr}}(d)$ and ${\bf{Tr}}(b) \not\subset {\bf{GL}^a}([a])$. 

\textbf{Algorithms \ref{negativeTrace}} and \textbf{\ref{positiveTrace}} add new atoms to some constants of $M$ and edges to some atoms of $M^{*}$. These constants and atoms existed before the algorithms are applied. Constants of $M$ and atoms of $M^{*}$ are finite and adding more than one new atom under a constant of $M$ and only under this constant has no effect in the traces or any other algebraically meaningful property. The same is true for the edges added to initially existing atoms of $M^{*}$. At some finite time it is possible to transform the original constraints into the simpler constraints which may or may not happen while running the algorithms but it can always happen, if needed, to enforce the trace constraints.

Enforcing negative trace constraints is carried out by adding new atoms to $M^{*}$. Adding new atoms can violate already holding positive trace constraints and fixing these imply adding edges to $M^{*}$ that can violate other negative trace constraints and so on. We are about to see that this process ends if it is possible to enforce the dual relations of $R$ in $M^{*}$. 

Assume that it is possible to enforce the duals of the relations of $R$, i.e. to enforce $[e] < [d]$ for $(d < e) \in R^{+}$ and $\neg([b] < [a])$ for $\neg(a < b) \in R^{-}$. Then we can enforce the positive constraints ${{\bf{GL}}^a}([e]) \subset {\bf{Tr}}(d)$ because it is always true ${{\bf{GL}}^a}([d]) \subset {\bf{Tr}}(d)$ and, for negative constraints, ${\bf{Tr}}(b) \not\subset {{\bf{GL}}^a}([a])$ follows from ${{\bf{GL}}^a}([b]) \subset {\bf{Tr}}(b)$ and $\neg([b] < [a])$. This proves that by adding atoms to $M$ and edges to $M^{*}$ we can enforce the trace constraints if it is possible to enforce the dual relations of $R$, which we can always do unless $R$ is inconsistent.\end{proof}

%%% optional fonts and color configuration
\SetAlFnt{\sffamily}
\renewcommand\ArgSty{\normalfont\sffamily}
\renewcommand\KwSty[1]{\textnormal{\textbf{\sffamily#1}}\unskip}
\SetAlCapFnt{\normalfont\sffamily\large}
\renewcommand\AlCapNameFnt{\sffamily\large}

\newpage

\section{Algorithms}
\label{algorithms}

%%% vertical rules in cyan color
\makeatletter
\renewcommand{\algocf@Vline}[1]{%     no vskip in between boxes but a strut to separate them, 
  \strut\par\nointerlineskip% then interblock space stay the same whatever is inside it
  \algocf@push{\skiprule}%        move to the right before the vertical rule
  \hbox{\bgroup\color{cyan}\vrule\egroup%
    \vtop{\algocf@push{\skiptext}%move the right after the rule
      \vtop{\algocf@addskiptotal #1}\bgroup\color{cyan}\Hlne\egroup}}\vskip\skiphlne% inside the block
  \algocf@pop{\skiprule}%\algocf@subskiptotal% restore indentation
  \nointerlineskip}% no vskip after
\renewcommand{\algocf@Vsline}[1]{%    no vskip in between boxes but a strut to separate them, 
  \strut\par\nointerlineskip% then interblock space stay the same whatever is inside it
  \algocf@bblockcode%
  \algocf@push{\skiprule}%        move to the right before the vertical rule
  \hbox{\bgroup\color{cyan}\vrule\egroup%               the vertical rule
    \vtop{\algocf@push{\skiptext}%move the right after the rule
      \vtop{\algocf@addskiptotal #1}}}% inside the block
  \algocf@pop{\skiprule}% restore indentation
  \algocf@eblockcode%
}
\makeatother
%%% end of optional fonts and color configuration

\SetKwProg{Fn}{Function}{}{}
\SetKwRepeat{Do}{do}{while}%

\begin{linenomath}
\begin{algorithm}
\ForEach{$(a \not< b) \in R^{-}$}{
  \If{  ${\bf{Tr}}(b) \subset {\bf{Tr}}(a)$  }{
    \Do{$ c \equal \emptyset $}{
  		 $c = findStronglyDiscriminantConstant(a, b)$\;
   
  		\If{$ c \equal \emptyset $}{
  
       		choose $h \in {\bf{C}}(M^{*})$ so $h \in  {\bf{GL}}^{c}([b]) \backslash  {\bf{GL}}([a])$\;
      		add new atom $\zeta$ to $M^{*}$ and edge $\zeta \rightarrow h$\;
   		}
     }
  	 add new atom $\phi$ to $M$ and edge $\phi \rightarrow c$\;
 } } 

\Fn{findStronglyDiscriminantConstant(a, b)}{
 calculate the set $ {\bf{\Omega}}(a) \equiv \{ [c]: c \in  {\bf{GL}}(a)  \cap  {\bf{C}}(M)  \} $\;
 initialize $U \equiv {\bf{Tr}}(b)$\;
 \While{$U \neq \emptyset$ }{
    choose atom $\zeta \in U$ and remove it from U\;
 	\If{  ${\bf{\Omega}}(a) \backslash {\bf{GU}}(\zeta ) $ not empty  }{
        choose $[c] \in {\bf{\Omega}}(a) \backslash {\bf{GU}}(\zeta ) $\;
    	\Return  c\;
   	}
 }
 \Return  $\emptyset$\;
}
\caption{enforce negative trace constraints}
\label{negativeTrace}
\end{algorithm}
\end{linenomath}

\newpage

\begin{linenomath}
\begin{algorithm}
\ForEach{$(d < e) \in R^{+}$}{
  \While{  ${\bf{Tr}}(e) \not\subset {\bf{Tr}}(d)$  }{
  	choose an atom $\zeta \in {\bf{Tr}}(e) \backslash {\bf{Tr}}(d)$ at random\;
    calculate ${\bf{\Gamma}}(\zeta, e)  \equiv \{ c\in {\bf{GL}}(e) \cap {\bf{C}}(M) :  \zeta \not\in \ {\bf{GL}}([c])  \}$\;
 	\eIf{ ${\bf{\Gamma}}(\zeta, e) = \emptyset $ }{
    	add edge $\zeta \rightarrow [d]$\;
   	}{
		choose $c \in {\bf{\Gamma}}(\zeta, e)$ at random\; 
        add new atom $\phi$ to $M$ and edge $\phi \rightarrow c$\;
    }
   
 } } 
\caption{enforce positive trace constraints}
\label{positiveTrace}
\end{algorithm}
\end{linenomath}

\begin{linenomath}
\begin{algorithm} 
  calculate $A \equiv {\bf{dis}}(a, b)  \equiv {{\bf{GL}}^a}(a) \backslash {\bf{GL}}(b)$\;
  \ForEach{$\phi \in A$}{
	initialize sets $U \equiv \emptyset$,  $B \equiv {{\bf{GL}}^a}(b)$ and $\Delta \equiv {{\bf{A}}(M^*)} \backslash {\bf{GL}}([\phi])$\;
    \Do{$\Delta \neq \emptyset$}{
  	   choose an atom $\epsilon \in B$\ at random\;
       calculate $\Delta' \equiv \Delta \cap {\bf{GL}}([\epsilon])$\;

 	   \If{$\Delta' \neq \Delta$ or $\Delta = \emptyset$}{
		  create new atom $\psi$\ and edges $\psi \rightarrow \phi$ and $\psi \rightarrow \epsilon$\;
          replace $\Delta$ by $\Delta'$\;
          add $\epsilon$ to $U$\;
   	   }
       substract $\epsilon$ from $B$;
    } 
  } 
  \ForEach{$\epsilon \in U$}{
      create new atom $\epsilon'$ and edge $\epsilon' \rightarrow \epsilon$\;
  }
  delete all atoms in $U \cup A$\;

\caption{Sparse Crossing of \emph{a} into \emph{b}}
\label{sparseCrossing}
\end{algorithm}
\end{linenomath}

\begin{linenomath}
\begin{algorithm} 
  initialize sets $Q \equiv \emptyset$ and $\Lambda \equiv {\bf{C}}(M)$\;
  \Do{$\Lambda \neq \emptyset$}{
    choose $c \in \Lambda$ at random and remove it from $\Lambda$\;
    calculate $S_c \equiv Q \cap {\bf{GL}}(c)$\;
	\eIf{$S_c = \emptyset$}{
		define $W_c \equiv {\bf{A}}(M^{*})$\;
	}{
    	calculate $W_c \equiv \cap_{\phi \in S_c} {{\bf{GL}}^a}([\phi])$\;
    }
    calculate $\Phi_c \equiv \{ [\phi] : \phi \in {{\bf{GL}}^a}(c) \}$\;
    \While{$W_c \neq {\bf{Tr}}(c)$}{
  	    choose an atom $\xi \in W_c \, \backslash \, {\bf{Tr}}(c)$ at random\;
        choose an atom $\phi$ such that $[\phi] \in \Phi_c \, \backslash \, {\bf{GU}}(\xi)$ at random\;
		add $\phi$ to set $Q$\;
		replace $W_c$ with $W_c  \cap{{\bf{GL}}^a}([\phi])$\;
    } 
  } 
  delete all atoms in the set ${\bf{A}}(M) \, \backslash \, Q$\;
\caption{atom set reduction}
\label{traceReduction}
\end{algorithm}
\end{linenomath}

\begin{linenomath}
\begin{algorithm} 
  initialize sets $Q \equiv \emptyset$ and $S \equiv R^{-}$\;
  \While{$S \neq \emptyset$}{
    choose $r \in S$ at random and remove it from $S$. Let $r \equiv \neg(a<b)$\; 
	\If{${\bf{dis}}_{M^{*}}([b], [a]) \cap Q = \emptyset$ }{
		choose an atom $\xi \in {\bf{dis}}_{M^{*}}([b], [a])$ and add it to Q\; 
	}
  } 
  delete all atoms in the set ${\bf{A}}(M^{*}) \, \backslash \, Q$\;
\caption{atom set reduction for the dual algebra}
\label{dualReduction}
\end{algorithm}
\end{linenomath}

\begin{linenomath}
\begin{algorithm} 
  let $R_p$ be a new or exisitng set of pinning relations\;
  \ForEach{$\phi \in M$}{
	  calculate the set $H = C(M) \backslash {\bf{U}}(\phi)$\;
      create the pinning term $T_{\phi} = \odot_{c\, \in H} \, c\, $\;
	  \ForEach{ $c \in C(M) \cap {\bf{U}}(\phi)$    }{
	      add $r \equiv \neg(c < T_{\phi})$  to the set $R_p$\;
      }
  }
\caption{generation of pinning terms and relations}
\label{dualReduction}
\end{algorithm}
\end{linenomath}

\newpage
Graphs are assumed to be transitively closed at all times. This requirement, however, can be delayed at some steps to speed up calculations. Always when atoms or edges are added to the graph of $M$ the corresponding duals and reverted edges should also be added to the graph of $M^{*}$.  When an element is deleted its dual should also be deleted from the graph of $M^{*}$.

\setcounter{equation}{0}
\renewcommand\theequation{A.\arabic{equation}}

\section{Exact atomizations} \label{perfectAtomizationAppendix}

Consider again our toy problem of the vertical lines. We want constant $v$ to satisfy $v < I$ if and only if $I$ is (the term of) an image that has a vertical line. Using subscript $i$ for rows and $j$ for columns we can write:
\begin{linenomath}
\begin{equation} 
(v<I)\Leftrightarrow \vee_{j}\wedge_{i}(c_{ij\,\bf{b}}<I)
\end{equation} 
\end{linenomath}
which simply states that the image should have a black pixel $c_{ij\,\bf{b}}$ at every row $i$ of some column $j$. The boldface index $\bf{b}$ stands for the particular value (color back). 

Recapitulating from {\bf{Section \ref{exactSolutions}}}, an element $b$ has an atom $\phi$ if and only if $b$ contains any of the constants that contain $\phi$. We say
\begin{linenomath}
\begin{equation} 
(\phi<b)\Leftrightarrow \vee_{k}(c_{\phi k}<b),
\end{equation} 
\end{linenomath}
where index $k$ at the disjunction runs along the constants $c_{\phi k}$ that contain atom $\phi$. Relation $a <b$ holds if and only if 
\begin{linenomath}
\begin{equation} 
(a<b)\Leftrightarrow\wedge_{\phi \in a} (\phi<b)\Leftrightarrow \wedge_{\phi \in a}\vee_{k}(c_{\phi k}<b),
\end{equation} 
\end{linenomath}
where the conjunction runs along all atoms in $a$.

\subsection{The map index} 

If we compare the solution of the vertical bar problem and the general form for $(a<b)$ we see that they differ only in the order of the connectors. The representation of elements in atomized semilattices corresponds with a first-order formula with a conjunction followed by a disjunction which is known as conjunctive normal form, CNF. We have to swap the connectors $\vee$ and $\wedge$ to understand how the vertical lines look represented in the semilattice. Interchanging connectors can be done by using the distributive law the same way we can interchange the multiplication and addition operators of linear algebra:
\begin{linenomath}
\begin{equation*} 
\otimes_{j}\oplus_{i} c_{ij}=\oplus_{j\rightarrow i}\otimes_{j} c_{ij}
\end{equation*} 
\end{linenomath}

We introduced the \enquote{map index} $j\rightarrow i$ in {\bf{Section \ref{exactSolutions}}} to represent an index that runs along all possible functions from j to i. For $\oplus_{j\rightarrow i}\otimes_{j} c_{ij}$ each summand is characterized by a particular function from $j$ to $i$. If $j$ takes "J" possible values and $i$ takes $I$ possible values the summation now has $I^J$ summands each summand a multiplication of $J$ factors. To make more explicit the functional dependence we can write $\oplus_{j\rightarrow i}\otimes_{j} c_{ij(i)}$ to emphasize that the value of $i$ on each factor depends upon the factor $j$ through a function $i(j)$ that is different for each summand. The handy map index has the following properties:
\begin{linenomath}
\begin{equation} 
\oplus_{i\rightarrow jk}=\oplus_{i\rightarrow j}\oplus_{i\rightarrow k}
\end{equation} 
\end{linenomath}
\begin{linenomath}
\begin{equation} 
\oplus_{i\rightarrow (j\rightarrow k)}=\oplus_{ij\rightarrow k}.
\end{equation} 
\end{linenomath}
These properties also apply to both, conjunction and disjunction. Unlike multiplication and addition, conjunction and disjunction are both distributive with respect to each other so we can interchange them in any order. 
\begin{linenomath}
\begin{equation} 
\wedge_{j}\vee_{i} c_{ij}=\vee_{j\rightarrow i}\wedge_{j} c_{ij(i)}
\end{equation} 
\end{linenomath}
\begin{linenomath}
\begin{equation} 
\vee_{j}\wedge_{i} c_{ij}=\wedge_{j\rightarrow i}\vee_{j} c_{ij(i)}
\end{equation} 
\end{linenomath}
We can now interchange connectors for the vertical line problem:
\begin{linenomath}
\begin{equation} 
(v<I)\Leftrightarrow \vee_{j}\wedge_{i}(c_{ij\,\bf{b}}<I) \Leftrightarrow \wedge_{j\rightarrow i}\vee_{j}(c_{i(j)j\,\bf{b}}<I) .
\end{equation} 
\end{linenomath}
From the structure of the CNF form we know that the exact embedding into a semilattice of the vertical line problem has $I^J$ atoms of the form:
\begin{linenomath}
\begin{equation} 
\phi_{j\rightarrow i}=\vee_{j}c_{i(j)j\,\bf{b}},
\end{equation} 
\end{linenomath}
where each atom is characterized by a function i(j). Each atom is in one black pixel per column and it is characterized by a particular choice of a row per column.

In this case we know the formal solution of the problem in advance and then we can work out the form of the {\emph{exact atomization}} using the map index. Usually we have examples and a general expression is unknown.

\subsection{Calculating atomizations for complex descriptions} 

The map index just introduced is powerful enough to characterize the form of any embedding provided that we have a first-order formula with or without quantifiers. We are dealing only with finite algebras so universal quantifiers can be treated as conjunctions and existential quantifiers as disjunctions. We first write the formula as a sequence of conjunctions and disjunctions. This is always possible by extending indexes and perhaps adding some trivial clauses that are always $true$ or always $false$. For example,
\begin{linenomath}
\begin{equation} 
[\vee_{i}\wedge_{j}(a_{ij}<I)] \wedge [\vee_{u}(b_{u}<I)] = \wedge_{s} \vee_{r=i \,\,\cup \,u}\wedge_{j}\,g_{srj},
\end{equation} 
\end{linenomath}
with $g_{srj}$ 
\begin{linenomath}
\begin{equation} 
g_{srij} = \left\{\begin{array}{lr}
s=0,\,\,\,r \in {\bf{i}} & a_{ij}<I\\
s=0,\,\,\,r \in {\bf{u}} & false\\
s=1,\,\,\,r \in {\bf{i}} & false\\
s=1,\,\,\,r \in {\bf{u}} & b_{r}<I
\end{array}\right\} ,
\end{equation} 
\end{linenomath}
The trick is simply to extend the scope of the index at the disjunction to $r=i \,\,\cup \,u$, so it can take all possible values of $i$ and $u$ by adding some trivial clauses equal to \textit{false}. To extend an index in a conjunction we would add extra \textit{true} clauses. 

Suppose we want to find the exact embedding for a problem with a solution:
\begin{linenomath}
\begin{equation} 
(h<I)\Leftrightarrow\vee_{a}\wedge_{b}\vee_{c}\neg\wedge_{d}\vee_{e}\, g_{abcde}
\end{equation} 
\end{linenomath}
where $g$ is a function that maps a tupla of indexes $abcde$ to \textit{true}, \textit{false} or some clause $(c_{k} < I)$, 
\begin{linenomath}
\begin{equation} 
g_{abcde}= \{c_{k}< I,\, true,\, false\}.
\end{equation} 
\end{linenomath}
To transform a chain of connectors to CNF, we first get rid of the negations: 
\begin{linenomath}
\begin{equation} 
\vee_{a}\wedge_{b}\vee_{c}\neg\wedge_{d}\vee_{e}\,g_{abcde}=\vee_{a}\wedge_{b}\vee_{c}\vee_{d}\wedge_{e}\neg g_{abcde}=\vee_{a}\wedge_{b}\vee_{cd}\wedge_{e}\, \bar{g}_{abcde},
\end{equation} 
\end{linenomath}
and then move the connectors where we want them by using the map index,
\begin{linenomath}
\begin{equation} 
\vee_{a}\wedge_{b}\vee_{cd}\wedge_{e}\, \bar{g}_{abcde}=\wedge_{a\rightarrow b}\vee_{a}\vee_{cd}\wedge_{e} \,\bar{g}_{ab(a)cde}=\wedge_{a\rightarrow b}\wedge_{acd\rightarrow e}\vee_{acd} \,\bar{g}_{ab(a)cde(acd)}.
\end{equation} 
\end{linenomath}
From the index structure of the conjunctions, we know that the exact model contains at most $B^A E^{ACD}$ atoms, each atom of the form
\begin{linenomath}
\begin{equation} 
\phi_{a\rightarrow b, \, acd\rightarrow e}=\vee_{acd} \,\bar{g}_{ab(a)cde(acd)},
\end{equation} 
\end{linenomath}
contained in at most $ACD$ constants, and characterized for two functions, $A:a\rightarrow b$ and $E:acd\rightarrow e$. 

The inverse problem looks very different and it can be much easier or harder to learn,
\begin{linenomath}
\begin{equation} 
(\bar{h}<I)\Leftrightarrow\neg\vee_{a}\wedge_{b}\vee_{c}\neg\wedge_{d}\vee_{e}\, g_{abcde}=\wedge_{a}\vee_{b}\wedge_{cd}\vee_{e}\, g_{abcde}=\wedge_{a}\wedge_{b\rightarrow cd}\vee_{be}\, g_{abc(b)d(b)e}.
\end{equation} 
\end{linenomath}
The exact model for the inverse problem contains at most $A (CD)^B$ atoms each atom contained on at most $BE$ constants, with the form
\begin{linenomath}
\begin{equation} 
\psi_{{\bf{a}}, \, b\rightarrow cd}=\vee_{be}\, g_{{\bf{a}}\, bc(b)d(b)e}.
\end{equation} 
\end{linenomath}
We say "at most" because the exact models may contain fewer atoms than the expected from the structure of the conjunction indexes. First, notice that disjunctions with trivial $true$ clause are always satisfied and never become atoms. Some atoms may be identical to others. Some other atoms are contained in a constant and in its inverse constant which become disjunctive clauses that are always satisfied so they can be ignored. Other atoms we can discard are the ones that are \emph{redundant} as in {\bf{Theorem \ref{redundantAtom}}}. Redundant atoms add nothing to the atomization that is not already required by other (smaller) atoms. Smaller atoms are contained in fewer constants than larger atoms. We say an atom is smaller than other if the other is larger as defined in {\bf{Appendix \ref{Notation}}}. 

From the form of the disjunctive clause $\psi_{{\bf{a}}, \, b\rightarrow cd}$ we see that atoms in this model are in at most $BE$ constants. Because $g_{abcde}$ maps to a clauses with a mapping that is not necessarily injective the same constant may appear multiple times in the disjunctive expression of an atom. Additionally $false$ clauses also result in missing constants so at the end an atom may be included in significantly less than $BE$ constants. Because of the difference in atom sizes it is possible for some atoms to be supersets of others. 

Consider that we potentially have a large set of symbols $g_{abcde}$ with as many as $ABCDE$ symbols that correspond with at most the number of constants defined for the problem. We should expect many repetitions in problems with many indexes (many connectors). When calculated using a computer we often find for many problems that their prefect models have by far fewer atoms than calculated from the conjunction indexes. In any case, the exact model is usually very large.
 
Interestingly, the fact that the excat model of a problem is larger than the exact model of another problem does not necessarily mean that the "larger" problem is harder to learn. In general the size of the atoms of a model is a much better indicator of problem hardness. The smaller the atoms the easier is to find an approximated solution to the problem. 

We finish this section with an interesting property. Any atom in a constant $x$ intersects in at least one constant any other atom (albeit reverted) of its inverse constant $\neg x$. For example, any two $\bar{\phi}_{a\rightarrow b, \, acd\rightarrow e}$ and $\psi_{{\bf{a}}, \, b\rightarrow cd}$ always intersect in the constant inclusion clause,
\begin{linenomath}
\begin{equation} 
\bar{g}_{\bf{a}b(\bf{a})c(b(\bf{a}))d(b(\bf{a}))e(\bf{a}c(b(\bf{a}))d(b(\bf{a})))}
\end{equation} 
\end{linenomath}
or the negation of this clause if we choose to revert $\bar{\psi}$ instead of $\phi$. To see why this is true, first notice that $\psi_{{\bf{a}}, \, b\rightarrow cd}$ sets a value $\bf{a}$ for index $a$. Once we have $\bf{a}$ fixed, we just need to look into the expression of $\bar{\phi}_{a\rightarrow b, \, acd\rightarrow e}$ to find out that fixing $\bf{a}$ sets a value $b(\bf{a})$ for $b$ which in turn, going back to $\psi_{{\bf{a}}, \, b\rightarrow cd}$, fixes $c(b(\bf{a}))$ and $d(b(\bf{a}))$ that finally sets the value $e(\bf{a}c(b(\bf{a}))d(b(\bf{a})))$ using again the expression of $\bar{\phi}_{a\rightarrow b, \, acd\rightarrow e}$. We have been jumping from one atom to the other selecting values for indexes until we find the intersecting clause. This clause corresponds always with a constant inclusion and never with a trivial $true$ or $false$ clause because atoms do not contain $true$ clauses. An atom may contain $false$ clauses but to intersect in a $false$ clause with the inverse of another atom requires a $true$ clause in this one.

\newpage

\section{Error is smaller the higher the compression} \label{compressionAndError}

\subsection{Derivation} \label{CompressionaAndErrorproof}

Assume that we sample $Q$ test questions from a distribution $D_{test}$ and that we have a learning algorithm that answers all the questions correctly. Let the failure rate be the probability for our algorithm to fail in one test question sampled using distribution $D_{test}$. 

The probability to have a failure rate greater than $\varepsilon$ and still answer the $Q$ questions correctly is bounded by:
\begin{linenomath}
\begin{equation} 
P(\text{Q tests correct} \, | \, \text{failure} > \varepsilon) < (1 - \varepsilon)^{Q}.
\end{equation} 
\end{linenomath}
Suppose that we have a set $\Omega$ of possible algorithms (or parameters) and we select one from this set. Assume the selected algorithm correctly responds the $Q$ test questions. We want to derive an upper bound for the failure rate $\varepsilon$ based on the fact that it responded to all the test questions correctly. We have:
\begin{linenomath}
\begin{equation} 
P(\text{failure} > \varepsilon \, | \, \text{Q tests correct}) < \frac{P(\text{failure} > \varepsilon) (1 - \varepsilon)^{Q}}{P(\text{Q tests correct})},
\end{equation} 
\end{linenomath}
where $P(\text{failure} > \varepsilon)$ is the probability to pick an algorithm from $\Omega$ that has an error rate larger than $\varepsilon$, and P(\text{Q tests correct}) is the probability to pick an algorithm that answers all $Q$ questions correctly. If $\varepsilon$ is small we may safely assume that $P(\text{failure} > \varepsilon) \approx 1$, and write:
\begin{linenomath}
\begin{equation} 
\delta \equiv \frac{ (1 - \varepsilon_{\delta})^{Q}}{P(\text{Q tests correct})}  ,
\end{equation} 
\end{linenomath}
where $\delta$ is (an overestimation of) the risk we are willing to accept for the error rate to be larger than $\varepsilon_{\delta}$. Solving for the error rate:
\begin{linenomath}
\begin{equation} 
\varepsilon_{\delta}= 1 - \delta^{\frac{1}{Q}} P(Q)^{\frac{1}{Q}}.
\end{equation} 
\end{linenomath}
The smaller the risk the larger is the error rate we have to accept. $\varepsilon_{\delta}$ has been derived from an upper bound of $P(\text{Q tests correct} \, | \, \text{failure} > \varepsilon)$ so the actual error rate we expect to measure is lower than $\varepsilon_{\delta}$.

Of course, this calculation is meaningless unless there is a well-defined distribution $p(\epsilon)$: 
\begin{linenomath}
\begin{equation} 
P(Q) = \Sigma_{\epsilon} p(\epsilon) p(Q | \epsilon) = \Sigma_{\epsilon} p(\epsilon) (1 - \epsilon)^{Q}
\end{equation} 
\end{linenomath}
where $p(\epsilon)$ is the probability to pick an algorithm that has an error rate equal to $\epsilon$ and the summation runs along all possible error rates.

If we don't know $P(Q)$ we cannot derive $\varepsilon_{\delta}$. It is tempting to use $\varepsilon_{\delta}= 1 - \delta^{\frac{1}{Q}}$ and, in fact, it may work well to approach the average value of $\epsilon$ when $Q$ is not too large. However, when $Q$ is large enough this approach dangerously underestimates $\varepsilon_{\delta}$ and cannot be used. 

It is also tempting to approximate $P(Q) \approx 0.5^Q$ if we know that the proportion of algorithms in $\Omega$ that are expected to do well in test questions is extremely small compared with the cardinal of $\Omega$. Even when the distribution $p(\epsilon)$ is very biased towards randomly responding algorithms for a sufficiently large value of $Q$ the distribution $P(Q)$ is always dominated by the algorithms that do well in test questions. If $Q$ is large enough $P(Q)$ becomes much larger than $0.5^Q$ and the approximation does not work. In general there is no way to derive an error rate unless we know $P(Q)$.

So, let's assume that we know $P(Q)$. As the cardinal of $Q$ grows we get $P(Q) << \delta$ quite rapidly for any reasonable $\delta$. If $Q$ is large enough, the term $P(Q)^{\frac{1}{Q}}$ dominates over $\delta^{\frac{1}{Q}}$ and $\varepsilon_{\delta}$ becomes independent of $\delta$. In general $P(Q)$ dominates unless we demand the risk $\delta$ to be extremely small, and there is no need for that. It is interesting and unintuitive that we get a meaningful value of $\varepsilon_{\delta}$ even if we let the risk to be as large as $\delta = 1$. When $Q$ is large there is a limit value:
\begin{linenomath}
\begin{equation} 
\varepsilon = 1 - P(Q)^{\frac{1}{Q}}.
\end{equation} 
\end{linenomath}
that is independent of the risk. 

Now that we know how to calculate an error rate from test example results we are going to apply a similar reasoning to training examples. 

Assume we sample $\emph{different}$ training examples from a distribution $D_{train}$. If multiple learning batches are used the algorithm may not remember well all the examples seen, particularly training examples seen in past epochs. Let's define $R$ as the number of training examples that have been correctly \enquote{retained} by the algorithm. When the error rate is small we expect the difference between $R$ and the total number of training examples to become small compared to $R$.

Again, $\varepsilon$ is the probability for our algorithm to fail in one test question randomly sampled using distribution $D_{test}$. We are going to assume test and train distributions equal, i.e. $D_{test} = D_{train}$.

We are interested in algebraic learning with semilattices here, so our algorithms in $\Omega$ are semilattice models. Imagine we have a random picking mechanism that selects one model among all models consistent with $R$. Assume the chosen model has $Z$ atoms. The probability to get a model with an error rate worse than $\varepsilon$ is given by:
\begin{linenomath}
\begin{equation} 
P(\text{failure} > \varepsilon \, | \, \text{R correct}\, \wedge \, \text{Z atoms} ) = 
\end{equation} 
\end{linenomath}
\begin{linenomath}
\begin{equation} 
= \frac{ P(\text{Z atoms} ) \, P(\text{failure} > \varepsilon \, \wedge \, \text{R correct}\, | \, \text{Z atoms} ) }{ P(\text{R correct}\, \wedge \, \text{Z atoms} ) }. 
\end{equation} 
\end{linenomath}
The models we can handle in practice are very small compared with the number of different atoms ($2^C$ for $C$ constants) a model could have, so realistic models are not very far (compared to $2^C$) from the minimal size of $Z$ for which there is some model consistent with $R$. In this range of $Z$ values we hypothesize
\begin{linenomath}
\begin{equation} 
P(\text{failure} > \varepsilon \, \wedge \, \text{R correct}\, | \, \text{Z atoms} ) \leq P(\text{failure} > \varepsilon \, \wedge \, \text{R correct}). 
\end{equation} 
\end{linenomath}
This inequality occurs when the proportion of small models that perform bad within the set of small models is not greater than the proportion of small models that perform bad in the set of large models. We can expect this to be the case based on the fact that there are many more large models than small models; with more atoms we get more models consistent with $R$ but we also get an even greater number of models inconsistent with $R$. 

Using the inequality above it is possible to derive an upper bound for the conditional probability:
\begin{linenomath}
\begin{equation} 
P(\text{failure} > \varepsilon \, | \, \text{R correct} \wedge \text{Z atoms}) < \frac{P(\text{failure} > \varepsilon) (1 - \varepsilon)^{R}}{P(\text{R correct} | \text{Z atoms})},
\end{equation} 
\end{linenomath}
which is almost the same result we got before for test examples with $P(\text{Q tests correct})$ replaced by $P(\text{R correct} \,|\, \text{Z atoms})$. Again, we can safely use the approach $P(\text{failure} > \varepsilon) \approx 1$ and replace the conditional probability in the denominator by: 
\begin{linenomath}
\begin{equation} 
P(\text{R correct}\, | \, \text{Z atoms}) = \frac{ |\Omega_{R\wedge Z}| }{ |\Omega_{Z}| },
\end{equation} 
\end{linenomath}
where $\Omega_{Z}$ is the number of models with $Z$ atoms, and $\Omega_{R\wedge Z}$ is the number of models with $Z$ atoms and consistent with $R$. There is a bound for $\Omega_{Z}$: 
\begin{linenomath}
\begin{equation} 
|\Omega_{Z}| < {{2^C}\choose{Z}},
\end{equation} 
\end{linenomath}
that we can use to get an upper bound for the probability: 
\begin{linenomath}
\begin{equation} 
P(\text{failure} > \varepsilon \, | \, \text{R correct} \wedge \text{Z atoms}) < \ \frac{ {{2^C}\choose{Z}} \, (1 - \varepsilon)^{R}} {|\Omega_{R\wedge Z}|}.
\end{equation} 
\end{linenomath}
The combinatorial number corresponds with all possible atomizations of size $Z$. It does not correspond with the number of possible models of size $Z$ because there are multiple atomizations that produce the same model. This is a consequence of \emph{redundant} atoms (see theorem \ref{redundantAtom}), but it provides an upper bound for $\Omega_{Z}|$.

If the risk we are willing to accept to get a bad performing model of size $Z$ is set to $\sigma$:
\begin{linenomath}
\begin{equation} 
\sigma \equiv \frac{ {{2^C}\choose{Z}} \, (1 - \varepsilon_{\sigma})^{R}} {\Omega_{R\wedge Z}},
\end{equation} 
\end{linenomath}
we can derive an upper bound for the error rate $\varepsilon_{\sigma}$. 

The logarithm of the combinatorial number can be estimated assuming $Z << 2^{C}$ with: 
\begin{linenomath}
\begin{equation} 
\ln {{2^C}\choose{Z}} \approx \ln(2) Z C + O(max(Z,\,C)),
\end{equation} 
\end{linenomath}
that can be derived from Stirling's factorial formula. Substituting this estimation in the equation above and using $ln(1 - \varepsilon) \approx -\varepsilon + O(\epsilon^2)$
\begin{linenomath}
\begin{equation} 
\ln(2)ZC - \varepsilon_{\sigma} \, R - \ln(|\Omega_{R\wedge Z}|) = \ln(\sigma).
\end{equation} 
\end{linenomath}
The error rate is dominated by $P(\text{R correct}\, | \, \text{Z atoms})$ and it becomes independent of the risk $\sigma$ for any reasonable value, just as it happened before with test examples. The quantity $log_{2} |\Omega_{R\wedge Z}|$ measures the degeneracy of the solutions. We have $1 \leq |\Omega_{R\wedge Z}| << |\Omega_{Z}|$, and even if $|\Omega_{R\wedge Z}|$ is a very large number, it is going to be very small compared to $|\Omega_{Z}|$. We expect: 
\begin{linenomath}
\begin{equation} 
O(\ln(|\Omega_{Z}|) - \ln(|\Omega_{R\wedge Z}|)) \approx O(\ln(|\Omega_{Z}|))
\end{equation} 
\end{linenomath}
so the contribution of $\ln(|\Omega_{R\wedge Z}|)$ is small (albeit not necessarily negligible) compared to $NC$. 

If we neglect $ln(|\Omega_{R\wedge Z}|)$ the following equation gives us the error rate we expect to get for a model selected using the random picking algorithm:
\begin{linenomath}
\begin{equation} 
\varepsilon R = \ln 2\, ZC.
\end{equation} 
\end{linenomath}
Reorganizing the equation and introducing the compression rate $\kappa$ we get for the random picking algorithm: \begin{linenomath}
\begin{equation} 
\kappa \equiv \frac{R}{Z}
\end{equation} 
\end{linenomath}
we finally get
\begin{linenomath}
\begin{equation} 
\varepsilon = \frac{\ln2 \,C}{\kappa}
\end{equation} 
\end{linenomath}
which says that error and compression rates are inversely proportional and their product depends only upon the number of constants or degrees of freedom of our data.

The random picking algorithm would actually be a valid learning algorithm if we could choose a low $Z$ value at will. We may do better than the random picking algorithm but what is actually easy is to do worse! We can do much worse, for example, if we use one of the memorizing algorithms described in {\bf{Section \ref{memorizingGeneralizing}}}. 

Experimental results suggest that the Sparse Crossing algorithm may learn faster than the random picking when the error is large. However, it seems that when the error rate gets small the Sparse Crossing algorithm asymptotically approaches the exact performance of the random picking algorithm. We also have to consider that the approximations made here for the random picking algorithm assume a low error rate, so we do not really know the performance of the random picking algorithm at high error rates.

When the input constants are divided in pairs, so the presence of one constant in the pair implies the absence of the other (like the white and black pixel constants) the number of different atoms is $3^{\frac{C}{2}}$ rather than $2^C$. Each atom can be either in one of the constants of the pair or in none of them: in total three states per constant pair. Atoms that have both constants of the same pair in its upper segment do not appear in simple classification problems. It is very easy to see why. Suppose we are classifying images. If an atom is in both, the white and its corresponding black pixel's constant, then it is in the lower segment of every term representing an image and has no use. In this case the proportionality law reads: 
\begin{linenomath}
\begin{equation} 
\varepsilon\kappa = {\frac{\ln\,3}{2}} C.
\end{equation} 
\end{linenomath}
This equation and its proportionality constant are in good agreement with experimental results. We compare the theoretical values with experimental results in the next section of the appendix and in section {\bf{Section \ref{ekIsConstant}}}. The inverse proportionality between error and compression rates is clear.

For classification problems for which there are symmetries of the input data that do not alter the hidden classes we can give a better estimation of the relation between error and compression rates. This is the subject of the next section.

\subsection{The role of symmetries} 
\label{symmetries}

In \textbf{Appendix \ref{perfectAtomizationAppendix}} we showed how to derive the atoms of a constant for which we have a formal description as a first order formula. The atoms can be described by combinatorial variations of other constants determined by the map index. We departed from a known formal description that uses some explicit indexes that map to constants. When learning from data the formal description is not known and the indexes are hidden but are still implicit in the atoms learned. Pairs of atoms of the exact atomization are related by one or multiple swappings of two constants. Two different values of the same index, map to two constants that can be swapped. The structure of the hidden problem gets reflected into the symmetries of the atoms. 

In the same way, if input data has a symmetry, meaning that some constants can be interchanged without affecting the hidden classes, the atoms also display the symmetry. To be more specific, consider the problem of separating images with an even count of vertical bars from images with an odd count. We can take an input image and permute the columns and also permute the rows without affecting in which class the image should be classified. In this case the atoms also manifest the same symmetry, i.e. we can apply the same permutations to an atom and obtain another atom of the exact atomization. 

This is potentially useful in practice. If a problem has a known symmetry new atoms can be derived and added to a model by applying the symmetry to the existing atoms. We get new atoms \enquote{for free} without the need to learn them, i.e. without the need of extensively train for all possible values that the symmetry can take. For example, we could use this technique to improve accuracy of translation-invariant pattern recognition with fewer training examples.

In section \ref{ekIsConstant} we studied the problem of separating even from odd using Sparse Crossing. We showed that the relation between error and compression fits well the theoretical predictions for the random picking at low error rates. The proportionality between error and the inverse of the compression rate is clear. For grids of size $7\times7$ and $10\times10$ the measured proportionality constant and the predicted proportionality constant for the random picking only differ in about 20\% and 35\% respectively. Not bad for an adimensional quantity that can take any value. 

We are going to use our knowledge of the symmetries of the even-versus-odd separation problem to improve our theoretical predictions. The term $ln(|\Omega_{R\wedge Z}|)$ that we considered small compared with $ln(|\Omega_{Z}|)$ corresponds with the logarithm of the number of atomizations with $Z$ atoms that satisfy $R$. This is a subtracting term that measures degeneracy of the solutions, so the larger it is the more efficient is the transformation of compression into accuracy. For each atom there are other $(d!)^2$ atoms in the exact atomization that correspond with a permutation of rows and a permutation of columns (where $d\times d$ is the dimension of the grid). For the number of solutions of $R$ with $Z$ atoms we should also expect to have $(d!)^2$ as a multiplying factor:
\begin{linenomath}
\begin{equation} 
|\Omega_{R\wedge Z}| \approx \alpha(R, Z) (d!)^{2Z},
\end{equation} 
\end{linenomath}
where $\alpha(R, Z)$ is some quantity larger than $1$. If we use this estimation we get:
\begin{linenomath}
\begin{equation} 
\frac{\ln(3)}{2} ZC - \varepsilon \, R - 2 \ln(d!) Z - \ln(\alpha(R, Z)) = 0.
\end{equation} 
\end{linenomath}
and solving for $\varepsilon\kappa$:
\begin{linenomath}
\begin{equation} 
\varepsilon \, \kappa = \frac{\ln(3)}{2} C - 2 \ln(d!) - \frac{\ln(\alpha(R, Z))}{Z}.
\end{equation} 
\end{linenomath}
Neglecting the last term and substituting $C = 2d^{2}$, we get the new proportionality constant:
\begin{linenomath}
\begin{equation} 
\varepsilon \, \kappa = \ln(3) d^{2} - 2 \ln(d!).
\end{equation} 
\end{linenomath}
With the new estimation the observed discrepancy between measured and experimental values of this constant drop to about $10\%$ for dimensions $7\times 7$ and $5\%$ for dimension $10\times 10$. Convergence to the theoretical prediction is reached when the error becomes small enough, see (\textbf{Figure \ref{fig:combinedResultsF}}). 
\begin{figure}
\centering
\hspace*{-1cm} 
\includegraphics[scale=0.3]{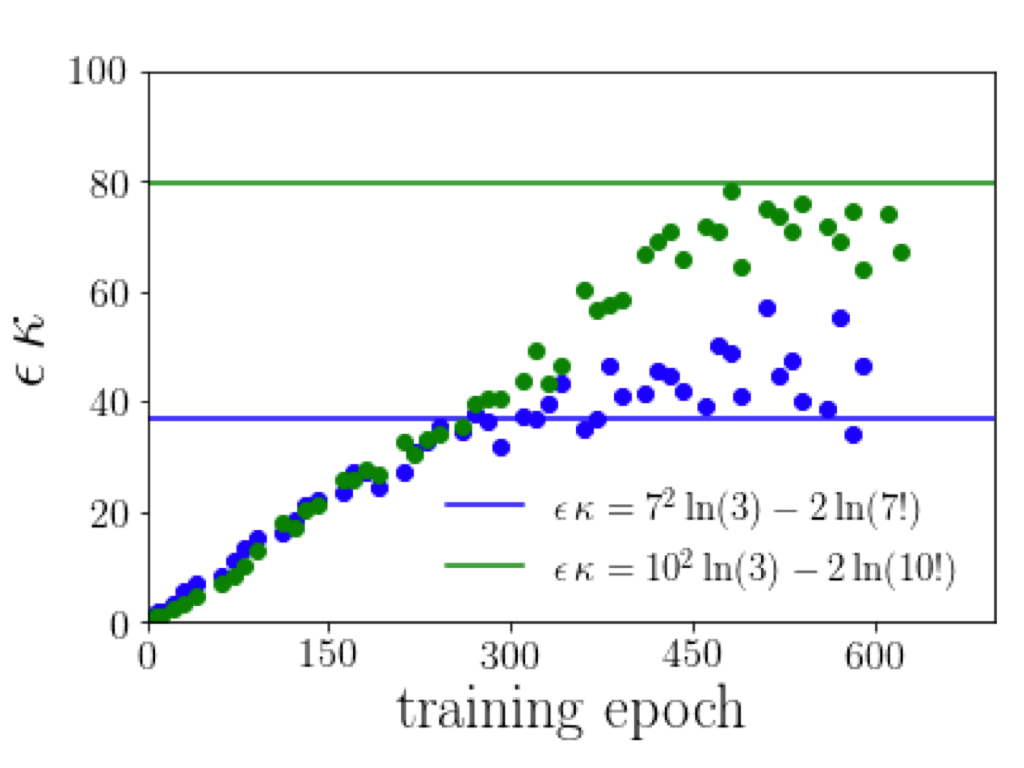} \caption{\textbf{Convergence of learning by Sparse Crossing to theoretical predictions in the problem of distinguishing whether an image has an even or odd number of vertical bars.} Lines indicate theoretical prediction at low error, to which Sparse Crossing approximately converges. blue: $7\times7$ images. green: $10\times10$ images.}
\label{fig:combinedResultsF}
\end{figure}

\nolinenumbers

\end{appendices}

%Bibliography

\bibliographystyle{unsrt}
\bibliography{algebraic}

\begin{thebibliography}{10}

\bibitem{Nilsson1991}
Nils~J. Nilsson.
\newblock {Logic and artificial intelligence}.
\newblock {\em Artificial Intelligence}, 47(1-3):31--56, jan 1991.

\bibitem{Pouly2011}
Marc Pouly, Jürg Kohlas, and {Wiley InterScience (Online service)}.
\newblock {\em {Generic Inference : a Unifying Theory for Automated
  Reasoning}}.
\newblock Wiley, 2011.

\bibitem{Burris}
Stanley. Burris and H.~P. Sankappanavar.
\newblock {\em {A course in universal algebra}}.
\newblock Springer-Verlag, 1981.

\bibitem{Tent2012}
Katrin Tent and Martin. Ziegler.
\newblock {\em {A course in model theory}}.
\newblock Cambridge University Press, 2012.

\bibitem{Lecun1998}
Y.~Lecun, L.~Bottou, Y.~Bengio, and P.~Haffner.
\newblock {Gradient-based learning applied to document recognition}.
\newblock {\em Proceedings of the IEEE}, 86(11):2278--2324, 1998.

\bibitem{Gent2017}
Ian~P Gent, Christopher Jefferson, and Peter Nightingale.
\newblock {Complexity of n-Queens Completion}.
\newblock {\em Journal of Artificial Intelligence Research}, 59:815--848, 2017.

\bibitem{Papert}
Dona Papert.
\newblock {Congruence Relations in Semi-Lattices}.
\newblock {\em Journal of the London Mathematical Society}, s1-39(1):723--729,
  jan 1964.

\bibitem{Bullier2001}
Jean Bullier.
\newblock {Integrated model of visual processing}.
\newblock {\em Brain Research Reviews}, 36(2-3):96--107, oct 2001.

\bibitem{Thorpe1996}
Simon Thorpe, Denis Fize, and Catherine Marlot.
\newblock {Speed of processing in the human visual system}.
\newblock {\em Nature}, 381(6582):520--522, jun 1996.

\bibitem{Bishop1995}
Christopher~M. Bishop.
\newblock {\em {Neural networks for pattern recognition}}.
\newblock Clarendon Press, 1995.

\bibitem{Grunwald}
Peter~D. Grünwald.
\newblock {\em {The minimum description length principle}}.
\newblock MIT Press, 2007.

\bibitem{Marriott}
Kim; Peter J.~Stuckey Marriott.
\newblock {Programming with constraints: An introduction. MIT Press}.

\bibitem{Donoho2006}
D.L. Donoho.
\newblock {Compressed sensing}.
\newblock {\em IEEE Transactions on Information Theory}, 52(4):1289--1306, apr
  2006.

\bibitem{Tao}
Emmanuel~J. Cand{\`{e}}s, Justin~K. Romberg, and Terence Tao.
\newblock {Stable signal recovery from incomplete and inaccurate measurements}.
\newblock {\em Communications on Pure and Applied Mathematics},
  59(8):1207--1223, aug 2006.

\bibitem{Skopljanac-Macina2014}
Frano {\v{S}}kopljanac-Ma{\v{c}}ina and Bruno Bla{\v{s}}kovi{\'{c}}.
\newblock {Formal Concept Analysis – Overview and Applications}.
\newblock {\em Procedia Engineering}, 69:1258--1267, jan 2014.

\bibitem{Todorcevic2010}
Stevo. Todorcevic.
\newblock {\em {Introduction to Ramsey spaces}}.
\newblock Princeton University Press, 2010.

\end{thebibliography}

\end{document}